\documentclass[oneside]{ut-thesis}

\usepackage[utf8]{inputenc}
\usepackage[comma,sort&compress,square]{natbib}

\usepackage{todonotes} %TODOs that look like post-it notes
\usepackage[]{xcolor} % NOTE: usually I would use dvipsnames but some other
\newif\ifcomments
\commentstrue % display author comments
\ifcomments
  \newcommand{\colornote}[3]{{\color{#1}\bf{#2: #3}\normalfont}}
\else
  \newcommand{\colornote}[3]{}
\fi

\usepackage{setspace} 
\usepackage{lipsum}
\usepackage{epigraph}
\usepackage{microtype}
\usepackage{graphicx}
\usepackage{subcaption}
\usepackage{booktabs} % for professional tables
\usepackage{amssymb}
\usepackage{amsmath}
\usepackage{tikz}
\usetikzlibrary{shapes.misc, positioning}
\usepackage{wrapfig,lipsum,booktabs}
\usepackage{arydshln}
\usepackage{hyperref}
\newcommand{\maximize}{\operatorname{maximize}}
\newcommand{\minimize}{\operatorname{minimize}}
\newcommand{\argmin}{\operatorname{argmin}}
\newcommand{\E}{\mathbb{E}}
\newcommand{\R}{\mathbb{R}}

\newcommand\independent{\protect\mathpalette{\protect\independenT}{\perp}}
\def\independenT#1#2{\mathrel{\rlap{$#1#2$}\mkern2mu{#1#2}}}
\RequirePackage{algorithm}
\usepackage{microtype}
\usepackage{graphicx}
\usepackage{subcaption}  % Elliot: hopefully this doesn't break anything...
\usepackage{booktabs} % for professional tables
\usepackage{multirow, boldline}
\usepackage{tikz}
\usetikzlibrary{shapes.misc, positioning}
\usepackage{amssymb}
\usepackage{amsmath}
\usepackage{natbib}
\usepackage{xcolor}
\newcommand{\argmax}{\operatorname{argmax}}
\usepackage[utf8]{inputenc} % allow utf-8 input
\usepackage[T1]{fontenc}    % use 8-bit T1 fonts
\usepackage{hyperref}       % hyperlinks
\usepackage{url}            % simple URL typesetting
\usepackage{booktabs}       % professional-quality tables
\usepackage{amsfonts}       % blackboard math symbols
\usepackage{nicefrac}       % compact symbols for 1/2, etc.
\usepackage{microtype}      % microtypography
\usepackage{enumitem}
\usepackage{lipsum}
\usepackage{amsmath}
\usepackage{amssymb}
\usepackage{amsthm}
\usepackage{bm}
\usepackage{dsfont}
\usepackage{graphicx}
\usepackage{wrapfig}
\usepackage{subcaption}
\usepackage{adjustbox}
\usepackage{algorithm, algcompatible}
\usepackage{multicol,changepage}
\usepackage[noend]{algpseudocode}
 \newcommand{\mybm}[1]{\scalebox{0.9}[1]{$\bm{#1}$}}
\newcommand{\indep}{\independent}
\newcommand{\given}{\,|\,}
\newcommand{\DO}{\textrm{do}}

\renewcommand{\S}{\mathcal{S}}
\newcommand{\M}{\mathcal{M}}
\renewcommand{\L}{\mathcal{L}}
\newcommand{\A}{\mathcal{A}}

\newcommand{\F}{\mathcal{F}}
\newcommand{\G}{\mathcal{G}}

\newcommand{\parents}{\textrm{Pa}}

\newcommand{\indicator}[1]{\mathds{1} (#1)}  % requires dsfont
\newcommand\restrict[1]{\raisebox{-.2ex}{$|$}_{#1}}
\newcommand\mydots{\ifmmode\ldots\else\makebox[1em][c]{.\hfil.\hfil.}\thinspace\fi}
\newcommand{\localSA}{\mathcal{L}^{(j\indep i)}}
\newtheorem*{assumption}{Assumption}
\newtheorem*{definition}{Definition}
\newtheorem{proposition}{Proposition}
\newtheorem{lemma}{Lemma}
\newtheorem{corollary}{Corollary}
\newtheorem{remark}{Remark}[section]
\makeatletter
\newenvironment{appdxProp}[1]
  {\count@\c@proposition
   \global\c@proposition#1 %
    \global\advance\c@proposition\m@ne
   \proposition}
  {\endproposition
   \global\c@proposition\count@}
\makeatother
\usepackage[utf8]{inputenc} % allow utf-8 input
\usepackage[T1]{fontenc}    % use 8-bit T1 fonts
\usepackage{hyperref}       % hyperlinks
\usepackage{url}            % simple URL typesetting
\usepackage{booktabs}       % professional-quality tables
\usepackage{amsfonts}       % blackboard math symbols
\usepackage{nicefrac}       % compact symbols for 1/2, etc.
\usepackage{microtype}      % microtypography
\usepackage{xspace}
\usepackage{enumitem}
\usepackage{lipsum}
\usepackage{amsmath}
\usepackage{amssymb}
\usepackage{amsthm}
\usepackage{bm}
\usepackage{dsfont}
\usepackage{graphicx}
\usepackage{wrapfig}
\usepackage{subcaption}
\usepackage[font=small]{caption}
\usepackage{adjustbox}
\usepackage{multicol,changepage}
\usepackage[noend]{algpseudocode}
\usepackage{algorithm}
\usepackage{tabularx,ragged2e}
\newcolumntype{L}[1]{>{\minwd l{#1}}l<{\endminwd}}
\newcolumntype{C}[1]{>{\minwd c{#1}}c<{\endminwd}}
\newcolumntype{R}[1]{>{\minwd r{#1}}r<{\endminwd}}
\usepackage{pifont}
\newcommand{\cmark}{\ding{51}}%
\newcommand{\xmark}{\ding{55}}%
\usepackage{float}
\floatstyle{plaintop}
\restylefloat{table}
\def\independenT#1#2{\mathrel{\rlap{$#1#2$}\mkern2mu{#1#2}}}
\newcommand{\lone}{\ell_1}
 \newcommand{\ad}[1]{\textsc{#1}}
\newcommand{\env}[1]{\texttt{#1}}
\newcommand{\lowerl}{{\Large \textbf{$\lrcorner$} }}
\newtheorem{theorem}{Theorem}

\makeatletter
\newenvironment{appdxTheorem}[1]
  {\count@\c@theorem
   \global\c@theorem#1 %
    \global\advance\c@theorem\m@ne
   \theorem}
  {\endproposition
   \global\c@theorem\count@}
\makeatother

\makeatletter
\newenvironment{appdxCorollary}[1]
  {\count@\c@corollary
   \global\c@corollary#1 %
    \global\advance\c@corollary\m@ne
   \corollary}
  {\endproposition
   \global\c@corollary\count@}
\makeatother

\usepackage{titlesec}
\titlespacing*{\section}
{0pt}{5pt plus 3pt minus 2pt}{4pt plus 2pt}
\titlespacing*{\subsection}
{0pt}{4pt plus 3pt minus 1pt}{3pt plus 2pt}
\title{
  Robust Machine Learning by Transforming and Augmenting Imperfect Training Data
}
\author{Elliot Creager}
\date{May 2023}
\degree{Doctor of Philosophy}
\department{Computer Science}
\gradyear{2023}

\begin{document}

  \frontmatter
  \maketitle

  \begin{abstract}
    Machine Learning (ML) is an expressive framework for turning data into computer programs.
Across many problem domains---both in industry and policy settings---the types of computer programs needed for accurate prediction or optimal control are difficult to write by hand.
On the other hand, collecting instances of desired system behavior may be relatively more feasible.
This makes ML broadly appealing, but also induces data sensitivities that often manifest as unexpected failure modes during deployment.
In this sense, the training data available tend to be imperfect for the task at hand.
This thesis explores several data sensitivities of modern machine learning and how to address them.
We begin by discussing how to prevent ML from codifying prior human discrimination measured in the training data, where we take a fair representation learning approach.
We then discuss the problem of learning from data containing spurious features, which provide predictive fidelity during training but are unreliable upon deployment.
Here we observe that insofar as standard training methods tend to learn such features, this propensity can be leveraged to search for partitions of training data that expose this inconsistency, ultimately promoting learning algorithms invariant to spurious features.
Finally, we turn our attention to reinforcement learning from data with insufficient coverage over all possible states and actions.
To address the coverage issue, we discuss how causal priors can be used to model the single-step dynamics of the setting where data are collected.
This enables a new type of data augmentation where observed trajectories are stitched together to produce new but plausible counterfactual trajectories.

  \end{abstract}

  \begin{dedication}
  This manuscript is dedicated to the memory of William B. Hooper.
  \end{dedication}

    \begin{acknowledgements}
   The research described in this thesis was conducted in the city now known as Toronto, Canada.
The Mohawk name for this place, ``Tkaronto'', indicates a meeting point between the trees and the water.
The land, water, and air that make up the surrounding region are governed by the Dish with One Spoon Wampum; throughout the generations, they have been cared for by many Indigenous peoples, who in turn have been subjected to material and cultural dispossession through a (still ongoing) project of colonization.
This is a human tragedy, and also an intellectual tragedy.
For example, Leanne Betasamosake Simpson argues that, by failing to engage with Indigenous ontologies that emphasize relationality, Western Scholars often unwittingly use concepts like ``abstraction'' (the bread and butter of Computer Science) towards extractive ends~\citep{simpson2017we}.
I have seen this pattern emerge frequently in research on technical approaches to mitigate harms in socio-technical systems, which typically exhibit nuanced and context-specific complexities.
In the Academy and beyond, much remains to be done when it comes to engaging with Indigenous (and other non-Western) ways of producing and preserving knowledge.
These efforts can be understood as a small part of the broader movement towards reconciliation.

On a personal level, I am grateful to this land, and the animal and plant life it supports, for providing a continual source of inspiration throughout my studies here.

\begin{center}
  \rule{0.7 \textwidth}{0.4pt}
\end{center}

Writing about gratitude is difficult for me because the readily available language tends to evoke an outstanding debt or obligation.
But these sentiments fail to capture my experience of collaborative research during grad school.
When I look back on things now and think of the many people who have helped me complete my studies, it does not suggest some unpaid debt of gratitude for this lesson or that favor.
However, I can say that these relationships, experiences, and memories have left a deep imprint on the person that I am today.
With humility and deep appreciation, I wish to acknowledge a few people who have left their mark on this thesis and its author.

My supervisor Rich Zemel has been enormously supportive of my research over these past six years, and has managed to do so with care, consistency, and a certain inscrutable deftness. 
Rich, you established our working relationship on a firm foundation of mutual respect, and always encouraged me to be curious and critical throughout my doctoral work.
I will try my best to do the same for my students.
Toni Pitassi, you are the paragon of modest brilliance.
Thank you for showing me the value of clearly articulating technical assumptions at the outset of a project, which I know see can lead to important insights in the research later on.
David Duvenaud, thank you for rounding out my thesis committee with a generous and convivial spirit.
You encouraged me to interpret my research against the backdrop of a bigger picture, and showed me the value in drawing connections to other fields.

Thank you to my many mentors throughout the years.
David Wingate and Noah Stein, you took a chance on me as a young researcher, and showed me the critical importance of patience and persistence in research.
Philippe Depalle, you taught me how to craft a research artifact, and showed me the small pleasures of discussing Fourier analysis over afternoon espressos.
Kevin Swersky, you made me feel welcome in a new research lab that was both impressive and intimidating.
I am amazed by your ability to bring a smile to every meeting, even when the work approaches drudgery.
Sheila McIlraith, more than anyone, you have taught me how to teach.
Thank you also for modeling the importance of compassion in building a research community, and for reminding me that impactful work need not adhere to the trends of the day.
Thank you to my guitar teacher Adam Solomon, whose music was a true source of joy for me during weary times, reminding me that all of us are simultaneously teachers and students.

Will Grathwohl, I am glad we could convince each other to take a leap of faith and move to Toronto for grad school.
Your ambition and scientific rigor are impressive, but I most appreciate your eagerness to feed a dozen people or more at the drop of a hat.
David Madras, I have learned so much from you: how to do the translation work that illuminates commonalities between disciplines, how to stay cool under deadline pressure, and how to find joy in creative pursuits outside of work.
J\"orn Jacobsen, you echoed this last lesson, and it was a treat to explore the culture of Toronto with you during your brief stay here.
At a time when I needed more guidance than I realized, your reminder, that life is slow, was well received.
Jesse Bettencourt, you endeared yourself to me starting from our very first semester.
As a teacher, your commitment to your students' success and well-being is admirable.
Your brilliance, wit, and kindness follow you outside the classroom and wherever you choose to go.
Your camaraderie is a treasure to me.

Thank you to the students who I have had the honor of mentoring, especially Robert Adragna, Frederik Tra\"uble, Dilys Dickson, Aisha Alaagib, Arjun Mani, Parand Alizadeh Alamdari, and Benjamin Eyre.
These were formative experiences for me; any guidance I may have offered you has already been reciprocated, and I look forward to following your work in the future.
Thank you to my wonderful lab mates Renjie Liao, Mengye Ren, Jake Snell, James Lucas, Jackson Wang, Kamyar Ghasemipour, and Eleni Triantafillou.
All of you helped keep spirits high in the lab, even when training the most unwieldy neural nets.
Thank you to Silviu Pitis and Animesh Garg for patiently explaining the fundamental algorithms of reinforcement learning to me time and again.
Thank you Shems Saleh, Yasmin McDowell, Mattie Tesfaldet and Taylor Killian for making Vector Institute a welcoming place for me to work.
Thank you to Ria Kalluri, Willie Agnew, Natalia Bilenko, Manuel Sabin, Agata Foryciarz, Micah Carroll, Susanne Kite, Sayash Kapoor, Blaine Harper, and everyone else who has helped me to understand the limits of approaching research from single institutional or disciplinary perspective.

To Tram Ngheim, Khashayar Mohammadi, and all of my close friends in Toronto and beyond, thank you for grounding me during those eight hours of each day not claimed by work or rest.
To my partner Claudia Edwards, the depth of your love and support over the years goes beyond a simple summary.
Needless to say, thank you, and I reciprocate.
Let us continue forward, helping one another along in the strange process of making dreams come true.
To my cats, Emu and Roti, thank you for keeping me warm when the work felt alienating or isolating.

Thank you to my family now and always for your endless love and support.
  \end{acknowledgements}
  
  \tableofcontents
  \listoftables
  \listoffigures
  
  \mainmatter

  \chapter{Introduction}\label{chap:intro}

\epigraph{``The purpose of a system is what it does.''
}
{Stafford Beer~\citep{beer_what_2002}}

\section{Turning Data Into Computer Programs}

Writing computer programs can be tedious.
Writing computer programs that work well can be difficult.
Machine Learning (ML) provides an attractive approach to writing computer programs that centers the role of data. 
For the purposes of ML, a good dataset is one that measures desired system behavior across the many contexts where the system will be deployed and expected to work reliably.
This thesis addresses the unfortunate reality that the available training data are, in this sense, typically never ``good'' enough.

To be sure: just like standard programming, ML also involves writing a lot of computer programs.
But \emph{these} programs---called learning algorithms---have a distinct role.
Say we have an ML ``task'', that is, an open-ended problem such as detecting objects within a visual scene or translating text from one language to another.
Rather than directly implementing a solution, the learning algorithm consumes training data to produce a ``model'' capable (hopefully) of tackling the task at hand.

ML has garnered much interest recently, resulting in its deepening integration into the social fabric.
A major catalyst of commercial and academic research on ML was the successful application of deep learning to challenging benchmark tasks, especially in the domain of Computer Vision~\citep{krizhevsky_imagenet_2012,rajpurkar_chexnet_2017}.
The connectionist paradigm of using simple learning algorithms to train very ``deep'' neural networks (i.e. models containing many free ``parameters'' to be tuned during the learning process) proved to be incredibly effective in practice, typically out-performing more bespoke learning algorithms that relied on hand-specified features or probabilistic assumptions.
The subsequent successful application of novel neural network architectures to tasks in Natural Language Processing~\citep{vaswani2017attention,devlin2018bert,radford_language_2019} added to the fervor around deep learning as a kind of panacea for high-quality predictive models.

\section{Data Imperfections}

The sheer pace of recent advancements in deep learning is staggering.
And yet the models that deep learning produces can be surprisingly brittle in practice, often failing in unexpected or problematic ways upon deployment~\citep{buolamwini_gender_2018,geirhos2020shortcut}.
In a sense, when any learning algorithm produces a model, it encapsulates certain statistical understandings about past observations.
Intuitively, this model is dataset-dependent, so a failure to acquire adequate training data (in terms of quantity, quality, or both) could have dire consequences, especially considering that people across the globe already interact with ML models as a matter of daily routine.

Deep learning has now moved into the forefront of many application areas, such as online content moderation and recommendation, automotive driver-assistance systems, and machine translation, just to name a few.
This trend reveals several curious tensions around data dependency in modern machine learning, which will be the focus of our attention in this manuscript.
Before we catalog these issues more completely, let us consider just one such tension: the \emph{underspecification} problem.
ML tasks often require predictions to be made using high-dimensional inputs such as images or text sequences.
In such cases, while the learning algorithm produces just one model during training, it turns out that there are many possible models capable of achieving the same degree of predictive performance.
However, upon deployment, especially if the statistics of the deployment differ from what was seeing during training, these models could behave very differently~\citep{damour_underspecification_2020,fisher2019all}.
So we say that the prediction problem itself, when instantiated with a specific training dataset, is ``underspecified''.
To make matters worse, making the network deeper exacerbates this issue~\citep{sagawa_investigation_2020}.
But the heuristic of making the network deeper and deeper is precisely what lead to improvements in predictive performance on benchmark tasks used by the research community, thus catalyzing the (contemporary) success of deep learning in the first place!

What can we make of this apparent contradiction?
Is it the case that as learning algorithms get ``deeper'', the resulting models become more shallow?
There is evidence to suggest that when it comes to aggregate performance, this tends not to be the case~\citep{bartlett_benign_2020,kaplan_scaling_2020}.
However, this thesis takes the position that aggregate measures of performance are poor indicators of the \emph{actual} reliability of ML methods.
In this regard, we can take inspiration from the cyberneticist position that ``the purpose of a system is what it does.''~\citep{beer_what_2002}.
ML models will be deployed in real-world scenarios, possibly involving direct interaction with human subjects, or indirect influence on the well-being of human subjects through automated decisions.
Thus, unexpected model failures can cause real, material harm to people, in spite of any better intentions during the design and training of the ML model.
Accordingly, ML researchers have a responsibility to address these failure modes and work towards ensuring equitable outcomes of emerging ML technologies.

There are many possible avenues for reducing the risk of such real-world harms, most notably in policy, regulation, investigative journalism, and community organizing.
To be certain, effort along all of these fronts is needed: ensuring that ML works reliably is not merely a technical concern.
Many of the most important questions are simultaneously technical \emph{and} non-technical in nature, such formulating prediction tasks (and articulating whether certain tasks should \emph{not} be solved~\citep{barocas_when_2020,hamidi_gender_2018}), and determining which variables to measure as inputs and targets for prediction~\citep{obermeyer2019dissecting}.
However, in this thesis, we narrow our focus and take a more technical approach to characterizing and mitigating failure modes of ML systems.
In particular, we focus on the sensitivity of learning algorithms to their training data.
We use the term ``data imperfections'' to describe insufficiencies in the training data that cause unexpected failure modes in ML learning models.

Many types of (possibility overlapping) data imperfections exist in practical settings.
When data are collected about human subjects, historically marginalized groups are often underrepresented (or overrepresented in association with negative outcomes), raising concerns that the model may produce unfair outcomes~\citep{blodgett2017racial,buolamwini2018gender,tomasev_2021_fairness,hutchinson-etal-2020-social}.
Other data imperfections can involve more subtle aspects of how inputs and targets are measured.
For example, the presence during training of ``spurious features'', such as imperceptible pixel-level patterns in high dimensional images, can lead to brittle models upon deployment~\citep{beery2018recognition,geirhos2020shortcut,degrave_ai_2021}.
In this thesis, the term ``data imperfections'' is used rather loosely.
Rather than providing a precise diagnosis about a specific dataset, it motivates a particular approach to designing learning algorithms, which starts by recognizing that the available training data cannot fully capture intended system behavior.

\section{Outline}
Throughout this manuscript, we will explore three variations on the common theme of training data imperfection.
    Each of the three central Chapters discusses to a unique way that available training data can be insufficient for the task at hand, and presents learning algorithms aimed at mitigating these insufficiencies.
In particular, we cover the following types of data imperfection: 
\begin{itemize}
  \item 
  \textbf{(Chapter \ref{chap:fair-reps}) The training data encodes prior human discrimination.}
  For ethical and/or legal reasons we seek a prediction model that avoids compounding these social biases, which we will do by incorporating demographic information in the learning process.
  In particular, we follow the approach of \emph{Fair Representation Learning} in order to mitigate algorithmic discrimination, and propose novel learning algorithms that promote structure in the learned representations conducive to this goal.
  This Chapter contains work previously published in two conference papers.
  The first is ``Learning Adversarially Fair and Transferable Representations''~\citep{madras2018learning}, which was co-authored with David Madras, Toniann Pitassi, and Richard Zemel.
  The second is ``Flexibly Fair Representation Learning by Disentanglement.''~\citep{creager_flexibly_2019}, which was co-authored with David Madras, J\"orn-Henrik Jacobsen, Marissa Weis, Kevin Swersky, Toniann Pitassi, and Richard Zemel.
  \item
  \textbf{(Chapter \ref{chap:eiil})The training data entails an \emph{underspecified} prediction problem.} There are many predictive functions achieving good predictive accuracy given the available data.
  Some such solutions may rely on \emph{spurious} features that are unreliable in deployment settings.
  Where available, metadata indicating the ``environment'' where each training example was collected can be incorporated as side information into the learning process to mitigate this issue.
  In order to address the more challenging but realistic setting where this side information is not directly available, we introduce a framework of \emph{Environment Inference} that searches for useful side information (environment labels) directly in the training data.
  This Chapter was adapted from the conference paper ``Environment Inference for Invariant Learning''~\citep{creager_environment_2021}, which was co-authored by J\"orn-Henrik Jacobsen and Richard Zemel.
  \item
  \textbf{(Chapter \ref{chap:coda}) The training data does not have sufficient support to solve the task at hand.}
  We shift our focus away from prediction and towards sequential decision making.
  In this context, dynamic programming approaches such as Q-learning provably converge to optimal policies, but only with sufficient coverage over state-action space.
  Many practical RL settings are data-constrained, and so these strict assumptions around coverage do not hold.
  We present data augmentation techniques that rely on prior knowledge over the interaction between objects to improve support of the training data, which can be used a pre-processing step to improve sample efficiency and out-of-distribution robustness of reinforcement learning (RL) agents.
This Chapter contains work previously published in two conference papers.
  The first is ``MoCoDA: Model-based Counterfactual Data Augmentation''~\citep{pitis_mocoda_2022}, which was co-authored with Silviu Pitis, Ajay Mandlekar, and Animesh Garg.
\end{itemize}

We then conclude in Chapter~\ref{chap:conc} by reflecting on the state of contemporary machine learning research and suggesting some fruitful directions for long-term future work.

This thesis comprises my own original research conducted at the University of Toronto.
Yet the ideas described herein were profoundly shaped, through a process of collaboration and cooperation, by the co-authors mentioned above.
Special commendation is due to Silviu Pitis for his leadership of the projects described in Chapter~\ref{chap:coda}, and David Madras for co-leading the LAFTR project described in Chapter~\ref{sec:laftr:laftr}.

\section{Topics Covered in This Manuscript}

\subsection{Robustness}
Our goal then is clear: we wish to derive new learning algorithms capable of producing reliable models in spite of training data imperfections.
In this thesis we will tend to describe such models as ``robust'', a term we use loosely to evoke sturdiness upon deployment in novel contexts.
Because the statistical properties of deployment data often differ (perhaps subtly) from training data, some explicit notion of out-of-distribution (OOD) generalization is often required for a model to be considered robust.
However, within the fields of ML and statistics, ``robust'' learning has several other (related) meanings.
Adversarial robustness~\citep{madry_towards_2019} seeks to the address known sensitivities of deep neural classifiers to small perturbations in the input space, such as imperceptible additive pixel noise for high-dimensional images~\citep{goodfellow_explaining_2015}.
Robust statistics~\citep{rousseeuw1991tutorial} is concerned with deriving inference methods capable of handling outliers or deviations from the assumed parametric form of the data distribution.
Robust optimization formulates mathematical programs where the resulting solution can tolerate measurement noise in the observed variables~\citep{ben-tal_robust_2009}.
In a similar vein, Distributionally Robust Optimization seeks to minimize expected risk over a family of related distributions~\citep{duchi2021statistics}, and has been applied to training deep neural networks~\citep{sagawa_distributionally_2020}.
In causal inference (see below for an overview), ``doubly robust'' estimators are a popular way to use both parametric methods (learned regressors) and non-parametric models (importance weights) to estimate interventional quantities from observational data.
The estimator is called \emph{doubly} robust in the sense that it remains unbiased (in the statistical sense) even if \emph{either} of these two components is biased (although this claim is nullified when both components contain bias).

Robustness serves as a unifying theme throughout the manuscript, and is discussed in each Chapter.
However, the precise design goal---what it means for the learning algorithm to produce a robust model---varies Chapter by Chapter according to the contexts described therein.
\subsection{Domain Generalization}
Learning models capable of OOD generalization is an exceedingly ambitious goal.
It is unrealistic to ensure the model will behave as expected in a deployment context whose properties are entirely unknown during training.
Domain Generalization (DG) is an attempt to make OOD-robust learning more tractable via a more structured training process~\citep{blanchard_generalizing_2011,muandet2013domain}.
In DG, data are collected from a variety of ``domains'', and the learning algorithms is given access to auxiliary labels indicating the domain origin of each training example.
The training domains are distinct from the deployment setting; otherwise no OOD generalization would be required.
Hopefully the statistical variation across training domains is representative of the \emph{type} of distribution shift expected during evaluation.
In this case, the domain labels can provide clues as to which features are unreliable (not predictive in all domains) and thus not worthy of including in the final model.
Domain Adaptation is a related problem setting allows for the model to be updated upon deployment using unlabeled input examples~\citep{patel2015visual}.

DG is the central focus of Chapter~\ref{chap:eiil}.
In Chapter~\ref{chap:coda}, we discuss the MoCoDA algorithm and how it can be used for task \emph{transfer} in offline Reinforcement Learning, where data from a source task is used to learn a policy on a target task.
Transfer learning is also highlighted in Chapter~\ref{chap:fair-reps}.
Both of these settings are indeed related to the goal of OOD generalization.
However, unlike most DG settings, they assume some degree of information about the target distribution (a target task reward function in the former case, and limited labeled target data in the latter).

\subsection{Algorithmic Fairness}
It is also useful to think about robustness not directly in statistical terms, but instead in terms of whether the model adheres to social and legal norms.
This conceptual framing is especially useful for addressing a more insidious type of data imperfection where ML models predict based on immutable characteristics of individuals (or their proxies) such as demographic group membership.
This is a tricky technical issue because we live in a stratified society where social status often provides a very strong predictive signal.
So this is not simply a matter of solving the underspecification problem; we may have to compromise on aggregate accuracy in order to find a model that does not rely on demographic information.
Algorithmic fairness, which seeks various technical mitigations to these issues---is the focus of Chapter~\ref{chap:fair-reps}, and also receives a cursory discussion in Chapter~\ref{chap:eiil}.

% macros
\newcommand{\DP}{DP\xspace}
\newcommand{\EqOdds}{EOdds\xspace}
\newcommand{\EqOpp}{EOpp\xspace}
\newcommand{\Adult}{Adult\xspace}
\newcommand{\Diabetes}{Diabetes\xspace}
\newcommand{\German}{German\xspace}
\newcommand{\HeritageHealth}{HeritageHealth\xspace}
\newcommand{\Yale}{Yale-B\xspace}

\paragraph{Concerns of Automated Discrimination}
The use of algorithms to automate high-stakes decision making raises a serious concern of automated discrimination against vulnerable populations~\citep{barocas2014datas}.
This concern is especially relevant for machine learning systems, as historical disadvantages at the population level are often reflected in data used for training~\citep{mayson2019bias,goel2021accuracy}.
Research on characterizing automated discrimination in ML systems has been greatly influenced by legal precedent in US case law, where concepts such as ``disparate treatment'' and ``disparate impact'' were introduced to justify rulings on Title VII disputes~\citep{barocas2014datas}.
ML systems may also produce predictions that reproduce social inequities even if they do not meet the strict requirement of discrimination in the legal sense; examples include unequal allocation of resources between demographic groups, or replication of harmful stereotypes~\citep{barocas2017problem}.
Thus, we find a variety of potentially problematic ML system behaviors, which have subsequently informed proposals for how ML systems can be made ``fair'' in various ways (discussed briefly below).
Interestingly, a post-hoc analysis of these fairness definitions suggests that each is underpinned with a distinct set of normative assumptions about when it is acceptable (or not) to treat people differently~\citep{binns2018fairness}.

Addressing algorithmic discrimination requires a nuanced look at the data collection process since, since models may produce ``unfair'' predictions even if they do not directly train on explicit instances of human discrimination.
For example, demographic subpopulations may be under-sampled in the training data, which can produce dramatic performance disparities across groups~\citep{buolamwini2018gender}.
It is also important to note that algorithmic systems can have malignant social effects even if they do not directly use machine learning; examples include harmful stereotyping in web search~\citep{noble2018algorithms} and matching algorithms that uphold the status quo in homeless services~\citep{eubanks2018automating}.
While the potential for automated discrimination in ML systems remains a serious concern, algorithmic systems sometimes enable a higher degree of scrutiny than human decision makers\footnote{Crucially, this relies on a transparent training process and interpretable model predictions, which are not found in every ML system.}, suggesting that detecting discrimination (and subsequent legal recourse) could be more tractable with some algorithmic systems~\citep{kleinberg2018discrimination}.

\paragraph{Definitions of ``Fairness''}
As suggested earlier, there are many non-technical means (community organizing, policymaking, etc.) by which algorithmic discrimination can be effectively addressed.
These efforts can be complemented by the pursuit of (statistically) ``fair'' machine learning models, since some forms of algorithmic discrimination manifest as performance disparities in data-driven prediction.
Accordingly, the ML community has proposed many formal definitions of ``fairness''\footnote{
    Interestingly, the standardized testing industry independently discovered
    many of these notions in the 1960s and 1970s \citep{hutchinson201950}.
  } in order to characterize and address disparities in data-driven prediction.
Broadly speaking, these definitions fall under the categories of \emph{Group Fairness} and \emph{Individual Fairness}.
Group fairness criterion can be understood as statistical independence conditions
  expressed via the target labels $Y \in \mathcal{Y}$, 
  sensitive attribute labels $A \in \mathcal{A}$, and the predictions---either 
  hard predictions $\hat Y \in \mathcal{Y}$ or soft scores $R \in \mathbb{R}$ (e.g.,
  $R = p(\hat Y = 1|X)$) where $\hat Y = f(X)$ or $R = f(X)$ are the model
  outputs.
Analyses are typically carried out in the setting of binary target labels and
  sensitive attributes, i.e., $\mathcal{Y} = \mathcal{A} = \{0, 1\}$.
\emph{Demographic Parity}\footnote{
    Also called \emph{statistical parity}.
  } (\DP) requires $\hat Y \perp A$ 
  \citep{calders2010three,kamishima2011fairness,dwork2012fairness}.
\citet{hardt2016equality} note that this requirement may be overly conservative
  from an accuracy perspective when the base rates of the target label differ
  between groups ($p(Y|A=0) \neq p(Y|A=1)$), and propose \emph{Equalized Odds}
  (\EqOdds) as an alternative, requiring $\hat Y \perp A | Y$.
One-sided Equalized Odds---$\hat Y \perp A | Y=1$ or $\hat Y \perp A | Y=0$---
  is called Equal Opportunity (\EqOpp), and is sometimes more precisely referred
  to as Equalized False Positive Rates for the $Y=1$ case and Equalized False
  Negative Rates for the $Y=0$ case.

Another category of important group fairness criteria relate to classifier
  \emph{calibration}, which describes models that output scores consistent with
  the true label probability: $p(Y|A=a, R=r)=r \ \forall \ \{a, r\}$.
Calibration was shown to be
  incompatible with certain parity-based fairness criteria in two influential
  and concurrent works \citet{kleinberg2016inherent, chouldechova2017fair}.

\citet{dwork2012fairness} proposed the canonical notion of individual fairness,
  whereby the classifier must ``treat similar individuals similarly''.
Formally, this entails specifying a similarity metric $d(X, X')$ for pairs of
  individuals, then enforcing a Lipschitz smoothness constraint on the
  classifier output $f(X)$ and $f(X')$ w.r.t. this metric $\forall X, X'$.
While intuitive and flexible\footnote{
    Given a sensitive attribute labels, this framework can also be made to 
    express demographic parity.
  }, this definition shunts the difficulty of defining ``fairness'' onto the
  difficulty of defining the metric $d$, and can thus be difficult to implement in practice.

\paragraph{Fair Classification}
Given a group fairness criteria and a classification task, learning a fair
  classifier can be naturally formulated as a constrained or regularized
  optimization problem.
\citet{calders2010three} describe how to modify naive Bayes with to achieve \DP.
\citet{kamishima2011fairness} describe a more general \DP \ regularization approach
  suitable for classifiers that output a probability score $R=f(X)$ with
  $R=p(\hat Y = Y|X)$.
\citet{hardt2016equality} describe a post-hoc method for converting a naive
  classifier into a \EqOdds \ or \EqOpp \ classifier.

\paragraph{Fair Representation Learning}
The goal of representation learning is broader than that of classification: it
  seeks a \emph{general purpose} representation of the data that admits
  efficient and accurate classification for a \emph{variety} of downstream tasks
  \citep{bengio2013representation}.
The list of admissible representation learning methodologies is also broad,
  encompassing unsupervised learning with autoencoders and variational
  autoencoders \citep{vincent2008extracting,kingma2013auto}, to fine-tuning the
  neural activations of a deep discriminative net \citep{huh2016makes}.

In the context of fair machine learning, group fairness constraints can be
  incorporated into representation learning algorithms to learn an encoding of
  the data that is invariant to the sensitive attributes.
If such an encoding is learned, it serves as a fairness bottleneck that guarantees
  any downstream predictions computed from the representation alone will satisfy
  the desired fairness criteria.

\subsection{Causal Inference}
Causal inference is at attempt to reason about fictitious data in a principled way.
While machine learning provides many expressive tools for modeling correlational patterns in data---from which accurate predictors can be derived---there are many compelling inference questions (typically concerning causes and effects) that cannot be directly answered using correlation patterns in the data.
For example, in order to improve patient health outcomes, we may wish to know:
``given the patient health outcomes observed in a clinical setting, how effective was a particular treatment in the patient population?''
This requires reasoning about events that did not happen, because in practical settings  the treatment was not administered in a uniform or random way.
Rather, each patient received an individualized treatment.
As it turns out, estimating treatment efficacy from standard conditional probabilities can point us in the wrong direction if there were confounding factors in the data causing some patient subpopulations to take the treatment at higher rates.

Fortunately there are methods for answering such questions by directly using real-world data, although they typically rely on strong (and often untestable) assumptions about the problem setting.
There are several influential frameworks for causal inference, most notably the causal graphical models of Pearl and colleagues~\citep{pearl2009causal} and the potential outcomes framework of Rubin and colleagues~\citep{rubin2005causal}.
Chapter \ref{chap:coda} makes use of causal graphs to encode domain knowledge about sequential decision-making problems.
\subsection{Reinforcement Learning}
Reinforcement Learning (RL) provides a data-driven approach to sequential decision making, suitable for producing models that must be deployed in a dynamically-changing setting.
An RL model, called ``policies'', is a mapping from observed state of the setting to an appropriate action.
Crucially, the learning algorithm receives a scalar reward signal every time an action is taken at a given state.
The overall goal is to learn a policy that maximizes long-term reward, there are algorithms that provably converge to an optimal policy under certain assumptions, such as a coverage requirement that all states and actions be visited with some frequency.
One of the main appeals of RL is its broad applicability, as the learning algorithm can learn an effective policy even without a detailed model of the physical domain (or other dynamics) where the policy is learning.
On the other hand, in settings where dynamics \emph{are} known, methods from planning or control theory may yield a solution more efficiently than RL

While there are many RL algorithms of theoretical and practical interest, two general approaches are worth mentioning:
The first approach is value iteration (e.g. Q-learning), where intermediate \emph{value functions}, which encapsulate how much reward a policy can accrue starting from a given state, are estimated using dynamic programming.
Chapter~\ref{chap:coda} addresses the coverage assumption of value iteration algorithms head-on, introducing causal graphical priors as a way to assist learning when coverage is lacking.
Policy gradient methods provide an alternate approach (not discussed in thesis) where the policy function is optimized directly using unbiased estimates of the long-term rewards.
Notably, neural networks have been applied in each of these general approaches as function approximators for the key mathematical objects---value functions, and policies, respectively---to derive useful policies in challenging settings ~\citep{maddison2014move}.
  \chapter{Representation Learning Approaches to Algorithmic Fairness}
\label{chap:fair-reps}
\section{The Role of Data Representations in Fair Machine Learning}
Is it possible for ML to produce models that provide a social utility without replicating inequities in the social context where data was collected?
That is our focus in this chapter, where we discuss algorithmic fairness in the context of classification problems~\citep{pedreshi_discrimination-aware_2008,kamiran_classifying_2009,calders2010three,kamishima2011fairness,hardt2016equality}, and ask \emph{how the data should be formatted} so that prior human discrimination is not directly automated in the learning process~\citep{dwork_fairness_2011,barocas_big_2016,obermeyer2019dissecting}.
We attempt to address this issue by playing to one of ML's strengths: its ability to synthesize a wide variety of data formats automatically by training neural networks on top of the raw data.
The idea is to formulate the learning problem in terms of both predictive accuracy and a coarse-grained notion of statistical fairness across known demographic groups during training.

Crucially, we emphasize that the model's internal \emph{representation} of the data---or example, the internal activations within a deep neural network~\citep{bengio_representation_2013}--- should be fair in this sense.
This is attractive because the learned representation provides a mapping (from the original data space to a new vector space with group fairness properties) that can be reused in new settings~\citep{zemel2013learning,louizos_variational_2016}.
We hope that this mapping will serve as a fairness bottleneck, enabling the use of a single large dataset to train a variety of ML models, each with a distinct prediction task.
It also opens up the possibility of interaction between entities (firms, research groups, or other data-owning collectives) to solve machine learning problems cooperatively.
Typically, these types of collaborations would entail some sort of data asymmetry: despite all parties sharing an interest in solving the task, only one party owns and maintains the data.
Fair representation learning enables the data owner to automatically format their data in a way that is suitable for training machine learning models, while also encouraging fairness in these downstream models.
\section{Encouraging Statistical Parity in Learned Representations with Adversarial Training}
\label{sec:laftr:laftr}

There are two implicit steps involved in every prediction task: acquiring data in a suitable form, and specifying an algorithm that learns to predict well given the data.
In practice these two responsibilities are often assumed by distinct parties.
For example, in online advertising the so-called \emph{prediction vendor} profits by selling its predictions (e.g., person $X$ is likely to be interested in product $Y$) to an advertiser, while the \emph{data owner} profits by selling a predicatively useful dataset to the prediction vendor \cite{dwork2012fairness}. 

Because the prediction vendor seeks to maximize predictive accuracy, it may (intentionally or otherwise) bias the predictions to unfairly favor certain groups or individuals.
The use of machine learning in this context is especially concerning because of its reliance on historical datasets that may include patterns of previous discrimination or other societal bias.
Thus there has been a flurry of recent work from the machine learning community focused on defining and quantifying these biases and proposing new prediction systems that mitigate their impact.

Meanwhile, the data owner also faces a decision that critically affects the predictions: what is the correct \emph{representation} of the data?
Often, this choice of representation is made at the level of data collection: feature selection and 
measurement. 
If we want to maximize the prediction vendor's utility, then the right choice is to simply collect and provide the prediction vendor with as much data as possible. 
However, assuring that prediction vendors learn only \textit{fair} predictors complicates the data owner's choice of representation, which must yield predictors that are never unfair but nevertheless have relatively high utility.

In this Section, we frame the data owner's choice as a representation learning problem with an adversary criticizing potentially unfair solutions.
Our contributions are as follows:
We connect common group fairness metrics (demographic parity, equalize odds, and equal opportunity) to adversarial learning by providing appropriate adversarial objective functions for each metric that upper bounds the unfairness of arbitrary downstream classifiers in the limit of adversarial training;
we distinguish our algorithm from previous approaches to adversarial fairness and discuss its suitability to fair classification due to the novel choice of adversarial objective and emphasis on representation as the focus of adversarial criticism;
we validate experimentally that classifiers trained naively (without fairness constraints) from representations learned by our algorithm achieve their respective fairness desiderata;
furthermore, we show empirically that these representations achieve \textit{fair transfer} --- they admit fair predictors on unseen tasks, even when those predictors are not explicitly specified to be fair. 
\subsection{Background: Fair Classification}
\label{sec:laftr:fair-classification}
In fair classification we have some data $X \in \mathbb{R}^n$, labels $Y \in \{0, 1\}$, and sensitive attributes $A \in \{0, 1\}$. 
The predictor outputs a prediction $\hat{Y} \in \{0, 1\}$. 
We seek to learn to predict outcomes that are accurate with respect to $Y$ but fair with respect to $A$; that is, the predictions are accurate but not biased in favor of one group or the other.

There are many possible criteria for group fairness in this context.
One is \textit{demographic parity}, which ensures that the positive outcome is given to the two groups at the same rate, i.e. $\Pr(\hat{Y} = 1 | A = 0) = \Pr(\hat{Y} = 1 | A = 1)$. 
However, the usefulness of demographic parity can be limited if the \textit{base rates} of the two groups differ, i.e. if $\Pr(Y = 1| A = 0) \neq  \Pr(Y = 1| A = 1)$. 
In this case, we can pose an alternate criterion by conditioning the metric on the ground truth $Y$, yielding \textit{equalized odds} and \textit{equal opportunity} \citep{hardt2016equality}; the former requires equal false positive and false negative rates between the groups while the latter requires only one of these equalities. Equal opportunity is intended to match errors in the ``advantaged'' outcome across groups; whereas \citet{hardt2016equality} chose $Y=1$ as the advantaged outcome, the choice is domain specific and we here use $Y=0$ instead without loss of generality.
Formally, this is $\Pr(\hat{Y} \neq Y | A = 0, Y = y) = \Pr(\hat{Y} \neq Y | A = 1, Y = y)\ \forall \ y \in \{0, 1\}$ (or just $y = 0$ for equal opportunity).

Satisfying these constraints is known to conflict with learning well-calibrated classifiers \citep{chouldechova2017fair,kleinberg2016inherent,pleiss2017fairness}.
It is common to instead optimize a relaxed objective \citep{kamishima2012fairness}, whose hyperparameter values negotiate a tradeoff between maximizing utility (usually classification accuracy) and fairness.

\subsection{Background: Adversarial Learning} \label{adv-background}

Adversarial learning is a popular method of training neural network-based models. 
\citet{goodfellow2014generative} framed learning a deep generative model as a two-player game between a generator $G$ and a discriminator $D$. 
Given a dataset $X$, the generator aims to fool the discriminator by generating convincing synthetic data, i.e., starting from random noise $z \sim p(z)$, $G(z)$ resembles $X$. 
Meanwhile, the discriminator aims to distinguish between real and synthetic data by assigning $D(G(z)) = 0$ and $D(X) = 1$.
Learning proceeds by the max-min optimization of the joint objective
\begin{equation}\nonumber
    V(D, G) \triangleq \mathbb{E}_{p(X)} [\log(D(X))] + \mathbb{E}_{p(z)}  [\log(1 - D(G(z)))]
,
\end{equation}
where $D$ and $G$ seek to maximize and minimize this quantity, respectively.

\subsection{Learning Adversarially Fair and Transferable Representations} \label{gen-model}

We assume a generalized model (Figure \ref{modelpic}), which seeks to learn a data representation $Z$ capable of reconstructing the inputs $X$, classifying the target labels $Y$, and protecting the sensitive attribute $A$ from an adversary.
Either of the first two requirements can be omitted by setting hyperparameters to zero, so the model easily ports to strictly supervised or unsupervised settings as needed.
This general formulation was originally proposed by \citet{edwards2015censoring}; below we address our specific choices of adversarial objectives and explore their fairness implications, which distinguish our work as more closely aligned to the goals of fair representation learning.

%%%%%%%%%%%%%%%%%%%%%%%%%%%%%%%%%%%%%%%%%%%%%%%%%%%%%%%%%%%%%%%%%%%%%%%%%%%%%%%%
% LAFTR graphical model
%%%%%%%%%%%%%%%%%%%%%%%%%%%%%%%%%%%%%%%%%%%%%%%%%%%%%%%%%%%%%%%%%%%%%%%%%%%%%%%%
\newcommand{\arrowWidth}{1.7pt}
\newcommand{\rectWidth}{0.5cm}
\newcommand{\circleSize}{0.9cm}
\begin{figure}[ht!]
\vskip 0.2in
\begin{center}
\begin{tikzpicture}
\path  (6.6,3.3) node[circle,draw,minimum size=\circleSize,align=center](a) {A}
(3,3.3) node[circle,draw,minimum size=\circleSize,align=center](z) {Z}
(-0.6,3.3) node[circle,draw,minimum size=\circleSize](y) {Y}
(3,0) node[circle,draw,minimum size=\circleSize,align=center](x) {X}
(2,1.6) node[rectangle,draw,minimum width=\rectWidth,minimum height=1cm](encoder) {\footnotesize \begin{tabular}{c} Encoder \\ $f(X)$ \end{tabular}}
(4,1.6) node[rectangle,draw,minimum width=\rectWidth,minimum height=1cm](decoder) {\footnotesize \begin{tabular}{c} Decoder \\ $k(Z, A)$ \end{tabular}}
(1.2,3.3) node[rectangle,draw,minimum width=\rectWidth,minimum height=1cm](classifier) {\footnotesize \begin{tabular}{c} Classifier \\ $g(Z)$ \end{tabular}}
(4.8,3.3) node[rectangle,draw,minimum width=\rectWidth,minimum height=1cm](adversary) {\footnotesize \begin{tabular}{c} Adversary \\ $h(Z)$ \end{tabular}};
\draw[->,>=stealth, line width=\arrowWidth] (x) -- (encoder);
\draw[->,>=stealth, line width=\arrowWidth] (decoder) -- (x);
\draw[->,>=stealth, line width=\arrowWidth] (z) -- (decoder);
\draw[->,>=stealth, line width=\arrowWidth] (encoder) -- (z);
\draw[->,>=stealth, line width=\arrowWidth] (z) -- (classifier);
\draw[->,>=stealth, line width=\arrowWidth] (classifier) -- (y);
\draw[->,>=stealth, line width=\arrowWidth] (z) -- (adversary);
\draw[->,>=stealth, line width=\arrowWidth] (adversary) -- (a);
\draw[->,>=stealth, line width=\arrowWidth] (a) to [out=-90,in=0,looseness=1.6] (decoder);
\end{tikzpicture}
\caption{
    Model for learning adversarially fair representations. The variables are data $X$, latent representations $Z$, sensitive attributes $A$, and labels $Y$. The encoder $f$ maps $X$ (and possibly $A$ - not shown) to $Z$, the decoder $k$ reconstructs $X$ from $(Z, A)$, the classifier $g$ predicts $Y$ from $Z$, and the adversary $h$ predicts $A$ from $Z$ (and possibly $Y$ - not shown). 
}
\label{modelpic}
\end{center}
\vskip -0.2in
\end{figure}
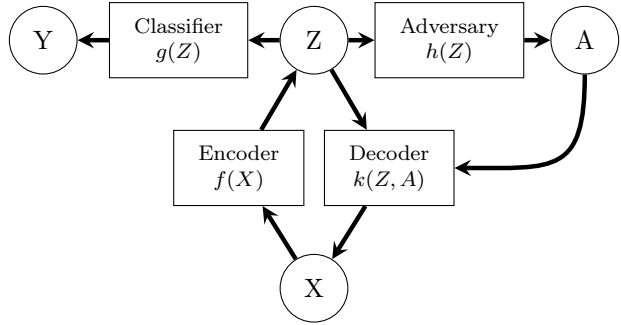
%%%%%%%%%%%%%%%%%%%%%%%%%%%%%%%%%%%%%%%%%%%%%%%%%%%%%%%%%%%%%%%%%%%%%%%%%%%%%%%%

The dataset consists of tuples $(X, A, Y)$ in $\mathbb{R}^{n}$, $\{0, 1\}$ and $\{0, 1\}$, respectively.
The encoder $f: \mathbb{R}^{n} \rightarrow \mathbb{R}^m$ yields the representations $Z$.
The encoder can also optionally receive $A$ as input.
The classifier and adversary\footnote{
In learning equalized odds or equal opportunity representations, the adversary $h: \mathbb{R}^{m} \times \{0, 1\} \rightarrow \{0, 1\}$ also takes the label $Y$ as input. 
} $g, h: \mathbb{R}^m \rightarrow \{0, 1\}$ each act on $Z$ and attempt to predict $Y$ and $A$, respectively.
Optionally, a decoder $k: \mathbb{R}^{m} \times \{0, 1\} \rightarrow \mathbb{R}^n$ attempts to reconstruct the original data from the representation and the sensitive variable.

The adversary $h$ seeks to maximize its objective $L_{Adv} (h(f(X, A)), A)$.
We discuss a novel and theoretically motivated adversarial objective in Sections \ref{gen-model}, whose exact terms are modified according to the fairness desideratum.

Meanwhile, the encoder, decoder, and classifier jointly seek to minimize classification loss and reconstruction error, and also minimize the adversary's objective. 
Let $L_C$ denote a suitable classification loss (e.g., cross entropy, $\ell_1$), and $L_{Dec}$ denote a suitable reconstruction loss (e.g., $\ell_2$).
Then we train the generalized model according to the following min-max procedure:
\begin{equation} \label{minmax}
\begin{aligned}
\underset{f, g, k}{\minimize} \medspace &\underset{h}{\maximize} \medspace \E_{X,Y,A} \left[ L(f, g, h, k) \right]
,
\end{aligned}
\end{equation}
with the combined objective expressed as
\begin{equation} \label{eq:objective}
\begin{aligned}
L(f, g, h, k) &= \alpha L_C (g(f(X, A)), Y)\\
&\quad + \beta L_{Dec}(k(f(X, A), A), X)\\
    &\quad\quad+ \gamma L_{Adv} (h(f(X, A)), A)
\end{aligned}
\end{equation}

The hyperparameters $\alpha, \beta, \gamma$ respectively specify a desired balance between utility, reconstruction of the inputs, and fairness.

Due to the novel focus on fair transfer learning, we call our model Learned Adversarially Fair and Transferable Representations (LAFTR).

\paragraph{Learning Objective}\label{sec:learning}
We realize $f$, $g$, $h$, and $k$ as neural networks and alternate gradient decent and ascent steps to optimize their parameters according to (\ref{eq:objective}). 
First the encoder-classifier-decoder group ($f,g,k$) takes a gradient step to minimize $L$ while the adversary $h$ is fixed, then $h$ takes a step to maximize $L$ with fixed ($f,g,k$).
Computing gradients necessitates relaxing the binary functions $g$ and $h$, the details of which are discussed below.

For shorthand we denote the adversarial objective $L_{Adv} (h(f(X, A)), A)$---whose functional form depends on the desired fairness criteria---as $L_{Adv}(h)$.
For demographic parity, we take the average absolute difference on each sensitive group $\mathcal{D}_0, \mathcal{D}_1$:
\begin{equation}\label{eq:adv-obj}
    L_{Adv}^{DP}(h) = 1 - \sum_{i \in \{0, 1\}} \frac{1}{|\mathcal{D}_i|}\sum_{(x, a) \in \mathcal{D}_i} | h(f(x, a)) - a |\\ 
\end{equation}
For equalized odds, we take the average absolute difference on each sensitive group-label combination $\mathcal{D}_0^0, \mathcal{D}_1^0, \mathcal{D}_0^1, \mathcal{D}_1^1$, where $\mathcal{D}_i^j = \{(x, y, a) \in \mathcal{D} | a = i, y = j\}$:
\begin{equation}\label{eq:adv-eo}
    L_{Adv}^{EO}(h) = 2 - \sum_{(i, j) \in \{0, 1\}^2} \frac{1}{|\mathcal{D}_i^j|}\sum_{(x, a) \in \mathcal{D}_i^j}  | h(f(x, a)) - a |\\ 
\end{equation}
To achieve equal opportunity, we need only sum terms corresponding to $Y = 0$. 

\paragraph{Motivation}
For intuition on this approach and the upcoming theoretical section, we return to the \emph{owner-vendor} framework previously discussed, with a data \textit{owner} who sells representations to a (prediction) \textit{vendor}.
Suppose the data owner is concerned about the unfairness in the predictions made by vendors who use their data. Given that vendors are strategic actors with goals, the owner may wish to guard against two types of vendors:
\begin{itemize}
\item The \textit{indifferent} vendor: this vendor is concerned with utility maximization, and doesn't care about the fairness or unfairness of their predictions.
\item The \textit{adversarial} vendor: this vendor will attempt to actively discriminate by the sensitive attribute.
\end{itemize}
In the adversarial model defined in Section \ref{gen-model}, the encoder is what the data owner really wants; this yields the representations which will be sold to vendors.
When the encoder is learned, the other two parts of the model ensure that the representations respond appropriately to each type of vendor: the classifier ensures utility by simulating an indifferent vendor with a prediction task, and the adversary ensures fairness by simulating an adversarial vendor with discriminatory goals.
It is important to the data owner that the model's adversary be as strong as possible---if it is too weak, the owner will underestimate the unfairness enacted by the adversarial vendor.

However, there is another important reason why the model should have a strong adversary, which is key to our theoretical results.
Intuitively, the degree of unfairness achieved by the adversarial vendor (who is optimizing for unfairness) will not be exceeded by the indifferent vendor.
Beating a strong adversary $h$ during training implies that downstream classifiers naively trained on the learned representation $Z$ must also act fairly.
Crucially, this fairness bound depends on the discriminative power of $h$; this motivates our use of the representation $Z$ as a direct input to $h$, because it yields a strictly more powerful $h$ and thus tighter unfairness bound than adversarially training on the predictions and labels alone as in \citet{zhang2018mitigating}.

\paragraph{Theoretical Properties} \label{theory}
We now draw a connection between our choice of adversarial objective and several common metrics from the fair classification literature.
We derive adversarial upper bounds on unfairness that can be used in adversarial training to achieve either demographic parity, equalized odds, or equal opportunity.

We are interested in quantitatively comparing two distributions corresponding to the learned group representations, so consider two distributions $\mathcal{D}_0$ and $\mathcal{D}_1$ over the same sample space $\Omega_{\mathcal{D}}$, as well as a binary test function $\mu: \Omega_{\mathcal{D}} \rightarrow \{0, 1\}$. 
$\mu$ is called a test since it can distinguish between samples from the two distributions according to the absolute difference in its expected value.
We call this quantity the \emph{test discrepancy} and express it as
\begin{equation}
    d_\mu(\mathcal{D}_0, \mathcal{D}_1) \triangleq | \underset{x \sim \mathcal{D}_0}{\E} \left[ \mu(x) \right] - \underset{x \sim \mathcal{D}_1}{\E} \left[ \mu(x) \right] |
.
\end{equation}
The \textit{statistical distance} (a.k.a. total variation distance) between distributions is defined as the maximum attainable test discrepancy \citep{cover2012elements}:
\begin{equation} \label{stat_dist}
\begin{aligned}
    \Delta^*(\mathcal{D}_0, \mathcal{D}_1) \triangleq \sup_{\mu} d_\mu(\mathcal{D}_0, \mathcal{D}_1)
.
\end{aligned}
\end{equation}

When learning fair representations we are interested in the distribution of $Z$ conditioned on a specific value of group membership $A \in \{0, 1\}$.
As a shorthand we denote the distributions $p(Z|A=0)$ and $p(Z|A=1)$ as $\mathcal{Z}_0$ and $\mathcal{Z}_1$, respectively.

\paragraph{Bounding Demographic Parity}
In supervised learning we seek a $g$ that accurately predicts some label $Y$; in fair supervised learning we also want to quantify $g$ according to the fairness metrics discussed in Section \ref{sec:laftr:fair-classification}.
For example, the demographic parity distance is expressed as the absolute expected difference in classifier outcomes between the two groups:
\begin{equation}
    \Delta_{DP}(g) \triangleq d_{g}(\mathcal{Z}_0,\mathcal{Z}_1) = | \E_{\mathcal{Z}_0}[g] - \E_{\mathcal{Z}_1}[g] |
    .
\end{equation}
Note that $\Delta_{DP}(g) \leq \Delta^*(\mathcal{Z}_0,\mathcal{Z}_1)$, and also that $\Delta_{DP}(g) = 0$ if and only if $g(Z) \independent A$, i.e., demographic parity has been achieved.

Now consider an adversary $h: \Omega_{\mathcal{Z}} \rightarrow \{0, 1\}$ whose objective (negative loss) function\footnote{
    This is equivalent to the objective expressed by Equation \ref{eq:adv-obj} when the expectations are evaluated with finite samples.    
} is expressed as 
\begin{equation}\label{eq:adv-reward}
    L_{Adv}^{DP}(h) \triangleq \E_{\mathcal{Z}_0}[1-h] + \E_{\mathcal{Z}_1}[h] - 1 
    .
\end{equation}
Given samples from $\Omega_{\mathcal{Z}}$ the adversary seeks to correctly predict the value of $A$, and learns by maximizing $L_{Adv}^{DP}(h)$.
Given an optimal adversary trained to maximize (\ref{eq:adv-reward}), the adversary's loss will bound $\Delta_{DP}(g)$ from above for any function $g$ learnable from $Z$.
Thus a sufficiently powerful adversary $h$ will expose through the value of its objective the demographic disparity of any classifier $g$.
We will later use this to motivate learning fair representations via an encoder $f: \mathcal{X} \rightarrow \mathcal{Z}$ that simultaneously minimizes the task loss and minimizes the adversary objective.

\textbf{Theorem.} \textit{
Consider a classifier $g:\Omega_{\mathcal{Z}}\rightarrow \Omega_{\mathcal{Y}}$ and adversary $h:\Omega_{\mathcal{Z}}\rightarrow \Omega_{\mathcal{A}}$ as binary functions, i.e,. $\Omega_{\mathcal{Y}}=\Omega_{\mathcal{A}}= \{0, 1\}$. 
Then $L_{Adv}^{DP}(h^*) \geq \Delta_{DP}(g)$: the demographic parity distance of $g$ is bounded above by the optimal objective value of $h$.
}

\textbf{Proof.} 
By definition $\Delta_{DP}(g) \geq 0$.
Suppose without loss of generality (WLOG) that $\E_{\mathcal{Z}_0}[g] \geq \E_{\mathcal{Z}_1}[g]$, i.e., the classifier predicts the ``positive'' outcome more often for group $A_0$ in expectation. 
Then, an immediate corollary is $\E_{\mathcal{Z}_1}[1-g] \geq \E_{\mathcal{Z}_0}[1-g]$, and we can drop the absolute value in our expression of the disparate impact distance:
\begin{equation}
\begin{aligned}
\Delta_{DP}(g) = \E_{\mathcal{Z}_0}[g] - \E_{\mathcal{Z}_1}[g] = \E_{\mathcal{Z}_0}[g] + \E_{\mathcal{Z}_1}[1 - g] - 1 
\end{aligned}
\end{equation}
where the second equality is due to $\E_{\mathcal{Z}_1}[g] = 1 - \E_{\mathcal{Z}_1}[1-g]$.
Now consider an adversary that guesses the opposite of $g$
, i.e., $h = 1 - g$. 
Then\footnote{
Before we assumed WLOG $\E_{\mathcal{Z}_0}[g] \geq \E_{\mathcal{Z}_1}[g]$. 
If instead $\E_{\mathcal{Z}_0}[g] < \E_{\mathcal{Z}_1}[g]$ then we simply choose $h = g$ instead to achieve the same result.
}, we have 
\begin{equation} \label{tight-DP-bound}
\begin{aligned}
    L_{Adv}^{DP}(h) = L_{Adv}^{DP}(1-g) &= \E_{\mathcal{Z}_0}[g] + \E_{\mathcal{Z}_1}[1-g] - 1 \\
    &= \Delta_{DP}(g)
\end{aligned}
\end{equation}
The optimal adversary $h^\star$ does at least as well as any arbitrary choice of $h$, therefore $L_{Adv}^{DP}(h^\star) \geq L_{Adv}^{DP}(h) = \Delta_{DP}$.
\hfill$\blacksquare$

In Appendix \ref{app:laftr:theory-eq-odds} we provide an analogous derivation for equalized odds~\citep{hardt2016equality}.

\paragraph{Additional points}
One interesting note is that in each proof, we provided an example of an adversary which was calculated only from the joint distribution of $Y, A$, and $\hat Y = g(Z)$---we did not require direct access to $Z$---and this adversary achieved a loss exactly equal to the quantity in question ($\Delta_{DP}$ or $\Delta_{EO}$). Therefore, if we only allow our adversary access to those outputs, our adversarial objective (assuming an optimal adversary), is equivalent to simply adding either $\Delta_{DP}$ or $\Delta_{EO}$ to our classification objective, similar to common regularization approaches  \citep{kamishima2012fairness,bechavod2017learning,madras2017predict,zafar2017fairness}.
Below we consider a stronger adversary, with direct access to the key intermediate learned representation $Z$. This allows for a potentially greater upper bound for the degree of unfairness, which in turn forces any classifier trained on $Z$ to act fairly.

In our proofs we have considered the classifier $g$ and adversary $h$ as binary functions.
In practice we want to learn these functions by gradient-based optimization, so we instead substitute their continuous relaxations $\tilde g, \tilde h: \Omega_{\mathcal{Z}} \rightarrow [0,1]$.
By viewing the continuous output as parameterizing a Bernoulli distribution over outcomes we can follow the same steps in our earlier proofs to show that in both cases (demographic parity and equalized odds)  $\E[L(\bar h^*)] \geq \E[\Delta(\bar g)]$, where $\bar h^*$ and $\bar g$ are randomized binary classifiers parameterized by the outputs of $\tilde h^*$ and $\tilde g$.

\paragraph{Comparison to \citet{edwards2015censoring}}
An alternative to optimizing the expectation of the randomized classifier $\tilde h$ is to minimize its negative log likelihood (NLL - also known as cross entropy loss), given by
\begin{equation}\label{eq:bernoulli}
L(\tilde h) = -\E_{Z,A} \left[ A \log \tilde h(Z) + (1-A) \log (1-\tilde h(Z)) \right]
.
\end{equation}
This is the formulation adopted by
\citet{ganin2016domain} and \citet{edwards2015censoring}, which propose maximizing (\ref{eq:bernoulli}) as a proxy for computing the statistical distance\footnote{
These papers discuss the bound on $\Delta_{DP}(g)$ in terms of the $\mathcal{H}$-divergence \cite{blitzer2006domain}, which is simply the statistical distance $\Delta^*$ up to a multiplicative constant.
}
$\Delta^*(\mathcal{Z}_0, \mathcal{Z}_1)$ during adversarial training. 

The adversarial loss we adopt here instead of cross-entropy is group-normalized $\ell_1$, defined in Equations $\ref{eq:adv-obj}$ and $\ref{eq:adv-eo}$.
The main problems with cross entropy loss in this setting arise from the fact that the adversarial objective should be calculating the test discrepancy.
However, the cross entropy objective sometimes fails to do so, for example when the dataset is imbalanced.
Furthermore, group normalized $\ell_1$ corresponds to a more natural relaxation of the fairness metrics in question.
It is important that the adversarial objective incentivizes the test discrepancy, as group-normalized $\ell_1$ does; this encourages the adversary to get an objective value as close to $\Delta^\star$ as possible, which is key for fairness.% (see Section \ref{intuition}).
In practice, optimizing $\ell_1$ loss with gradients can be difficult, so while we suggest it as a suitable theoretically-motivated continuous relaxation for our model (and present experimental results), there may be other suitable options beyond those considered in this work.

\subsection{Experiments}\label{sec:laftr:experiments}

\paragraph{Fair classification}\label{sec:fair-classification}
\begin{figure*}[ht!]
\centering
\begin{subfigure}[t]{0.33\textwidth}
\centering
\includegraphics[width=\textwidth]{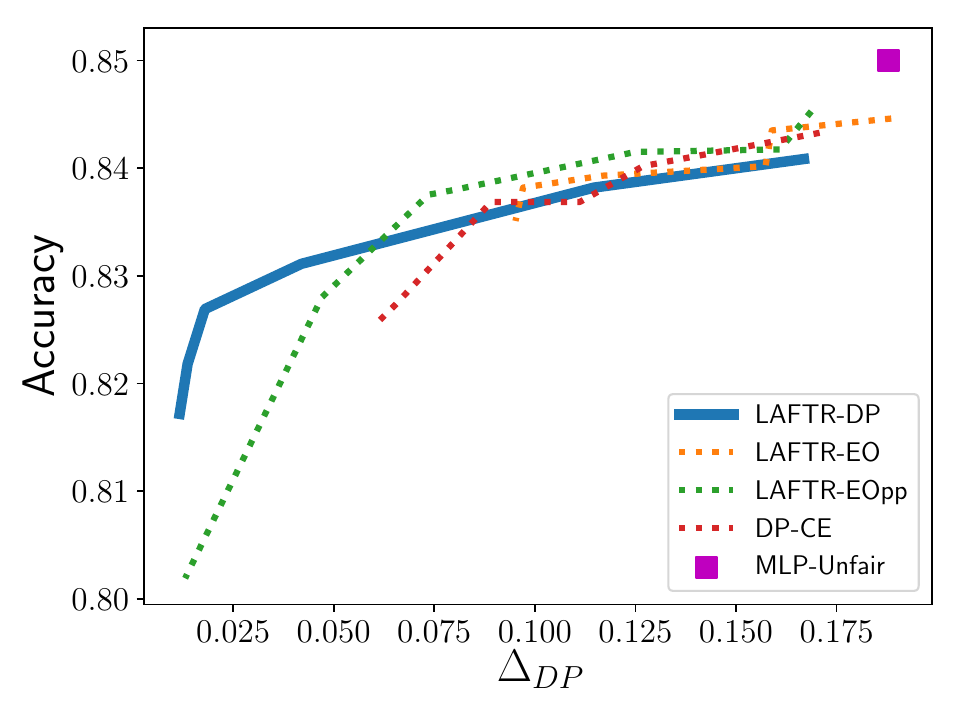}
\caption{Tradeoff between accuracy and $\Delta_{DP}$}
\label{results:pareto-DP}
\end{subfigure}%
\hfill
\begin{subfigure}[t]{0.33\textwidth}
\centering
\includegraphics[width=\textwidth]{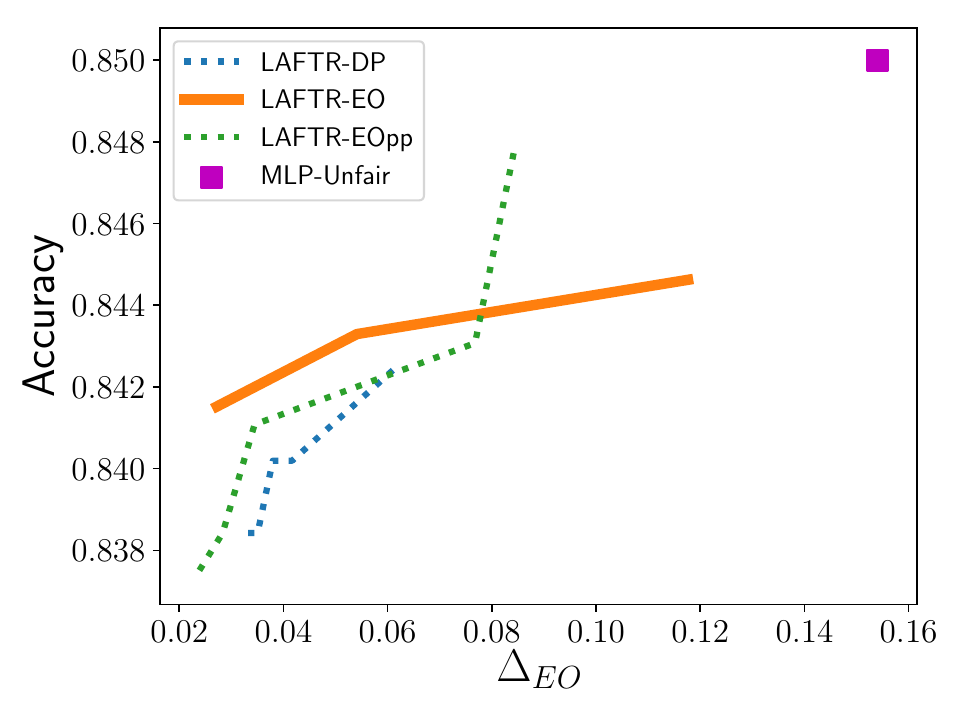}
\caption{Tradeoff between accuracy and $\Delta_{EO}$}
\label{results:pareto-EO}
\end{subfigure}%
\hfill
\begin{subfigure}[t]{0.33\textwidth}
\centering
\includegraphics[width=\textwidth]{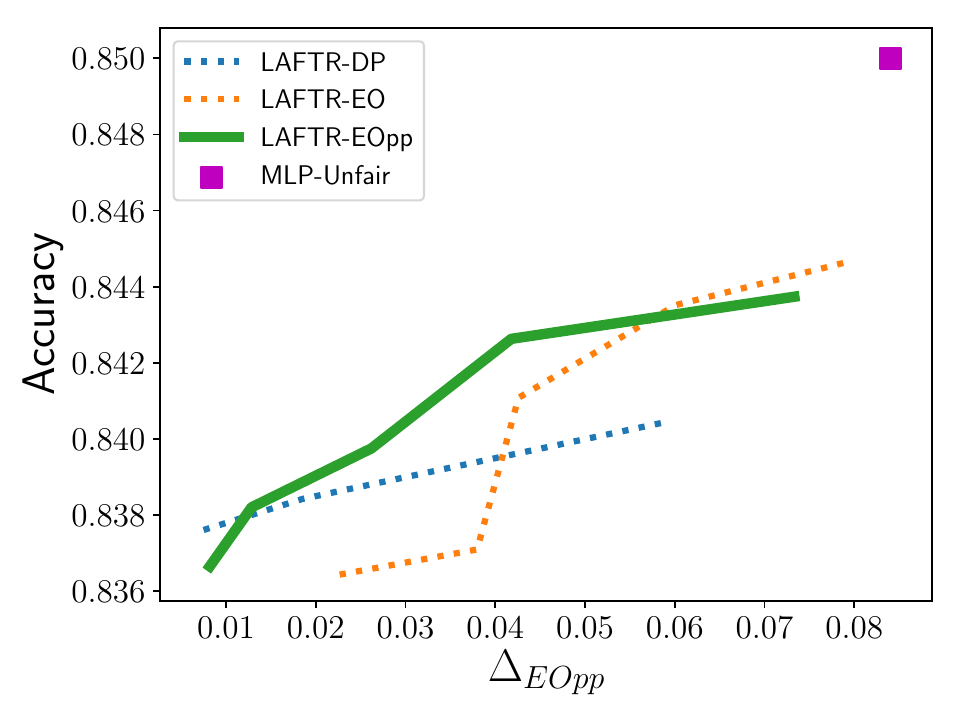}
\caption{Tradeoff between accuracy and $\Delta_{EOpp}$}
\label{results:pareto-EOpp}
\end{subfigure}%
    \caption{
        Accuracy-fairness tradeoffs for various fairness metrics ($\Delta_{DP}$, $\Delta_{EO}$, $\Delta_{EOpp}$), and LAFTR adversarial objectives $(L_{Adv}^{DP}, L_{Adv}^{EO}, L_{Adv}^{EOpp})$ on fair classification of the Adult dataset.
Upper-left corner (high accuracy, low $\Delta$) is preferable.
Figure \ref{results:pareto-DP} also compares to a cross-entropy adversarial objective \citep{edwards2015censoring}, denoted DP-CE.
Curves are generated by sweeping a range of fairness coefficients $\gamma$, taking the median across 7 runs per $\gamma$, and computing the Pareto front.
In each plot, the bolded line is the one we expect to perform the best.
Magenta square is a baseline MLP with no fairness constraints.
see Algorithm \ref{alg:laftr} and Appendix \ref{sec:training-details}.
    }
\label{results:pareto-fairness}

\end{figure*}

LAFTR seeks to learn an encoder yielding fair representations, i.e., the encoder's outputs can be used by third parties with the assurance that their naively trained classifiers will be reasonably fair and accurate.
Thus we evaluate the quality of the encoder according to the following training procedure, also described in pseudo-code by Algorithm 1.
Using labels $Y$, sensitive attribute $A$, and data $X$, we train an encoder using the adversarial method outlined in Section \ref{gen-model}, receiving both $X$ and $A$ as inputs.
We then freeze the learned encoder; from now on we use it only to output representations $Z$. 
Then, using unseen data, we train a classifier on top of the frozen encoder. The classifier learns to predict $Y$ from $Z$ --- note, this classifier is not trained to be fair. 
We can then evaluate the accuracy and fairness of this classifier on a test set to assess the quality of the learned representations.

During Step 1 of Algorithm 1, the learning algorithm is specified either as a baseline (e.g., unfair MLP) or as LAFTR, i.e., stochastic gradient-based optimization of (Equation \ref{minmax}) with one of the three adversarial objectives described in Section \ref{gen-model}.
When LAFTR is used in Step 1, all but the encoder $f$ are discarded in Step 2.
For all experiments we use cross entropy loss for the classifier (we observed training instability with other classifier losses).
The classifier $g$ in Step 3 is a feed-forward MLP trained with SGD.
See Appendix \ref{sec:training-details} for details.
\begin{algorithm}[tb]\captionsetup{labelfont={sc,bf}}
    \caption{Evaluation scheme for fair classification ($Y'=Y$) \& transfer learning ($Y' \neq Y$).}
   \label{alg:laftr}
% NOTE: for some reason \STATE had to be changed to \State in order for this
%       algo box to be compiled as a ut-thesis document
\begin{algorithmic}
   \State {\bfseries Input:} data $X \in \Omega_X$, sensitive attribute $A \in \Omega_A$, labels $Y, Y' \in \Omega_Y$, representation space $\Omega_Z$.
   \State {\bfseries Step 1:} Learn an encoder $f: \Omega_X \rightarrow \Omega_Z$ using data $X$, task label $Y$, and sensitive attribute $A$.
   \State {\bfseries Step 2:} Freeze $f$.
   \State {\bfseries Step 3:} Learn a classifier (without fairness constraints) $g: \Omega_Z \rightarrow \Omega_Y$ on top of $f$, using  data $f(X')$, task label $Y'$, and sensitive attribute $A'$.
   \State {\bfseries Step 3:} Evaluate the fairness and accuracy of the composed classifier $g \circ f: \Omega_X \rightarrow \Omega_Y$ on held out test data, for task $Y'$.
\end{algorithmic}
\end{algorithm}

We evaluate the performance of our model\footnote{See \href{https://github.com/VectorInstitute/laftr}{https://github.com/VectorInstitute/laftr} for code.
}on fair classification on the UCI Adult dataset\footnote{https://archive.ics.uci.edu/ml/datasets/adult}, which contains over 40,000 rows of information describing adults from the 1994 US Census. 
We aimed to predict each person's income category (either greater or less than 50K/year). 
We took the sensitive attribute to be gender, which was listed as Male or Female. 

Figure \ref{results:pareto-fairness} shows classification results on the Adult dataset.
Each sub-figure shows the accuracy-fairness trade-off (for varying values of $\gamma$; we set $\alpha = 1, \beta = 0$ for all classification experiments) evaluated according to one of the group fairness metrics: $\Delta_{DP}$, $\Delta_{EO}$, and $\Delta_{EOpp}$.
For each fairness metric, we show the trade-off curves for LAFTR trained under three adversarial objectives: $L_{Adv}^{DP}$, $L_{Adv}^{EO}$, and $L_{Adv}^{EOpp}$.
We observe, especially in the most important regiment for fairness (small $\Delta$), that the adversarial objective we propose for a particular fairness metric tends to achieve the best trade-off.
Furthermore, in Figure \ref{results:pareto-DP}, we compare our proposed adversarial objective for demographic parity with the one proposed in \citep{edwards2015censoring}, finding a similar result.

For low values of un-fairness, i.e., minimal violations of the respective fairness criteria, the LAFTR model trained to optimize the target criteria obtains the highest test accuracy.
While the improvements are somewhat uneven for other regions of fairness-accuracy space (which we attribute to instability of adversarial training), this demonstrates the potential of our proposed objectives.
However, the fairness of our model's learned representations are not limited to the task it is trained on.
We now turn to experiments which demonstrate the utility of our model in learning fair representations for a variety of tasks.

\paragraph{Transfer Learning}
In this section, we show the promise of our model for \textit{fair transfer learning}.
As far as we know, beyond a brief introduction in \citet{zemel2013learning}, we provide the first in-depth experimental results on this task, which is pertinent to the common situation where the data owner and vendor are separate entities.

We examine the Heritage Health dataset\footnote{https://www.kaggle.com/c/hhp}, which comprises insurance claims and physician records relating to the health and hospitalization of over 60,000 patients.
We predict the Charlson Index, a comorbidity indicator that estimates the risk of patient death in the next several years. 
We binarize the (non negative) Charlson Index as zero/nonzero. 
We took the sensitive variable as binarized age (thresholded at 70 years old). 
This dataset contains information on sex, age, lab test, prescription, and claim details. 

The task is as follows: using data $X$, sensitive attribute $A$, and labels $Y$, learn an encoding function $f$ such that given unseen $X'$, the representations produced by $f(X', A)$ can be used to learn a fair predictor for new task labels $Y'$, even if the new predictor is being learned by a vendor who is indifferent or adversarial to fairness.
This is an intuitively desirable condition: if the data owner can guarantee that predictors learned from their representations will be fair, then there is no need to impose fairness restrictions on vendors, or to rely on their goodwill.

The original task is to predict Charlson index $Y$ fairly with respect to age $A$.
The transfer tasks relate to the various primary condition group (PCG) labels, each of which indicates a patient's insurance claim corresponding to a specific medical condition.
PCG labels $\{Y'\}$ were held out during LAFTR training but presumably correlate to varying degrees with the original label $Y$.
The prediction task was binary: did a patient file an insurance claim for a given PCG label in this year? For various patients, this was true for zero, one, or many PCG labels. 
There were 46 different PCG labels in the dataset; we considered only used the 10 most common---whose positive base rates ranged from 9-60\%---as transfer tasks.

Our experimental procedure was as follows. 
To learn representations that transfer fairly, we used the same model as described above, but set our reconstruction coefficient $\beta = 1$.
Without this, the adversary will stamp out any information not relevant to the label from the representation, which will hurt transferability.
We can optionally set our classification coefficient $\alpha$ to 0, which worked better in practice.
Note that although the classifier $g$ is no longer involved when $\alpha = 0$, the target task labels are still relevant for either equalized odds or equal opportunity transfer fairness.

We split our test set ($\sim20,000$ examples) into transfer-train, -validation, and -test sets. 
We trained LAFTR ($\alpha = 0, \beta = 1$, $\ell_2$ loss for the decoder) on the full training set, and then only kept the encoder. 
In the results reported here, we trained using the equalized odds adversarial objective described in Section \ref{gen-model}; similar results were obtained with the other adversarial objectives.
Then, we created a feed-forward model which consisted of our frozen, adversarially-learned encoder followed by an MLP with one hidden layer, with a loss function of cross entropy with no fairness modifications. 
Then, $\forall \ i \in 1 \dots 10$, we trained this feed-forward model on PCG label $i$ (using the transfer-train and -validation) sets, and tested it on the transfer-test set.
This procedure is described in Algorithm \ref{alg:laftr}, with $Y'$ taking 10 values in turn, and $Y$ remaining constant ($Y \neq Y'$).

We trained four models to test our method against. 
The first was an MLP predicting the PCG label directly from the data (Target-Unfair), with no separate representation learning involved and no fairness criteria in the objective---this provides an effective upper bound for classification accuracy. 
The others all involved learning separate representations on the original task, and freezing the encoder as previously described; the internal representations of MLPs have been shown to contain useful information \cite{hinton2006reducing}.
These (and LAFTR) can be seen as the values of \textsc{ReprLearn} in Alg. \ref{alg:laftr}.
In two models, we learned the original $Y$ using an MLP (one regularized for fairness \citep{bechavod2017learning}, one not; Transfer-Fair and -Unfair, respectively) and trained for the transfer task on its internal representations. 
As a third baseline, we trained an adversarial model similar to the one proposed in \cite{zhang2018mitigating}, where the adversary has access only to the classifier output $\hat{Y} = g(Z)$ and the ground truth label (Transfer-Y-Adv), to investigate the utility of our adversary having access to the underlying representation, rather than just the joint classification statistics $(Y, A, \hat{Y})$.

%%%%%%%%%%%%%%%%%%%%%%%%%%%%%%%%%%%%%%%%%%%%%%%%%%%%%%%%%%%%%%%%%%%%%%%%%%%%%%%%
% LAFTR transfer learning results (chart)
%%%%%%%%%%%%%%%%%%%%%%%%%%%%%%%%%%%%%%%%%%%%%%%%%%%%%%%%%%%%%%%%%%%%%%%%%%%%%%%%
\begin{figure}[ht!]
\vskip 0.2in
\begin{center}
\centerline{\includegraphics[width=0.5\textwidth]{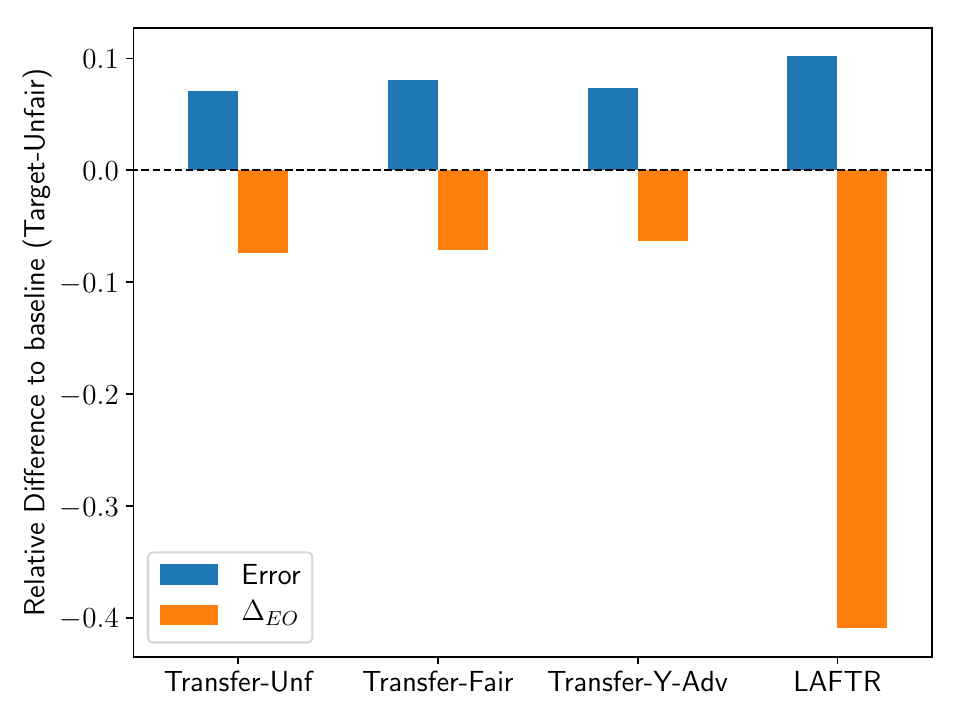}}
\caption{Fair transfer learning on Health dataset. Displaying average across 10 transfer tasks of relative difference in error and $\Delta_{EO}$ unfairness (the lower the better for both metrics), as compared to a baseline unfair model learned directly from the data. -0.10 means a 10\% decrease. Transfer-Unf and -Fair are MLP's with and without fairness restrictions respectively, Transfer-Y-Adv is an adversarial model with access to the classifier output rather than the underlying representations, and LAFTR is our model trained with the adversarial equalized odds objective.}
\label{transfer-learn-results-pic}
\end{center}
\vskip -0.2in
\end{figure}
%%%%%%%%%%%%%%%%%%%%%%%%%%%%%%%%%%%%%%%%%%%%%%%%%%%%%%%%%%%%%%%%%%%%%%%%%%%%%%%%

We report our results in Figure \ref{transfer-learn-results-pic} and Table \ref{health-transfer}.
In Figure \ref{transfer-learn-results-pic}, we show the relative change from the high-accuracy baseline learned directly from the data for both classification error and $\Delta_{EO}$.
LAFTR shows a clear improvement in fairness; it improves $\Delta_{EO}$ on average from the non-transfer baseline, and the relative difference is an average of $\sim$20\%, which is much larger than other baselines.
We also see that LAFTR's loss in accuracy is only marginally worse than other models.

%%%%%%%%%%%%%%%%%%%%%%%%%%%%%%%%%%%%%%%%%%%%%%%%%%%%%%%%%%%%%%%%%%%%%%%%%%%%%%%%
% LAFTR transfer learning results (table)
%%%%%%%%%%%%%%%%%%%%%%%%%%%%%%%%%%%%%%%%%%%%%%%%%%%%%%%%%%%%%%%%%%%%%%%%%%%%%%%%
\begin{table}[!h]
\caption{Results from Figure \ref{transfer-learn-results-pic} broken out by task. $\Delta_{EO}$ for each of the 10 transfer tasks is shown, which entails identifying a primary condition code that refers to a particular medical condition. Most fair on each task is bolded. All model names are abbreviated from Figure \ref{transfer-learn-results-pic}; ``TarUnf'' is a baseline, unfair predictor learned directly from the target data without a fairness objective.}
\vskip 0.15in
\begin{center}
\begin{small}
\begin{sc}
\tabcolsep=0.08cm
\begin{tabular}{cccccc}
\toprule
Tra. Task  & TarUnf & TraUnf & TraFair & TraY-AF& LAFTR \\
\midrule
MSC2a3 & 0.362 & 0.370 &  0.381 & 0.378 & \textbf{0.281}\\
METAB3 & 0.510 & 0.579 &  \textbf{0.436} & 0.478 &0.439\\
ARTHSPIN & 0.280 & 0.323 &  0.373 & 0.337 &\textbf{0.188}\\
NEUMENT & 0.419 & 0.419 &  0.332 & 0.450 &\textbf{0.199}\\
RESPR4 & 0.181 & 0.160 &  0.223 & 0.091 &\textbf{0.051}\\
MISCHRT & 0.217 & 0.213 &  0.171 & 0.206 &\textbf{0.095}\\
SKNAUT & 0.324 & \textbf{0.125} &  0.205 & 0.315 & 0.155\\
GIBLEED & 0.189 & 0.176 &  0.141 & 0.187 & \textbf{0.110}\\
INFEC4 & 0.106 & 0.042 &  0.026 & \textbf{0.012} & 0.044\\
TRAUMA & 0.020 & 0.028 & 0.032 & 0.032 & \textbf{0.019}\\
\bottomrule
\end{tabular}
\end{sc}
\end{small}
\end{center}
\label{health-transfer}
\vskip -0.1in
\end{table}
%%%%%%%%%%%%%%%%%%%%%%%%%%%%%%%%%%%%%%%%%%%%%%%%%%%%%%%%%%%%%%%%%%%%%%%%%%%%%%%%

A fairly-regularized MLP (``Transfer-Fair'') does not actually produce fair representations during transfer; on average it yields similar fairness results to transferring representations learned without fairness constraints.
Another observation is that the output-only adversarial model (``Transfer Y-Adv'') produces similar transfer results to the regularized MLP.
This shows the practical gain of using an adversary that can observe the representations.
Since transfer fairness varied much more than accuracy, we break out the results of Fig. \ref{transfer-learn-results-pic} in Table \ref{health-transfer}, showing the fairness outcome of each of the 10 separate transfer tasks.
We note that LAFTR provides the fairest predictions on 7 of the 10 tasks, often by a wide margin, and is never too far behind the fairest model for each task.
The unfair model TraUnf achieved the best fairness on one task.
We suspect this is due to some of these tasks being relatively easy to solve without relying on the sensitive attribute by proxy.
Since the equalized odds metric is better aligned with accuracy than demographic parity \citep{hardt2016equality}, high accuracy classifiers can sometimes achieve good $\Delta_{EO}$ if they do not rely on the sensitive attribute by proxy. 
Because the data owner has no knowledge of the downstream task, however, our results suggest that using LAFTR is safer than using the raw inputs;
LAFTR is relatively fair even when TraUnf is the most fair, whereas TraUnf is dramatically less fair than LAFTR on several tasks.

%%%%%%%%%%%%%%%%%%%%%%%%%%%%%%%%%%%%%%%%%%%%%%%%%%%%%%%%%%%%%%%%%%%%%%%%%%%%%%%%
% LAFTR kenrel measurements of representations
%%%%%%%%%%%%%%%%%%%%%%%%%%%%%%%%%%%%%%%%%%%%%%%%%%%%%%%%%%%%%%%%%%%%%%%%%%%%%%%%
\begin{table}[!h]
\caption{Transfer fairness, other metrics. Models are as defined in Figure \ref{transfer-learn-results-pic}.  MMD is calculated with a Gaussian RBF kernel ($\sigma = 1$). AdvAcc is the accuracy of a separate MLP trained on the representations to predict the sensitive attribute; due to data imbalance an adversary predicting 0 on each case obtains accuracy of approximately 0.74.}
\vskip 0.15in
\begin{center}
\begin{small}
\begin{sc}
\begin{tabular}{ccc}
\toprule
Model & MMD & AdvAcc \\
\midrule
Transfer-Unfair & $1.1 \times 10^{-2}$  & 0.787\\
Transfer-Fair & $1.4 \times 10^{-3}$ & 0.784 \\
Transfer-Y-Adv ($\beta = 1$) & $3.4 \times 10^{-5}$ & 0.787 \\
Transfer-Y-Adv ($\beta = 0$) & $1.1 \times 10^{-3}$ & 0.786 \\
LAFTR & $\mathbf{2.7 \times 10^{-5}}$ & \textbf{0.761}\\
\bottomrule
\end{tabular}
\end{sc}
\end{small}
\end{center}
\label{health-transfer-other}
\vskip -0.1in
\end{table}
%%%%%%%%%%%%%%%%%%%%%%%%%%%%%%%%%%%%%%%%%%%%%%%%%%%%%%%%%%%%%%%%%%%%%%%%%%%%%%%%

We provide coarser metrics of fairness for our representations in Table \ref{health-transfer-other}. We give two metrics: maximum mean discrepancy (MMD) \citep{gretton2007kernel}, which is a general measure of distributional distance; and adversarial accuracy (if an adversary is given these representations, how well can it learn to predict the sensitive attribute?). 
In both metrics, our representations are more fair than the baselines. We give two versions of the ``Transfer-Y-Adv'' adversarial model ($\beta = 0, 1$); note that it has much better MMD when the reconstruction term is added, but that this does not improve its adversarial accuracy, indicating that our model is doing something more sophisticated than simply matching moments of distributions.

\section{Flexible Fairness: Disentanglement Priors for Subgroup-fair Prediction}
\label{sec:ffvae:ffvae}

The previous Section emphasized transferability as a key property of fair representations, which provides a type of \emph{flexibility} in terms of downstream tasks.
However, this approach is \emph{inflexible} with respect to sensitive attributes: while a single learned representation can adapt to the prediction of different task labels~$y$, the single sensitive attribute $a$ for \emph{all} tasks must be specified at train time.
Mis-specified or overly constraining train-time sensitive attributes could negatively affect performance on downstream prediction tasks.
Can we instead learn a \textit{flexibly fair} representation that can be \emph{adapted}, at test time, to be fair to a variety of protected groups and their intersections?
That is the focus of this Section.

A flexibly fair representation should satisfy two criteria.
Firstly, the structure of the latent code should facilitate \textit{simple} adaptation, allowing a practitioner to easily adapt the representation to a variety of fair classification settings, where each task may have a different task label $y$ and sensitive attributes $a$.
Secondly, the adaptations should be \textit{compositional}: the representations can be made fair with respect to conjunctions of sensitive attributes, to guard against performance disparities between demographic \emph{subgroups} (e.g., a classifier that is fair to women but not Black women over the age of 60).
This type of subgroup unfairness has been observed in commercial machine learning systems \citep{buolamwini2018gender}.

Towards these goals, we draw inspiration from the disentangled representation literature, where the goal is for each dimension of the representation (also called the ``latent code'') to correspond to no more than one semantic factor of variation in the data (for example, independent visual features like object shape and position) \citep{higgins2016beta, locatello2018challenging}.
Our method uses multiple sensitive attribute labels at train time to induce a disentangled structure in the learned representation, which allows us to easily eliminate their influence at test time.
Importantly, at test time our method does not require access to the sensitive attributes, which can be difficult to collect in practice due to legal restrictions \citep{elliot2008new,division1983}.
The trained representation permits simple and composable modifications at test time that eliminate the influence of sensitive attributes, enabling a wide variety of downstream tasks.

\subsection{Background: Variational Autoencoders}
\label{sec:ffvae:background-vae}
The vanilla Variational Autoencoder (VAE) \citep{kingma2013auto} is typically implemented with an isotropic Gaussian prior $p(z) = \mathcal{N}(0, I)$.
The objective to be maximized is the Evidence Lower Bound (a.k.a., the ELBO), 
\begin{align}
    L_{\text{VAE}}(p,q) = \E_{q(z|x)} \left[ \log p(x|z) \right] - D_{KL} \left[ q(z|x) || p(z)  \right], \nonumber
\end{align}
which bounds the data log likelihood $\log p(x)$ from below for any choice of $q$.
The encoder and decoder are often implemented as Gaussians 
\begin{align}
    q(z|x) &= \mathcal{N}(z|\mu_q(x),  \Sigma_q(x)) \nonumber \\
    p(x|z) &= \mathcal{N}(x|\mu_p(z), \Sigma_p(z)) \nonumber
\end{align}
whose distributional parameters are the outputs of neural networks $\mu_q(\cdot)$, $\Sigma_q(\cdot)$, $\mu_p(\cdot)$, $\Sigma_p(\cdot)$, with the $\Sigma$ typically exhibiting diagonal structure.
For modeling binary-valued pixels, a Bernoulli decoder $p(x|z)=\text{Bernoulli}(x|\theta_p(z))$ can be used.
The goal is to maximize $L_{\text{VAE}}$---which is made differentiable by re-parameterizing samples from $q(z|x)$---w.r.t. the network parameters.

\paragraph{$\beta$-VAE}
\citet{higgins2016beta} modify the VAE objective:
\begin{align}
    L_{\beta\text{VAE}}(p,q) = \E_{q(z|x)} \left[ \log p(x|z) \right] - \beta D_{KL} \left[ q(z|x) || p(z) \right ]. \nonumber
\end{align}
The hyperparameter $\beta$ allows the practitioner to encourage the variational distribution $q(z|x)$ to reduce its KL-divergence to the isotropic Gaussian prior $p(z)$.
With $\beta>1$ this objective is a valid lower bound on the data likelihood. 
This gives greater control over the model's adherence to the prior.
Because the prior factorizes per dimension $p(z) = \prod_j p(z_j)$, \citet{higgins2016beta} argue that increasing $\beta$ yields ``disentangled'' latent codes in the encoder distribution $q(z|x)$.
Broadly speaking, each dimension of a properly disentangled latent code should capture no more than one semantically meaningful factor of variation in the data.
This allows the factors to be manipulated in isolation by altering the per-dimension values of the latent code.
Disentangled autoencoders are often evaluated by their sample quality in the data domain, but we instead emphasize the role of the encoder as a representation learner to be evaluated on downstream fair classification tasks.

\paragraph{FactorVAE and $\beta$-TCVAE}
\citet{kim2018disentangling} propose a different variant of the VAE objective:
\begin{align}%\label{eq:factorvae}
    L_{\text{FactorVAE}}(p,q) = L_{\text{VAE}}(p,q) - \gamma D_{KL}(q(z)||\prod_{j} q(z_j)). \nonumber
\end{align}
The main idea is to encourage factorization of the aggregate posterior $q(z)=\E_{p^{\text{data}}(x)} \left[ q(z|x) \right]$ so that $z_i$ correlates with $z_j$ if and only if $i=j$.
The authors propose a simple trick to generate samples from the aggregate posterior $q(z)$ and its marginals $\{q(z_j)\}$ using shuffled minibatch indices, then approximate the $D_{KL}(q(z)||\prod_{j} q(z_j))$ term using the cross entropy loss of a classifier that distinguishes between the two sets of samples, which yields a mini-max optimization.

\citet{chen2018isolating} show that the $D_{KL}(q(z)||\prod_{j} q(z_j))$ term above---a.k.a. the ``total correlation'' of the latent code---can be naturally derived by decomposing the expected KL divergence from the variational posterior to prior:
\begin{align} \nonumber% \label{eq:kl-decomposition}
    \E_{p^{\text{data}}(x)}[&D_{KL}(q(z|x)||p(z))] = \\
    &\quad\quad D_{KL}(q(z|x)p^{\text{data}}(x)||q(z)p^{\text{data}}(x)) \nonumber \\
        &\quad\quad+ D_{KL}(q(z)||\prod_j q(z_j))
     \nonumber\\
    &\quad\quad + \sum_j D_{KL} \left[ q(z_j) || p(z_j) \right]. \nonumber
\end{align}
They then augment the decomposed ELBO to arrive at the same objective as \citet{kim2018disentangling}, but optimize using a biased estimate of the marginal probabilities $q(z_j)$ rather than with the adversarial bound on the KL between aggregate posterior and its marginals.

\subsection{Flexibly Fair VAE} \label{sec:ffvae:method}

%%%%%%%%%%%%%%%%%%%%%%%%%%%%%%%%%%%%%%%%%%%%%%%%%%%%%%%%%%%%%%%%%%%%%%%%%%%%%%%%
% FFVAE graphical model
%%%%%%%%%%%%%%%%%%%%%%%%%%%%%%%%%%%%%%%%%%%%%%%%%%%%%%%%%%%%%%%%%%%%%%%%%%%%%%%%
\renewcommand{\arrowWidth}{1.7pt}
\renewcommand{\rectWidth}{0.4cm}
\renewcommand{\circleSize}{0.7cm}
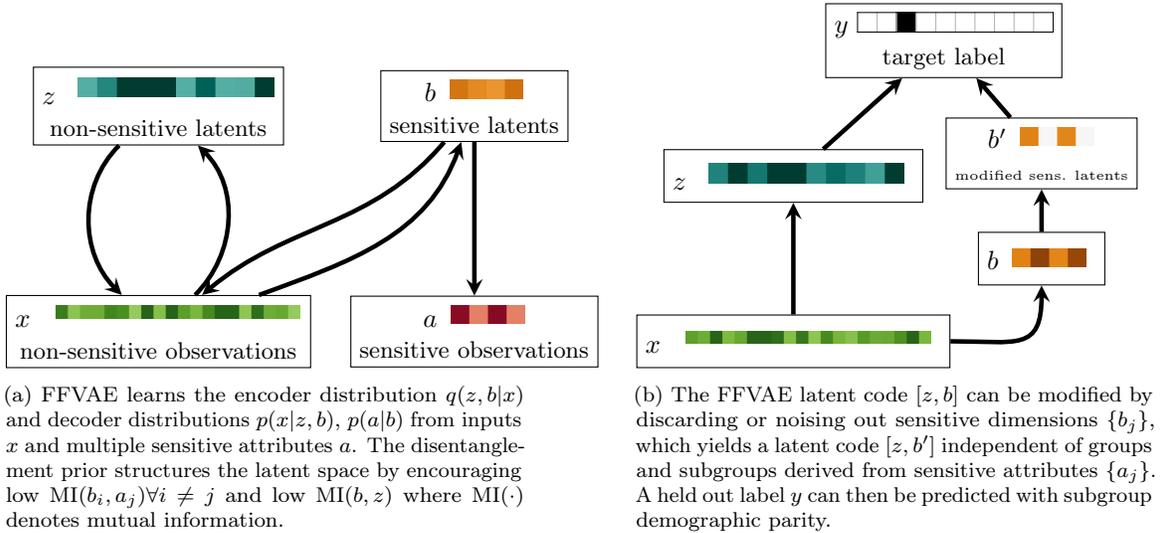
\begin{figure*}[htp]
\begin{center}
\begin{subfigure}[t]{0.45\textwidth}
\begin{tikzpicture}
\path (0.0,0.0) node[rectangle,draw,minimum size=\circleSize,align=center](x) {
    $x$ \includegraphics[scale=0.2]{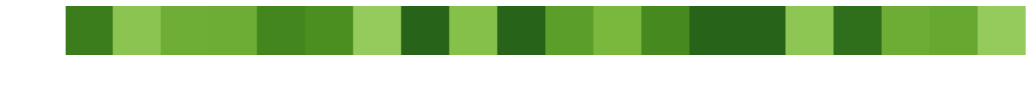}\\\small non-sensitive observations
}
(4.2,0.0) node[rectangle,draw,minimum size=\circleSize,align=center](a) {
    $a$ \includegraphics[scale=0.06]{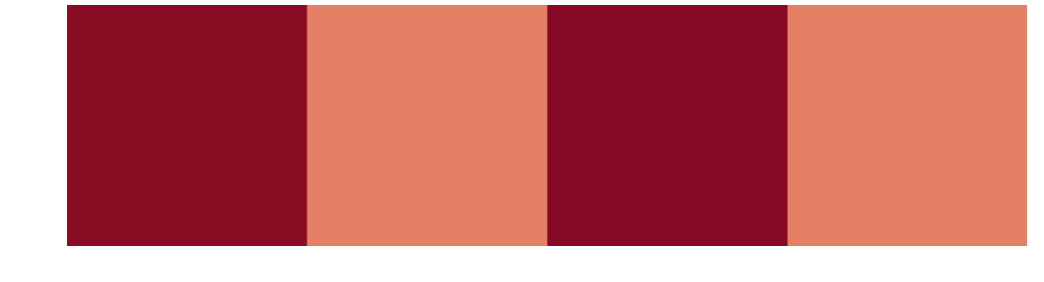}\\\small sensitive observations
}
(0.0,3.0) node[rectangle,draw,minimum size=\circleSize,align=center](z) {
    $z$ \includegraphics[scale=0.16]{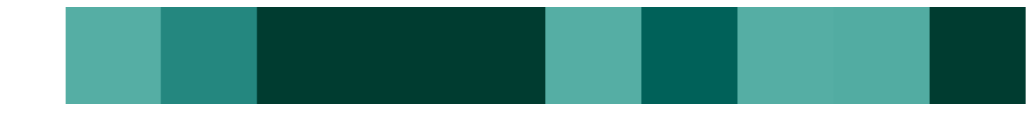}\\\small non-sensitive latents
}
(4.2,3.0) node[rectangle,draw,minimum size=\circleSize,align=center](b) {
    $b$ \includegraphics[scale=0.06]{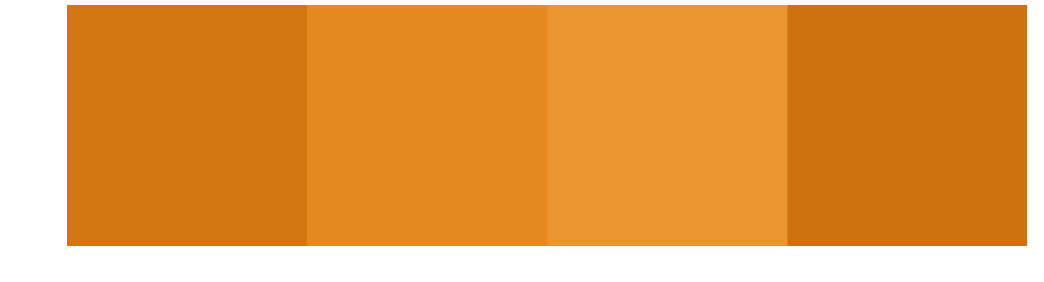}\\\small sensitive latents
};
\draw[->,>=stealth, line width=\arrowWidth] (x) to [out=45,in=-45,looseness=1.0] (z);
\draw[->,>=stealth, line width=\arrowWidth] (z) to [out=-135,in=135,looseness=1.0] (x);
\draw[->,>=stealth, line width=\arrowWidth] (x) to [out=20,in=-110,looseness=1.0] (b);
\draw[->,>=stealth, line width=\arrowWidth] (b) to [out=-130,in=40,looseness=1.0] (x);
\draw[->,>=stealth, line width=\arrowWidth] (b) -- (a);
\end{tikzpicture}
    \caption{
    FFVAE learns the encoder distribution $q(z,b|x)$ and decoder distributions $p(x|z,b)$, $p(a|b)$ from inputs $x$ and multiple sensitive attributes $a$.
    The disentanglement prior structures the latent space by encouraging low $\text{MI}(b_i,a_j)\forall i \neq j$ and low $\text{MI}(b,z)$ where $\text{MI}(\cdot)$ denotes mutual information.
    }
\label{fig:1a}
\end{subfigure}
\hfill
\begin{subfigure}[t]{0.45\textwidth}
\begin{tikzpicture}
\path
(0.0-1.8,0.0) node[rectangle,draw,minimum size=\circleSize,align=center](x) {
    $x$ \includegraphics[scale=0.2]{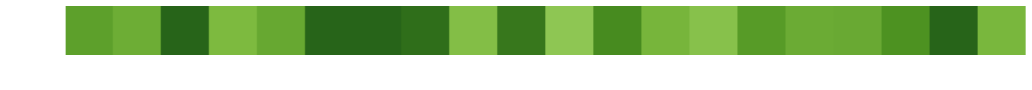}
}
(0.0-1.8,2.2) node[rectangle,draw,minimum size=\circleSize,align=center](z) {
    $z$ \includegraphics[scale=0.16]{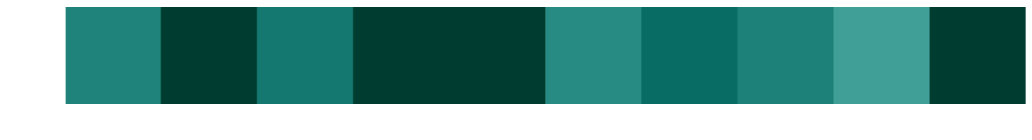}
}
(4.0-2.5,1.1) node[rectangle,draw,minimum size=\circleSize,align=center](b) {
    $b$ \includegraphics[scale=0.06]{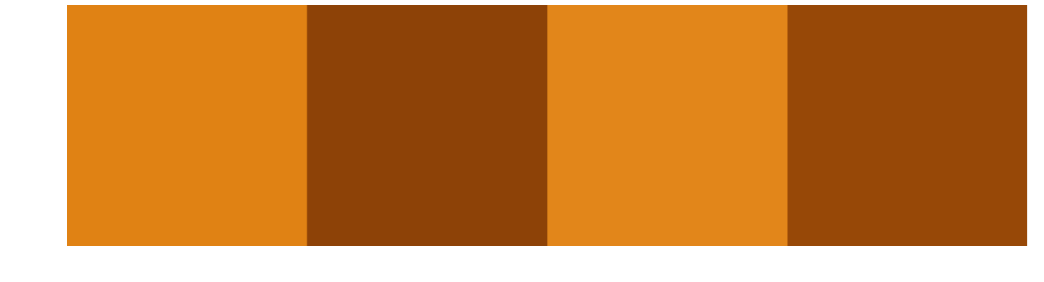}
}
(4.0-2.5,2.5) node[rectangle,draw,minimum size=\circleSize,align=center](bprime) {
    $b'$ \includegraphics[scale=0.06]{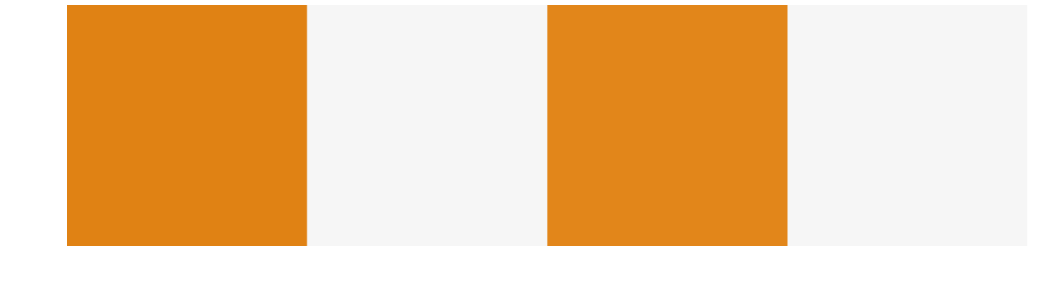}\\\tiny modified sens. latents
}
(2.0-1.8,4.0) node[rectangle,draw,minimum size=\circleSize,align=center](y) {
    $y$ \includegraphics[scale=0.16]{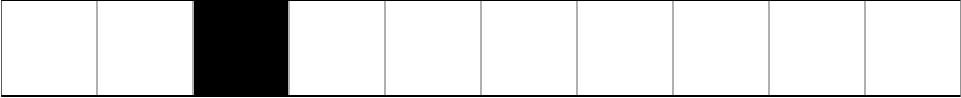}\\\small target label
};
\draw[->,>=stealth, line width=\arrowWidth] (x) -- (z);
\draw[->,>=stealth, line width=\arrowWidth] (x) to [out=0,in=-90,looseness=1.6] (b);
\draw[->,>=stealth, line width=\arrowWidth] (b) -- (bprime);
\draw[->,>=stealth, line width=\arrowWidth] (bprime) -- (y);
\draw[->,>=stealth, line width=\arrowWidth] (z) -- (y);
\end{tikzpicture}
    \caption{The FFVAE latent code $[z,b]$ can be modified by discarding or noising out sensitive dimensions $\{b_j\}$, which yields a latent code $[z, b']$ independent of groups and subgroups derived from sensitive attributes $\{a_j\}$.
A held out label $y$ can then be predicted with subgroup demographic parity.
}
\label{fig:1b}
\end{subfigure}
\caption{
    Data flow at train time (\ref{fig:1a}) and test time (\ref{fig:1b}) for our model, Flexibly Fair VAE (FFVAE).
}
\label{modelpic}
\end{center}
\vskip -0.2in
\label{fig:1}
\end{figure*}
%%%%%%%%%%%%%%%%%%%%%%%%%%%%%%%%%%%%%%%%%%%%%%%%%%%%%%%%%%%%%%%%%%%%%%%%%%%%%%%%

We want to learn fair representations that---beyond being useful for predicting many test-time task labels $y$---can be adapted \textit{simply} and \textit{compositionally} for a variety of sensitive attributes settings $a$ after training.
We call this property \textit{flexible fairness}.
Our approach to this problem involves inducing structure in the latent code that allows for easy manipulation.
Specifically, we isolate information about each sensitive attribute to a specific subspace, while ensuring that the latent space factorizes these subspaces independently. 

\paragraph{Notation}
We employ the following notation:
\begin{itemize}
    \item $x \in \mathcal{X}$: a vector of non-sensitive attributes, for example, the pixel values in an image or row of features in a tabular dataset;
    \item $a \in \{0,1\}^{N_a}$: a vector of binary sensitive attributes;
    \item $z \in \R^{N_z}$: non-sensitive subspace of the latent code;
    \item $b \in \R^{N_b}$: sensitive subspace of the latent code\footnote{
    In our experiments we used $N_b = N_a$ (same number of sensitive attributes as sensitive latent dimensions) to model binary sensitive attributes. 
    But categorical or continuous sensitive attributes can also be accommodated. 
    }.
\end{itemize}
For example, we can express the VAE objective in this notation as
\begin{align}
    L_{\text{VAE}}(p,q) &= \E_{q(z,b|x,a)} \left[ \log p(x,a|z,b) \right] \nonumber \\
    &\quad - D_{KL} \left[ q(z,b|x,a) || p(z,b)  \right]. \nonumber
\end{align}

In learning a flexibly fair representations $[z,b]=f([x,a])$, we aim to satisfy two general properties: \emph{disentanglement} and \emph{predictiveness}.
We say that $[z,b]$ is \emph{disentangled} if its aggregate posterior factorizes as $q(z,b)=q(z)\prod_j q(b_j)$ and is \emph{predictive} if each $b_i$ has high mutual information with the corresponding $a_i$.
Note that under the disentanglement criteria the dimensions of $z$ are free to co-vary together, but must be independent from all sensitive subspaces ${b_j}$.
We have also specified factorization of the latent space in terms of the aggregate posterior $q(z,b)=\E_{p^{\text{data}}(x)} [q(z,b|x)]$, to match the global independence criteria of group fairness.

\paragraph{Desiderata}
We can formally express our desiderata as follows:
\begin{itemize}
    \item $z \perp b_j \medspace \forall \medspace j$ (disentanglement of the non-sensitive and sensitive latent dimensions);
    \item $b_i \perp b_j \medspace \forall \medspace i\neq j$ (disentanglement of the various different sensitive dimensions);
    \item $\text{MI}(a_j,  b_j)$ is large $\forall \medspace j$ (predictiveness of each sensitive dimension);
\end{itemize}
where $\text{MI}(u,v)=\E_{p(u,v)} \log \frac{p(u,v)}{p(u)p(v)}$ represents the mutual information between random vectors $u$ and $v$.
We note that these desiderata differ in two ways from the standard disentanglement criteria.
The predictiveness requirements are stronger: they allow for the injection of external information into the latent representation, requiring the model to structure its latent code to align with that external information.
However, the disentanglement requirement is less restrictive since it allows for correlations between the dimensions of $z$.
Since those are the non-sensitive dimensions, we are not interested in manipulating those at test time, and so we have no need for constraining them.

If we satisfy these criteria, then it is possible to achieve demographic parity with respect to some $a_i$ by simply removing the dimension $b_i$ from the learned representation i.e. use instead $[z,b]\backslash b_i$.
We can alternatively replace $b_i$ with independent noise.
This adaptation procedure is simple and compositional: if we wish to achieve fairness with respect to a conjunction of binary attributes\footnote{
$\wedge$ and $\vee$ represent logical \emph{and} and \emph{or} operations, respectively.
} $a_i \wedge a_j \wedge a_k$, we can simply use the representation $[z,b]\backslash \{b_i, b_j, b_k\}$.

By comparison, while FactorVAE may disentangle dimensions of the aggregate posterior---$q(z)=\prod_j q(z_j)$---it does not automatically satisfy flexible fairness, since the representations are not predictive, and cannot necessarily be easily modified along the attributes of interest.

\paragraph{Distributions}
We propose a variation to the VAE which encourages our desiderata, building on methods for disentanglement and encouraging predictiveness. 
Firstly, we assume assume a variational posterior that factorizes across $z$ and $b$:
\begin{align}
    q(z,b|x) &= q(z|x)q(b|x).
\end{align}
The parameters of these distributions are implemented as neural network outputs, with the encoder network yielding a tuple of parameters for each input: $(\mu_q(x), \Sigma_q(x), \theta_q(x))=\text{Encoder}(x)$.
We then specify $q(z|x)=\mathcal{N}(z|\mu_q(x), \Sigma_q(x))$ and $q(b|x)=\delta(\theta_q(x))$ (i.e., $b$ is non-stochastic)\footnote{
We experimented with several distributions for modeling $b|x$ stochastically, but modeling this uncertainty did not help optimization or downstream evaluation in our experiments.
}.

Secondly, we model reconstruction of $x$ and prediction of $a$ separately using a factorized decoder:
\begin{align}
    p(x,a|z,b)=p(x|z,b)p(a|b)
\end{align}
where $p(x|z,b)$ is the decoder distribution suitably chosen for the inputs $x$, and 
\begin{align}
    p(a|b)=\prod_j\text{Bernoulli}(a_j|\sigma(b_j))
\end{align}
is a factorized binary classifier that uses $b_j$ as the logit for predicting $a_j$ ($\sigma(\cdot)$ represents the sigmoid function).
Note that the $p(a|b)$ factor of the decoder requires no extra parameters.

Finally, we specify a factorized prior $p(z,b)=p(z)p(b)$ with $p(z)$ as a standard Gaussian and $p(b)$ as Uniform.

\paragraph{Learning Objective}
Using the encoder and decoder as defined above, we present our final objective:
\begin{align} \label{eq:ffvae}
    \nonumber L_{\text{FFVAE}}(p,q) &= \E_{q(z,b|x)} [ \log p(x|z,b) + \alpha \log p(a|b)] \nonumber \\
    &\quad\quad - \gamma D_{KL}(q(z,b)||q(z)\prod_j q(b_j))
     \nonumber\\
    &\quad\quad - D_{KL} \left[ q(z,b|x) || p(z,b) \right].
\end{align}
It comprises the following four terms, respectively:
a \textit{reconstruction} term which rewards the model for successfully modeling non-sensitive observations;
a \textit{predictiveness} term which rewards the model for aligning the correct latent components with the sensitive attributes; 
a \textit{disentanglement} term which rewards the model for decorrelating the latent dimensions of $b$ from each other and $z$;
and a \textit{dimension-wise KL} term which rewards the model for matching the prior in the latent variables.
We call our model \mbox{FFVAE} for Flexibly Fair VAE (see Figure \ref{modelpic} for a schematic representation).

The hyperparameters $\alpha$ and $\gamma$ control aspects relevant to flexible fairness of the representation. $\alpha$ controls the alignment of each $a_j$ to its corresponding $b_j$ (predictiveness), whereas $\gamma$ controls the aggregate independence in the latent code (disentanglement).

The $\gamma$-weighted total correlation term is realized by training a binary adversary  to approximate the log density ratio $\log \frac{q(z,b)}{q(z)\prod_j q(b_j)}$.
The adversary attempts to classify between ``true'' samples from the aggregate posterior $q(z,b)$ and ``fake'' samples from the product of the marginals $q(z)\prod_j q(b_j)$ (see Appendix \ref{sec:disc_approx} for further details).
If a strong adversary can do no better than random chance, then the desired independence property has been achieved.

We note that our model requires the sensitive attributes $a$ at training time but not at test time.
This is advantageous, since often these attributes can be difficult to collect from users, due to practical and legal restrictions, particularly for sensitive information \citep{elliot2008new,division1983}.

\subsection{Experiments}
\label{sec:ffvae:experiments}
\paragraph{Evaluation Criteria}

We evaluate the learned encoders with an ``auditing'' scheme on held-out data. 
The overall procedure is as follows:
\begin{enumerate}
    \item \textbf{Split data} into a \textit{training} set (for learning the encoder) and an \textit{audit} set (for evaluating the encoder).
    \item \textbf{Train} an encoder/representation using the training set.
    \item \textbf{Audit} the learned encoder. Freeze the encoder weights and train an MLP to predict some task label given the (possibly modified) encoder outputs on the audit set.
\end{enumerate}

To evaluate various properties of the encoder we conduct three types of auditing tasks---\emph{fair classification}, \emph{predictiveness}, and \emph{disentanglement}---which vary in task label and representation modification.
The \emph{fair classification audit} \citep{madras2018learning} trains an MLP to predict $y$ (held-out from encoder training) given $[z,b]$ with appropriate sensitive dimensions removed, and evaluates accuracy and $\Delta_{DP}$ on a test set.
We repeat for a variety of demographic subgroups derived from the sensitive attributes.
The \emph{predictiveness audit} trains classifier $C_i$ to predict sensitive attribute $a_i$ from $b_i$ alone.
The \emph{disentanglement audit} trains classifier $C_{\backslash i}$ to predict sensitive attribute $a_i$ from the representation with $b_i$ removed (e.g. $[z, b]\backslash b_i$).
If $C_i$ has low loss, our representation is predictive; if $C_{\backslash i}$ has high loss, it is disentangled.

\paragraph{Synthetic Data} \label{sec:synthetic}

%%%%%%%%%%%%%%%%%%%%%%%%%%%%%%%%%%%%%%%%%%%%%%%%%%%%%%%%%%%%%%%%%%%%%%%%%%%%%%%%
% results-dsprites-pareto
\newcommand{\figWidth}{0.24\textwidth}
\begin{figure*}[ht!]
\begin{subfigure}[t]{\figWidth}
\includegraphics[width=\textwidth]{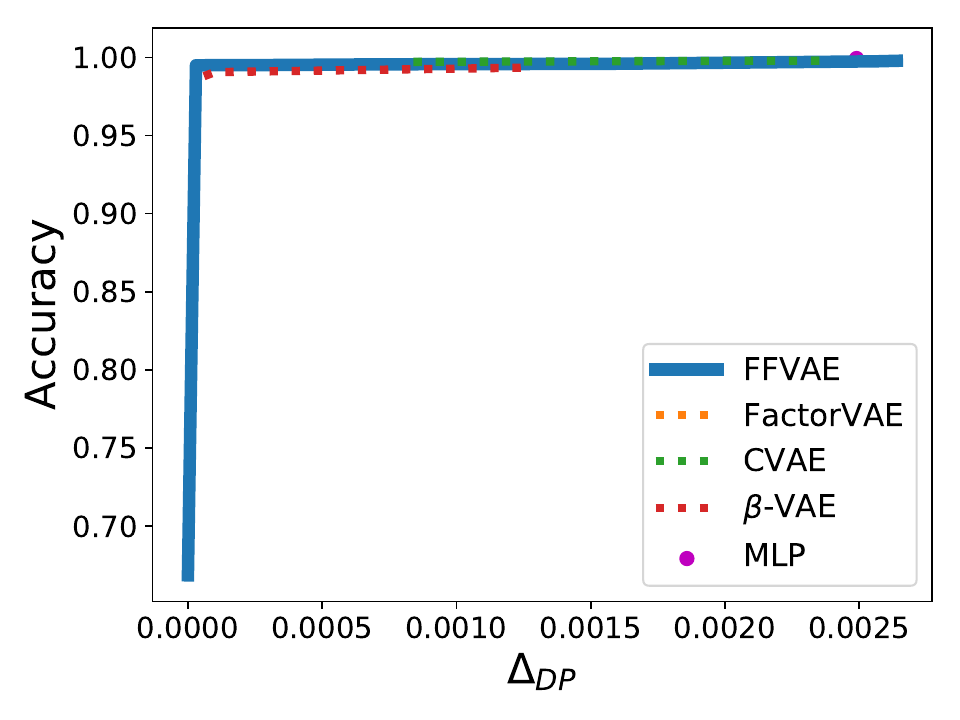}
\caption{$a$ = Scale}
\label{fig:dsprites-pareto-a2-y4}
\end{subfigure}
\hfill
\begin{subfigure}[t]{\figWidth}
\includegraphics[width=\textwidth]{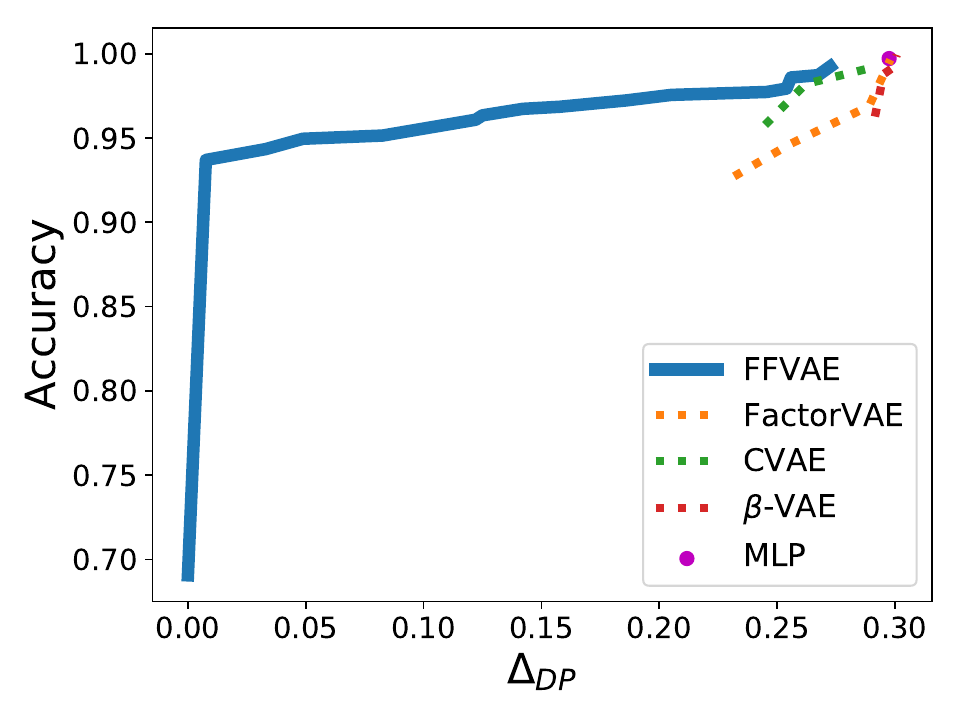}
\caption{$a$ = Shape}
\label{fig:dsprites-pareto-a1-y4}
\end{subfigure}
\hfill
\begin{subfigure}[t]{\figWidth}
\includegraphics[width=\textwidth]{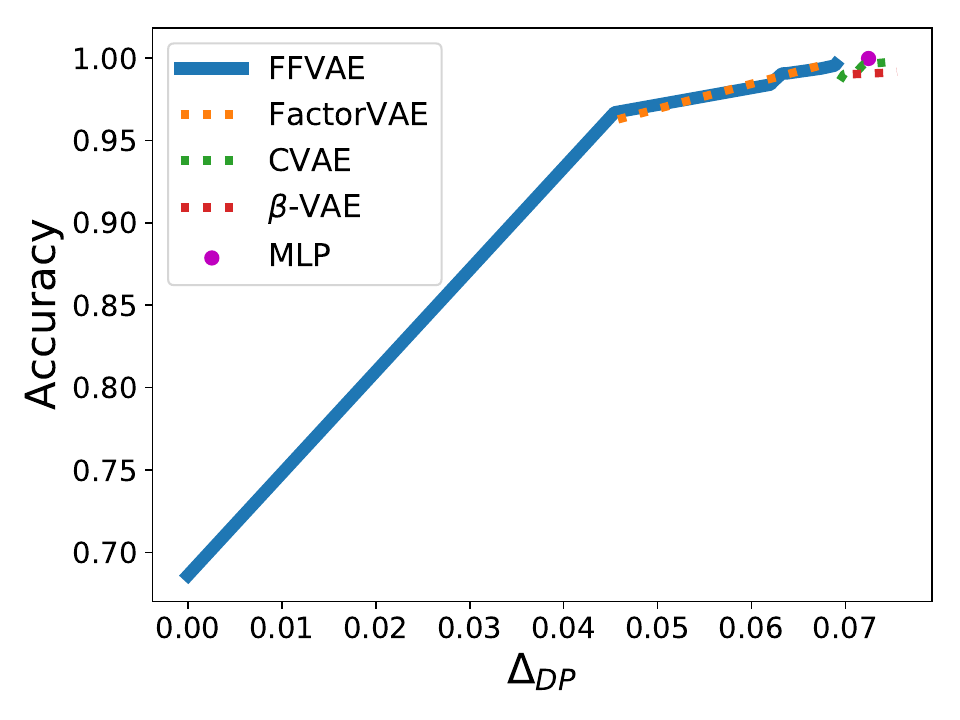}    
\caption{$a$ = Shape $\wedge$ Scale}
\label{fig:dsprites-pareto-a1and2-y4}
\end{subfigure}
\hfill
\begin{subfigure}[t]{\figWidth}
\includegraphics[width=\textwidth]{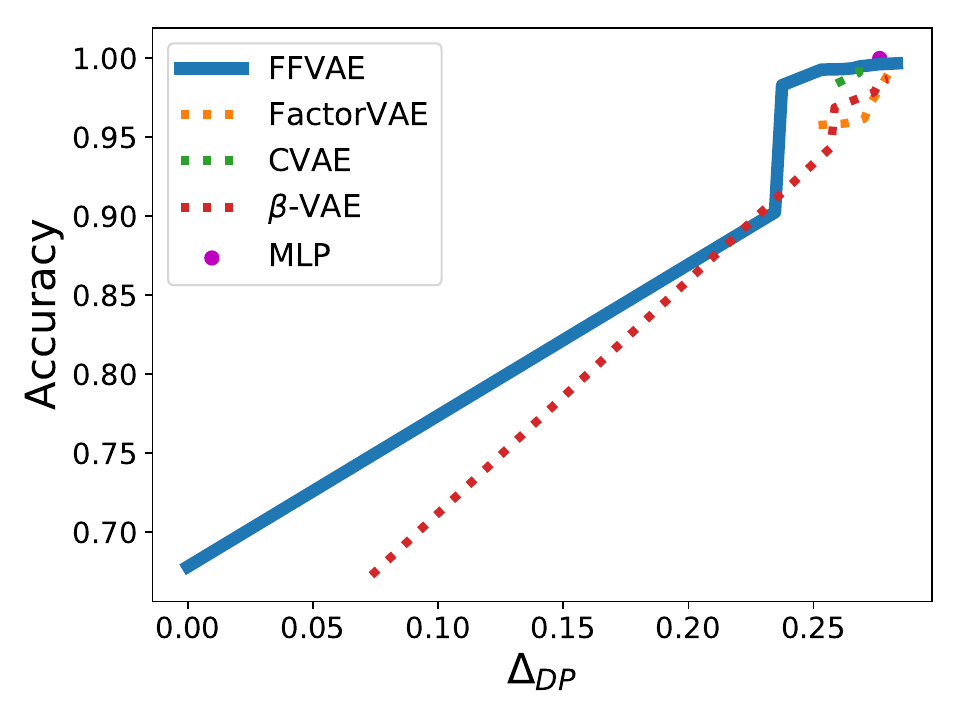}
\caption{$a$ = Shape $\vee$ Scale}
\label{fig:dsprites-pareto-a1or2-y4}
\end{subfigure}
\caption{
    Fairness-accuracy tradeoff curves, DSpritesUnfair dataset.
    We sweep a range of hyperparameters for each model and report Pareto fronts.
    Optimal point is the top left hand corner --- this represents perfect accuracy and fairness.
    MLP is a baseline classifier trained directly on the input data.
    For each model, encoder outputs are modified to remove information about $a$.
    $y$ = XPosition for each plot.
    }
    \label{fig:dsprites-pareto}
\end{figure*}
%%%%%%%%%%%%%%%%%%%%%%%%%%%%%%%%%%%%%%%%%%%%%%%%%%%%%%%%%%%%%%%%%%%%%%%%%%%%%%%%

\subparagraph{DSpritesUnfair Dataset}

%%%%%%%%%%%%%%%%%%%%%%%%%%%%%%%%%%%%%%%%%%%%%%%%%%%%%%%%%%%%%%%%%%%%%%%%%%%%%%%%
% results-dsprites-audit
%%%%%%%%%%%%%%%%%%%%%%%%%%%%%%%%%%%%%%%%%%%%%%%%%%%%%%%%%%%%%%%%%%%%%%%%%%%%%%%%
\begin{figure}[ht!]
\centering
\includegraphics[width=0.4\textwidth]{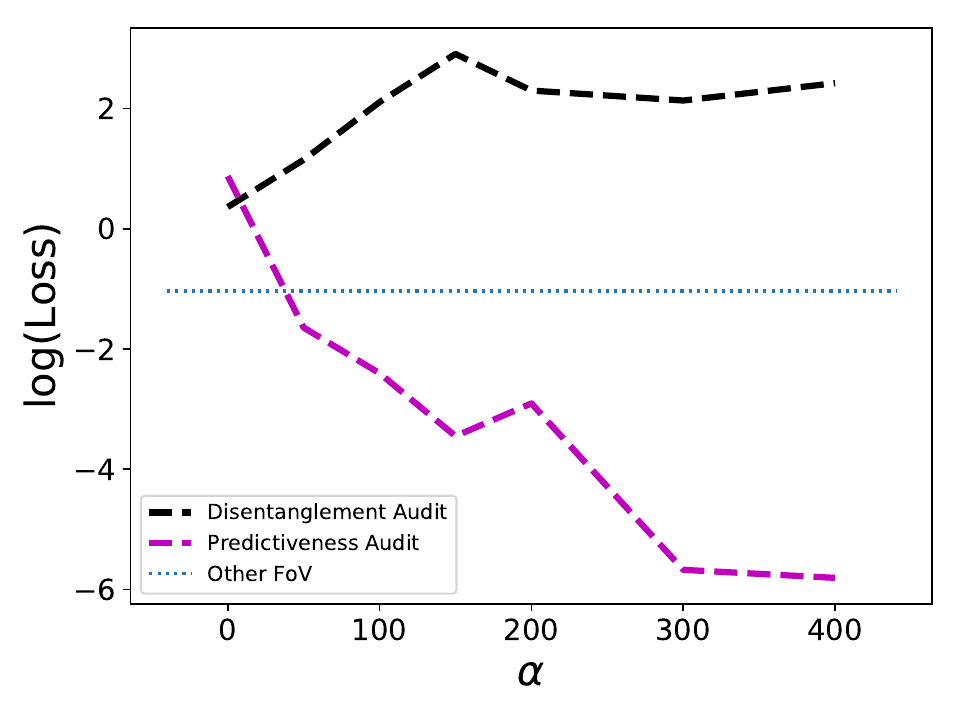}
\vspace{-0.2cm}
\caption{
    Black and pink dashed lines respectively show FFVAE disentanglement audit (the higher the better) and predictiveness audit (the lower the better) as a function of $\alpha$.
    These audits use $A_i$=Shape (see text for details).
    The blue line is a reference value---the log loss of a classifier that predicts $A_i$ from the other 5 DSprites factors of variation (FoV) alone, ignoring the image---representing the amount of information about $A_i$ inherent in the data.
    }
    \label{fig:dsprites-audit}
\end{figure}
%%%%%%%%%%%%%%%%%%%%%%%%%%%%%%%%%%%%%%%%%%%%%%%%%%%%%%%%%%%%%%%%%%%%%%%%%%%%%%%%

The DSprites dataset\footnote{\scriptsize \url{https://github.com/deepmind/dsprites-dataset}} contains $64 \times 64$-pixel images of white shapes against a black background, and was designed to evaluate whether learned representations have disentangled sources of variation.
The original dataset has several categorical factors of variation---Scale, Orientation, XPosition, YPosition---that combine to create $700,000$ unique images.
We binarize the factors of variation to derive sensitive attributes and labels, so that many images now share any given attribute/label combination 
(See Appendix \ref{sec:dsprites-training-details} for details).
In the original DSprites dataset, the factors of variation are sampled uniformly.
However, in fairness problems, we are often concerned with correlations between attributes and the labels we are trying to predict (otherwise, achieving low $\Delta_{DP}$ is aligned with standard classification objectives).
Hence, we sampled an ``unfair'' version of this data (DSpritesUnfair) with correlated factors of variation; in particular Shape and XPosition correlate positively.
Then a non-trivial fair classification task would be, for instance, learning to predict shape without discriminating against inputs from the left side of the image.

\subparagraph{Baselines}
To test the utility of our predictiveness prior, we compare our model to $\beta$-VAE (VAE with a coefficient $\beta \geq 1$ on the KL term) and FactorVAE, which have disentanglement priors but no predictiveness prior.
We can also think of these as FFVAE with $\alpha = 0$.
To test the utility of our disentanglement prior, we also compare against a version of our model with $\gamma = 0$, denoted CVAE.
This is similar to the class-conditional VAE \citep{kingma2014semi}, with sensitive attributes as labels --- this model encourages predictiveness but no disentanglement.

\subparagraph{Fair Classification}
We perform the fair classification audit using several group/subgroup definitions for models trained on DSpritesUnfair (see Appendix \ref{sec:dsprites-training-details} for training details), and report fairness-accuracy tradeoff curves in Fig. \ref{fig:dsprites-pareto}.
In these experiments, we used Shape and Scale as our sensitive attributes during encoder training.
We perform the fair classification audit by training an MLP to predict $y=$``XPosition''---which was not used in the representation learning phase---given the modified encoder outputs, and repeat for several sensitive groups and subgroups.
We modify the encoder outputs as follows:
When our sensitive attribute is $a_i$ we remove the associated dimension $b_i$ from $[z,b]$; when the attribute is a conjunction of $a_i$ and $a_j$, we remove both $b_i$ and $b_j$.
For the baselines, we simply remove the latent dimension which is most correlated with $a_i$, or the two most correlated dimensions with the conjunction.
We sweep a range of hyperparameters to produce the fairness-accuracy tradeoff curve for each model.
In Fig. \ref{fig:dsprites-pareto}, we show the ``Pareto front'' of these models: points in ($\Delta_{DP}$, accuracy)-space for which no other point is better along both dimensions.
The optimal result is the top left hand corner (perfect accuracy and $\Delta_{DP} = 0$).

Since we have a 2-D sensitive input space, we show results for four different sensitive attributes derived from these dimensions: $\{a=\text{``Shape''}, a=\text{``Scale''}, a=\text{``Shape''}\vee\text{``Scale''}, a=\text{``Shape''}\wedge\text{``Scale''}\}$.
Recall that Shape and XPosition correlate in the DSpritesUnfair dataset.
Therefore, for sensitive attributes that involve Shape, we expect to see an improvement in $\Delta_{DP}$.
For sensitive attributes that do not involve Shape, we expect that our method does not hurt performance at all --- since the attributes are uncorrelated in the data, the optimal predictive solution also has $\Delta_{DP} = 0$.

When group membership $a$ is uncorrelated with label $y$ (Fig. \ref{fig:dsprites-pareto-a2-y4}), all models achieve high accuracy and low $\Delta_{DP}$ ($a$ and $y$ successfully disentangled).
When $a$ correlates with $y$ by design (Fig. \ref{fig:dsprites-pareto-a1-y4}), we see the clearest improvement of the FFVAE over the baselines, with an almost complete reduction in $\Delta_{DP}$ and very little accuracy loss.
The baseline models are all unable to improve $\Delta_{DP}$ by more than about 0.05, indicating that they have not effectively disentangled the sensitive information from the label.
In Figs. \ref{fig:dsprites-pareto-a1and2-y4} and \ref{fig:dsprites-pareto-a1or2-y4}, we examine conjunctions of sensitive attributes, assessing FFVAE's ability to flexibly provide multi-attribute fair representations.
Here FFVAE exceeds or matches the baselines accuracy-at-a-given-$\Delta_{DP}$ almost everywhere;
by disentangling information from multiple sensitive attributes, FFVAE enables flexibly fair downstream classification.

\subparagraph{Disentanglement and Predictiveness}
Fig. \ref{fig:dsprites-audit} shows the FFVAE disentanglement and predictiveness audits (see above for description of this procedure).
This result aggregates audits across all FFVAE models trained in the setting from Figure \ref{fig:dsprites-pareto-a1-y4}.
The classifier loss is cross-entropy, which is a lower bound on the mutual information between the input and target of the classifier.
We observe that increasing $\alpha$ helps both predictiveness and disentanglement in this scenario.
In the disentanglement audit, larger $\alpha$ makes predicting the sensitive attribute from the modified representation (with $b_i$ removed) more difficult.
The horizontal dotted line shows the log loss of a classifier that predicts $a_i$ from the other DSprites factors of variation (including labels not available to FFVAE); this baseline reflects the correlation inherent in the data.
We see that when $\alpha = 0$ (i.e. FactorVAE), it is slightly more difficult than this baseline to predict the sensitive attribute.
This is due to the disentanglement prior.
However, increasing $\alpha > 0$ increases disentanglement benefits in FFVAE beyond what is present in FactorVAE.
This shows that encouraging predictive structure can help disentanglement through isolating each attribute's information in particular latent dimensions.
Additionally, increasing $\alpha$ improves predictiveness, as expected from the objective formulation.
We further evaluate the disentanglement properties of our model in Appendix \ref{sec:mig} using the Mutual Information Gap metric \cite{chen2018isolating}.

\paragraph{Celebrity Faces} \label{sec:celeba}

%%%%%%%%%%%%%%%%%%%%%%%%%%%%%%%%%%%%%%%%%%%%%%%%%%%%%%%%%%%%%%%%%%%%%%%%%%%%%%%%
% results-celeba-pareto
%%%%%%%%%%%%%%%%%%%%%%%%%%%%%%%%%%%%%%%%%%%%%%%%%%%%%%%%%%%%%%%%%%%%%%%%%%%%%%%%
\newcommand{\celebAFigWidth}{0.15\textwidth}
\begin{figure}[ht!]
\begin{subfigure}[t]{\celebAFigWidth}
\includegraphics[width=\textwidth]{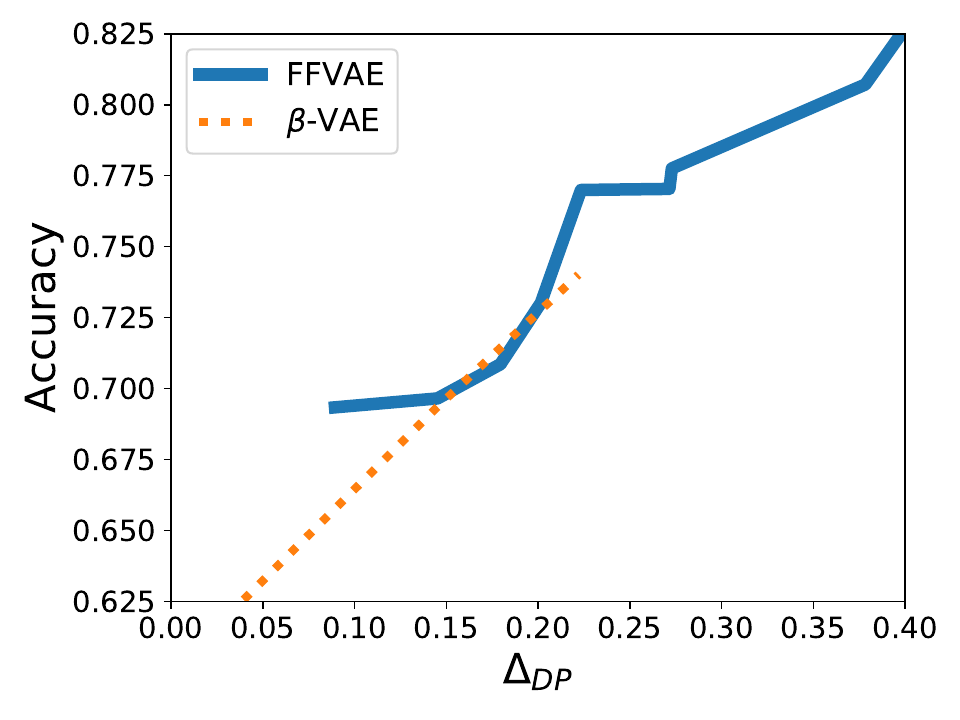}
\caption{$a$ = C}
\label{fig:predict_C}
\end{subfigure}
\hfill
\begin{subfigure}[t]{\celebAFigWidth}
\includegraphics[width=\textwidth]{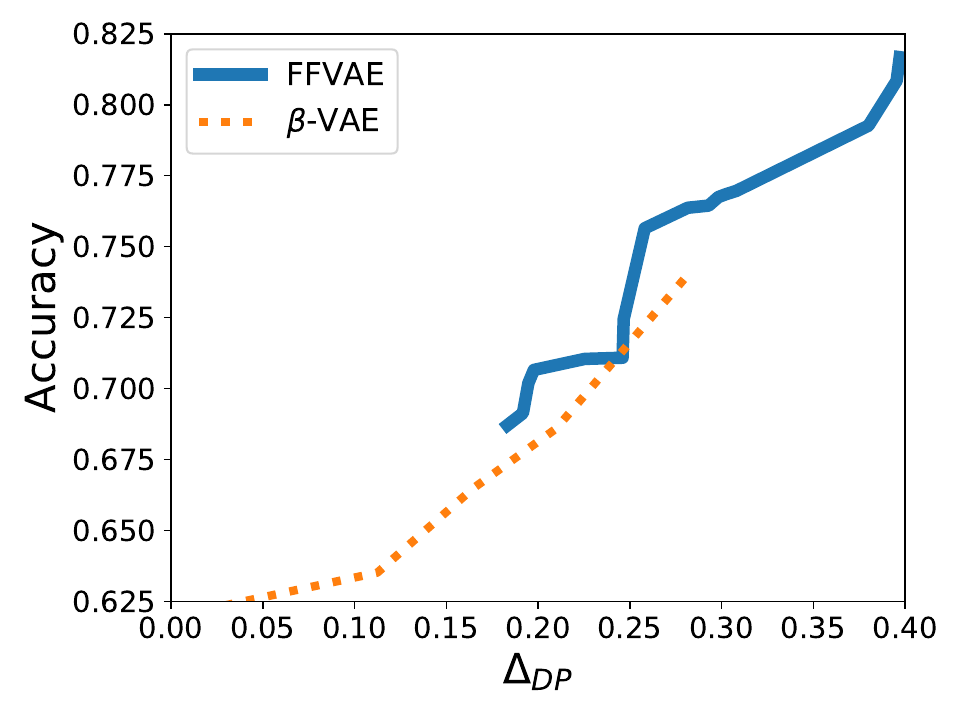}
\caption{$a$ = E}
\label{fig:predict_E}
\end{subfigure}
\hfill
\begin{subfigure}[t]{\celebAFigWidth}
\includegraphics[width=\textwidth]{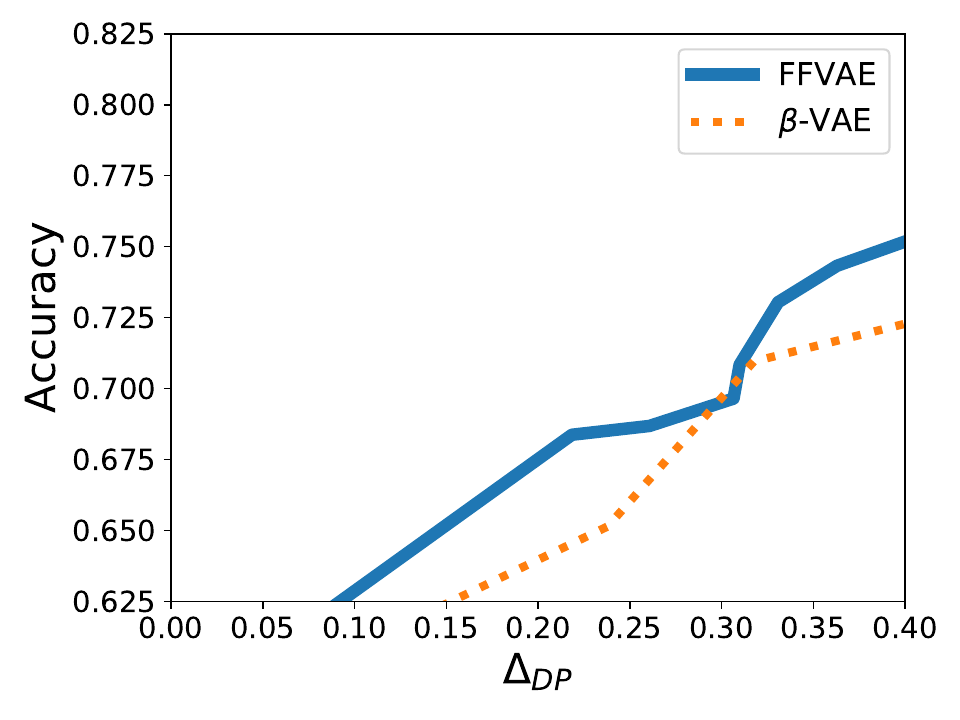}
\caption{$a$ = M}
\label{fig:predict_M}
\end{subfigure}
\hfill
\begin{subfigure}[t]{\celebAFigWidth}
\includegraphics[width=\textwidth]{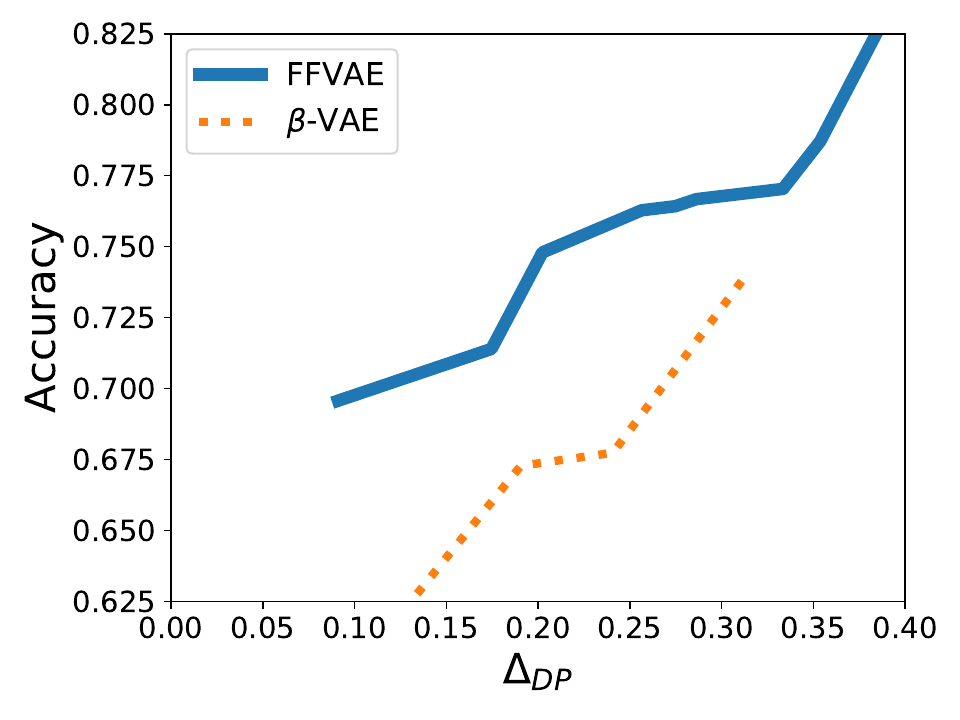}
\caption{$a$ = C $\wedge$ E}
\label{fig:predict_C-AND-E}
\end{subfigure}
\hfill
\begin{subfigure}[t]{\celebAFigWidth}
\includegraphics[width=\textwidth]{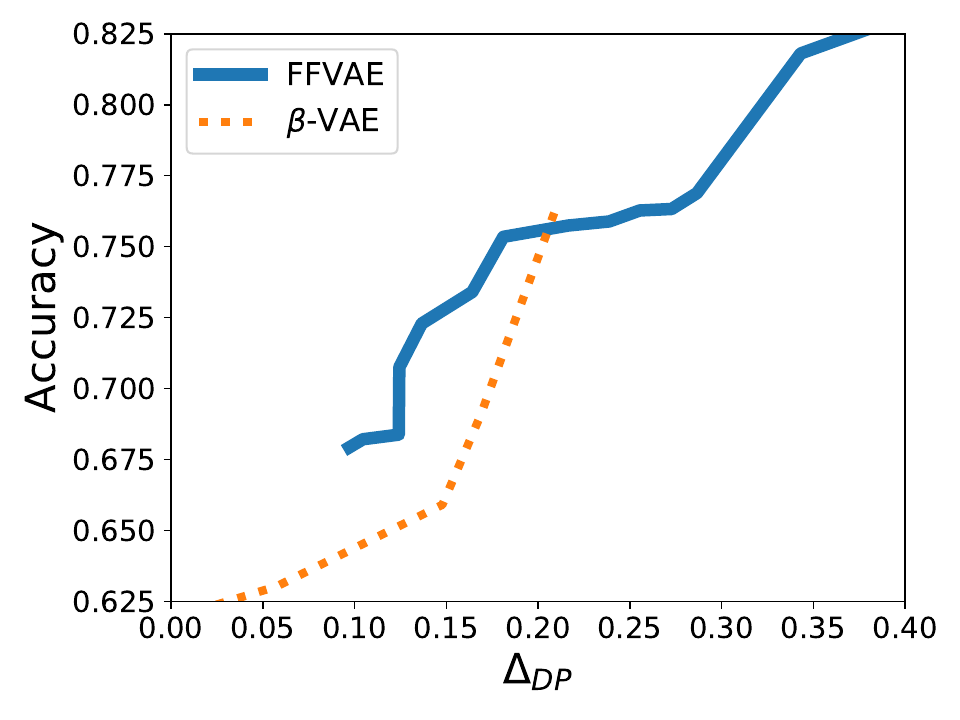}
\caption{$a$ = C $\wedge \neg$ E}
\label{fig:predict_C-AND-NOT-E}
\end{subfigure}
\hfill
\begin{subfigure}[t]{\celebAFigWidth}
\includegraphics[width=\textwidth]{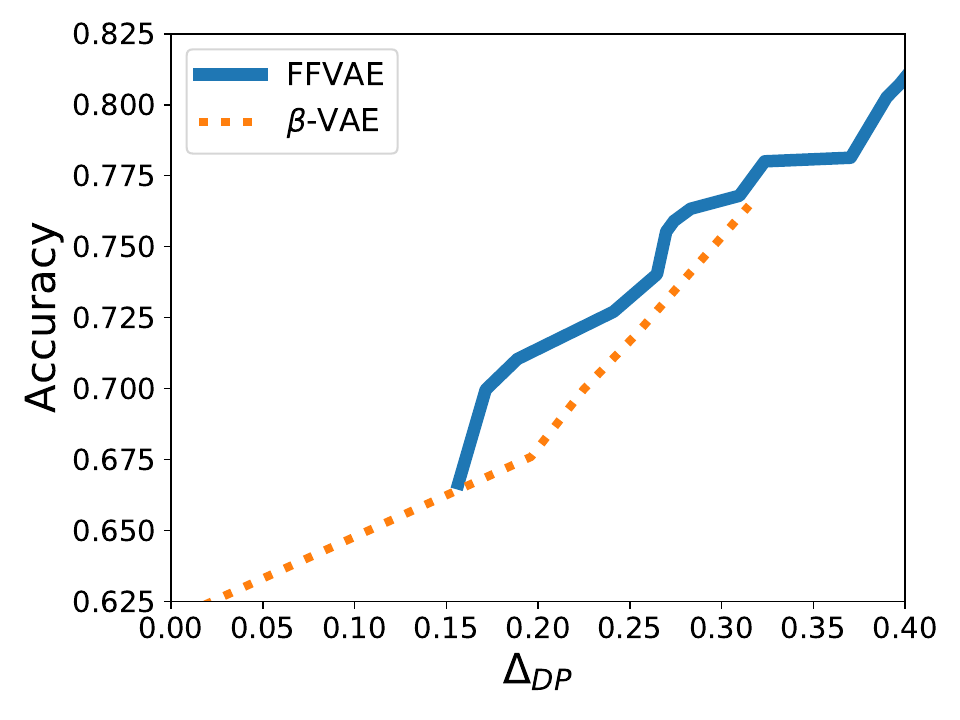}
\caption{$a$ = $\neg$ C $\wedge$ E}
\label{fig:predict_NOT-C-AND-E}
\end{subfigure}
\hfill
\begin{subfigure}[t]{\celebAFigWidth}
\includegraphics[width=\textwidth]{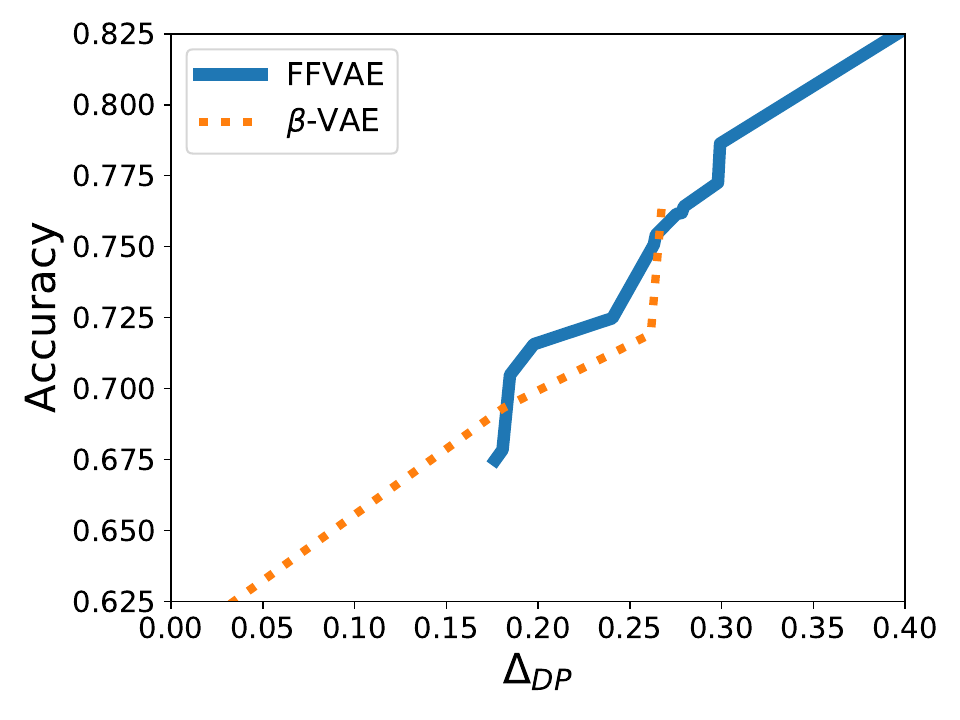}
\caption{$a$ = $\neg$ C $\wedge \neg$ E}
\label{fig:predict_NOT-C-AND-NOT-E}
\end{subfigure}
\hfill
\begin{subfigure}[t]{\celebAFigWidth}
\includegraphics[width=\textwidth]{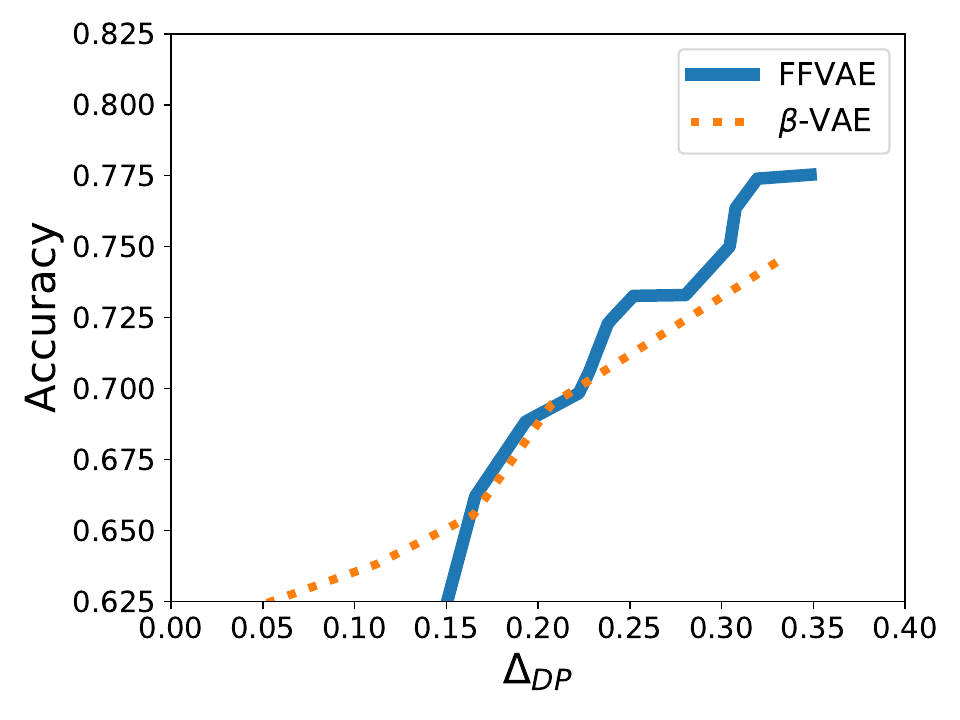}
\caption{$a$ = C $\wedge$ M}
\label{fig:predict_C-AND-M}
\end{subfigure}
\hfill
\begin{subfigure}[t]{\celebAFigWidth}
\includegraphics[width=\textwidth]{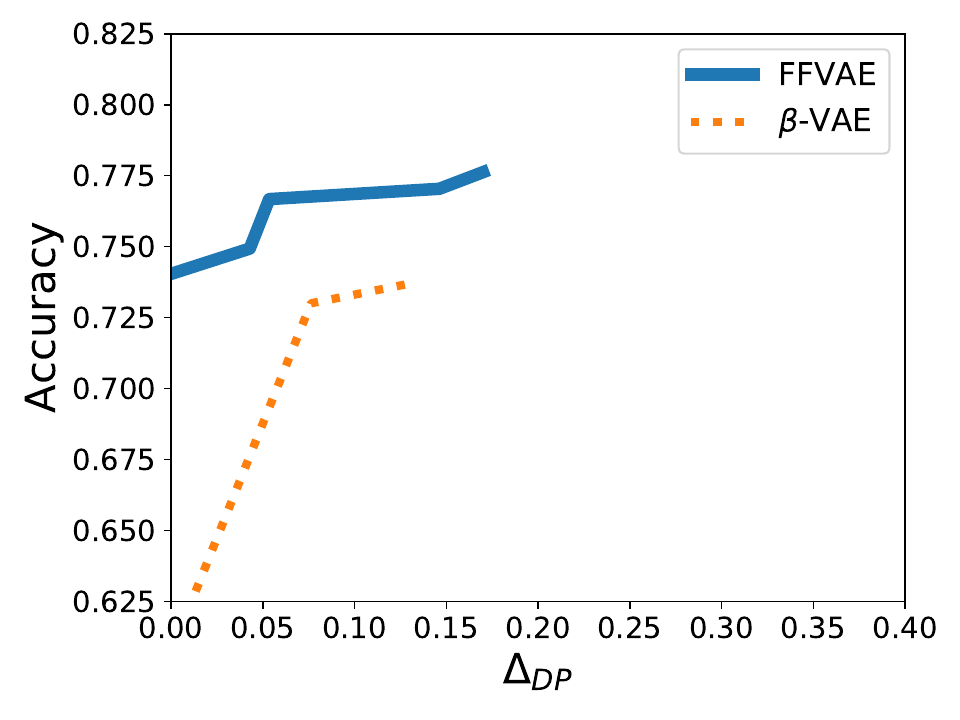}
\caption{$a$ = C $\wedge \neg$ M}
\label{fig:predict_C-AND-NOT-M}
\end{subfigure}
\hfill
\begin{subfigure}[t]{\celebAFigWidth}
\includegraphics[width=\textwidth]{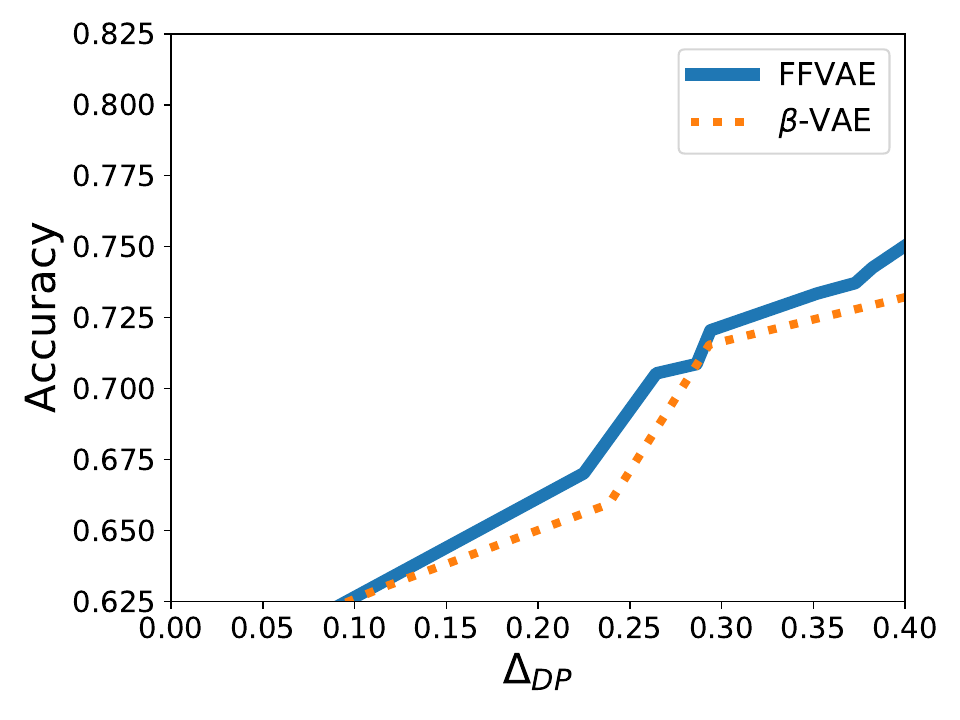}
\caption{$a$ = $\neg$ C $\wedge$ M}
\label{fig:predict_NOT-C-AND-M}
\end{subfigure}
\hfill
\begin{subfigure}[t]{\celebAFigWidth}
\includegraphics[width=\textwidth]{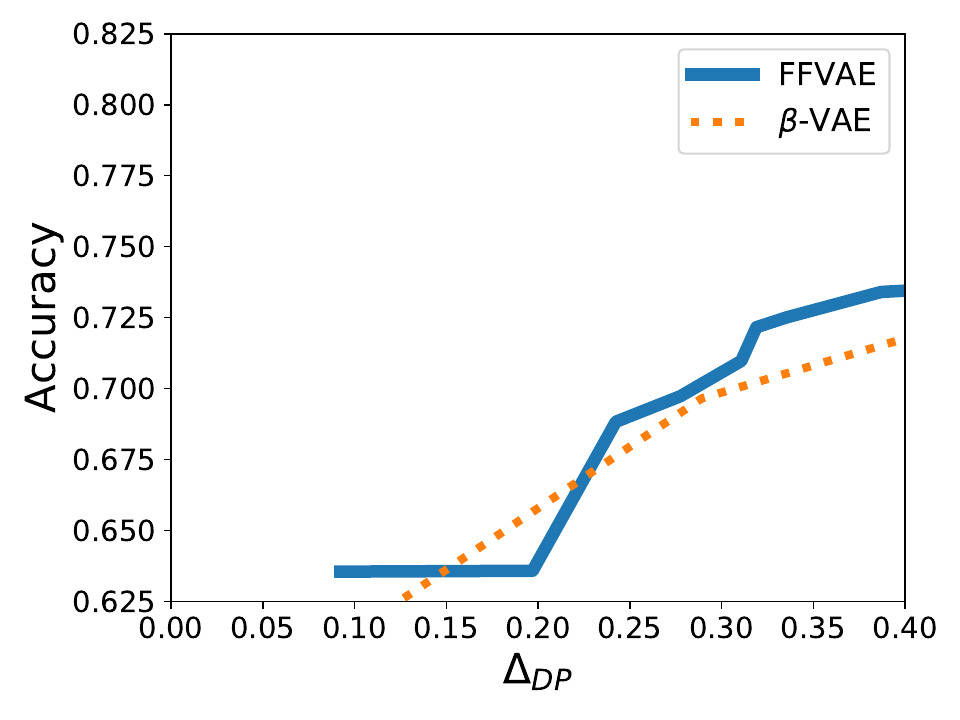}
\caption{$a$ = $\neg$ C $\wedge \neg$ M}
\label{fig:predict_NOT-C-AND-NOT-M}
\end{subfigure}
\hfill
\begin{subfigure}[t]{\celebAFigWidth}
\includegraphics[width=\textwidth]{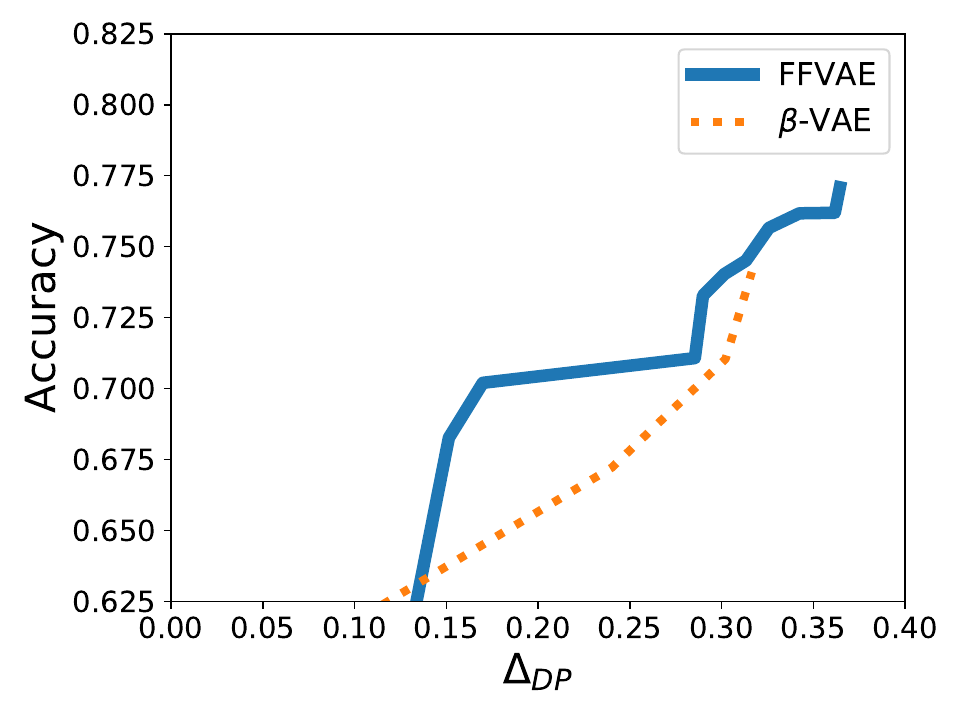}
\caption{$a$ = $\neg$ E $\wedge$ M}
\label{fig:predict_E-AND-M}
\end{subfigure}
\hfill
\begin{subfigure}[t]{\celebAFigWidth}
\includegraphics[width=\textwidth]{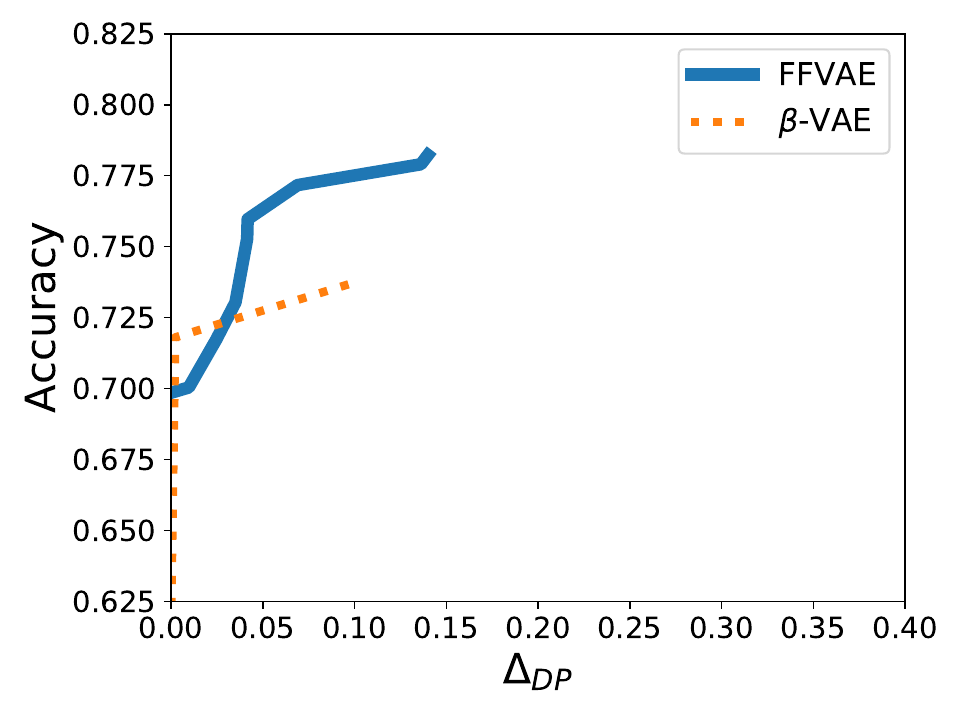}
\caption{$a$ = E $\wedge \neg$ M}
\label{fig:predict_E-AND-NOT-M}
\end{subfigure}
\hfill
\begin{subfigure}[t]{\celebAFigWidth}
\includegraphics[width=\textwidth]{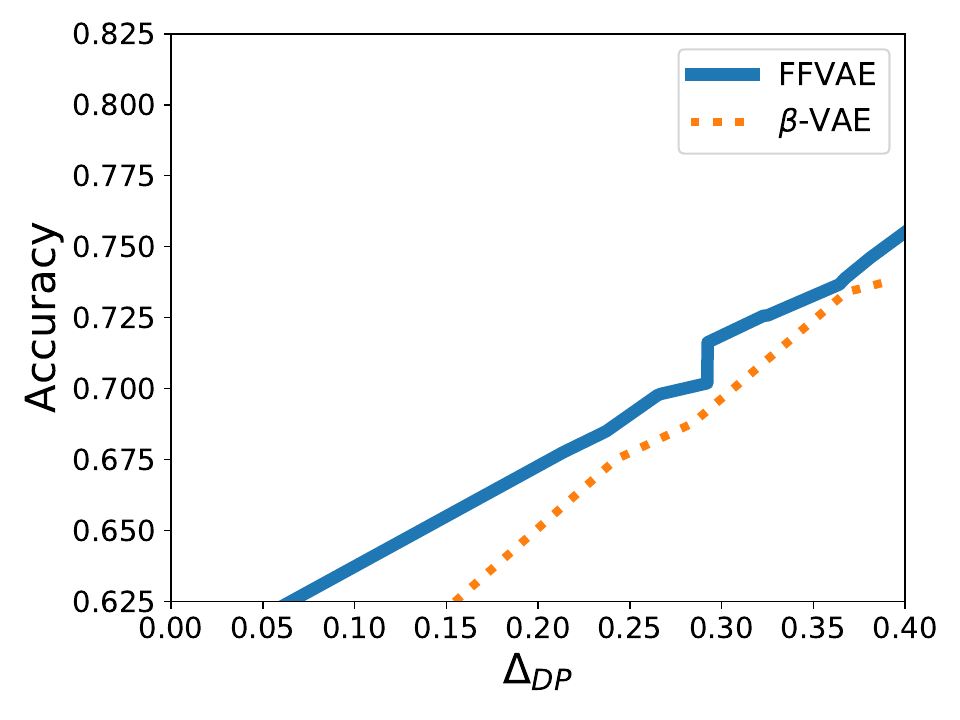}
\caption{$a$ = $\neg$ E $\wedge$ M}
\label{fig:predict_NOT-E-AND-M}
\end{subfigure}
\hfill
\begin{subfigure}[t]{\celebAFigWidth}
\includegraphics[width=\textwidth]{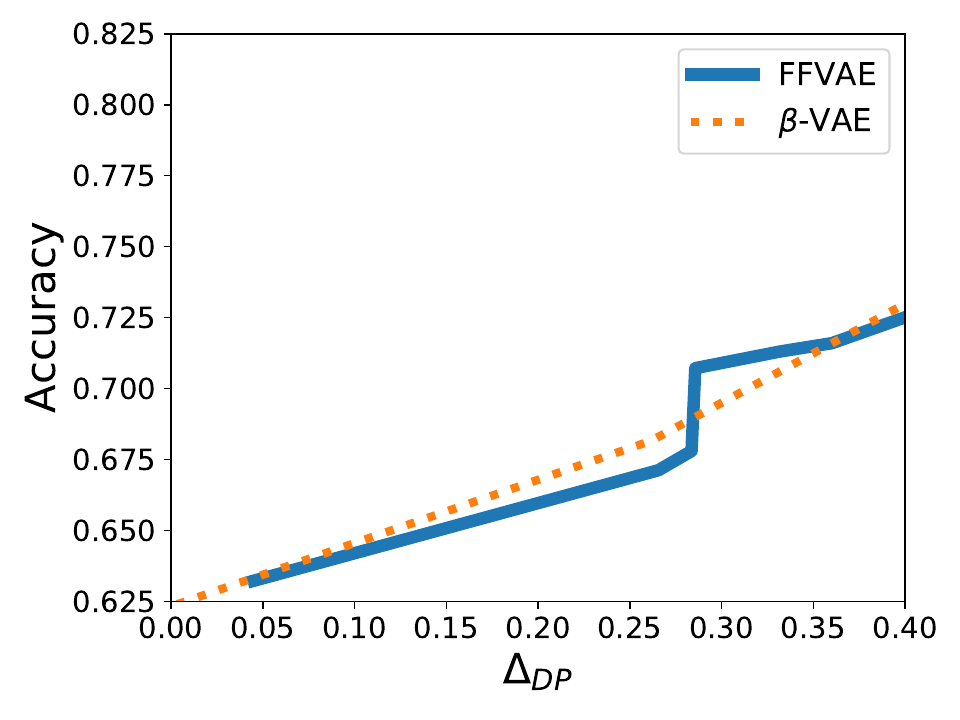}
\caption{$a$ = $\neg$ E $\wedge \neg$ M}
\label{fig:predict_NOT-E-AND-NOT-M}
\end{subfigure}
\hfill
\caption{
    Celeb-A subgroup fair classification results.
    Sensitive attributes: Chubby (C), Eyeglasses (E), and Male (M).
    $y$ = HeavyMakeup. 
    }
    \label{fig:celeba-pareto}
\end{figure}
%%%%%%%%%%%%%%%%%%%%%%%%%%%%%%%%%%%%%%%%%%%%%%%%%%%%%%%%%%%%%%%%%%%%%%%%%%%%%%%%

\subparagraph{Dataset}
The CelebA\footnote{\scriptsize \url{http://mmlab.ie.cuhk.edu.hk/projects/CelebA.html}} dataset contains over $200,000$ images of celebrity faces.
Each image is associated with 40 human-labeled binary attributes (OvalFace, HeavyMakeup, etc.).
We chose three attributes, Chubby, Eyeglasses, and Male as sensitive attributes\footnote{
We chose these attributes because they co-vary relatively weakly with each other (compared with other attribute triplets), but strongly with other attributes.
Nevertheless the rich correlation structure amongst all attributes makes this a challenging fairness dataset; it is difficult to achieve high accuracy and low $\Delta_{DP}$.
}, and report fair classification results on 3 groups and 12 two-attribute-conjunction subgroups only (for brevity we omit three-attribute conjunctions).
To our knowledge this is the first exploration of fair representation learning algorithms on the Celeb-A dataset.
As in the previous sections we train the encoders on the train set, then evaluate performance of MLP classifiers trained on the encoded test set.

\subparagraph{Fair Classification}

We follow the fair classification audit procedure described above, where the held-out label HeavyMakeup---which was not used at encoder train time---is predicted by an MLP from the encoder representations.
When training the MLPs we take a fresh encoder sample for each minibatch (statically encoding the dataset with one encoder sample per image induced overfitting).
We found that training the MLPs on encoder means (rather than samples) increased accuracy but at the cost of very unfavorable $\Delta_{DP}$.
We also found that FactorVAE-style adversarial training does not scale well to this high-dimensional problem, so we instead optimize Equation \ref{eq:ffvae} using the biased estimator from \citet{chen2018isolating}.
Figure \ref{fig:celeba-pareto} shows Pareto fronts that capture the fairness-accuracy tradeoff for FFVAE and $\beta$-VAE.

While neither method dominates in this challenging setting, FFVAE achieves a favorable fairness-accuracy tradeoff across many of subgroups.
We believe that using sensitive attributes as side information gives FFVAE an advantage over $\beta$-VAE in predicting the held-out label.
In some cases (e.g., $a$=$\neg$E$\wedge$M) FFVAE achieves better accuracy at all $\Delta_{DP}$ levels, while in others (e.g., $a$=$\neg$C$\wedge\neg$E)
, FFVAE did not find a low-$\Delta_{DP}$ solution.
We believe Celeb-A--with its many high dimensional data and rich label correlations---is a useful test bed for subgroup fair machine learning algorithms, and we are encouraged by the reasonably robust performance of FFVAE in our experiments.
\section{Related Work}

\paragraph{Algorithmic fairness}
Interest in fair machine learning is burgeoning as researchers seek to define and mitigate unintended harm in automated decision making systems.
Definitional works have been broadly concerned with \textit{group fairness} or \textit{individual fairness}.
\citet{dwork2012fairness} discussed individual fairness within the owner-vendor framework we utilize. 
\citet{zemel2013learning} encouraged elements of both group and individual fairness via a regularized objective.
An intriguing body of recent work unifies the individual-group dichotomy by exploring fairness at the intersection of multiple group identities, and among small subgroups of individuals \cite{kearns2017preventing,hebert2017calibration}.

\citet{calmon2017optimized} and \citet{hajian2015discrimination} explored fair machine learning by pre- and post-processing training datasets.
\citet{mcnamara2017provably} provides a framework where the data producer, user, and regulator have separate concerns, and discuss fairness properties of representations. 
\citet{louizos2015variational} give a method for learning fair representations with deep generative models by using maximum mean discrepancy \citep{gretton2007kernel} to eliminate disparities between the two sensitive groups.

\paragraph{Adversarial training}
Adversarial training for deep generative modeling was popularized by \citet{goodfellow2014generative} and applied to deep semi-supervised learning \citep{salimans2016improved,odena2016semi} and segmentation \citep{luc2016semantic}, although similar concepts had previously been proposed for unsupervised and supervised learning \citep{schmidhuber1992learning,gutmann2010noise}.
\citet{ganin2016domain} proposed adversarial representation learning for domain adaptation, which resembles fair representation learning in the sense that multiple distinct data distributions (e.g., demographic groups) must be expressively modeled by a single representation.

\citet{edwards2015censoring} made this connection explicit by proposing adversarially learning a classifier that achieves demographic parity. 
This work is the most closely related to LAFTR, although our choice of balanced $\ell_1$ learning objective (rather than standard cross entropy loss) is more closely aligned to the goal of ensuring statistical parity in the learned representations.
Recent work has explored the use of adversarial training to other notions of group fairness.
\citet{beutel2017data} explored the particular fairness levels achieved by the algorithm from \citet{edwards2015censoring}, and demonstrated that they can vary as a function of the demographic unbalance of the training data.
In work concurrent to ours, \citet{zhang2018mitigating} use an adversary which attempts to predict the sensitive variable solely based on the classifier output, to learn an equal opportunity fair classifier.
Whereas they focus on fairness in classification outcomes, in our work we allow the adversary to work directly with the learned representation, which we show yields fair and transferable representations that in turn admit fair classification outcomes.
%
% LAFTR ^^^^^^^^^^^

% FFVAE ............
%
\paragraph{Disentanglement}
The search of independent latent components that explain observed data has long been a focus on the probabilistic modeling community \cite{comon1994independent,hyvarinen2000independent,bach2002kernel}.
In light of the increased prevalence of neural networks models in many data domains, the machine learning community has renewed its interest in learned features that ``disentangle'' semantic factors of data variation.
The introduction of the $\beta$-VAE \citep{higgins2016beta}, as discussed in Section \ref{sec:ffvae:background-vae}, motivated a number of subsequent studies  that examine why adding additional weight on the KL-divergence of the ELBO encourages disentangled representations \citep{alemi2018fixing, burgess2018understanding}.
\citet{chen2018isolating,kim2018disentangling} and \citet{esmaeili2018structured} argue that decomposing the ELBO and penalizing the total correlation increases disentanglement in the latent representations. 
\citet{locatello2018challenging} conduct extensive experiments comparing existing unsupervised disentanglement methods and metrics. 
They conclude pessimistically that learning disentangled representations requires inductive biases and possibly additional supervision, but identify fair machine learning as a potential application where additional supervision is available by way of sensitive attributes.

\paragraph{Compositional representations}
FFVAE provides multi-attribute fair representation learning, which we accomplish by using sensitive attributes as labels to induce a factorized structure in the aggregate latent code.
\citet{bose2018compositional} proposed a compositional fair representation of graph-structured data.
\citet{kingma2014semi} previously incorporated (partially-observed) label information into the VAE framework to perform semi-supervised classification.
Several recent VAE variants have incorporated label information into latent variable learning for image synthesis \cite{klys2018learning} and single-attribute fair representation learning \cite{song2018learning,botros2018hierarchical,moyer2018invariant}.
\section{Discussion}
This Chapter presented some approaches for mitigating algorithmic discrimination by learning a data representation that serves as a bottleneck, promoting group parity in downstream learning processes that consume the representation.
This approach has found considerable purchase within the research community: firms have implemented fair representation learning in their products and open-source software~\citep{beutel_putting_2019,bellamy_ai_2018}, while researchers have extended or analyzed the adversarial approach proposed in this Chapter (discussed concurrently by other research groups \cite{beutel2017data,zhang2018mitigating}) in a number of directions~\citep{
    agarwal2021towards,
    quadrianto2019discovering,
    adel2019one,
    song2018learning,
    moyer2018invariant,
    samadi2018price,
    zhao2019inherent}.
We orient our discussion here towards articulating known limitations of the representation learning approach, which we feel is especially important given that the goal of fair ML is to ensure that emerging technologies confer social equal benefits regardless of social position.

Fair representation learning has close ties to OOD-robustness~\citep{farnadi2022algorithmic} and causality~\citep{dieng2020algorithmic,madras_fairness_2019}.
In fact, recent work~\citep{makar_causally_2022,makar_fairness_2022} provides a causal interpretation for the group-normalized losses of per-group losses discussed in Section~\ref{gen-model}.
From this view point, a dataset relecting societal bias is thought of as resulting from a \emph{confounding} process on the observations.
Thus, fair learning amounts to using importance weights (i.e. normalizing group losses by their subpopulation size) to emulate training on unconfounded data, which provides robustness to distribution shift.
Thus, when used in an anti-causal prediction setting~\citep{schoelkopf_causal_2012}, LAFTR apparently provides some implicit causal inference!

This also suggests fair representations may naturally deliver a degree of robustness to distribution shift, which is supported by our theory (Section~\ref{gen-model}) and experiments (Section~\ref{sec:laftr:experiments}).
However, this natural OOD-robustness is somewhat limited, as fair representations are now known to be vulnerable to worst-case distribution shifts~\citep{lechner_impossibility_2021, lechner_inherent_2022}.
\footnote{
Such counterexamples can be quite ``worst-case'' relative to the data bias being addressed, and may be unrealistic. For example, if upon deployment there is a \emph{reversal} of which groups are privileged in the sense that their data correlate with ``positive'' labels (like loan approval), then a ``fair'' representation learned on the original data may lose its fairness properties upon deployment.
}
In general, fair regularizers and metrics (like other differentiable aggregate quantities such as cross-entropy loss) can be manipulated adversarially~\citep{aivodji_fairwashing_2019,shamsabadi_washing_2022,solans_poisoning_2020,ghosh_subverting_2022}.
These results provide a cautionary tale: fair representations should be deployed with extreme care, which could include a characterization of which distribution shifts are \emph{likely} during deployment or a security audit to assess vulnerability to adversarial actions like data poisoning.

One obstacle to applying fair representation learning within contemporary workflows is the somewhat onerous reliance on demographic labels during training.
More recent representation learning approaches eschew labeled supervision, instead ``pre-training'' a large network using unsupervised and self-supervised learning objectives over massive corpora of web-scraped data~\citep{radford_learning_2021}.
This pre-training process provides high-fidelity representations (sometimes called ``foundation models'') that can be adapted for use in a variety of downstream prediction tasks.
However, the reliance on web data for pre-training can induce harmful stereotypical associations in the learned representations \cite{hundt_robots_2022,bianchi_easily_2023}.

Addressing representational biases in pre-trained representations remains an important open question;
while the discussions in this Chapter focused on fairness in a (variational) autoencoder representations, which are derived using a relatively simpler training paradigm, they may shed some light on how to proceed.
The key to encouraging fairness with FFVAE (Section~\ref{sec:ffvae:ffvae}) was the use of \emph{supervised} disentanglement priors, which cannot be readily implemented in the pre-training regime due to the scale of data required.
Unforunately, VAEs with unsupervised disentanglement priors can struggle to learn disentangled representations when trained on data with strong correlation patterns~\citep{trauble_disentangled_2021}.
This suggests that updating the learning objective during pre-training (without explicit supervision) is unlikely to mitigate representational bias.
Perhaps representation bias in pre-trained models could be addressed after the pre-training process by using human feedback to inform a more targeted adaptation (along the lines of~\citet{bolukbasi2016man,santurkar2021editing}).
One important obstacle along this research direction is the lack of suitable measures of representation bias~\citep{bommasani2021opportunities}.
There is also a concern that despite adaptation methods providing a superficial improvement on representation fairness metrics, they could still admit performance disparities in downstream prediction~\citep{gonen2019lipstick,zhou-etal-2021-challenges}.

On the positive side, the bottleneck advocated by fair representation learning could ultimately be beneficial for other types of model failures.
This is especially relevant to the setting where models are severely ``over-parameterized'', that is, the number of model parameters (perhaps greatly) exceeds the number of available data points.
By definition, over-parameterized prediction problems are underspecified~\citep{damour_underspecification_2020}.
Despite this issue, neural networks tend to generalize well to \emph{i.i.d.} test data~\citep{bartlett_benign_2020,frei_benign_2022}, although they may generalize poorly for subpopulations that are underrepresented in the training data~\citep{sagawa_investigation_2020, wald_malign_2023}.
In this regard, the bottleneck effect of fair representation learning could be doubly effective, mitigating not only predictive disparities in prediction but also underspecification.
This is because (unlike when models are fine-tuned or trained from scratch), adapting smaller models (say, linear probes) on top of learned representation
mitigates the over-parameterization issue by greatly reducing the number of parameters being optimized.
Future work is needed to more completely characterize this issue.
  \chapter{Leveraging Shortcut Learning to Learn Robust Features}\label{chap:eiil}
\section{Towards Out-of-distribution Generalization}

The previous Chapter discussed the tendency of ML to exhibit disparate performance across demographic groups, even in spite of high aggregate accuracy.
These group disparities in model performance can be interpreted in terms of the model's sensitivity to \emph{distribution shift}: after training the model on a convex mixture of (non-overlapping) subpopulations, group fairness metrics can be expressed in terms of model behavior on shifted distributions where only one mixture component is present.\footnote{
    This specific type of distribution shift is called ``subpopulation shift''.
}
In fact, many types of model failure beyond algorithmic discrimination can be framed in terms of a lack of robustness to distribution shift.
That is precisely the purpose of this Chapter, where we will address the out-of-distribution (OOD) generalization problem more broadly.

The ML research community has demonstrated impressive predictive performance on controlled benchmarks with \emph{i.i.d.} evaluation data~\citep{he_delving_2015,rajpurkar_chexnet_2017}.
However, distribution shift abounds in practical settings, and subsequently model performance can severely degrade during deployment, even to below-chance predictions~\citep{geirhos2020shortcut}.
Tiny perturbations can derail classifiers, as shown by adversarial examples~\citep{szegedy2013intriguing} and common image corruptions~\citep{hendrycks2019benchmarking}.
Even new test sets collected from the same data acquisition pipeline induce distribution shifts that significantly harm
performance~\citep{recht2019imagenet, engstrom2020identifying}. 

\subsection{The Role of Side Information in Learning Robust Features}
Many approaches have been proposed to overcome the brittleness of supervised learning---e.g.  Empirical Risk Minimization (ERM)~\citep{vapnik_principles_1991}---in the face of distribution shifts.
Robust optimization aims to achieve good performance on any distribution close to the training distribution~\citep{goodfellow2014explaining,duchi2021statistics,madry2017towards}.
Invariant learning on the other hand tries to go one step further, to generalize to distributions potentially far away from the training distribution.
Such methods pose OOD generalization in terms of learning ``robust features'' that are reliable in all domains of interest~\citep{peters2016causal}, a framing that we follow in this Chapter.

Unfortunately, common invariant learning methods typically come at a serious disadvantage: they require datasets to be partitioned into multiple domains or environments.\footnote{We use ``domains'', ``environments'' and ``groups''/``subgroups'' interchangeably.}
Environment assignments should implicitly define variation the algorithm should become invariant or robust to, but often such environment labels are unavailable at training time, either because they are difficult to obtain or due to privacy limitations. 
In some cases, relevant side-information or metadata---human annotations, or device ID used to take a medical image, hospital or department ID, etc.---may be abundant, but it remains unclear how best to specify environments based on this information \citep{srivastava2020robustness}.

A similar issue arises in mitigating algorithmic unfairness, where sensitive attributes may be difficult to define in practice \citep{hanna2020towards}, or their values may be impossible to collect.
The focus of this Chapter is to overcome the difficulty of manual environment specification by developing a new method inspired by fairness approaches for unknown group memberships~\citep{kim2019multiaccuracy,lahoti2020fairness}.
Our approach is precisely to leverage the tendency of ERM to prefer ``shortcut'' features~\citep{ferrari_recognition_2018,geirhos_shortcut_2020} (also referred to as ``spurious'' features~\citep{sober_venetian_2001,pearl_causal_2009,lopez-paz_dependence_2016}), which we show can be used to partition the training data into \emph{inferred} environments suitable for use with standard invariant learning approaches.

\section{Environment Inference for Invariant Learning}

\subsection{Background: Invariant Learning}\label{sec:exchanging-lessons}

\paragraph{Notation}
%%%%%%%%%%%%%%%%%%%%%%%%%%%%%%%%%%%%%%%%%%%%%%%%%%%%%%%%%%%% Notation macros
\newcommand{\pobs}{p^{obs}}
\newcommand{\phand}{p^{obs}}
\newcommand{\Xsp}{\mathcal{X}}
\newcommand{\Ysp}{\mathcal{Y}}
\newcommand{\Zsp}{\mathcal{H}}
\newcommand{\Esp}{\mathcal{E}^{obs}}
\newcommand{\RR}{\mathbb{R}}
\newcommand{\Normal}{\mathcal{N}}
\newcommand{\obs}{\mathbf{x}} % observations
\newcommand{\cf}{\mathbf{v}} % causal feature
\newcommand{\ncf}{\mathbf{z}} % non-causal feature
\newcommand{\pcf}{\hat{\mathbf{w}_v}} % predicted causal feature
\newcommand{\pncf}{\hat{\mathbf{w}_z}} % predicted non-causal feature
\newcommand{\lbl}{\mathbf{y}} % label
\newcommand{\spa}{\RR^{N}} % space of observations
\newcommand{\hsp}{\RR^{N/2}} % half-space of observations
\newcommand{\lblnoise}{\theta_y}
%%%%%%%%%%%%%%%%%%%%%%%%%%%%%%%%%%%%%%%%%%%%%%%%%%%%%%%%%%%%
Let $\Xsp$ be the input space, $\Esp$ the set of training environments (a.k.a. ``domains''), $\Ysp$ the target space.
Let $x, y, e \sim \pobs(x, y, e)$ 
be observational data, with $x \in \Xsp$, $y \in \Ysp$, and $e \in \Esp$.
$\Zsp$ denotes a representation space, from which a classifier $w \circ \Phi$ (that maps to the logit space of $\Ysp$ via a linear map $w$) can be applied.
${\Phi: \Xsp \rightarrow \Zsp}$ denotes the parameterized mapping or ``model'' that we optimize.
We refer to $\Phi(x) \in \Zsp$ as the ``representation'' of example $x$.
$\hat y \in \Ysp$ denotes a hard prediction derived from the classifier by stochastic sampling or probability thresholding.
$\ell: \Zsp \times \Ysp \rightarrow \R$ denotes a scalar loss, which guides the learning.

The empirical risk minimization (ERM) solution is found by minimizing the global risk, expressed as the expected loss over the observational distribution:
\begin{equation}
    C^{ERM}(\Phi) = \E_{\pobs(x, y, e)}[\ell(\Phi(x), y)].
\end{equation}

\paragraph{Representation Learning with Environment Labels}
Domain generalization is concerned with achieving low error rates on unseen test distributions $p(x, y|e_{test})$ for $e_{test} \notin \Esp$.
Domain adaption is a related problem where model parameters can be adapted at test time using unlabeled data.
\emph{Invariant Learning} approaches such as Invariant Risk Minimization (IRM) \citep{arjovsky2019invariant} and Risk Extrapolation (REx) \citep{krueger2020out} were proposed to overcome the limitations of adversarial domain-invariant representation learning \citep{zhao2019learning} by discovering invariant relationships between inputs and targets across domains.
Invariance serves as a proxy for causality, as features representing ``causes'' of target labels rather than effects will generalize well under intervention.
In IRM, a representation $\Phi(x)$ is learned that performs optimally within each environment---and is thus invariant to the choice of environment $e \in \Esp$---with the ultimate goal of generalizing to an unknown test dataset $p(x,y|e_{test})$.
Because optimal classifiers under standard loss functions can be realized via a conditional label distribution ($f^*(x) = \E[y|x]$), an invariant representation $\Phi(x)$ must satisfy the following \emph{Environment Invariance Constraint}: 
\begin{align}
\label{eq:irm}
    \E[y|\Phi(x) = h,e_1] = \E[y|\Phi(x) = h,e_2] \nonumber\\  \, \, \, \, \forall \, h \in \Zsp \, \, \, \, \forall \,  e_1, e_2 \in  \Esp.
\tag{\textsc{EIC}}
\end{align}
Intuitively, the representation $\Phi(x)$ encodes features of the input $x$ that induce the same conditional distribution over labels across each environment.
This is closely related to the notion of ``group sufficiency'' studied in the fairness literature \citep{liu2018implicit}.% (see Appendix \ref{sec:fairness_connections}).

Because trivial representations such as mapping all $x$ onto the same value may satisfy environment invariance, other objectives must be introduced to encourage the predictive utility of $\Phi$.
\citet{arjovsky2019invariant} propose IRM as a way to satisfy (\ref{eq:irm}) while achieving a good overall risk.
As a practical instantiation, the authors introduce IRMv1, a regularized objective enforcing simultaneous optimality of the same classifier $w \circ \Phi$ in all environments;\footnote{
$w \circ \Phi$ yields a classification decision via linear weighting on the representation features.}
here w.l.o.g. $w=\bar w$ is a constant scalar multiplier of 1.0 for each output dimension.
Denoting by $R^e = \E_{\pobs(x, y|e)}[\ell]$ the \mbox{per-environment} risk, the objective to be minimized is
\begin{equation}
    \label{eq:irmv1}
    C^{IRM}(\Phi) = 
    \sum_{e \in \Esp}
    R^e(\Phi)
    + \lambda ||
    \nabla_{\bar w}
    R^e(\bar w \circ \Phi)||.
    \tag{IRMv1}
\end{equation}

\paragraph{Robust Optimization}
Another approach at generalizing beyond the training distribution is robust optimization \citep{ben2009robust}, where one aims to minimize the worst-case loss for every subset of the training set, or other well-defined perturbation sets around the data \citep{duchi2021statistics, madry2017towards}.
Rather than optimizing a notion of invariance, Distributionally Robust Optimization (DRO) \citep{duchi2021statistics} seeks good performance for all nearby distributions by minimizing the worst-case loss: ${\max_q \E_{q}[\ell] \ \text{s.t.} \  D(q||p) < \epsilon}$, where $D$ denotes similarity between two distributions (e.g., $\chi^2$ divergence) and $\epsilon$ is a hyperparameter.
The objective can be computed as an expectation over $p$ via per-example importance weights ${\gamma_i = \frac{q(x_i, y_i)}{p(x_i, y_i)}}$.
GroupDRO operationalizes this principle by sharing importance weights across training examples, using environment labels to define relevant groups for this parameter sharing.
This can be expressed as an expected risk under a worst-case distribution over group proportions:
\[
C^{GroupDRO}(\Phi) = \max_g \E_{g(e)}[R^e(\Phi)]
\]
This is a promising approach towards tackling distribution shift with deep nets \citep{sagawa2019distributionally}, and we show in our experiments how environment inference enables application of GroupDRO to improve over standard learning without requiring group labels.

\paragraph{Limitations of Invariant Learning}
While the use of invariant learning to tackle domain generalization is still relatively nascent, several known limitations merit discussion.
IRM can provably find an invariant predictor that generalizes OOD, but only under restrictive assumptions, such as linearity of the data generative process and access to many environments \citep{arjovsky2019invariant}.
However, most benchmark datasets are in the non-linear regime; \citet{rosenfeld2020risks} demonstrated that for some non-linear datasets, the IRMv1 penalty term induces multiple optima, not all of which yield invariant predictors.
Nevertheless, IRM has found empirical success in some high dimensional non-linear classification tasks (e.g. CMNIST) using just a few environments \citep{arjovsky2019invariant,koh2020wilds}.
On the other hand, it was recently shown that, using careful and fair model selection strategies across a suite of image classification tasks, neither IRM nor other invariant learners consistently beat ERM in OOD generalization \citep{gulrajani2020search}.
This last study underscores the importance of \emph{model selection} in any domain generalization approach, which we discuss further below.
 
\subsection{Invariance Without Environment Labels}\label{sec:methods}
In this Section we propose a novel invariant learning framework that does not require a priori domain/environment knowledge. 
This framework is useful in algorithmic fairness scenarios when demographic makeup is not directly observed; it is also applicable in standard machine learning settings when relevant environment information is either unavailable or not clearly identified.
In both cases, a method that sorts training examples $\mathcal{D}$ into environments that maximally separate the spurious features---i.e. inferring populations $\mathcal{D}_1 \cup \mathcal{D}_2=\mathcal{D}$---can facilitate effective
invariant learning.

\paragraph{Environment Inference for Invariant Learning}

Our aim is to find environments that maximally violate the invariant learning
principle.
We can then evaluate the quality of these inferred environments
by utilizing them in an invariant learning method.
Our overall algorithm EIIL is a two-stage process: (1) Environment Inference (EI): infer the environment assignments; and (2) Invariant Learning (IL): run invariant learning given these assignments.

The primary goal of invariant-learning is to find features that are domain-invariant,
i.e, that reliably predict the true class regardless of the domain.
The EI phase aims to identify domains that help uncover these features.
This phase depends on a reference classifier
$\tilde \Phi$; which maps inputs $X$ to outputs $Y$, and defines a putative
set of invariant features.
This model could be
found using ERM on $\pobs(x,y)$, for example.
Environments are then derived that partition the mapping of the reference
model which maximally violate the invariance principle, i.e.,
where for the reference classifier the same feature vector is associated
with examples of different classes.
While any of the aforementioned invariant learning objectives can be
incorporated into the EI phase, the invariance principle or group-sufficiency---as expressed in (\ref{eq:irm})---is a natural fit, since it explicitly depends on learned feature representations $\Phi$.

To realize an EI phase focused on the invariance principle, we utilize the IRM objective (\ref{eq:irmv1}).
We begin by noting that the per-environment risk $R^e$ depends implicitly on the manual environment labels from the dataset.
For a given environment $e'$, we denote $\indicator{e_i = e'}$ as an indicator that example $i$ is assigned to that environment, and re-express the per-environment risk as:

\begin{align}
    R^{e}(\Phi) &= \frac{1}{\sum_{i'}\indicator{e_{i'} = e}
} \sum_i \indicator{e_i = e}
 \ell(\Phi(x_i), y_i)
\end{align}

Now we relax this risk measure to search over the space of environment assignments.
We replace the manual assignment indicator $\indicator{e_{i} = e'}$, with a probability distribution $\mathbf{q}_i(e'):=q(e'|x_i, y_i)$, representing a soft assignment of the $i$-th example to the $e'$-th environment.
To \emph{infer} environments, we optimize $q(e|x_i,y_i)$ so that it captures the worst-case environments for a fixed classifier $\Phi$.
This corresponds to maximizing w.r.t. $\mathbf{q}$ the following soft relaxation of the regularizer\footnote{
We omit the average risk term as we are focused on maximally violating (\ref{eq:irm}) regardless of the risk.
} from $C^{IRM}$:
\begin{align}
    C^{EI}(\Phi, \mathbf{q}) &= ||\nabla_{\bar w} \tilde R^e(\bar w \circ \Phi, \mathbf{q})||, \label{eq:ei-obj} \\
    \tilde R^e(\Phi, \mathbf{q}) &= \frac{1}{N} \sum_i \mathbf{q}_i(e) \ell(\Phi(x_i), y_i) \label{eq:soft-risk}
\end{align}
where $\tilde R^e$ represents a soft per-environment risk that can pass gradients to the environment assignments $\mathbf{q}$.
See Algorithm \ref{algo:eiil} in Appendix \ref{sec:pseudocode} for pseudocode.
j
To summarize, EIIL involves the following sequential\footnote{
We also tried jointly training $\Phi$ and $\mathbf{q}$ using alternating updates, as in GAN training, but did not find empirical benefits. 
This formulation introduces optimization and conceptual difficulties, e.g. ensuring that invariances apply to all environments discovered throughout learning.
} approach:
\begin{enumerate}
    \item Input \emph{reference model} $\tilde \Phi$;
    \item Fix $\Phi \leftarrow \tilde \Phi$ and optimize the EI objective to infer environments:  ${\mathbf{q}^* = \argmax_\mathbf{q} C^{EI}(\tilde \Phi, \mathbf{q})}$; 
    \item Fix $\mathbf{\tilde q} \leftarrow \mathbf{q}^*$ and optimize the IL objective to yield the new model: $\Phi^* = \argmin_\Phi C^{IL}(\Phi, \mathbf{\tilde q})$ 
\end{enumerate}

In our experiments we consider binary environments and parameterize the $\mathbf{q}$ as a vector of probabilities for each example in the training data.\footnote{Note that under this parameterization, when optimizing the inner loop with fixed $\Phi$ the number of parameters equals the number of data points (which is small relative
to standard neural net training).
We leave amortization of $q$ to future work.}
EIIL is applicable more broadly to any environment-based
invariant learning objective through the choice of $C^{IL}$ in Step 3.
We present experiments using $C^{IL}\in \{C^{IRM}, C^{GroupDRO}\}$, and leave a more complete exploration to future work.

\paragraph{Analyzing the Inferred Environments}
\label{sec:analysis}
To characterize the ability of EIIL to generalize to unseen test data, we now examine the inductive bias for generalization provided by the reference model $\tilde \Phi$.
We state the main result here and defer the proofs to Appendix \ref{sec:proofs}.
Consider a dataset with some feature(s) $z$ which are spurious, and other(s) $v$ which are valuable/invariant/causal w.r.t. the label $y$.
Our proof considers binary features/labels and two environments, but the same argument extends to other cases.
Our goal is to find a model $\Phi$ whose representation $\Phi(v, z)$ is invariant w.r.t. $z$ and focuses solely on $v$.

\begin{proposition}\label{thm:1}
Consider environments that differ in the degree to which the label $y$ agrees with the spurious features $z$: $\mathbb{P}(\indicator{y=z}|e_1) \neq \mathbb{P}(\indicator{y=z}|e_2)$:
then a reference model $\tilde \Phi = \Phi_{Spurious}$ that is invariant to valuable features $v$ and solely focuses on spurious features $z$ maximally violates the invariance principle (\ref{eq:irm}).
Likewise, consider the case with fixed representation $\Phi$ that focuses on the spurious features: then a choice of environments that maximally violates (\ref{eq:irm}) is $e_1 = \{v,z,y|\indicator{y = z}\}$ and $e_2 = \{v,z,y|\indicator{y \neq z}\}$.
\end{proposition}

If environments are split according to agreement of $y$ and $z$,
then the constraint from (\ref{eq:irm}) is satisfied by a representation that ignores $z$: $\Phi(x) \perp z$.
Unfortunately this requires a priori knowledge of either the spurious feature $z$ or a reference model $\tilde \Phi = \Phi_{Spurious}$ that extracts it.
When the sub-optimal solution
$\Phi_{Spurious}$ is not a priori known, it will sometimes be recovered directly from the training data; for example in CMNIST we find that $\Phi_{ERM}$ approximates $\Phi_{Color}$.
This allows EIIL to find environment partitions providing the starkest possible contrast for invariant learning.

Even if environment partitions are available, it may be possible to improve performance by inferring new partitions from scratch.
It can be shown 
that the environments provided in the {CMNIST} dataset \citep{arjovsky2019invariant} do not maximally violate (\ref{eq:irm}) for a reference model $\tilde \Phi = \Phi_{Color}$, and are thus not maximally informative for learning to ignore color.
Accordingly, EIIL improves test accuracy for IRM compared with the hand-crafted environments (Table \ref{tab:table_teaser}).

If $\tilde \Phi = \Phi_{ERM}$ focuses on a mix of $z$ and $v$, EIIL may still find environment partitions that enable effective invariant learning, as we find in the Waterbirds dataset, but they are not guaranteed to maximally violate (\ref{eq:irm}). 

\paragraph{Binned Environment Invariance}

We can derive a heuristic algorithm for EI that maximizes violations of the invariance principle by stratifying examples into discrete bins (i.e. confidence bins for 1-D representations), then sorting them into environments within each bin.
This algorithm provides insight into both the EI task and the relationship between the IRMv1 regularizer and the invariance principle.
We define bins in the space of the learned representation $\Phi(x)$,
indexed by $b$; $s_{ib}$ indicates whether example $i$ is in bin $b$.
The intuition behind the algorithm is that a simple approach can separate the examples in a bin to achieve the maximal value of the (\ref{eq:irm}).

The degree to which the environment assignments violate (\ref{eq:irm}) can be expressed as follows, which can then be approximated in terms of the bins:
{\small
\begin{eqnarray}
\Delta \rm{EIC} & = & (E[y|\Phi(x),e_1] - E[y|\Phi(x),e_2])^2 \label{eq:softEIC} \nonumber\\ 
& \approx & \sum_b (\sum_i s_{ib} y_i \mathbf{q}_i(e=e_1) - \sum_i s_{ib} y_i \mathbf{q}_i(e=e_2))^2 \nonumber\label{eq:binEIC}
\end{eqnarray}
}
Inspection of this objective leads to a simple algorithm: assign all the $y=1$ examples to one
environment, and $y=-1$ examples to the other. This results in the
expected values of $y$ equal to $\pm 1$, which achieves the maximum possible value of $\Delta \rm{EIC}$ per bin.\footnote{For very confident reference models, where few confidence bins are populated, splitting based on $y$ relates to partitioning the error cases. This splitting strategy is not the only possible solution, as multiple global optima exist.
}

This binning leads to an important insight into the relationship 
between the IRMv1 regularizer and (\ref{eq:irm}). Despite the analysis in \citet{arjovsky2019invariant}, this link is not completely
clear \citep{kamath2021does, rosenfeld2020risks}.
However, in the situation considered here, with binary classes, we can use this binning approach to show a tight link between the two objectives:
finding an environment assignment that maximizes the
violation of our softened IRMv1 regularizer (Equation \ref{eq:ei-obj}) also maximizes the violation
of the softened Environment Invariance Constraint ($\Delta \rm{EIC}$); see Appendix \ref{sec:majmin} for the proof.
This binning approach highlights the dependence on the reference model, as the bins are defined in its learned $\Phi$ space; the reference model also played a key role in the analysis above. We analyze it empirically in Section \ref{sec:secondary-results}.

\section{Experiments}\label{sec:experiments}
We proceed by describing the remaining datasets under study in Section \ref{sec:datasets}.
We then present the main results measuring the ability of EIIL to handle distribution shift in Section \ref{sec:primary-results}, and offer a more detailed analysis of the EIIL solution and its dependence on the reference model in Section \ref{sec:secondary-results}.
See \mbox{\url{https://github.com/ecreager/eiil}} for code.

\paragraph{Model selection}
Tuning hyperparameters when train and test distributions differ is a difficult open problem \citep{krueger2020out,gulrajani2020search}.
Where possible, we reuse effective hyperparameters for IRM and GroupDRO found by previous authors.
Because these works allowed limited validation samples for hyperparameter tuning (all baseline methods benefit fairly from this strategy), these results represent an optimistic view on the ability for invariant learning.
As discussed above, the choice of reference classifier is of crucial importance when deploying EIIL; if worst-group performance can be measured on a validation set, this could be used to tune the hyperparameters of the reference model (i.e. model selection subsumes reference model selection).
See Appendix \ref{sec:experimental_details} for further discussion.

\subsection{Datasets}\label{sec:datasets}
\paragraph{CMNIST}
CMNIST is a noisy digit recognition task\footnote{MNIST digits are grouped into $\{0,1,2,3,4\}$ and $\{5,6,7,8,9\}$ so the CMNIST target label $y$ is binary.}
where color is a spurious feature that correlates with the label at train time but anti-correlates at test time, with the correlation strength at train time varying across environments \citep{arjovsky2019invariant}.
In particular, the two training environments have $Corr(color, label) \in \{0.8,0.9\}$ while the test environment has $Corr(color, label)=0.1$.
Crucially, label noise is applied by flipping $y$ with probability $\lblnoise = 0.25$.
This implies that shape (the invariant feature) is marginally less reliable than color in the training data, so ERM ignores shape to focus on color and suffers from below-chance test performance.

\paragraph{Waterbirds}

To evaluate whether EIIL can infer useful environments in a more challenging setting with high-dimensional images, we turn to the Waterbirds dataset \citep{sagawa2019distributionally}.
Waterbirds is a composite dataset that combines $4,795$ bird images from the CUB dataset \citep{wah2011caltech} with background images from the Places dataset \citep{zhou2017places}.
It examines the proposition (which frequently motivates invariant learning approaches) that modern networks often learn spurious background features (e.g. green grass in pictures of cows) that are predictive of the label at train time but fail to generalize in new contexts \citep{beery2018recognition,geirhos2020shortcut}.
The target labels are two classes of birds---``landbirds'' and ``waterbirds'' respectively coming from dry or wet habitats---superimposed on land and water backgrounds.
At training time, landbirds and waterbirds are most frequently paired with land and water backgrounds, respectively, but at test time the 4 subgroup combinations are uniformly sampled.
To mitigate failure under distribution shift, a robust representation should learn primarily features of the bird itself, since these are invariant, rather than features of the background.
Beyond the increase in dimensionality, this task  differs from CMNIST in that the ERM solution does not fail catastrophically at test time, and in fact can achieve 97.3\% average accuracy.
However, because ERM optimizes average loss, it suffers in performance on the worst-case subgroup (waterbirds on land, which has only $56$ training examples).

\paragraph{CivilComments-WILDS}

We apply EIIL to a large and challenging comment toxicity prediction task with important fairness implications \citep{borkan2019nuanced}, where ERM performs poorly on comments associated with certain identity groups.
We follow the procedure and data splits of \citet{koh2020wilds} to finetune DistilBERT embeddings \citep{sanh2019distilbert}.
EIIL uses an ERM reference classifier and its inferred environments are fed to a GroupDRO invariant learner.
Because the large training set ($N_{train}=269,038$) increases the convergence time for gradient-based EI, we deploy the binning heuristic discussed in Section \ref{sec:methods}, 
which in this instance finds environments that correspond to the error and non-error cases of the reference classifier.
While ERM and EIIL do not have access to the sensitive group labels, we note that worst-group validation accuracy is used to tune hyperparameters for all methods.
See Appendix \ref{sec:experimental_details} for details.
We also compare against a GroupDRO (oracle) learner that has access to group labels.

\subsection{Results}\label{sec:primary-results}

\paragraph{CMNIST}

%%%%%%%%%%%%%%%%%%%%%%%%%%%%%%%%%%%%%%%%%%%%%%%%%%%%%%
% Main CMNIST result
%%%%%%%%%%%%%%%%%%%%%%%%%%%%%%%%%%%%%%%%%%%%%%%%%%%%%%
\renewcommand{\cmark}{\color{green}\ding{51}}
\renewcommand{\xmark}{\color{red}\ding{55}}
\begin{table}[h]
\centering
\begin{tabular}{cccc}
\toprule
{Method} & Handcrafted &      Train &       Test \\
 & Environments &      &        \\
\midrule
ERM       & \xmark &  \textbf{86.3 $\pm$ 0.1} &  13.8 $\pm$ 0.6 \\
IRM  & \cmark & 71.1 $\pm$ 0.8 &  65.5 $\pm$ 2.3 \\
\textbf{EIIL}
& \xmark   &  73.7 $\pm$ 0.5 &  \textbf{68.4 $\pm$ 2.7} \\
\bottomrule
\end{tabular}
\caption{
Accuracy (\%) on CMNIST, a digit classification task where color is a spurious feature correlated with the label during training but anti-correlated at test time.
EIIL exceeds test-time performance of IRM \emph{without} knowledge of pre-specified environment labels, by instead finding worst-case environments using aggregated data and a reference classifier.
}
\label{tab:table_teaser}
\end{table}
%%%%%%%%%%%%%%%%%%%%%%%%%%%%%%%%%%%%%%%%%%%%%%%%%%%%%%

IRM was previously shown to learn an invariant representation on this dataset, allowing it to generalize relatively well to the held-out test set whereas ERM fails dramatically \citep{arjovsky2019invariant}.
It is worth noting that label noise makes the problem challenging, so even an oracle classifier can achieve at most 75\% test accuracy on this binary classification task.
To realize EIIL in our experiments, we discard the environment labels, and run the procedure described in Section \ref{sec:methods}
with ERM as the reference model and IRM as the invariant learner used in the final stage.
We find that EIIL's environment labels are very effective for invariant learning, ultimately \emph{outperforming} standard IRM using the environment labels provided in the dataset 
(Table \ref{tab:table_teaser}).
This suggests that in this case the EIIL solution approaches the \emph{maximally informative} set of environments discussed in Proposition \ref{thm:1}.

\paragraph{Waterbirds}

\citet{sagawa2019distributionally} demonstrated that ERM suffers from poor worst-group performance on this dataset, and that GroupDRO can mitigate this performance gap if group labels are available.
In this dataset, group labels should be considered as oracle information, meaning that the relevant baseline for EIIL is standard ERM.
The main contribution of \citet{sagawa2019distributionally} was to show \emph{how} deep nets can be optimized for the GroupDRO objective using their online algorithm that adaptively adjusts per-group importance weights.
In our experiment, we combine this insight with our EIIL framework to show that distributionally robust neural nets can be realized without access to oracle information.
We follow the same basic procedure as above,\footnote{For this dataset, environment inference worked better with reference models that were not fully trained. We suspect this is because ERM focuses on the easy-to-compute features like background color early in training, precisely the type of bias EIIL can exploit to learn informative environments.} in this case using GroupDRO as the downstream invariant learner for which EIIL's inferred labels will be used.
        
%%%%%%%%%%%%%%%%%%%%%%%%%%%%%%%%%%%%%%%%%%%%%%%%%%%%%
% waterbirds main table
%%%%%%%%%%%%%%%%%%%%%%%%%%%%%%%%%%%%%%%%%%%%%%%%%%%%%
\begin{table}[ht]
\centering
\begin{tabular}{lccc}
\toprule
{Method} &  Train (avg) & Test (avg) & Test (worst group) \\
\midrule
ERM & 100.0 &  \textbf{97.3} &  60.3 \\
EIIL  & 99.6 &  96.9 & \textbf{78.7} \\
\midrule
GroupDRO (oracle) & 99.1  & 96.6 &  84.6 \\
\bottomrule
\end{tabular}
\caption{Accuracy (\%) on the Waterbirds dataset.
EIIL strongly outperforms ERM on worst-group performance, approaching the performance of the GroupDRO algorithm proposed by \citet{sagawa2019distributionally}, which requires oracle access to group labels.
In this experiment we feed environments inferred  by EIIL into a GroupDRO learner.
}
\label{tab:waterbirds_results}
\end{table}
%%%%%%%%%%%%%%%%%%%%%%%%%%%%%%%%%%%%%%%%%%%%%%%%%%%%%%

%%%%%%%%%%%%%%%%%%%%%%%%%%%%%%%%%%%%%%%%%%%%%%%%%%%%%%
% waterbirds inferred envs figure
%%%%%%%%%%%%%%%%%%%%%%%%%%%%%%%%%%%%%%%%%%%%%%%%%%%%%%
\begin{figure}[b!]
\centering
\includegraphics[width=.65\textwidth]{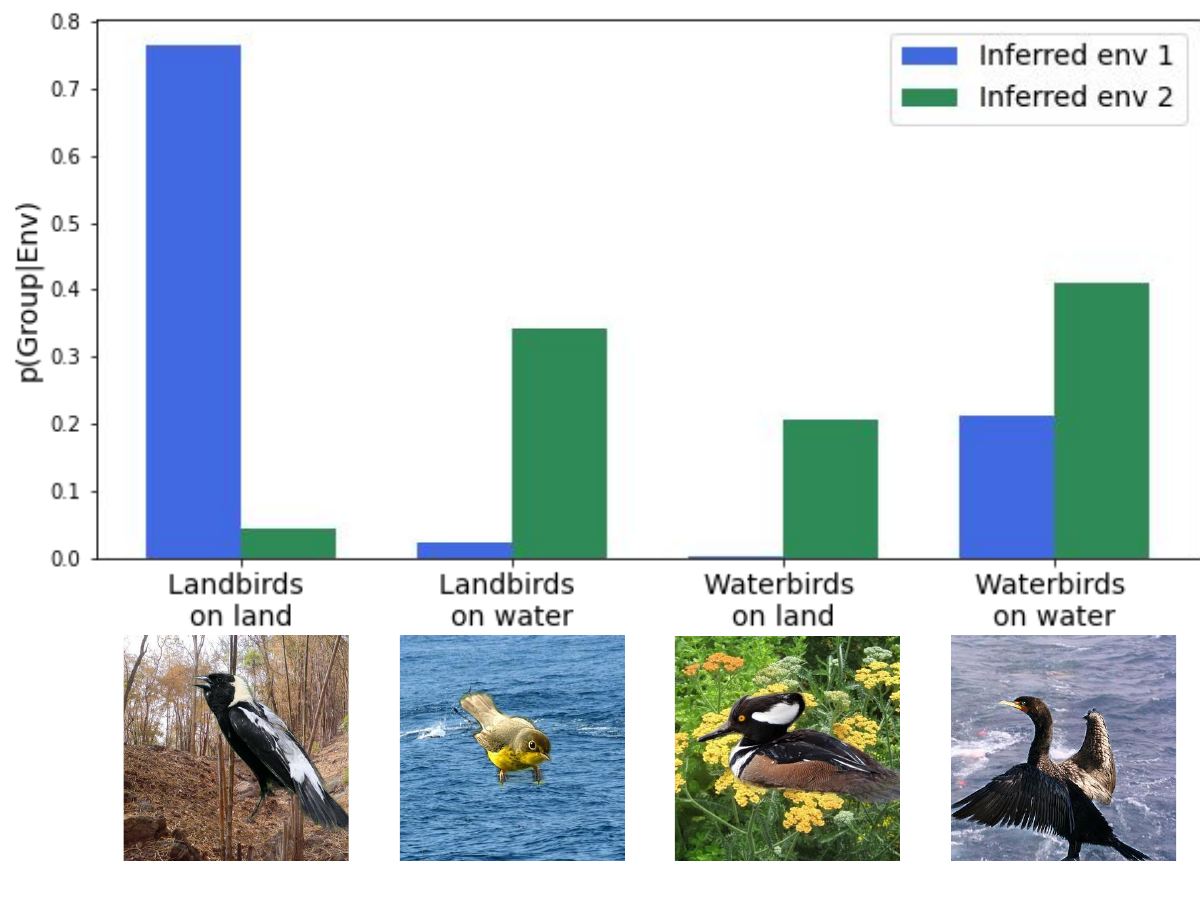}
\caption[]{After using EIIL to directly infer two environments from the Waterbirds dataset, we examine the proportion of each subgroup (available in the original dataset but not used by EIIL) present in the inferred environment.
}
\label{fig:waterbirds_inferred_envs}
\end{figure}
%%%%%%%%%%%%%%%%%%%%%%%%%%%%%%%%%%%%%%%%%%%%%%%%%%%%%%

EIIL is significantly more robust than the ERM baseline (Table \ref{tab:waterbirds_results}), raising worst-group test accuracy by 18\% with only a 1\% drop in average accuracy.
In Figure \ref{fig:waterbirds_inferred_envs} we plot the distribution of subgroups for each inferred environment, showing that the minority subgroups (landbirds on water and waterbirds on land) are mostly organized in the same inferred environment.
This suggests the possibility of leveraging environment inference for interpretability to automatically discover a model's performance discrepancies on subgroups, which we leave for future work.

\paragraph{CivilComments-WILDS}

%%%%%%%%%%%%%%%%%%%%%%%%%%%%%%%%%%%%%%%%
% CivilComments table
\begin{table}[h]
\vspace{-.1cm}
\centering
\begin{tabular}{lcccccc}
\toprule
{Method} &        Train (avg) & Test (avg) & Test (worst group) \\
\midrule
ERM  &  96.0 $\pm$ 1.5 &   \textbf{92.0 $\pm$ 0.4} &  61.6 $\pm$ 1.3 \\
EIIL &  97.0 $\pm$ 0.8 &   90.5 $\pm$ 0.2 &  \textbf{67.0 $\pm$ 2.4} \\
\midrule
GroupDRO (oracle) &  93.6 $\pm$ 1.3 &  89.0 $\pm$ 0.3 &  69.8 $\pm$ 2.4 \\
\bottomrule
\end{tabular}
\caption{
EIIL improves worst-group accuracy in the CivilComments-WILDS toxicity prediction task,
without access to group labels.
}
\end{table}
%%%%%%%%%%%%%%%%%%%%%%%%%%%%%%%%%%%%%%%%

Without knowledge of which comments are associated with which groups, EIIL improves worst-group accuracy over ERM with only a modest cost in average accuracy, approaching the oracle GroupDRO solution (which requires group labels).

\subsection{Influence of the Reference Model}\label{sec:secondary-results}
As discussed in Section \ref{sec:analysis}, the ability of EIIL to find useful environments---partitions yielding an invariant representation when used by an invariant learner---depends on its ability to exploit variation in the predictive distribution of a reference classifier.
Here we study the influence of the reference classifier on the final EIIL solution.
We return to the CMNIST dataset, which provides a controlled sampling setup where particular ERM solutions can be induced to serve as reference for EIIL.

EIIL was shown to outperform IRM \emph{without access to environment labels} in the standard CMNIST dataset (Table \ref{tab:table_teaser}), which has label noise of $\lblnoise=0.25$.
Because $Corr(color, label$) is 0.85 (on average) for the train set, this amount of label noise implies that color is the most predictive feature on aggregated training set (although its predictive power varies across environments).
ERM, even with access to infinite data, will focus on color given this amount of label noise to achieve an average train accuracy of 85\%.
However we can implicitly control the ERM solution $\Phi_{ERM}$ by tuning $\theta_y$, an insight that we use to study the dependence of EIIL on the reference model $\tilde \Phi = \Phi_{ERM}$.

%%%%%%%%%%%%%%%%%%%%%%%%%%%%%%%%%%%%%%%%%%%%%%%%%%%%%%%%%%%%%%%%%%%%%%%%%%%%%%%%
% follow-up CMNIST plot - analysis of ref model
%%%%%%%%%%%%%%%%%%%%%%%%%%%%%%%%%%%%%%%%%%%%%%%%%%%%%%%%%%%%%%%%%%%%%%%%%%%%%%%%
\newcommand{\cmnistFigWidth}{0.4\textwidth}
\begin{figure*}[ht!]
\centering
\begin{subfigure}[t]{\cmnistFigWidth}
\includegraphics[width=\textwidth]{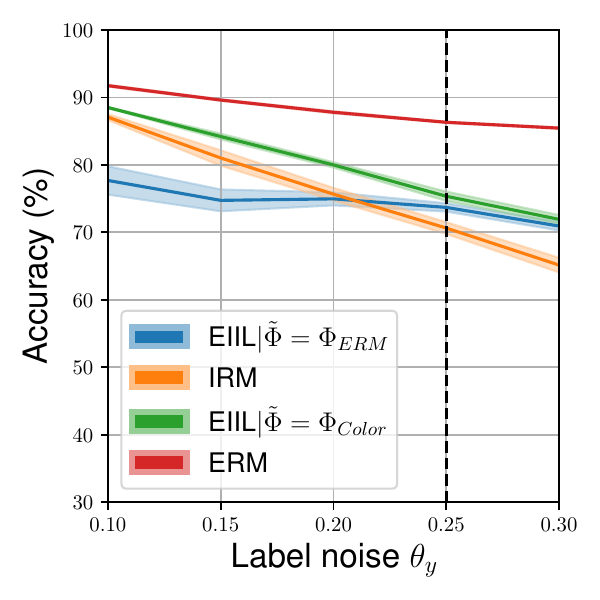}
\caption{Train accuracy}
\label{fig:label_noise_sweep_train_accs}
\end{subfigure}
\begin{subfigure}[t]{\cmnistFigWidth}
\includegraphics[width=\textwidth]{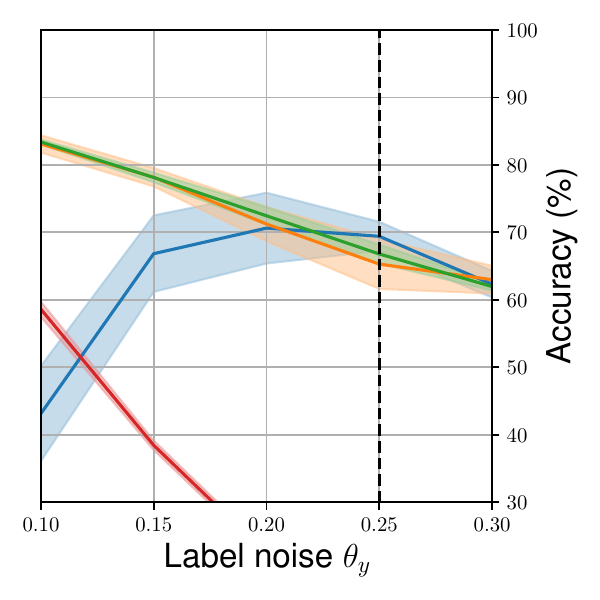}
\caption{Test accuracy}
\label{fig:label_noise_sweep_test_accs}
\end{subfigure}
\caption{
CMNIST with varying label noise $\lblnoise$.
Under high label noise (${\lblnoise > .2}$), where the spurious feature color correlates to label \emph{more} than shape on the train data, EIIL matches or exceeds the test performance of IRM \emph{without relying on hand-crafted environments}.
Under medium label noise (${.1 < \lblnoise < .2}$), EIIL is worse than IRM but better than ERM, the logical approach if environments are not available.
Under low label noise (${\lblnoise < .1}$), where color is \emph{less} predictive than shape at train time, ERM  performs well and EIIL  fails.
The vertical dashed black line indicates the default setting of $\theta_y=0.25$, which we report in Table
\ref{tab:table_teaser}.
}
\label{fig:label_noise_sweep}
\end{figure*}
%%%%%%%%%%%%%%%%%%%%%%%%%%%%%%%%%%%%%%%%%%%%%%%%%%%%%%%%%%%%%%%%%%%%%%%%%%%%%%%%

Figure \ref{fig:label_noise_sweep} shows the results of our study.
We find that EIIL generalizes better than IRM  with sufficiently high label noise $\lblnoise > .2$, but generalizes poorly under low label noise.
This is precisely because ERM learns the color feature under high label noise, and the shape feature under low label noise.
We verify this conclusion by evaluating EIIL when $\tilde \Phi = \Phi_{Color}$, i.e. a hand-coded color-based predictor as reference, which does well across all settings of $\lblnoise$.

We saw in the Waterbirds experiment that it is not a strict requirement that ERM fail completely in order for EIIL to succeed.
However, this controlled study highlights the importance of the reference model in the ability of EIIL to find environments that emphasize the right invariances, which leaves open the question of how to effectively choose a reference model for EIIL in general.
One possible way forward is by using validation data that captures the \emph{kind} of distribution shift we expect at test time, without exactly producing the test distribution, e.g. as in the WILDS benchmark \citep{koh2020wilds}.
In this case we could choose to run EIIL with a reference model that exhibits a large generalization gap between the training and validation distributions.
\section{Related Work}\label{sec:eiil:related-work}

\paragraph{Domain adaptation and generalization}

Beyond the methods discussed above, a variety of recent works have approached the domain generalization problem from the lens of learning invariances in the training data.
Adversarial training is a popular approach for learning representations invariant \citep{zhang2017aspect,hoffman2018cycada,ganin2016domain} or conditionally invariant \citep{li2018deep} to the environment.
However, this approach has limitations in settings where distribution shift affects the marginal distribution over labels \citep{zhao2019learning}.

\citet{arjovsky2019invariant} proposed IRM to mitigate the effect of test-time label shift, which was inspired by applications of causal inference to select invariant features  \citep{peters2016causal}.
\citet{krueger2020out} proposed the related Risk Extrapolation (REx) principle, which dictates a stronger preference to exactly equalize $R^e \ \forall \ e$ (e.g. by penalizing variance across $e$ as in their practical algorithm V-REx), which is shown to improve generalization in several settings.\footnote{
Analogous to V-REx, \citet{williamson2019fairness} adapt Conditional Variance at Risk (CVaR) \citep{rockafellar2002conditional} to equalize risk across demographic groups.
}

Several large-scale benchmarks have been proposed to highlight difficult open problems in this field, including the use of real-world data \citep{koh2020wilds}, handling subpopulation shift \citep{santurkar2020breeds}, and model selection \citep{gulrajani2020search}.

\paragraph{Leveraging a reference classifier}
A number of methods have been proposed that improve performance by exploiting the mistakes of a pre-trained auxiliary model, as we do when inferring environments for the invariant learner using $\tilde \Phi$.
\citet{nam2020learning} jointly train a ``biased'' model $f_B$ and a ``debiased'' model $f_D$, where the relative cross-entropy losses of $f_B$ and $f_D$ on each training example determine their importance weights in the overall training objective for $f_D$.
\citet{sohoni2020no} infer a different set of ``hidden subclasses'' for each class label $y \in \Ysp$, Subclasses computed in this way are then used as group labels for training a GroupDRO model, so the overall two-step process corresponds to certain choices of EI and IL objectives. 

\citet{liu2021just} and \citet{dagaev2021too} concurrently proposed to compute importance weights for the primary model using an ERM reference, which can be seen as a form of distributionally robust optimization where the worst-case distribution only updates once.
\citet{dagaev2021too} use the confidence of the reference model to assign importance weights to each training example.
\citet{liu2021just} split the training examples into two disjoint groups based on the errors of ERM, akin to our EI step, with the per-group importance weights treated as a hyperparameter for model selection (which requires a subgroup-labeled validation set).
We note that the implementation of EIIL using binning, discussed in Section \ref{sec:methods}, can also realize an error splitting behavior. 
In this case, both methods use the same disjoint groups of training examples towards slightly different ends: we train an invariant learner, whereas \citet{liu2021just} train a cross-entropy classifier with fixed per-group importance weights.

\paragraph{Algorithmic fairness}

Our work draws inspiration from a rich body of work on learning fair classifiers in the absence of demographic labels \citep{hebert2017calibration,kearns2018preventing,hashimoto2018fairness,kim2019multiaccuracy,lahoti2020fairness}.
Generally speaking, these works seek a model that performs well for group assignments that are the worst case according to some fairness criterion.
Environment inference serves a similar purpose for our method, but with a slightly different motivation: rather than learn an fair model in an online way that provides favorable in-distribution predictions, we learn discrete data partitions as an intermediary step, which enables use of invariant learning methods to tackle distribution shift.

Adversarially Reweighted Learning (ARL) \citep{lahoti2020fairness} is most closely related to ours, since they emphasize subpopulation shift as a key motivation.
Whereas ARL uses a DRO objective that prioritizes stability in the loss space, we explore environment inference to encourage invariance in the learned representation space.
We see these as complementary approaches that are each suited to different types of distribution shift.

When group labels are available for training, we can establish links between algorithmic fairness and related methods for invariant learning.
Early approaches to learning fair representations \citep{zemel2013learning,edwards2015censoring,louizos2015variational,zhang2018mitigating,madras2018learning} leveraged statistical independence regularizers from domain adaptation
\citep{ben2010theory,ganin2016domain,tzeng2017adversarial,long2018conditional}, noting that marginal or conditional independence from domain to prediction relates to the fairness notions of demographic parity $\hat y \perp e$ \citep{dwork2012fairness} and equal opportunity $\hat y \perp e | y$ \citep{hardt2016equality}.

Recall that (\ref{eq:irm}) involves an environment-specific conditional label expectation given a data representation ${\E[y|\Phi(x) = h,e]}$.
Objects of this type have been closely studied in the fair machine learning literature, where $e$ now denotes a  ``sensitive'' attribute indicating membership in a protected demographic group (age, race, gender, etc.), and the vector representation $\Phi(x)$ is typically replaced by a scalar score\footnote{
For binary classification, score-based and representation-based approaches are closely related since scores are commonly implemented as (or can be interpreted as) as the linear mapping of a data representation:  $S(x) = w \circ \Phi(x)$.
} $S(x) \in \R$. 
Noting that $\sigma(S(x))$ represents the probability of the model prediction, $\E[y|S(x),e]$ can now be interpreted as a \emph{calibration curve} that must be regulated according to some fairness constraint.
\citet{chouldechova2017fair} showed that equalizing this calibration curve across groups is often incompatible with a common fairness constraint, demographic parity, while \citet{liu2018implicit} studied ``group sufficiency'' of classifiers with strongly convex losses, concluding that ERM naturally finds group sufficient solutions without fairness constraints.

\section{Discussion}
This Chapter presented EIIL, a framework that leverages a non-robust reference model to partition the training data, which can enable robust downstream learning.
In general, learning to generalize OOD will require additional assumptions or inductive biases beyond what are typically used in supervised learning.
Therefore, EIIL should be understood as a robust learning approach that uses a different inductive bias than standard Domain Generalization: rather than relying on human-specified side information (environment labels), EIIL assumes access to a reference model that is \emph{not} robust under the distribution shifts of interest.
In this sense, EIIL does not provide any free lunch, as the sensitivity of OOD performance to the reference model is a known issue (see Section~\ref{sec:secondary-results}).
However, since EIIL was published as a conference paper, subsequent efforts have extended EIIL or otherwise addressed these limitations through complementary approaches.

\citet{ahmed2021systematic} demonstrated that EIIL can be extended to effectively handle ``systematic'' generalization \citep{bahdanau2018systematic} on a semi-synthetic foreground/background task similar to the Waterbirds dataset that we study.
To do so, they proposed a novel invariance regularizer for use in the IL stage, which is based on matching class-conditioned average predictive distributions across environments.
We note that this is closely related to the equalized odds criterion commonly used in fair classification \citep{hardt2016equality}.

\citet{mishra_repeated_2022} provide an insightful information-theoretic analysis of environment inference, where EI amounts to choosing a majority group (i.e. the inferred environment with the most examples) such that the reference representation $\Phi(x)$ has more information about the target label $y$ within the majority group than it does in aggregate (in the original training distribution).
They then formulate a procedure that improves the downstream invariant learner through a process of iterative repeated EIIL.
Meanwhile, \citet{zhang_rich_2022} showed that iteratively growing a set of invariant features (similar to repeated EIIL) is helpful in finding a good initialization for invariant learning.
In early prototyping efforts, we tried a rudimentary version of repeated EIIL without the information theoretic interpretation, but found it to be unstable.
The practical success of repeated EIIL demonstrated by \citet{mishra_repeated_2022} suggests connections between environment-free invariant learning and boosting~\citep{kearns_toward_1992,goos_agnostic_2001,kalai_agnostic_2008}.
This connection is quite intuitive---boosting was a key inspiration for works on fair learning without demographic labels~\citep{hebert2017calibration,kearns2018preventing}, which in turn inspired EIIL---and could be further explored in the future.

All neural networks, even ones trained using unconstrained supervised learning (e.g. ERM), have internal representations.
The arguments presented in this Chapter (like those in Chapter~\ref{chap:fair-reps}) follow the intuition that because ERM learns shortcuts, then using its internal representations could cause unintended downstream model failures, thus motivating more specialized representation learning algorithms.
It is worth pausing to note that this position is not a universal consensus, especially in light of the fact that specialized representation learners often underperform relative to ERM on in-distribution data.\footnote{This is analogous to the fairness-accuracy tradeoffs discussed in Chapter~\ref{chap:fair-reps}} 

One noteworthy countervailing argument is that that representation learning should be done by ERM or other unconstrained representation learners---e.g. self-supervised learners like CLIP~\citep{radford_learning_2021}---then adapted for the task at hand using data from the target distribution~\citep{menon_long-tail_2021,kirichenko_last_2022}.
It would seem that while ERM predictions rely on shortcut features~\citep{geirhos2020shortcut}, their internal representations contain more than just shortcut features~\citep{menon_long-tail_2021}.
Indeed, ~\citet{eyre_towards_2022} found that even invariant representation learners like IRM~\citep{arjovsky2019invariant} can absorb shortcuts (in addition to robust features) into their internal representations, although these are typically projected out by the subsequent classification layer.
Researchers will continue to study the properties of invariant and unconstrained representations.
In the meantime, practitioners wondering whether one should be preferred over the other should consider whether labeled data can be acquired from the target domain, and also whether distribution shifts can be expected after model deployment.

  \chapter{Data Augmentation for Robust Reinforcement Learning}\label{chap:coda}

\section{Contemporary RL Often Lacks Data Coverage}

Chapters \ref{chap:fair-reps} and \ref{chap:eiil} provided variations on the theme of realizing robust prediction by addressing troublesome correlation patterns between observed inputs and targets in the training data.
In this Chapter, we shift our focus from the static setting of supervised learning to \emph{sequential decision making}, where our decision-making model is deployed in a dynamically changing world.
In particular, we will follow the approach of Reinforcement Learning (RL)~\citep{sutton2018reinforcement} by requiring our decision-making model---hereafter referred to as the ``agent''---to learn a mapping from observed states to pre-defined actions that maximizes a reward signal.

Let us formalize this setup using a Markov Decision Process (MDP), described by tuple $\langle \S, \A, P, R, \gamma \rangle$ consisting of the state space, action space, transition function, reward function, and discount factor, respectively~\citep{puterman2014markov,sutton2018reinforcement}. 
We denote individual states and actions using lowercase $s \in \S$ and $a \in \A$, and variables using the uppercase $S$ and $A$ (e.g., $s \in \textrm{range}(S) \subseteq \S$). 
A policy $\pi : \S \times \A \to [0, 1]$ defines a probability distribution over the agent's actions at each state, and an agent is typically tasked with learning 
a parameterized policy that maximizes the expected reward in the long run.

In order to act optimally in a given observed state $s$, it is thus necessary to reason about how candidate actions $a \in \A$ will affect rewards accrued in the distant future.
For a given tuple $(s, a)$ within the state-action space, the long-term prospects for a policy $\pi$ are encapsulated by its value function
\begin{equation}
\label{eq:state-action-value}
Q_\pi(s, a)	= \mathbb{E}_{P,\pi}[
	R(s_{t}, s_{a}) + 
	\gamma R(s_{t+1}, s_{a+1}) + 
	\gamma^2 R(s_{t+2}, s_{a+2}) + 
	...|s_t=s, a_t=a
	].
\end{equation}
The infinite summand in Equation \ref{eq:state-action-value} presents an apparent computational hurdle.
One popular approach to resolving this issue is dynamic programming~\citep{bellman1957dynamic,kirk2004optimal}, which exploits the recursive structure in the summands.
For example, the Bellman optimality condition provides a relation~\citep{bellman1957dynamic} provides a relation between the value at the current state-action pair $(s_t, a_t)$ and its successor state $s_{t+1}$ under the optimal policy $\pi^*$:
\[
	Q_{\pi^*}(s_t, a_t) = \E[r_{t} + \gamma \max_{a'} Q_{\pi^*}(s_{t+1}, a')|s_t,a_t].
\]
The Q-learning algorithm~\citep{watkins1992q} exploits this relation to derive a fixed-point update to an approximation of $Q$:
\[
	Q(s_t, a_t) \leftarrow Q(s_t, a_t) + \alpha [r_{t+1} + \gamma \max_{a'} Q(s_{t+1},a') - Q(s_t, a_t)]
\]
where $\alpha$ is a hyperparameter representing the step size.
For appropriate settings of $\alpha$, this approximation is known to converge to $Q_{\pi^*}(s,a)$, from which the optimal policy $\pi^*$ can be derived~\citep{watkins1992q}.

At first glance this suggests an optimistic view on robust RL.
Because the agent will converge to an optimal policy \emph{regardless} of the training data distribution, it would seem that previously discussed concerns about fairness and shortcut learning---where models were seen to be highly sensitive to particularities in how training data is collected---simply do not apply in the RL setting.
However, in practice, we find that RL algorithms are indeed sensitive to training data.
This is because the available data typically do not satisfy the key assumptions underpinning classic convergence guarantees.
In particular, with most dynamic programming algorithms,  ``all that is required for correct convergence is that all (state-action) pairs continue to be updated''~\citep[Sec. 6.5]{sutton2018reinforcement}.
In other words, the training data must have sufficient \emph{coverage}: all vales of $(s, a)$ must be visited with reasonable frequency.
This is a tall order in real-world settings for two reasons:
first, the state space grows exponentially in the number of observed features, of which there are many in contemporary tasks of interest;
second, in safety critical applications like clinical health care and robotics, it may be unethical to deploy an exploratory policy (e.g. prescribing random treatments to patients) in order to realize this coverage condition~\citep{gottesman_guidelines_2019,brunke_safe_2021}.

This Chapter addresses the data coverage issue in RL through the lens of data augmentation:
by leveraging causal priors (typically specified by domain experts, although sometimes learned directly from data) we can improve the quantity and quality of available training data.

\section{The Role of Data Augmentation in RL}
Data augmentation has become a go-to tool in the arsenal of RL methods, especially in settings with large state spaces, or when using neural networks to approximate value functions during learning.
Take for example Hindsight Experience Replay (HER)~\citep{andrychowicz2017hindsight}, which is designed for goal-conditioned RL, where 
the agent must learn to act optimally for any number of distinct goals $g \in \G$.
Including the goal---perhaps a continuous vector indicating the target position where an object must be placed---as a special type of state in the MDP amplifies the size of the overall state-action space, and in fact \citet{andrychowicz2017hindsight} show that many goal-conditioned continuous control tasks simply cannot be solved using this naive approach.
To resolve this difficulty, they propose an application of \emph{goal relabeling}~\citep{kaelbling1993learning}, which effectively reframes the observed data so far.
For experiences where the agent fails to achieve the desired goal $g \in \G$, the experience is \emph{relabeled} as a success under an alternate goal $g' \neq g$, where $g'$ could indicate the coordinate where the object was dropped in reality (regardless of the agent's original goal).

Learning to act from images presents an analogous coverage issue, as the number of pixels in images space is huge.
In this case, training the agent on replicated observed experiences with superficial perturbations on the images (random crops, etc.) has been shown to improve the sample efficiency of RL~\citep{laskin2020reinforcement,kostrikov2020image}.

RL methods are often categorized as being either ``model-free'' or ``model-based''.
Here the ``model'' in question is a learned dynamics model $P_\theta: \S \times \A \rightarrow \S$ trained to approximate the MDP's true transition function $P$ (which may be inaccessible to the agent).
Model-free methods directly consume observed $(s, a, s')$ data to learn values and policies.
On the other hand, model-based methods like Dyna~\citep{sutton1991dyna} first fit $P_\theta: \S \times \A \rightarrow \S$ fit to the observed data $(s, a, s')$ then add samples from it to the observed experience (thus augmenting the available data) during the learning process.

\section{Amplifying the Experience Buffer Using Causal Priors}
\subsection{Why is Data Augmentation a Useful Heuristic in RL?}\label{sec:coda:why_does_data_aug_work}
Data augmentation for RL is evidently effective, but why does it work?
The key insight of this Chapter is that existing data augmentation methods can be interpreted as implicit counterfactual samplers.
This involves formulating a causal graphical model~\citep{pearl2009causal}
over single-step transition diagrams
We then observe that whenever this graph composes two or more \emph{independent mechanisms}, then this independence justifies resampling one mechanism's values to those consistent with the data distribution.
For example, with HER~\citep{andrychowicz2017hindsight} despite the policies being goal-conditioned, once the data is collected, the goals are independent of the other next-state dynamics (e.g. trajectories of objects) in the scene
 (See Fig.~\ref{fig_coda_examples}).

\begin{figure}[!t]
	\centering
	\includegraphics[width=\textwidth]{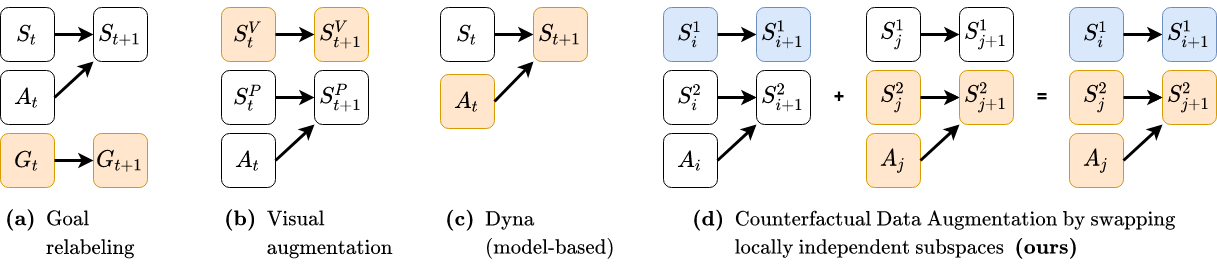}
	\vspace{-\baselineskip}
	\caption{\textit{Four instances of CoDA; orange nodes are relabeled, noise variables omitted for clarity.} \textbf{First:} Goal relabeling~\citep{kaelbling1993learning}, including HER~\citep{andrychowicz2017hindsight}, augments transitions with counterfactual goals. \textbf{Second:} Visual feature augmentation~\citep{andrychowicz2020learning,laskin2020reinforcement} uses domain knowledge to change visual features $S_t^V$ (such as textures, lighting, and camera positions) that the designer knows do not impact the physical state $S^P_{t+1}$. \textbf{Third:} Dyna~\citep{sutton1991dyna}, including MBPO~\citep{janner2019trust}, augments real states with new actions and resamples the next state using a learned dynamics model. \textbf{Fourth (ours):} Given two transitions that share local causal structures, we propose to swap connected components to form new transitions.}
	\label{fig_coda_examples}
\end{figure}

This Chapter formalizes and discusses a generic recipe for Counterfacutual Data Augmentation (CoDA), and applies this recipe to improve RL in settings where the independent mechanisms correspond to objects in the underlying MDP.

\subsection{An Opportunity to Augment: Local Independence of Objects}
High-dimensional dynamical systems are often composed of simple subprocesses that affect one another through sparse interaction. 
If the subprocesses \emph{never} interacted, an agent could realize significant gains in sample efficiency by globally factoring the dynamics and modeling each subprocess independently~\citep{guestrin2003efficient,hallak2015off}.
In most cases, however, the subprocesses do \emph{eventually} interact and so the prevailing approach is to model the entire process using a monolithic, unfactored model. 
In this Section, we take advantage of the observation that \textit{locally---during the time between their interactions---the subprocesses are causally independent}. By locally factoring dynamic processes in this way, we are able to capture the benefits of factorization even when their subprocesses interact on the global scale. 

Consider a game of billiards, where each ball can be viewed as a separate physical subprocess.
Predicting the opening break is difficult because all balls are mechanically coupled by their initial placement.
Indeed, a dynamics model with dense coupling amongst balls may seem sensible when considering the expected outcomes over the course of the game, as each ball has a non-zero chance of colliding with the others. But at any given timestep, interactions between balls are usually sparse.

One way to take advantage of sparse interactions between otherwise disentangled entities is to use a structured state representation together with a graph neural network or other message passing transition model that captures the local interactions~\citep{goyal2019recurrent, kipf2019contrastive}. When it is tractable to do so, such architectures can be used to model the world dynamics directly, producing transferable, task-agnostic models. In many cases, however, the underlying processes are difficult to model precisely, and model-free~\citep{laskin2020reinforcement,wang2019benchmarking} or task-oriented model-based ~\citep{farahmand2017value,oh2017value} approaches are less biased and exhibit superior performance. In this Section we argue that \emph{knowledge of whether or not local interactions occur is useful in and of itself}, and can be used to generate causally-valid counterfactual data even in absence of a forward dynamics model. 
In fact, if two trajectories have the same local factorization in their transition dynamics, then under mild conditions we can produce new counterfactually plausible data using our proposed \textbf{Counterfactual Data Augmentation (CoDA)} technique, wherein factorized subspaces of observed trajectory pairs are swapped (Figure \ref{fig_coda_intro_diagram}). This lets us sample from a
counterfactual data distribution by stitching together subsamples from observed transitions. Since CoDA acts only on the agent's training data, it is compatible with any agent architecture (including unfactored ones). 

In the remainder of this Section, we formalize this data augmentation strategy and discuss how it can improve performance of model-free RL agents in locally factored tasks. 
\begin{figure}[!t]
	\centering
	\hspace{-0.02\textwidth}
	\includegraphics[width=0.98\textwidth]{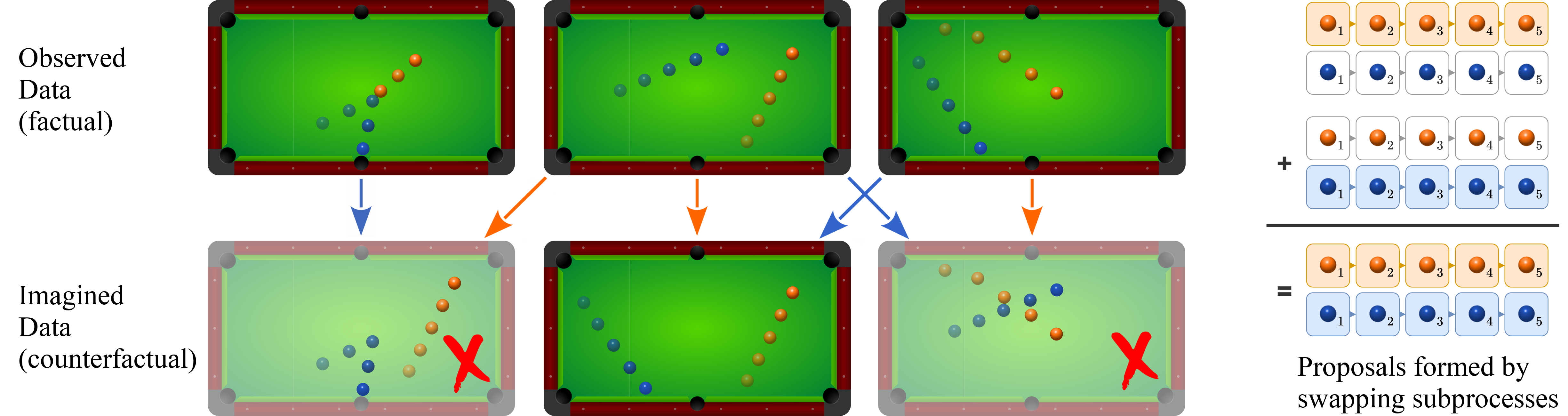}
	\hspace{0.02\textwidth}
	\caption{\textbf{Counterfactual Data Augmentation (CoDA)}. Given 3 factual samples, knowledge of the local causal structure lets us mix and match factored subprocesses to form counterfactual samples. The first proposal is rejected because one of its factual sources (the blue ball) is not locally factored.	The third proposal is rejected because it is not itself factored. The second proposal is accepted, and can be used as additional training data for a reinforcement learning agent.
	}
	\label{fig_coda_intro_diagram}
\end{figure}
\subsection{Background: Factored MDPs}\label{sec:coda:factored-mdps}
The basic model for decision making in a controlled dynamic process is a Markov Decision Process (MDP), described by tuple $\langle \S, \A, P, R, \gamma \rangle$ consisting of the state space, action space, transition function, reward function, and discount factor, respectively~\citep{puterman2014markov,sutton2018reinforcement}. Note that MDPs generalize uncontrolled Markov processes (set $A = \emptyset$), so that our work applies also to sequential prediction. 
We denote individual states and actions using lowercase $s \in \S$ and $a \in \A$, and variables using the uppercase $S$ and $A$ (e.g., $s \in \textrm{range}(S) \subseteq \S$). 
A policy $\pi : \S \times \A \to [0, 1]$ defines a probability distribution over the agent's actions at each state, and an agent is typically tasked with learning a parameterized policy $\pi_\theta$ that maximizes value $\mathbb{E}_{P,\pi}\sum_t\gamma^tR(s_t, a_t)$. 

In most non-trivial cases, the state $s \in \S$ can be described as an object hierarchy together with global context. 
For instance, this decomposition will emerge naturally in any simulated process or game that is defined using a high-level programming language (e.g., the commonly used Atari~\citep{bellemare2013arcade} or Minecraft~\citep{johnson2016malmo} simulators). 
In this Chapter we consider MDPs with a single, known top-level decomposition of the state space $\S = \S^1 \oplus \S^2 \oplus \dots \oplus \S^n$ for fixed $n$, leaving extensions to hierarchical decomposition and multiple representations~\citep{higgins2018scan,esmaeili2019structured}, dynamic factor count $n$~\citep{zaheer2017deep}, and (learned) latent representations~\citep{burgess2019monet} to future work. The action space might be similarly decomposed: $\A = \A^1 \oplus \A^2 \oplus \dots \oplus \A^m$. 
Given such state and action decompositions, we may model time slice $(t, t\!+\!1)$ using a structural causal model (SCM) $\M_t = \langle V_t, U_t, \F \rangle$~\citep[Ch. 7]{pearl2009causal} with directed acyclic graph (DAG) $\G$, where:
\begin{itemize}[leftmargin=0.3in,labelsep=0.1in]
\item $V_t = \{V^i_{t[+1]}\}_{i=0}^{2n+m}\!=\!\{S^1_t \mydots S^n_t\!, A^1_t \mydots A^m_t\!, S^1_{t+1} \mydots S^n_{t+1}\}$ are the nodes (variables) of $\G$.
\item $U_t = \{U^i_{t[+1]}\}_{i=0}^{2n+m}$ is a set of noise variables, one for each $V^i$, determined by the initial state, past actions, and environment stochasticity. We assume that noise variables at time $t+1$ are independent from other noise variables: $U^i_{t+1} \indep U^j_{t[+1]} \forall i,j$. The instance $u = (u^1, u^2, \dots, u^{2n + m})$ of $U_t$ denotes an individual realization of the noise variables. 
\item $\F = \{f^i\}_{i=0}^{2n+m}$ is a set of functions (``structural equations'') that map from $U^i_{t[+1]} \times \parents(V^i_{t[+1]})$ to $V^i_{t[+1]}$, where $\parents(V^i_{t[+1]}) \subset V_t\setminus V^i_{t[+1]}$ are the parents of $V^i_{t[+1]}$ in $\G$; hence each $f^i$ is associated with the set of incoming edges to node $V^i_{t[+1]}$ in $\G$; see, e.g., Figure \ref{fig_local_factorization} (center). 
\end{itemize} 

Note that while $V_t$, $U_t$, and $\M_t$ are indexed by $t$ (their distributions change over time), the structural equations $f^i \in \F$ and causal graph $\G$ represent the global transition function $P$ and apply at all times $t$. To reduce clutter, we drop the subscript $t$ on $V$, $U$, and $\M$ when no confusion can arise. 

Critically, we require the set of edges in $\G$ (and thus the number of inputs to each $f^i$) to be \emph{structurally minimal} \citep[Remark 6.6]{peters2017elements}.

\begin{assumption}[Structural Minimality]
$V^j \in \parents(V^i)$ if and only if there exists some $\{u^i, v^{-ij}\}$ with ${u^i \in \textrm{range}(U^i), v^{-ij} \in \textrm{range}(V\setminus\{V^i, V^j\})}$ and pair ${(v^j_1, v^j_2)}$ with ${v^j_1, v^j_2 \in \textrm{range}(V^j)}$ such that $v^i_1 = f^i(\{u^i, v^{-ij}, v^j_1\}) \not= f^i(\{u^i, v^{-ij}, v^j_2 \}) = v^i_2$.
\end{assumption}

Intuitively, structural minimality says that $V^j$ is a parent of $V^i$ if and only if setting the value of $V^j$ can have a nonzero \emph{direct} effect\footnote{Thus parentage does describe knock-on effects, e.g. $V_1$ on $V_3$ in the Markov chain ${V_1 \rightarrow V_2 \rightarrow V_3}$.} on the child $V^i$ through the structural equation $f^i$. The structurally minimal representation is unique~\citep{peters2017elements}.

Given structural minimality, we can think of edges in $\G$ as representing global causal dependence.
The probability distribution of $S^i_{t+1}$ is fully specified by its parents $\parents(S^i_{t+1})$ together with its noise variable $U_i$; that is, we have $P(S^i_{t+1} \given S_t, A_t) = P(S^i_{t+1} \given \parents(S^i_{t+1}))$ so that $S^i_{t+1} \indep V^j\given \parents(S^i_{t+1})$ for all nodes $V^j \not\in  \parents(S^i_{t+1})$. 
We call an MDP with this structure a \textit{factored MDP}~\citep{kearns1999efficient}. %
When edges in $\G$ are sparse, factored MDPs admit more efficient solutions than unfactored MDPs~\citep{guestrin2003efficient}.
\subsection{Local Causal Models (LCMs)}\label{subsection_locally_factored}
\paragraph{Limitations of Global Models}
Unfortunately, even if states and actions can be cleanly decomposed into several nodes, in most practical scenarios the DAG $\G$ is fully connected (or nearly so): since the $f^i$ apply globally, so too does structural minimality, and edge $(S^i_k, S^j_{k+1})$ at time $k$ is present so long as there is a single instance---at any time $t$, no matter how unlikely---in which $S^i_t$ influences $S^j_{t+1}$. In the words of Andrew Gelman, ``\textit{there are (almost) no true zeros}''~\citep{gelman2011causality}. As a result, the factorized causal model $\M_t$, based on globally factorized dynamics, rarely offers an advantage over a simpler causal model that treats states and actions as monolithic entities (e.g.,~\citep{buesing2018woulda}).

\paragraph{LCMs}
Our key insight is that for each pair of nodes $(V^i_t, S^j_{t+1})$ with $V^i_t \in \parents(S^j_{t+1})$ in $\G$, there often exists a large subspace $\localSA \subset \S \times \A$ for which $S^j_{t+1} \indep 
V^i_t \given \parents(S^j_{t+1})\setminus V^i_t, (s_t, a_t) \in \localSA$.
For example, in case of a two-armed robot (Figure \ref{fig_local_factorization}), there is a large subspace of states in which the two arms are too far apart to influence each other physically. 
Thus, if we restrict our attention to $(s_t, a_t) \in \localSA$, we can consider a \textit{local} causal model $\M^{\localSA}_t$ whose local DAG $\G^{\localSA}$ is strictly sparser than the global DAG $\G$, as the structural minimality assumption applied to $\G^{\localSA}$ implies that there is no edge from $V^i_t$ to $S^j_{t+1}$. More generally, for any subspace $\L \subseteq S \times A$, we can induce the Local Causal Model (LCM) $\M^\L_t = \langle V^\L_t, U^\L_t, \F^\L \rangle$ with DAG $\G^\L$ from the global model $\M_t$ as:
\begin{itemize}[leftmargin=0.3in,labelsep=0.1in]
\item $V^\L_t = \{V^{\L,i}_{t[+1]}\}_{i=0}^{2n+m}$, where $P(V^{\L,i}_{t[+1]}) = P(V^i_{t[+1]} \given (s_t, a_t) \in \L)$.
\item $U^\L_t = \{U^{\L,i}_{t[+1]}\}_{i=0}^{2n+m}$, where $P(U^{\L,i}_{t[+1]}) = P(U^i_{t[+1]}\given (s_t, a_t) \in \L)$.
\item $\F^\L = \{f^{\L,i}\}_{i=0}^{2n+m}$, where $f^{\L,i} = f^i\restrict{\L}$ ($f^i$ with range of input variables restricted to $\L$). 
Due to structural minimality, the signature of $f^{\L,i}$ may shrink (as the range of the relevant variables is now restricted to $\L$), and corresponding edges in $\G$ will not be present in $\G^\L$.\footnote{As a trivial example, if $f^i$ is a function of binary variable $V^j$, and $\L = \{(s, a)\given V^j = 0\}$, then $f^{\L,i}$ is not a function of $V^j$ (which is now a constant), and there is no longer an edge from $V^j$ to $V^i$ in $\G^\L$.}
\end{itemize} 
\begin{figure}[!t]
	\centering
	 \begin{minipage}{0.02\textwidth}
			\hfill
	 \end{minipage}%
    \begin{minipage}{0.16\textwidth}
		\includegraphics[width=\textwidth]{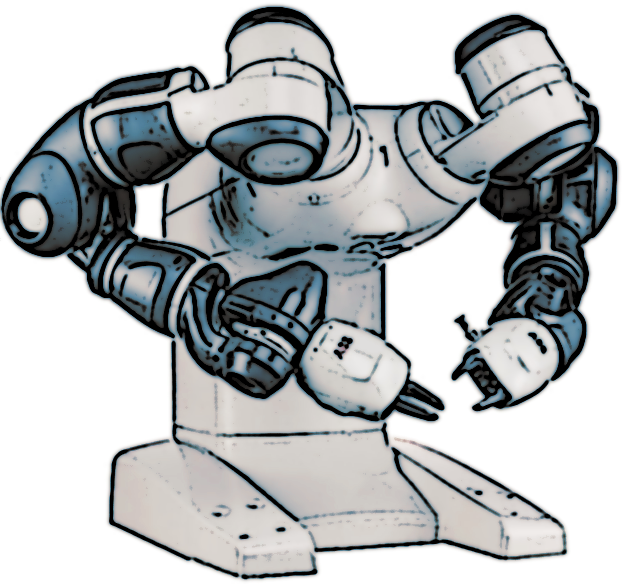}
    \end{minipage}%
    \begin{minipage}{0.05\textwidth}
		\hfill
    \end{minipage}%
    \begin{minipage}{0.77\textwidth}
		\includegraphics[width=\textwidth]{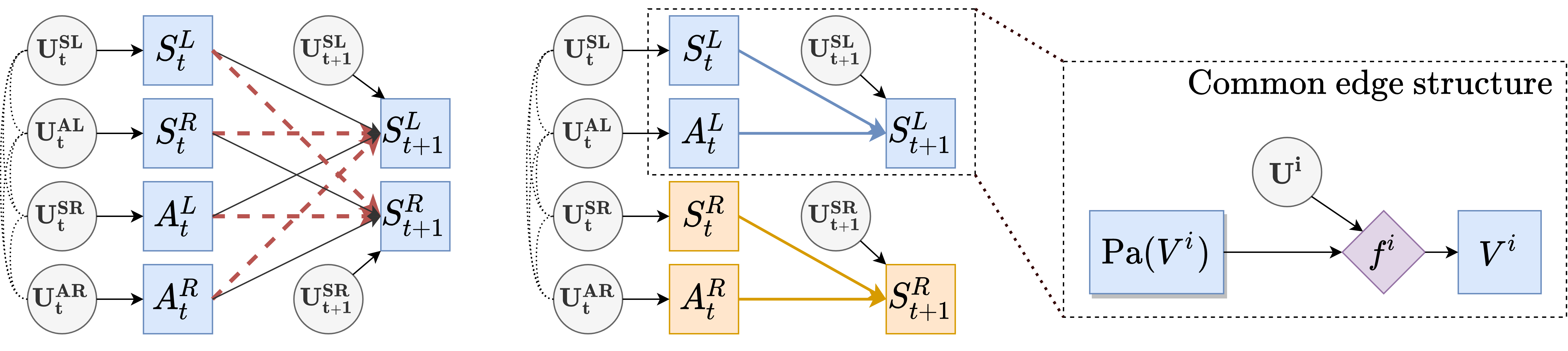}
	\end{minipage}
	\caption{A two-armed robot (\textbf{left}) might be modeled as an MDP whose state and action spaces decompose into left and right subspaces: $\mathcal{S} = \mathcal{S}^L \oplus \mathcal{S}^R, \mathcal{A} = \mathcal{A}^L \oplus \mathcal{A}^R$. Because the arms can touch, the global causal model (\textbf{center left}) between time steps is fully connected, even though left-to-right and right-to-left connections (dashed red edges) are rarely active. By restricting our attention to the subspace of states in which left and right dynamics are independent we get a local causal model (\textbf{center right}) with two components that can be considered separately for training and inference.}
	\label{fig_local_factorization}
\end{figure}

In case of the two-armed robot, conditioning on the arms being far apart simplifies the global DAG to a local DAG with two connected components (Figure \ref{fig_local_factorization}). This can make counterfactual reasoning considerably more efficient: given a factual situation in which the robot's arms are far apart, we can carry out separate counterfactual reasoning about each arm.

\paragraph{Leveraging LCMs} To see the efficiency therein, consider a general case with global causal model $\M$.
To answer the counterfactual question, ``\textit{what might the transition at time $t$ have looked like if component $S_t^i$ had value $x$ instead of value $y$?}'', we would ordinarily apply Pearl's do-calculus to $\M$ to obtain submodel $\M_{\DO(S_t^i = x)} = \langle V, U, \F_x \rangle$, where $\F_x = \F\setminus f^i \cup \{S_t^i = x\}$ and incoming edges to $S_t^i$ are removed from $\G_{\DO(S_t^i = x)}$~\citep{pearl2009causal}. The component distributions at time $t+1$ can be computed by reevaluating each function $f^j$ that depends on $S_t^i$. When $S_t^i$ has many children (as is often the case in the global $\G$), this requires one to estimate outcomes for many structural equations 
$\{f^j | V^j \in \text{Children}(V^i_t)\}$. 
But if both the original value of $S_t$ (with $S_t^i = y$) and its new value (with $S_t^i = x$) are in the set $\L$, the intervention is ``within the bounds'' of local model $\M^\L$ and we can instead work directly with local submodel $\M^\L_{\DO(S_t^i = x)}$ (defined accordingly). The validity of this follows from the definitions: since $f^{\L,j} = f^j\restrict{\L}$ for all of $S_t^i$'s children, the nodes $V_t^k$ for $k \not= i$ at time $t$ are held fixed, and the noise variables at time $t+1$ are unaffected, the distribution at time $t+1$ is the same under both models. When $S_t^i$ has fewer children in $\M^\L$ than in $\M$, this reduces the number of structural equations that need to be considered. 

\subsection{A Generic Recipe for Counterfactual Data Augmentation}
We hypothesize that local causal models will have several applications, and potentially lead to improved agent designs, algorithms, and interpretability. In this Section we focus on improving off-policy learning in RL by exploiting causal independence in local models for \textbf{Counterfactual Data Augmentation (CoDA)}. CoDA augments real data by making counterfactual modifications to a subset of the causal factors at time $t$, leaving the rest of the factors untouched. Following the logic outlined in the Subsection \ref{sec:coda:factored-mdps}, this can understood as manufacturing ``fake'' data samples using the counterfactual model $\M^{[\L]}_{\DO(S_t^{i\mydots j} = x)}$, where we modify the causal factors $S_t^{i\mydots j}$ and resample their children. While this is always possible using a model-based approach if we have good models of the structural equations, it is particularly nice when the causal mechanisms are independent, as we can do counterfactual reasoning 
directly by reusing subsamples from observed trajectories. 

\begin{definition} 
The causal mechanisms represented by subgraphs $\G_i, \G_j \subset \G$ are \textbf{independent} when $\G_i$ and $\G_j$ are disconnected in $\G$.
\end{definition}

When $\G$ is divisible into two (or more) connected components, we can think of each subgraph as an independent causal mechanism that can be reasoned about separately. 

As briefly mentioned in Section~\ref{sec:coda:why_does_data_aug_work}, existing data augmentation techniques can be interpreted as specific instances of CoDA (Figure \ref{fig_coda_examples}). For example, goal relabeling~\citep{kaelbling1993learning}, as used in Hindsight Experience Replay (HER)~\citep{andrychowicz2017hindsight} and Q-learning for Reward Machines~\citep{icarte2018using}, exploits the independence of the goal dynamics $G_t \mapsto G_{t+1}$ (identity map) and the next state dynamics $S_t \times A_{t} \mapsto S_{t+1}$ in order to relabel the goal variable $G_t$ with a counterfactual goal. While the goal relabeling is done model-free, we typically assume knowledge of the goal-based reward mechanism $G_t \times S_t \times A_t \times S_{t+1} \mapsto R_{t+1}$ to relabel the reward, ultimately mixing model-free and model-based reasoning. Similarly, visual feature augmentation, as used in reinforcement learning from pixels~\citep{laskin2020reinforcement,kostrikov2020image} and sim-to-real transfer~\citep{andrychowicz2020learning}, exploits the independence of the physical dynamics $S^P_t \times A_t \mapsto S^P_{t+1}$ and visual feature dynamics $S^V_t \mapsto S^V_{t+1}$ such as textures and camera position, assumed to be static ($S^V_{t+1} = S^V_t$), to counterfactually augment visual features. Both goal relabeling and visual data augmentation rely on \textit{global} independence relationships.  

We propose a novel form of data augmentation. In particular, we observe that whenever an environment transition is within the bounds of some local model $\M^\L$ whose graph $\G^\L$ has the locally independent causal mechanism $\G_i$ as a disconnected subgraph (note: $\G_i$ itself need not be connected), that transition contains an unbiased sample from $\G_i$. Thus, 
given two transitions in $\L$, 
we may mix and match the samples of $\G_i$ to generate counterfactual data, 
so long as the resulting transitions are themselves in $\L$. 

\begin{remark} How much data can we generate using our CoDA algorithm? If we have $n$ independent samples from subspace $\L$ whose graph $\G^\L$ has $m$ connected components, we have $n$ choices for each of the $m$ components, for a total of $n^m$ CoDA samples---an \textbf{exponential} increase in data! One might term this the ``blessing of independent subspaces.''
\end{remark}
\begin{remark} Our discussion has been at the level of a single transition (time slice $(t, t+1)$), which is consistent with the form of data that RL agents typically consume. But we could also use CoDA to mix and match locally independent components over several time steps (see, e.g., Figure \ref{fig_coda_intro_diagram}).
\end{remark}
\begin{remark} \label{remark_distributional_shift} As is typical, counterfactual reasoning changes the data distribution. While off-policy agents are typically robust to distributional shift, future work might explore different ways to control or prioritize the counterfactual data distribution~\citep{schaul2015prioritized,kumar2020discor}. We note, however, that certain prioritization schemes may introduce selection bias~\citep{hernan2020causal}, effectively entangling otherwise independent causal mechanisms (e.g., HER's ``future'' strategy~\citep{andrychowicz2017hindsight} may introduce ``hindsight bias''~\citep{lanka2018archer,schroecker2020universal}).
\end{remark}
\begin{remark} The global independence relations relied upon by goal relabeling and image augmentation are incredibly general, as evidenced by their wide applicability. We posit that certain local independence relations are similarly general. For example, the physical independence of objects separated by space (the billiards balls of Figure \ref{fig_coda_intro_diagram}, the two-armed robot of Figure \ref{fig_local_factorization}, and the environments used in Section \ref{sec:coda:experiments}), and the independence between an agent's actions and the truth of (but not belief about) certain facts the agent is ignorant of (e.g., an opponent's true beliefs).
\end{remark}

\paragraph{Implementing CoDA}

\newcommand{\Var}[1]{\texttt{#1}}
\begin{algorithm*}[t]\small
	\caption{Mask-based Counterfactual Data Augmentation (CoDA)} \label{alg_coda}
	\begin{small}\vspace{-0.15in}
	\begin{multicols}{2}\noindent
		\begin{algorithmic}
			\Function{Coda}{transition $\Var{t1}$, transition $\Var{t2}$}:
			\vspace{0.01in}
    		\State $\Var{s1}, \Var{a1}, \Var{s1'} \gets \Var{t1}$
	    	\State $\Var{s2}, \Var{a2}, \Var{s2'} \gets \Var{t2}$
	    	\State $\Var{m1}, \Var{m2} \gets \textsc{Mask}(\Var{s1}, \Var{a1}), \textsc{Mask}(\Var{s2}, \Var{a2})$
	    	\State $\Var{D1} \gets  \textsc{Components}(\Var{m1})$
 		    \State $\Var{D2} \gets  \textsc{Components}(\Var{m2})$
		    \State $\Var{d} \gets$ random sample from $(\Var{D1}$ $\cap$ $\Var{D2})$
		    \State $\Var{\~s}, \Var{\~a}, \Var{\~s'} \gets \textrm{copy}(\Var{s1}, \Var{a1}, \Var{s1'})$
		    \State $\Var{\~s[d]}, \Var{\~a[d]}, \Var{\~s'[d]} \gets \Var{s2[d]}, \Var{a2[d]}, \Var{s2'[d]}$
		    \State $\Var{\~D} \gets \textsc{Components}(\textsc{Mask}(\Var{\~s}, \Var{\~a}))$
		    \State \textbf{return} $(\Var{\~s},\Var{\~a},\Var{\~s'})$ \textbf{if} $\Var{d} \in \Var{\~D}$ \textbf{else} $\emptyset$
			\EndFunction
		\end{algorithmic}
		\columnbreak
		\begin{algorithmic}
		\Function{Mask}{state \Var{s}, action \Var{a}}:
			\begin{adjustwidth}{\algorithmicindent}{0pt}
			Returns $(n+m)\times(n)$ matrix indicating if the $n$ next state components (columns) locally depend on the $n$ state and $m$ action components (rows). 
			\end{adjustwidth}
		\EndFunction
		\Statex
		\vspace{-0.01in}
		\Function{Components}{mask \Var{m}}:
			\begin{adjustwidth}{\algorithmicindent}{0pt}
			Using the mask as the adjacency matrix for $\G^\L$ (with dummy columns for next action), finds the set of connected components $C = \{C_j\}$, and returns the set of independent components \newline $D = \{\G_i = \bigcup_k \mathcal{C}^i_{k}\given\mathcal{C}^i \subset \textrm{powerset}(C)\}$.
			\end{adjustwidth}
    	\EndFunction
		\end{algorithmic}
		\end{multicols}\vspace{-0.12in}
	\end{small}
\end{algorithm*}

We implement CoDA, as outlined above and visualized in Figure \ref{fig_coda_examples}(d), as a function of two factual transitions and a mask function $M(s_t, a_t): \S \times \A \to \{0, 1\}^{(n+m)\times n}$ that represents the adjacency matrix of the sparsest local causal graph $\G^\L$ such that $\L$ is a neighborhood of $(s_t, a_t)$.%
\footnote{If Jacobian $\partial P / \partial x$ exists at $x = (s_t, a_t)$, the ground truth $M(s_t, a_t)$ equals $|(\partial P / \partial x)^T| > 0$.} 
We apply $M$ to each transition to obtain local masks $\texttt{m}_1$ and $\texttt{m}_2$, compute their connected components, and 
swap independent components $\G_i$ and $\G_j$ (mutually disjoint and collectively exhaustive groups of connected components)
between the transitions to produce a counterfactual proposal. We then apply $M$ to the counterfactual $(\tilde s_t, \tilde a_t)$ to validate the proposal---if the counterfactual mask $\texttt{\~ m}$ shares the same graph partitions as $\texttt{m}_1$ and $\texttt{m}_2$, we accept the proposal as a CoDA sample. See Algorithm \ref{alg_coda}. 

Note that masks $\texttt{m}_1$, $\texttt{m}_2$ and $\texttt{\~ m}$ correspond to different neighborhoods $\L_1, \L_2$ and $\tilde\L$, so it is not clear that we are ``within the bounds'' of any model $\M^\L$ as was required in Subsection \ref{sec:coda:factored-mdps} for valid counterfactual reasoning. 
To correct this discrepancy we use the following proposition and additionally require the causal mechanisms (subgraphs) for independent components $\G_i$ and $\G_j$ to share structural equations in each local neighborhood: $f^{\L_1,i} = f^{\L_2,i} = f^{\tilde\L,i}$ and $f^{\L_1,j} = f^{\L_2,j} = f^{\tilde\L,j}$.\footnote{To see why this is not trivially true, imagine there are two rooms, one of which is icy. In either room the ground conditions are locally independent of movement dynamics, but not so if we consider their union.}  This makes our reasoning valid in the local subspace $\L^* = \L_1 \cup \L_2 \cup \tilde\L$.
See Appendix \ref{app:coda:proposition} for proof. 

\begin{proposition}\label{proposition_union_of_local_sets}
The causal mechanisms represented by $\G_i, \G_j \subset \G$ are independent in $\G^{\L_1 \cup \L_2}$ if and only if  $\G_i$ and $\G_j$ are independent in both $\G^{\L_1}$ and $\G^{\L_2}$, and $f^{\L_1,i} = f^{\L_2,i}, f^{\L_1,j} = f^{\L_2,j}$.
\end{proposition}

Since CoDA only modifies data within local subspaces, this biases the resulting replay buffer to have more factorized transitions.
In our experiments below, we specify the ratio of observed-to-counterfactual data heuristically to control this selection bias, but find that off-policy agents are reasonably robust to large proportions of CoDA-sampled trajectories.
We leave a full characterization of the selection bias in CoDA to future studies, noting that knowledge of graph topology was shown to be useful in mitigating selection bias for causal effect estimation \citep{bareinboim2012controlling, bareinboim2014recovering}.

\paragraph{Inferring local factorization} 
While the ground truth mask function $M$ may be available in rare cases as part of a simulator, the general case either requires a domain expert to specify an approximate causal model (as in goal relabeling and visual data augmentation) or requires the agent to learn the local factorization from data.
Given how common independence due to physical separation of objects is, the former option will often be available. 
In the latter case, we note that the same data could also be used to learn a forward model. Thus, there is an implicit assumption in the latter case that learning the local factorization is easier than modeling the dynamics. We think this assumption is rather mild, as an accurate forward dynamics model would subsume the factorization, and we provide some empirical evidence of its validity in Section \ref{sec:coda:experiments}.

Learning the local factorization is similar to conditional causal structure discovery~\citep{spirtes2000causation, runge2019inferring, peters2017elements}, conditioned on neighborhood
$\L$ of $(s_t, a_t)$, 
except that the same structural equations must be applied globally (if the structural equations were conditioned on 
$\L$, 
Proposition \ref{proposition_union_of_local_sets} would fail). 
As there are many algorithms for general structure discovery~\citep{spirtes2000causation,runge2019inferring}, and the arrow of time simplifies the inquiry~\citep{granger1969investigating,peters2017elements}, there may be many ways to approach this problem.
In the first of our two experiments we train a single-head set transformer architecture~\citep{vaswani2017attention,lee2018set}.
to do next-state prediction, then derive a locally conditioned network mask $M(s_t, a_t)$ as the product of the learned (input-specific) attention masks (see Appendix~\ref{app:coda:batch_rl_details} for details.)
The second experiment uses heuristic LCM that is specified by hand.
We expect that CoDA will see an increase in performance an applicability alongside future improvements in (local) causal discovery in dynamic settings.

\subsection{Experiments}
\label{sec:coda:experiments}
Our experiments evaluate CoDA in the online, batch, and goal-conditioned settings, in each case finding that CoDA significantly improves agent performance as compared to non-CoDA baselines. Since CoDA only modifies an agent's training data, we expect these improvements to extend to other off-policy task settings in which the state space can be accurately disentangled. Below we outline our experimental design and results, deferring specific details and additional results to Appendix \ref{app:coda:training_details}.

\begin{figure}[!b]
	\centering
	\includegraphics[width=\textwidth]{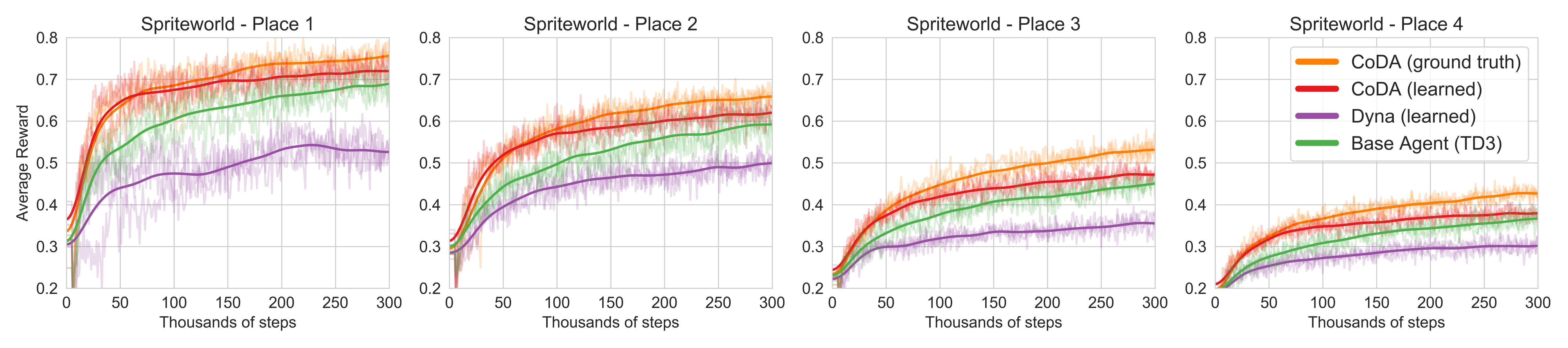}
	\caption{\textbf{Standard online RL} (3 seeds): CoDA with the ground truth mask always performs the best, validating our basic idea. CoDA with a pre-trained model also offers a significant early boost in sample efficiency and maintains its lead over the base TD3 agent throughout training. Using the same model to generate data directly (a la Dyna~\citep{sutton1991dyna}) performs poorly, suggesting significant model bias.}
	\label{fig_rl_results_plot}
\end{figure}

{\small
\begin{table*}[t]
\newcommand{\smol}[1]{{\scriptsize\texttt{#1}}}
\centering\small
\begin{tabular}{c@{\hskip 0.2in}cc@{\hskip 0.16in}|@{\hskip 0.16in}cccc}
	\toprule
$|\mathcal{D}|$	&Real data& MBPO&
\multicolumn{4}{c}{Ratio of Real:CoDA [:MBPO] data (ours)}\\
$(1000s)$ &   1\smol{R}  &   1\smol{R}:1\smol{M} &  1\smol{R}:1\smol{C}  &  1\smol{R}:3\smol{C}  &  1\smol{R}:5\smol{C} & 1\smol{R}:3\smol{C}:1\smol{M} \\
	\midrule
$\hphantom{0}25$ & $13.2 \pm 0.7$  & $18.5 \pm 1.5$  & $43.8 \pm 2.8$  & $40.9 \pm 2.5$  & $38.4 \pm 4.9$  & $\mybm{46.8} \pm 3.1$  \\
$\hphantom{0}50$ & $22.8 \pm 3.0$  & $36.6 \pm 4.3$  & $66.6 \pm 3.8$  & $64.4 \pm 3.1$  & $62.5 \pm 3.5$  & $\mybm{70.4} \pm 3.8$  \\
$\hphantom{0}75$ & $43.2 \pm 4.9$  & $46.0 \pm 4.7$  & $73.4 \pm 2.8$  & $\mybm{76.7} \pm 2.6$  & $75.0 \pm 3.4$  & $74.6 \pm 3.2$  \\
$100$ & $63.0 \pm 3.1$  & $66.4 \pm 4.9$  & $77.8 \pm 2.0$  & $\mybm{82.7} \pm 1.5$  & $76.6 \pm 3.0$  & $73.7 \pm 2.9$  \\
$150$ & $77.4 \pm 1.2$  & $72.6 \pm 5.6$  & $82.2 \pm 1.8$  & $\mybm{85.8} \pm 1.4$  & $84.2 \pm 1.0$  & $79.7 \pm 3.6$  \\
$250$ & $78.2 \pm 2.7$  & $77.9 \pm 2.4$  & $85.0 \pm 2.9$  & $\mybm{87.8} \pm 1.8$  & $87.0 \pm 1.0$  & $78.3 \pm 4.9$  \\
	\bottomrule
	\vspace{-0.12in}
\end{tabular}
\caption{\textbf{Batch RL} (10 seeds): Mean success ($\pm$ standard error, estimated using 1000 bootstrap resamples) on \texttt{Pong} environment. CoDA with learned masking function more than doubles the effective data size, resulting in a 3x performance boost at smaller data sizes. Note that a 1\smol{R}:5\smol{C} Real:CoDA ratio performs slightly worse than a 1\smol{R}:3\smol{C} ratio due to distributional shift (Remark \ref{remark_distributional_shift}).
}\label{fig_batch_rl_and_fetch}
\vspace{-\baselineskip}
\end{table*}
}

\paragraph{Batch RL}

\begin{wrapfigure}{R}{0.17\textwidth}
	\centering
    \captionsetup{width=0.16\textwidth}
	\includegraphics[width=0.16\textwidth,height=0.15\textwidth]{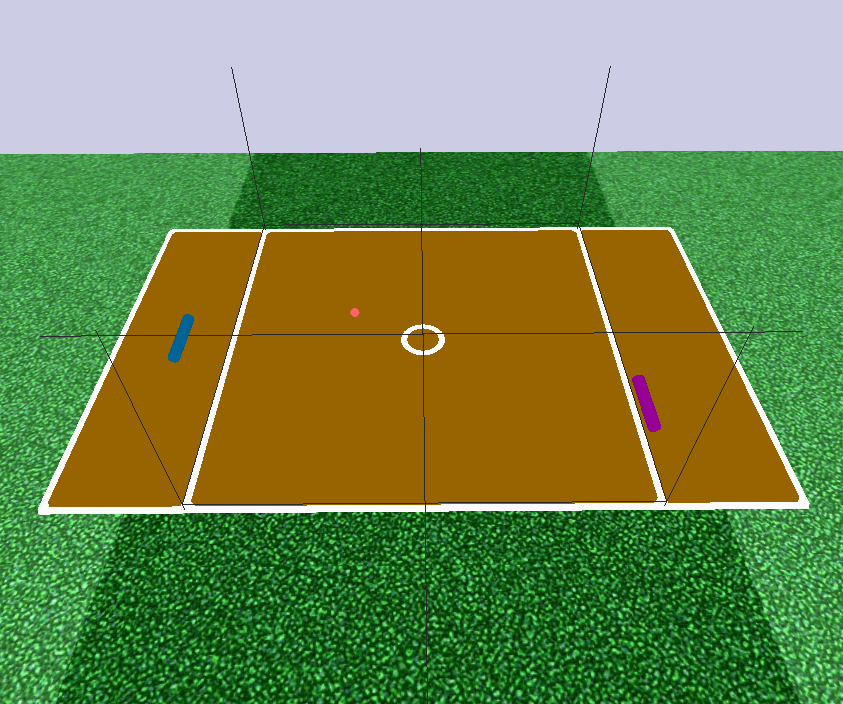}
	\label{fig_env_pong}
    \vspace{-\baselineskip}
\end{wrapfigure}

A natural setting for CoDA is batch-constrained RL, where an agent has access to an existing transition-level dataset, but cannot collect more data via exploration~\citep{fujimoto2018off,levine2020offline}.  Another reason why CoDA is attractive in this setting is that there is no \textit{a priori} reason to prefer the given batch data distribution to a counterfactual one.  
For this experiment we use a continuous control \texttt{Pong} environment based on \texttt{RoboschoolPong}~\citep{klimov2017roboschool}. The agent must hit the ball past the opponent, receiving reward of +1 when the ball is behind the opponent's paddle, -1 when the ball is behind the agent's paddle, and 0 otherwise. 
Since our transformer model performed poorly when used as a dynamics model, our Dyna baseline for batch RL adopts a state-of-the-art architecture~\citep{janner2019trust} that employs a 7-model ensemble (MBPO).
We collect datasets of up to 250,000 samples from an pre-trained policy with added noise. For each dataset, we train both mask and reward functions (and in case of MBPO, the dynamics model) on the provided data and use them to generate different amounts of counterfactual data. We also consider combining CoDA with MBPO, by first expanding the dataset with MBPO and then applying CoDA to the result. We train the same TD3 agent on the expanded datasets in batch mode for 500,000 optimization steps. The results in Table \ref{fig_batch_rl_and_fetch} show that with only 3 state factors (two paddles and ball), applying CoDA is approximately equivalent to doubling the amount of real data.
\begin{figure*}[!h]
    \centering
    \begin{minipage}{0.24\textwidth}
    \centering
    \includegraphics[width=\textwidth]{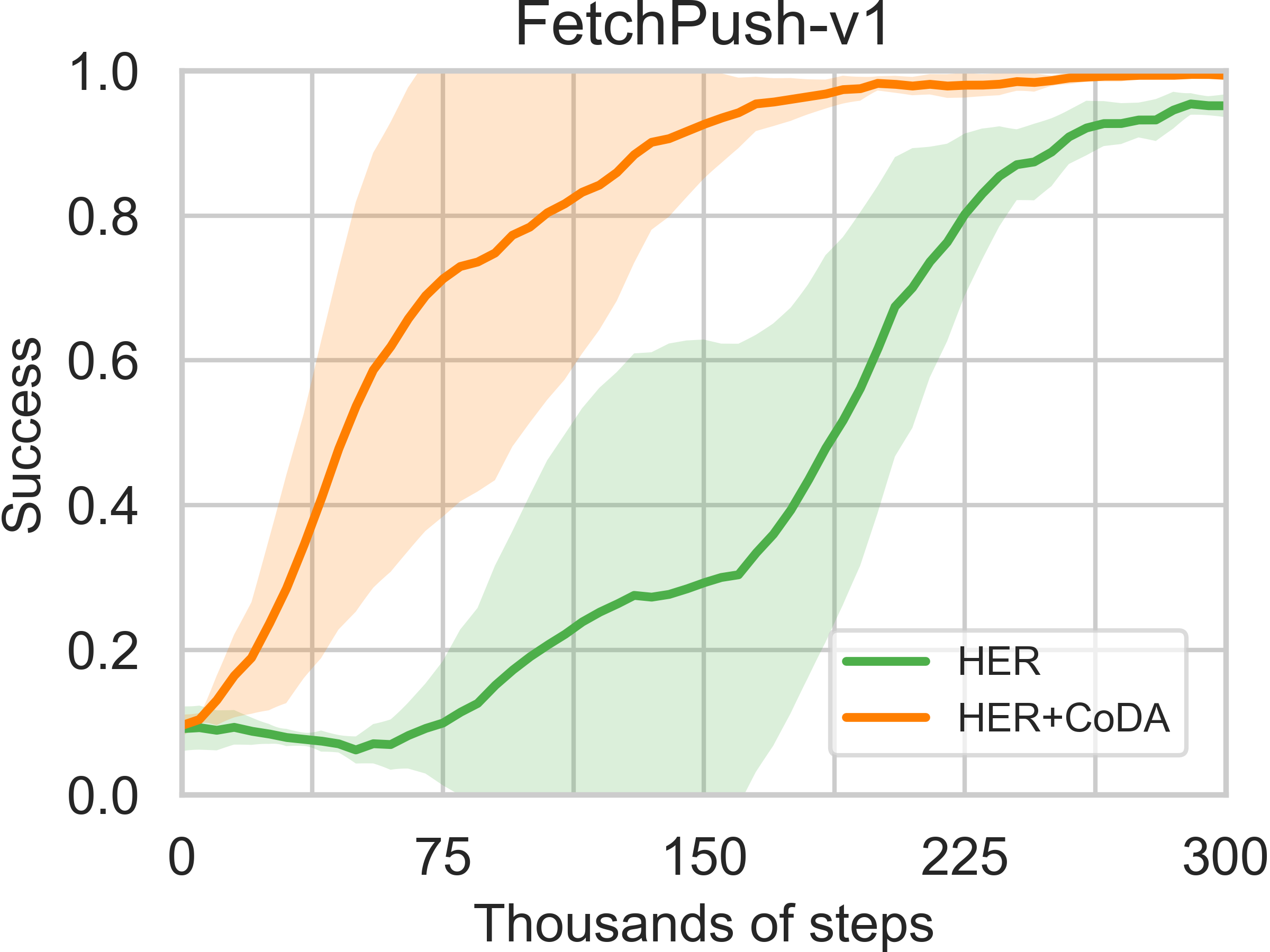}
    \end{minipage}%
    \begin{minipage}{0.01\textwidth}
   	\hphantom{X}
    \end{minipage}%
    \begin{minipage}{0.24\textwidth}
    \centering
    \includegraphics[width=\textwidth]{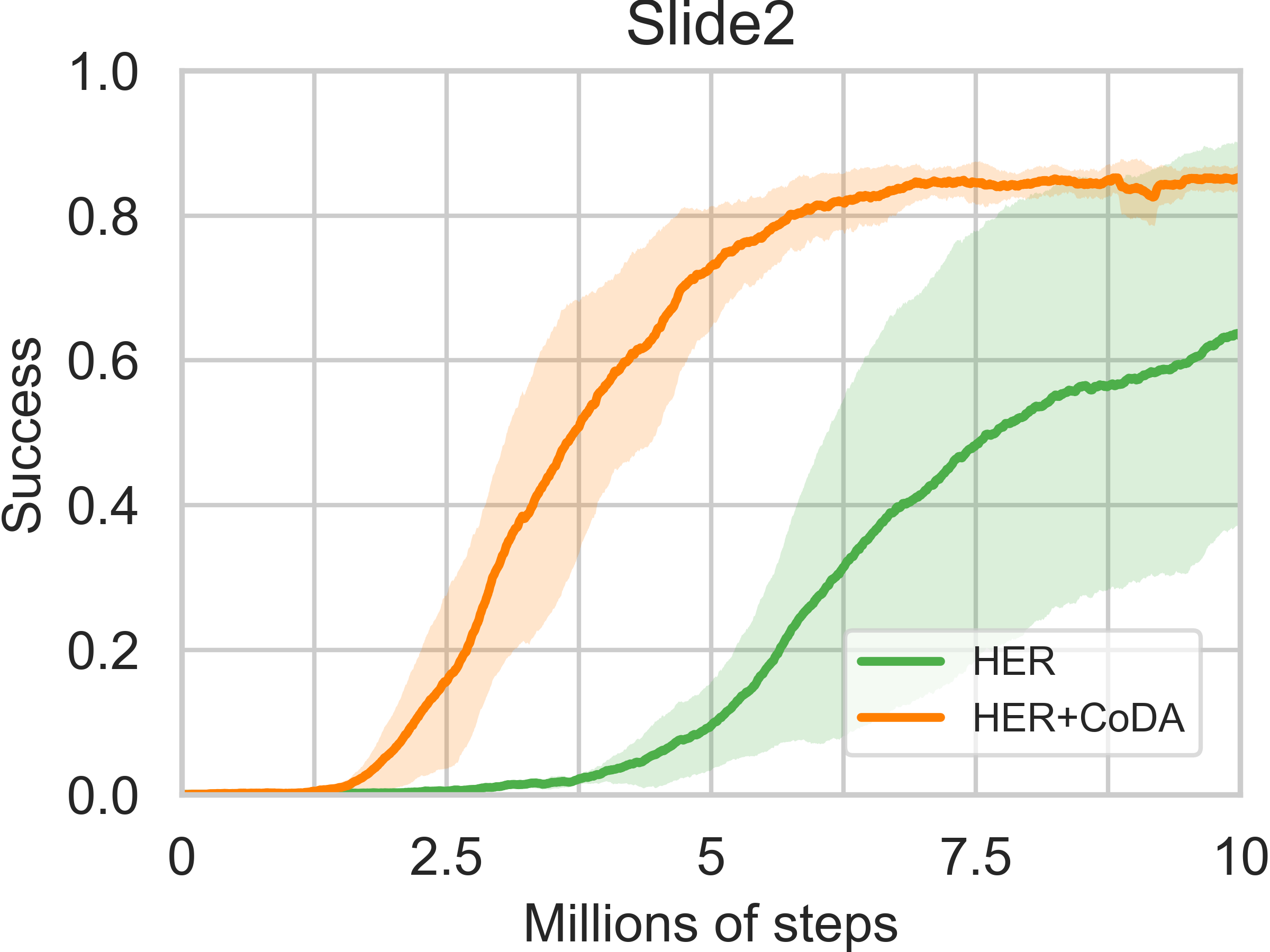}
    \end{minipage}%
    \begin{minipage}{0.04\textwidth}
	\hphantom{X}
    \end{minipage}%
    \begin{minipage}{0.47\textwidth}
    \vspace{0.08in}
    \renewcommand{\figurename}{Fig.}
    \caption{\textbf{Goal-conditioned RL} (5 seeds): In \texttt{FetchPush} and the challenging \texttt{Slide2} environment, a HER agent whose dataset has been enlarged with CoDA approximately doubles the sample efficiency of the base HER agent.}\label{fig_goal_conditioned_results}
    \end{minipage}
    \vspace{-\baselineskip}
\end{figure*}
\begin{wrapfigure}{r}{0.17\textwidth}
    \vspace{-\baselineskip}
	\centering
    \captionsetup{width=0.16\textwidth}
	\includegraphics[width=0.16\textwidth,height=0.15\textwidth]{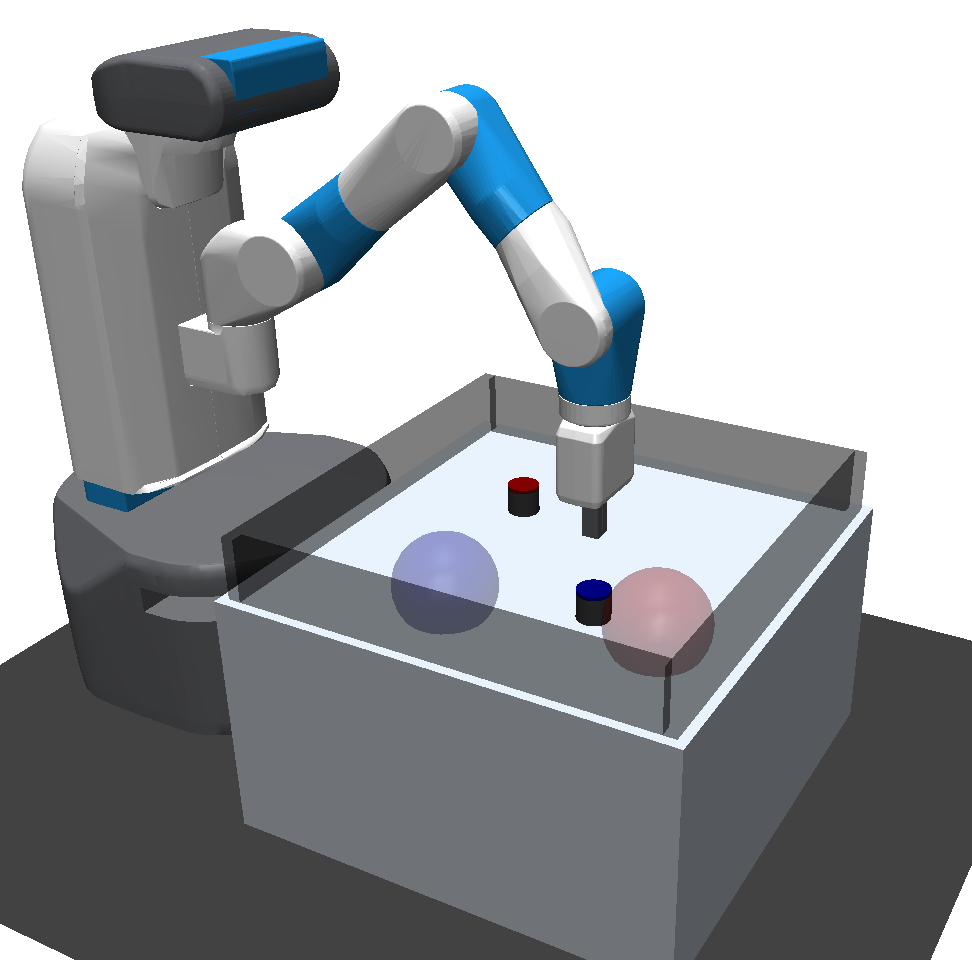}
	\label{fig_env_fetch}
    \vspace{-\baselineskip}
\end{wrapfigure}
\paragraph{Goal-conditioned RL} As HER~\citep{andrychowicz2017hindsight} is an instance of prioritized CoDA that greatly improves sample efficiency in sparse-reward tasks, can our unprioritized CoDA algorithm further improve HER agents? We use HER to relabel goals on real data only, relying on random CoDA-style goal relabeling for CoDA data.
After finding that CoDA obtains state-of-the-art results in \texttt{FetchPush-v1}~\citep{plappert2018multi}, we show that CoDA also accelerates learning in a novel and significantly more challenging \texttt{Slide2} environment, where the agent must slide two pucks onto their targets (Figure \ref{fig_goal_conditioned_results}). For this experiment, we specified a heuristic mask using domain knowledge (``objects are disentangled if more than 10cm apart'') that worked in both \texttt{FetchPush} and \texttt{Slide2} despite different dynamics.
\section{Extending the Experience Buffer Using Dynamics Models that Generalize OOD}
%
% reset the method name
\newcommand{\methodName}{MoCoDA\xspace}
\newcommand{\methodNameFull}{Model-based Counterfactual Data Augmentation\xspace}
\subsection{CoDA as an Implicit Generative Model}\label{sec:coda:implicit_generative_model}
CoDA is a process of stitching together sub-trajectories of real experience to generate novel but plausible trajectories, thus amplifying the number of available training data.
Stitching together data samples in this way is a form of implicit generative modeling that will interpolate the available training distribution. 
However, in more challenging settings there may be \emph{zero} support in crucial subspaces of the state-action space.
In this case, CoDA is unlikely to address the data coverage issue.
This Section describes an extension to CoDA capable of \emph{extrapolating} the training distribution.
We do so modeling the transition dynamics explicitly using neural networks, then use object-centric causal priors to sample from a modified model with expanded support.

\subsection{Motivating Model-based CoDA}

Modern reinforcement learning (RL) algorithms have demonstrated remarkable success in several different domains such as games~\cite{mnih2015human, silver2017mastering} and robotic manipulation~\cite{kalashnikov2018scalable, andrychowicz2020learning}. By repeatedly attempting a single task through trial-and-error, these algorithms can learn to collect useful experience and eventually solve the task of interest. 
However, designing agents that can generalize in \textit{off-task} and \textit{multi-task} settings remains an open and challenging research question. 
This is especially true in the offline and zero-shot settings, in which the training data might be unrelated to the target task, and may lack sufficient coverage over possible states.

One way to enable generalization in such cases is through structured representations of states, transition dynamics, or task spaces. These representations can be directly learned, sourced from known or learned abstractions over the state space, or derived from causal knowledge of the world. Symmetries present in such representations enable compositional generalization to new configurations of states or tasks, either by building the structure into the function approximator or algorithm \citep{kipf2019contrastive,veerapaneni2020entity,goyal2019recurrent,nangue2020boolean}, or by using the structure for data augmentation \citep{andrychowicz2017hindsight,laskin2020reinforcement,pitis2020counterfactual}.

Here we extend past work on structure-driven data augmentation by using a locally factored model of the transition dynamics to generate counterfactual training distributions.
This enables agents to generalize beyond the support of their original training distribution, including to novel tasks where learning the optimal policy requires access to states never seen in the experience buffer.
Our key insight is that a learned dynamics model that accurately captures local causal structure (a ``locally factored'' dynamics model) will predictably exhibit good generalization performance outside the empirical training distribution.
We propose \methodNameFull (\methodName), which generates an augmented state-action distribution  where its locally factored dynamics model is likely to perform well, then applies its dynamics model to generate new transition data.
By training the agent's policy and value modules on this augmented dataset, they too learn to generalize well out-of-distribution. 

\subsection{Object-centric Models of Transition Dynamics}\label{sec:mocoda:theory}
\paragraph{Sample Complexity of Training a Locally Factored Dynamics Model}
We now provide an original adaptation of an elementary result from model-based RL to the \textit{locally} factored setting, to show that factorization can exponentially improve sample complexity. We note that several theoretical works have shown that the Factored MDP (FMDP) structure can be exploited to obtain similarly strong sample complexity bounds in the FMDP setting. Our goal here is not to improve upon these results, but to adapt a small part (model-based generalization) to the significantly more general locally factored setting and show that local factorization is enough for (1) \textit{exponential gains in sample complexity} and (2) \textit{out-of-distribution generalization} with respect to the empirical joint, to a set of states and actions that may be exponentially larger than the empirical set. Note that the following discussion applies to tabular RL, but we apply our method to continuous domains. 

\textbf{Notation}.\ \ We work with finite state and action spaces ($|\S|, |\A| < \infty$) and assume that there are $m$ local subspaces $\L$ of size $|\L|$, such that $m|\L| = |\S||\A|$. For each subspace $\L$, we assume transitions factor into $k$ causal mechanisms $\{P_i\}$, each with the same number of possible children, $|c_i|$, and the same number of possible parents, $|\parents_i|$. Note $m\Pi_i|c_i| = |\S|$  (child sets are mutually exclusive) but $m\Pi_i|\parents_i| \geq |\S||\A|$ (parent sets may overlap).

\begin{theorem}\label{theorem_one}
Let $n$ be the number of empirical samples used to train the model of each local causal mechanism, $P_{i, \theta}^\L$ at each configuration of parents $\parents_i = x$. There exists constant $c$ such that, if
$$
n \geq \frac{ck^2|c_i|\log(|\S||\A|/\delta)}{\epsilon^2},
$$
then, with probability at least $1-\delta$, we have:
$$\max_{(s, a)} \Vert P(s, a) - P_{\theta}(s, a) \Vert_1 \leq  \epsilon.$$
\end{theorem}

\begin{proof}[Sketch of Proof]
We apply a concentration inequality to bound the $\ell_1$ error for fixed parents and extend this to a bound on the $\ell_1$ error for a fixed $(s, a)$ pair. The conclusion follows by a union bound across all states and actions. See Appendix \ref{appdx_proposition} for details. 
\end{proof}

To compare to full-state dynamics modeling, we can translate the sample complexity from the per-parent count $n$ to a total count $N$. Recall $m\Pi_i|c_i| = |\S|$, so that $|c_i| = (|\S|/m)^{1/k}$, and $m\Pi_i|\parents_i| \geq |\S||\A|$. We assume a small constant overlap factor $v \geq 1$, so that $|\parents_i| = v(|\S||\A|/m)^{1/k}$. We need the total number of component visits to be $n|\parents_i|km$, for a total of $nv(|\S||\A|/m)^{1/k}m$ state-action visits, assuming that parent set visits are allocated evenly, and noting that each state-action visit provides $k$ parent set visits. This gives:

\begin{corollary}\label{corollary_one}
To bound the error as above, we need to have
$$N \geq \frac{cmk^2(|\S|^2|\A|/m^2)^{1/k}\log(|\S||\A|/\delta)}{\epsilon^2},$$
total train samples, where we have absorbed the overlap factor $v$ into constant $c$.
\end{corollary}

Comparing this to the analogous bound for full-state model learning (\citet{agarwal2019reinforcement}, Prop. 2.1): 

$$N \geq \frac{c|\S|^2|\A|\log(|\S||\A|/\delta)}{\epsilon^2},$$

we see that we have gone from super-linear $O(|\S|^2|\A|\log(|\S||\A|))$ sample complexity in terms of $|S||A|$, to the exponentially smaller $O(mk^2(|\S|^2|\A|/m^2)^{1/k}\log(|\S||\A|))$. 

This result implies that \textit{for large enough $|\S||\A|$ our model \textit{must} generalize to unseen states and actions}, since the number of samples needed ($N$) is exponentially smaller than the size of the state-action space ($|\S||\A|$). In contrast, if it did not, then sample complexity would be $\Omega(|\S||\A|)$. 

\begin{remark}
The global factorization property of FMDPs is a strict assumption that rarely holds in reality. Although local factorization is broadly applicable and significantly more realistic than the FMDP setting, it is not without cost. In FMDPs, we have a single subspace ($m=1$). In the locally factored case, the number of subspaces $m$ is likely to grow exponentially with the number of factors $k$, as there are exponentially many ways that $k$ factors can interact. To be more precise, there are $k2^k$ possible bipartite graphs from $k$ nodes to $k$ nodes. Nevertheless, by comparing bases ($2 \ll |\S||\A|$), we see that we still obtain exponential gains in sample complexity from the locally factored approach.
\end{remark}

\paragraph{Training Value Functions and Policies for Out-of-Distribution Generalization}

In the previous subsection, we saw that a locally factored dynamics model provably generalizes outside of the empirical joint distribution. A natural question is whether such \textit{local factorization can be leveraged to obtain similar results for value functions and policies}?

We will show that the answer is \textit{yes}, but perhaps counter-intuitively, it is not achieved by directly training the value function and policy on the empirical distribution, as is the case for the dynamics model. The difference arises because learned value functions, and consequently learned policies, involve the long horizon prediction $\mathbb{E}_{P,\pi}\sum_{t=0}^{\infty}\gamma^tr(s_t, a_t)$, which may not benefit from the local sparsity of $\G^\L$. When compounded over time, sparse local structures can quickly produce an entangled long horizon structure (cf. the ``butterfly effect''). 
Intuitively, even if several pool balls are far apart and locally disentangled, future collisions are central to planning and the optimal policy depends on the relative positions of all balls. This applies even if rewards are factored (e.g., rewards in most pool variants) \citep{sodhani2022improving}. 

We note that, although temporal entanglement may be exponential in the branching factor of the unrolled causal graph, it's possible for the long horizon structure to stay sparse (e.g., $k$ independent factors that never interact, or long-horizon disentanglement between decision relevant and decision irrelevant variables~\citep{huang2022action}). It's also possible that other regularities in the data will allow for good out-of-distribution generalization. Thus, we cannot claim that value functions and policies will never generalize well out-of-distribution (see \citet{veerapaneni2020entity} for an example when they do). Nevertheless, we hypothesize that exponentially fast entanglement does occur in complex natural systems, making direct generalization of long horizon predictions difficult. 

Out-of-distribution generalization of the policy and value function can be achieved, however, by leveraging the generalization properties of a locally factored dynamics model. We propose to do this by generating out-of-distribution states and actions (the ``parent distribution''), and then applying our learned dynamics model to generate transitions that are used to train the policy and value function. We call this process \methodNameFull (\methodName). 
\subsection{Model-based CoDA Using Neural Networks}
\label{sec:mocoda:method}

In the previous section, we discussed how locally factored dynamics model can generalize beyond the empirical dataset to provide accurate predictions on an augmented state-action distribution we call the ``parent distribution''.
We now seek to leverage this out-of-distribution generalization in the dynamics model to bootstrap the training of an RL agent.
Our approach is to control the agent's training distribution $P(s, a, s')$ via the locally factored dynamics $P_\phi(s'|s,a)$ and the  parent distribution $P_\theta(s,a)$ (both trained using experience data).
This allows us to sample \emph{augmented} transitions (perhaps unseen in the experience data) for consumption by a downstream RL agent.
We call this framework \methodName, and summarize it using the following three-step process:

\begin{enumerate}
    \item[\textbf{Step 1}\ ] Given known parent sets, generate appropriate parent distribution $P_\theta(s, a)$.
    \item[\textbf{Step 2}\ ] Apply a learned dynamics model $P_\phi(s'|s,a)$ to parent distribution to generate ``augmented dataset'' of transitions $(s, a, s')$.
    \item[\textbf{Step 3}\ ] Use augmented dataset $s,a,s' \sim P_\theta P_\phi$ (alongside experience data, if desired) to train an off-policy RL agent on the (perhaps novel) target task. 
\end{enumerate}

Figure \ref{fig:block-diagram} illustrates this framework in a block diagram. An instance of \methodName is realized by specific choices at each step. For example, the original CoDA method \citep{pitis2020counterfactual} is an instance of \methodName, which (1) generates the parent distribution by uniformly swapping non-overlapping parent sets, and (2) uses subsamples of empirical transitions as a locally factored dynamics model. CoDA works when local graphs have non-overlapping parent sets, but it does not allow for control over the parent distribution and does not work in cases where parent sets overlap. \methodName generalizes CoDA, alleviating these restrictions and allowing for significantly more design choices, discussed next. 

\subsection{Generating the Parent Distribution}\label{subsection_generating_ad}

\begin{figure}[!t]
  \centering
  \includegraphics[width=0.72\textwidth]{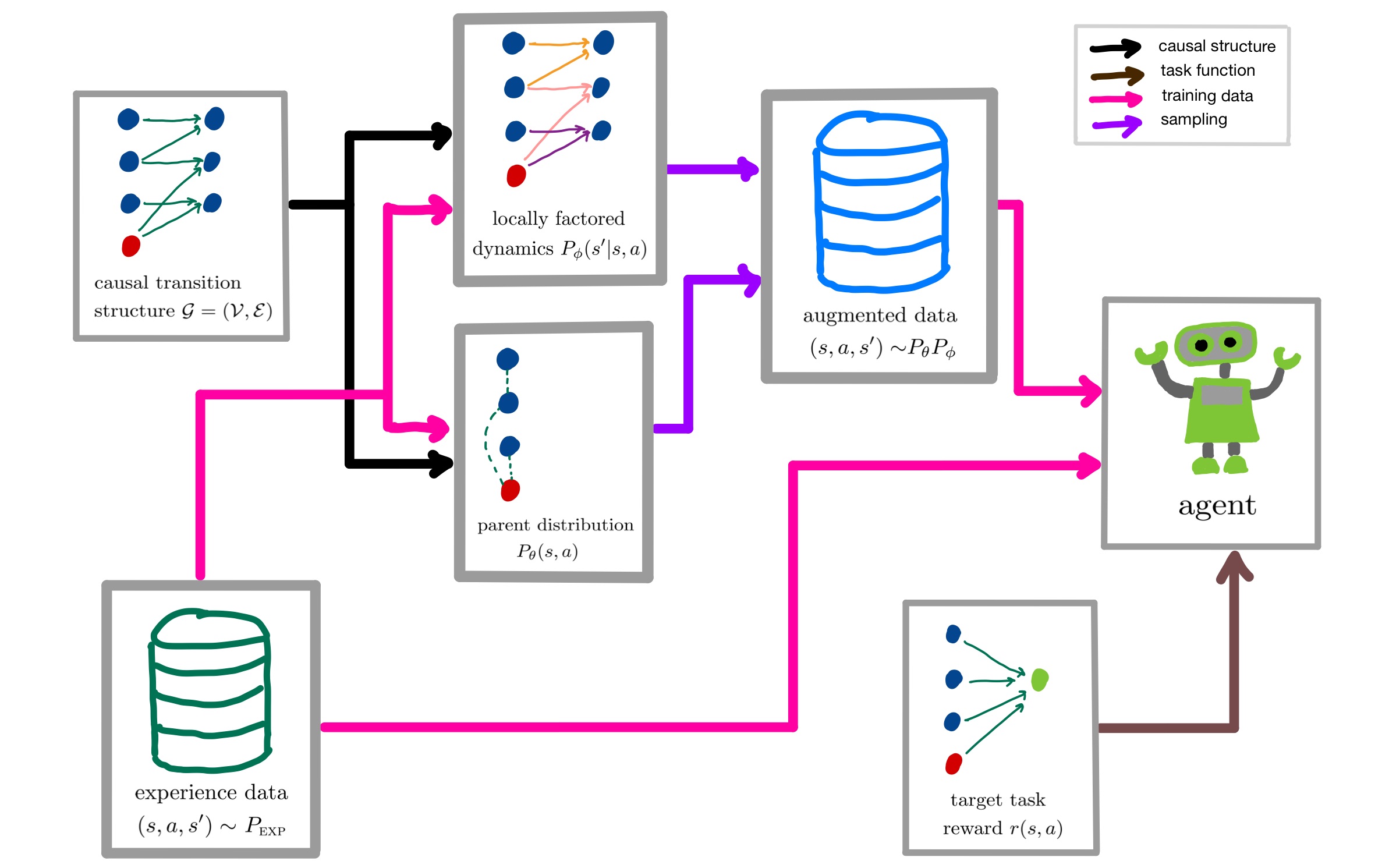}
\caption{\textbf{Training an RL Agent with \methodName}: We use the empirical dataset to train parent distribution model, $P_\theta(s, a)$ and locally factored dynamics model $P_\phi(s'\given s, a)$, both informed by the local structure. The dynamics model is applied to the parent distribution to produce augmented dataset $P_\theta P_\phi$. The augmented \& empirical datasets are labeled with the target task reward, $r(s, a)$ and fed into the RL algorithm as training data.}
    \label{fig:block-diagram}
\vspace{0.3\baselineskip}
\end{figure}

What parent distribution (Step 1) should be used to generate the augmented dataset? We describe some options below, noting that our proposals (\ad{Mocoda}, \ad{Mocoda-U}, \ad{Mocoda-P}) rely on knowledge of (possibly local) parent sets---i.e., they require the state to be decomposed into objects. 

\textbf{Baseline Distributions.}\ \ \ \ If we restrict ourselves to states and actions in the empirical dataset (\textbf{\ad{Emp}}) or short-horizon rollouts that start in the empirical state-action distribution (\textbf{\ad{Dyna}}), as is typical in Dyna-style approaches \citep{sutton2018reinforcement,janner2019trust}, we limit ourselves to a small neighborhood of the empirical state-action distribution. This forgoes the opportunity to train our off-policy RL agent on out-of-distribution data that may be necessary for learning the target task.

Another option is to sample random state-actions from $\S\times\A$ (\textbf{\ad{Rand}}). While this provides coverage of all state-actions relevant to the target task, there is no guarantee that our locally factorized model generalizes well in \ad{Rand}. The proof of Theorem \ref{theorem_one} shows that our model only generalizes well to a particular $(s, a)$ if each component generalizes well on the configurations of each parent set in that $(s, a)$. In context of Theorem \ref{theorem_one}, this occurs only if the empirical data used to train our model contained at least $n$ samples for each set of parents in $(s, a)$. This suggests focusing on data whose parent sets have sufficient support in the empirical dataset. 

\textbf{The \ad{Mocoda} distribution.}\ \ \ \ We do this by constraining the marginal distribution of each parent set (within local neighborhood $\L$) in the augmented distribution to match the corresponding marginal in the empirical dataset. As there are many such distributions, in absence of additional information, it is sensible to choose the one with maximum entropy \citep{jaynes1957information}. We call this maximum entropy, marginal matching distribution the \textbf{\ad{Mocoda}} augmented distribution.
 We propose an efficient way to generate the $\ad{Mocoda}$ distribution using a set of Gaussian Mixture Models, one for each parent set distribution. We sample parent sets one at a time, conditioning on any previous partial samples due to overlap between parent sets. This process is detailed in Appendix \ref{app:coda:mocoda-implementation-details}.

\textbf{Weaknesses of the \ad{Mocoda} distribution.}\ \ \ \ Although our locally factored dynamics model is likely to generalize well on \ad{Mocoda}, there are a few reasons why training our RL agent on $\ad{Mocoda}$ in Step 3 may yield poor results. First, if there are empirical imbalances within parent sets (some parent configurations more common than others), these imbalances will appear in $\ad{Mocoda}$. Moreover, multiple such imbalances will compound exponentially, so that $(s, a)$ tuples with rare parent combinations will be extremely rare in \ad{Mocoda}, even if the model generalizes well to them. 
Second, $\textrm{Support}(\ad{Mocoda})$ may be so large that it makes training the RL algorithm in Step 3 inefficient. Finally, the cost function used in RL algorithms is typically an expectation over the training distribution, and optimizing the agent in irrelevant areas of the state-action space may hurt performance. The above limitations suggest that rebalancing $\ad{Mocoda}$ might improve results.

\textbf{\ad{Mocoda-U} and \ad{Mocoda-P}.}\ \ \ \ To mitigate the first weakness of \ad{Mocoda} we might skew $\ad{Mocoda}$ toward the uniform distribution over its support, $\mathcal{U}(\textrm{Support}(\ad{Mocoda}))$. Although this is possible to implement using rejection sampling when $k$ is small, exponential imbalance makes it impractical when $k$ is large. A more efficient implementation reweights the GMM components used in our \ad{Mocoda} sampler. We call this approach (regardless of implementation) $\textbf{\ad{Mocoda-U}}$. To mitigate the second and third weaknesses of \ad{Mocoda}, we need additional knowledge about the target task---e.g., domain knowledge or expert trajectories. We can use such information to define a prioritized parent distribution \textbf{\ad{Mocoda-P}} with support in \textrm{Support}({\ad{Mocoda}}), which can also be obtained via rejection sampling (perhaps on \ad{Mocoda-U} to also relieve the initial imbalance). 

\subsection{The Choice of Dynamics Model and RL Algorithm}

Once we have a parent distribution, $P_\theta(s, a)$, we generate our augmented dataset by applying dynamics model $P_\phi(s'\given s, a)$. The natural choice in light of the discussion in Section \ref{sec:mocoda:theory} is a locally factored model. This requires knowledge of the local factorization, which is more involved than the parent set knowledge used to generate the \ad{Mocoda} distribution and its reweighted variants. We note, however, that a locally factored model may not be strictly necessary for \methodName, so long as the underlying dynamics are factored. Although unfactored models do not perform well in our experiments, we hypothesize that a good model with enough in-distribution data and the right regularization might learn to implicitly respect the local factorization. The choice of model architecture is not core to our work, and we leave exploration of this possibility to future work. 

\textbf{Masked Dynamics Model.}\ \ \ \ In our experiments, we assume access to a mask function $M: \S \times \A \to \{0, 1\}^{(|\S|+|\A|)\times |\S|}$ (perhaps learned \citep{kipf2018neural,pitis2020counterfactual}), which maps states and actions to the adjacency map of the local graph $\G^\L$. Given this mask function, we design a dynamics model $P_\phi$ that accepts $M(s, a)$ as an additional input and respects the causal relations in the mask (i.e., mutual information $I(X^i_t; X^j_{t+1} \given (S_t, A_t)\setminus X^i_t) = 0$ if $M(s_t, a_t)_{ij} = 0$). There are many architectures that enforce this constraint. In our experiments we opt for a simple one, which first embeds each of the $k$ parent sets: $f = [f_i(\parents_i)]_{i=1}^k$, and then computes the $j$-th child as a function of the sum of the masked embeddings, $g_j(M(s, a)_{\cdot,j}\cdot f)$. See Appendix \ref{app:coda:mocoda-implementation-details} for further implementation details.

\textbf{The RL Algorithm.}\ \ \ \ After generating an augmented dataset by applying our dynamics model to the augmented distribution, we label the data with our target task reward and use the result to train an RL agent. \methodName works with a wide range of algorithms, and the choice of algorithm will depend on the task setting. For example, our experiments are done in an offline setup, where the agent is given a buffer of empirical data, with no opportunity to explore. For this reason, it makes sense to use offline RL algorithms, as this setting has proven challenging for standard online algorithms \citep{levine2020offline}. 

\begin{remark}
 The rationales for (1) regularizing the policy toward the empirical distribution in offline RL algorithms, and (2) training on the \ad{Mocoda} distribution, are compatible: in each case, we want to restrict ourselves to state-actions where our models generalize well. By using \ad{Mocoda} we expand this set \textit{beyond} the empirical distribution. Thus, when we apply offline RL algorithms in our experiments, we train their offline component (e.g., the action sampler in \ad{BCQ}~\citep{fujimoto2019off} or the BC constraint in \ad{TD3-BC}~\citep{fujimoto2021minimalist}) on the expanded \ad{Mocoda} training distribution.
\end{remark}
\subsection{Experiments}
\label{sec:mocoda:empirical}

\textbf{Hypotheses}\ \ \ \ Our experiments are aimed at finding support for two critical hypotheses:

\begin{enumerate}
    \item[\textbf{H1}\ \ ] Dynamics models, especially ones sensitive to the local factorization, are able to generalize well in the \ad{Mocoda} distribution.
    \item[\textbf{H2}\ \ ] This out-of-distribution generalization can be leveraged via data augmentation to train an RL agent to solve out-of-distribution tasks.
\end{enumerate}

Note that support for \textbf{H2} provides implicit support for \textbf{H1}.

\textbf{Domains}\ \ \ \ We test \methodName on two continuous control domains. First is a simple, but controlled, \env{2D Navigation} domain, where the agent must travel from one point in a square arena to another. States are 2D $(x, y)$ coordinates and actions are 2D $(\Delta x, \Delta y)$ vectors. In most of the state space, the sub-actions $\Delta x$ and $\Delta y$ affect only their respective coordinate. In the top right quadrant, however, the $\Delta x$ and $\Delta y$ sub-actions each affect \textit{both} $x$ and $y$ coordinates, so that the environment is locally factored. The agent has access to empirical training data consisting of left-to-right and bottom-to-top trajectories that are restricted to a \lowerl shape of the state space (see the \ad{Emp} distribution in Figure \ref{fig_toy_visualization}). We consider a target task where the agent must move from the bottom left to the top right. In this task there is sufficient empirical data to solve the task by following the $\lowerl$ shape of the data, but learning the optimal policy of going directly via the diagonal requires out-of-distribution generalization.

\begin{wrapfigure}{r}{0.23\textwidth}
    \vspace{-\baselineskip}
	\centering
    \captionsetup{width=0.22\textwidth}
	\includegraphics[width=0.22\textwidth,height=0.17\textwidth]{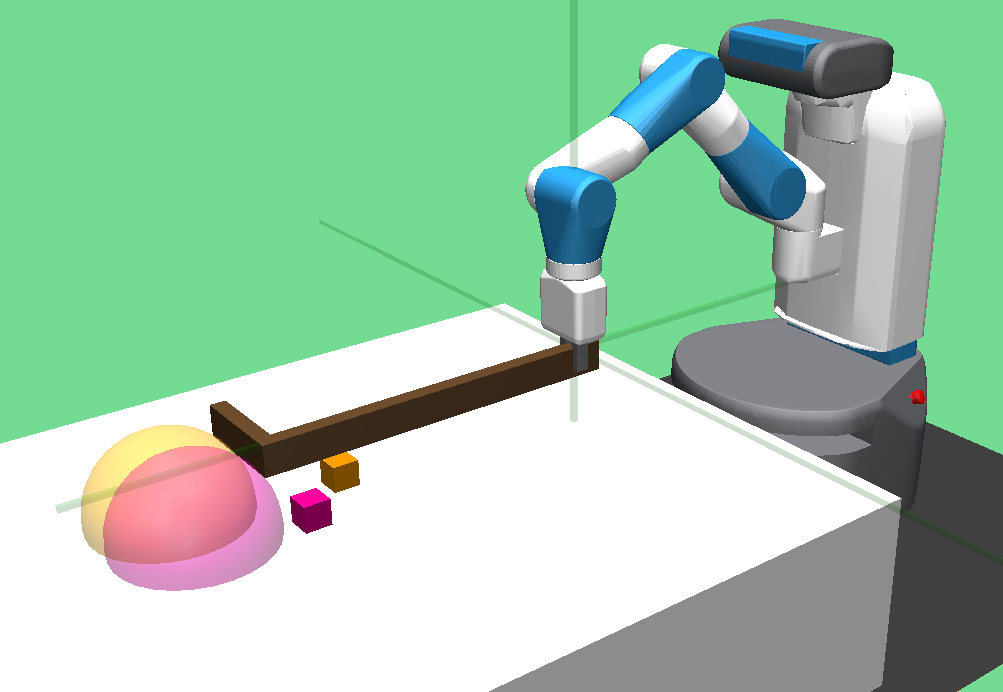}
	\label{fig_env_fetch}
    \vspace{-\baselineskip}
\end{wrapfigure}
Second, we test \methodName in a challenging \env{HookSweep2} robotics domain based on Hook-Sweep \citep{kurenkov2020ac}, in which a Fetch robot must use a long hook to sweep two boxes to one side of the table (either toward or away from the agent). The boxes are initialized near the center of the table, and the empirical data contains trajectories of the agent sweeping exactly one box to one side of the table, leaving the other in the center. The target task requires the agent to generalize to states that it has never seen before (both boxes together on one side of the table). This is particularly challenging because the setup is entirely offline (no exploration), where poor out-of-distribution generalization typically requires special offline RL algorithms that constrain the agent's policy to the empirical distribution \citep{levine2020offline,agarwal2020optimistic,kumar2020conservative,fujimoto2021minimalist}.

\newcommand{\buffer}{\hspace{1.7cm}}
\begin{figure}[!b]
	\centering
	\includegraphics[width=\textwidth]{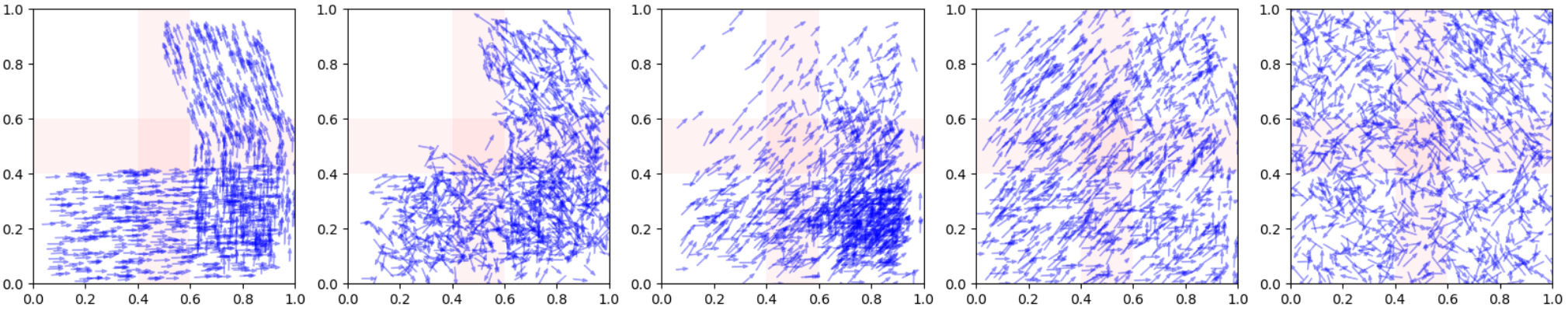}\\
	{\small\hspace{0.4cm}\ad{Emp}\hspace{0.4cm}\buffer\ad{Dyna}\buffer\ad{Mocoda}\hspace{-0.4cm}\buffer\ad{Mocoda-U}\hspace{0.05cm}\buffer\ad{Rand}\hfill}\\[6pt]
	\caption{\textbf{2D Navigation Visualization.} (Best viewed with 2x zoom) Blue arrows represent transition samples as a vector from $(x_t, y_t)$ to $(x_{t+1}, y_{t+1})$. Shaded red areas mark the edges of the initial states of empirical trajectories and the center of the square. We see that 5-step rollouts (\ad{Dyna}) do not fill in the center (needed for optimal policy), and fail to constrain actions to those that the model generalizes well on. For \ad{Mocoda}, we see the effect of compounding dataset imbalance discussed in Subsection \ref{subsection_generating_ad}, which is resolved by \ad{Mocoda-U}.}
	\label{fig_toy_visualization}
\vspace{-0\baselineskip}
\end{figure}

{\small
\tabcolsep=0pt\def\arraystretch{1.3}
\begin{table*}[!t]
\caption{\textbf{2D Navigation Dynamics Modeling Results:} Mean squared error $\pm$ 
\label{tab:toy-mse}
std. dev. over 5 seeds, scaled by 1e2 for clarity (best model boldfaced). The locally factored model experienced less performance degradation out-of-distribution, and performed better on all distributions, except for the empirical distribution (\ad{Emp}) itself.}\label{fig_toy_mse_results}
\newcommand{\smol}[1]{{\scriptsize\texttt{#1}}} % hehe lil smol txt
\centering\small
\begin{tabularx}{\textwidth}{>{\centering\arraybackslash}p{3.2cm}| *5{>{\Centering}X}}\toprule	

&\multicolumn{5}{c}{\textbf{Generalization Error (MSE $\times 1e2$)} (lower is better)}\\
\cline{2-6}
Model Architecture&   \ad{Emp}  &  \ad{Dyna} &  \ad{Rand} & \ad{\textbf{MoCoDA}}  &  \ad{\textbf{MoCoDA-U}}   \\
	\midrule
Not Factored & \textbf{0.14 $\pm$ 0.04} &  2.41 $\pm$ 0.29 &  4.4 $\pm$ 0.31  & 0.95 $\pm$ 0.06 &  1.29 $\pm$ 0.15  \\
Globally Factored & 0.36 $\pm$ 0.01 &  2.09 $\pm$ 0.28  &  3.17 $\pm$ 0.3 &  0.41 $\pm$ 0.02 &  0.51 $\pm$ 0.02 \\
Locally Factored & 0.23 $\pm$ 0.1 &  \textbf{1.47 $\pm$ 0.2}7   &  \textbf{2.03 $\pm$ 0.19} &  \textbf{0.33 $\pm$ 0.11} &  \textbf{0.46 $\pm$ 0.11} \\
	\bottomrule
\end{tabularx}
\vspace{0.5\baselineskip}
\end{table*}
}

{\small
\tabcolsep=0pt\def\arraystretch{1.3}
\begin{table*}[!t]
\caption{\textbf{2D Navigation Offline RL Results:} Average steps to completion $\pm$ std. dev. over 5 seeds for various RL algorithms (best distribution in each row boldfaced), where average steps was computed over the last 50 training epochs. Training on \ad{Mocoda} and \ad{Mocoda-U} improved performance in all cases. Interestingly, even using \ad{Rand} improves performance, indicating the importance of training on out-of-distribution data. Note that this is an offline RL task, and so \ad{SAC} (an algorithm designed for online RL) is not expected to perform well.}
\label{fig_toy_batchrl_results}

\newcommand{\smol}[1]{{\scriptsize\texttt{#1}}}
\centering\small
\begin{tabularx}{\textwidth}{>{\centering\arraybackslash}p{3.1cm}| *5{>{\Centering}X}}\toprule	

\multicolumn{5}{c}{\textbf{Average Steps to Completion} (lower is better)}\\
\cline{2-6}
RL Algorithm&   \ad{Emp}  &  \ad{Rand} & \ad{\textbf{MoCoDA}}  &  \ad{\textbf{MoCoDA-U}}  & \ad{CoDA}~\citep{pitis2020counterfactual} \\
	\midrule
\env{SAC} (online RL) & 53.1 $\pm$ 9.8  &  \textbf{27.6 $\pm$ 1.1} &  38.8 $\pm$ 18.3 &  41.3 $\pm$ 17.7 & 35.1 $\pm$ 18.1\\
\env{BCQ} & 58.5 $\pm$ 10.1  &  31.7 $\pm$ 2.4 &   \textbf{22.8 $\pm$ 0.4} &  24.8 $\pm$ 4.2 & 25.0 $\pm$ 0.4 \\
\env{CQL} & 45.8 $\pm$ 4.0 &  27.6 $\pm$ 1.3 &  22.8 $\pm$ 0.2 &  \textbf{22.7 $\pm$ 0.3} & 23.6 $\pm$ 0.5 \\
\env{TD3-BC} & 40.0 $\pm$ 16.1 &  26.1 $\pm$ 0.8 &  21.0 $\pm$ 0.7 &  \textbf{20.7 $\pm$ 0.8} & 21.4 $\pm$ 0.6 \\
	\bottomrule
\end{tabularx}
\vspace{1.2\baselineskip}
\end{table*}
}

\textbf{Directly comparing model generalization error.}\ \ \ \ In the \env{2D Navigation} domain we have access to the ground truth dynamics, which allows us to directly compare generalization error on variety of distributions, visualized in Figure \ref{fig_toy_visualization}. We compare three different model architectures: unfactored, globally factored (assuming that the $(x, \Delta x)$ and $(y, \Delta y)$ causal mechanisms are independent everywhere, which is not true in the top right quadrant), and locally factored. The models are each trained on a empirical dataset of 35000 transitions for up to 600 epochs, which is early stopped using a validation set of 5000 transitions. The results are shown in Table \ref{fig_toy_mse_results}. We find strong support for \textbf{H1}: even given the simple dynamics of \env{2d Navigation}, it is clear that the locally factored model is able to generalize better than a fully connected model, particularly on the \ad{Mocoda} distribution, where performance degradation is minimal. We note that the \ad{Dyna} distribution was formed by starting in \ad{Emp} and doing 5-step rollouts with \textit{random} actions. The random actions produce out-of-distribution data to which no model (not even the locally factored model) can generalize well to. 

\vspace{0.2\baselineskip}
\textbf{Solving out-of-distribution tasks.}\ \ \ \ We apply the trained dynamics models to several base distributions and compare the performance of RL agents trained on each dataset. To ensure improvements are due to the augmented dataset and not agent architecture, we train several different algorithms, including: \env{SAC} \citep{haarnoja2018soft}, \env{BCQ} \citep{fujimoto2019off} (with DDPG \citep{lillicrap2015continuous}), \env{CQL} \citep{kumar2020conservative} and \env{TD3-BC} \citep{fujimoto2021minimalist}. 

The results on \env{2D Navigation} are shown in Table \ref{fig_toy_batchrl_results}. We see that for all algorithms, the use of the \ad{Mocoda} and \ad{Mocoda-U} augmented datasets greatly improve the average step count, providing support for \textbf{H2} and suggesting that using these datasets allows the agents to learn to traverse the diagonal of the state space, even though it is out-of-distribution with respect to \ad{Emp}. This is consistent with a qualitative assessment of the learned policies, which confirms that agents trained on the \lowerl\hspace{-0.12cm}-shaped \ad{Emp} distribution learn a \lowerl\hspace{-0.12cm}-shaped policy, whereas agents trained on \ad{Mocoda} and \ad{Mocoda-U} learn the optimal (diagonal) policy.

The results on the more complex \env{HookSweep2} environment, shown in Table \ref{table_fetch_batchrl_results}, provide further support for \textbf{H2}. On this environment, only results for \env{BCQ} and \env{TD3-BC} are shown, as the other algorithms failed on all datasets. For \env{HookSweep2} we used a prioritized \ad{Mocoda-P} parent distribution, as follows: knowing that the target task involves placing two blocks, we applied rejection sampling to \ad{Mocoda} to make the marginal distribution of the joint block positions approximately uniform over its support. The effect is to have good representation in all areas of the most important state features for the target task (the block positions). The visualization in Figure \ref{fig_fetch_parent_distributions} makes clear why training on \ad{Mocoda} or \ad{Mocoda-P} was necessary in order to solve this task: the base \ad{Emp} distribution simply does not have sufficient coverage of the goal space.  

\vfill

\begin{figure}[!t]
	\vspace{-0.5\baselineskip}
	\caption{\textbf{HookSweep2 Visualization:} Stylized visualization of the distributions \ad{Emp} (left), \ad{Mocoda} (center), and \ad{Mocoda-P} (right). Each figure can be understood as a top down view of the table, where a point is a plotted if the two blocks are close together on the table. The distribution \ad{Emp} does not overlap with the green goal areas on the left and right, and so the agent is unable to learn. In the \ad{Mocoda} distribution, the agent gets some success examples. In the \ad{Mocoda-P} distribution, state-actions are reweighted so that the joint distribution of the two block positions is approximately uniform, leading to more evenly distributed coverage of the table.}
	\label{fig_fetch_parent_distributions}
	\vspace{-0\baselineskip}
\end{figure}

{\small
\tabcolsep=0pt\def\arraystretch{1.3}
\begin{table*}[!t]
\newcommand{\smol}[1]{{\scriptsize\texttt{#1}}}
\centering\small
\begin{tabularx}{\textwidth}{>{\centering\arraybackslash}p{3cm}| *5{>{\Centering}X}}
\toprule	
&\multicolumn{5}{c}{\textbf{Average Success Rate} (higher is better)}\\
\cline{2-6}
RLAlgorithm&  \vspace{-3px} \ad{Emp}  &\vspace{-3px}  \ad{\textbf{MoCoDA}}  & \vspace{-3px}  \ad{\textbf{MoCoDA-P}} & \ad{{MoCoDA}} (not factored) & \ad{{MoCoDA-P}} (not factored) \\
\midrule
\env{BCQ} & 2.0 $\pm$ 1.6 &  20.7 $\pm$ 4.1 &  \textbf{64.7 $\pm$ 4.1} & 14.0 $\pm$ 3.3 & 15.3 $\pm$ 4.1 \\
\env{TD3-BC} & 0.7 $\pm$ 0.9 &  38.7 $\pm$ 7.5 &  \textbf{84.0 $\pm$ 2.8} & 29.3 $\pm$ 3.8 & 26.0 $\pm$ 1.6 \\
	\bottomrule
\end{tabularx}
	\vspace{0.02in}
	\caption{\textbf{HookSweep2 Offline RL Results:} Average success percentage ($\pm$ std. dev. over 3 seeds), where the average was computed over the last 50 training epochs. 
	\env{SAC} and \env{CQL} (omitted) were unsuccessful with all datasets.
	We see that \ad{Mocoda} was necessary for learning, and that results improve drastically with \ad{Mocoda-P}, which re-balances \ad{Mocoda} toward a uniform distribution in the box coordinates (see Figure \ref{fig_fetch_parent_distributions}).
	Additionally, we show results from an ablation, which generates the MoCoDA datasets using a fully connected dynamics model. While this still achieves some success, it demonstrates that using a locally-factored model is important for OOD generalization. In this case the more OOD MoCoDA-P distribution does not help, suggesting that the fully connected model is failing to produce useful OOD transitions.
	}
	\label{table_fetch_batchrl_results}
	\vspace{\baselineskip}
\end{table*}
}
\section{Related Work}
\paragraph{Exploiting structure in MDPs}
Factored MDPs~\citep{guestrin2003efficient,hallak2015off,weber2019credit} consider MDPs where state variables are only influenced by a fixed subset of ``parent'' variables at the previous timestep.
The notion of ``context specific independence'' (CSI), which was used to compactly represent single factors of a Bayes net \citep{boutilier2013context} or MDP \citep{boutilier1995exploiting} for efficient inference and model storage,\footnote{Note that these methods were proposed to efficiently encode conditional probability tables at the graph nodes, requiring that all variables considered be discrete; CoDA on the other hand works in continuous tasks.} is closely related to the local factorizations we study in this Section.
CSI can be understood as going one step beyond CoDA, exploiting not only knowledge of the local factorization, but also the structural equations at play in the local factorization; CSI could be leveraged for model-based RL approaches where faithful models of factored dynamics can be realized.
Object-oriented and relational approaches to RL and prediction~\citep{diuk2008object,goyal2019recurrent,kipf2019contrastive,li2019towards,zambaldi2018relational,zhu2018object} represent the dynamics as a set of interacting entities. Factored actions and policies have been used to formulate dimension-wise policy gradient baselines in standard and multi-agent settings~\citep{foerster2018counterfactual, lowe2017multi,wu2018variance}. 

\paragraph{Causality and RL}
A growing body of work applies causal reasoning the RL setting to improve sample efficiency, interpretability and learn better representations~\citep{lu2018deconfounding,madumal2019explainable,rezende2020causally}. Particularly relevant is the work by \citet{buesing2018woulda}, which improves sample efficiency by using a causal model to sample counterfactual trajectories, thereby reducing variance of off-policy gradient estimates in a guided policy search framework. 
These counterfactuals use coarse-grained representations at the trajectory level, while our approach uses factored representations within a single transition.
Batch RL~\citep{levine2020offline,fujimoto2018off,mandel2014offline} and  more generally off-policy RL~\citep{watkins1992q,munos2016safe} are counterfactual by nature, and are particularly important when it is costly or dangerous to obtain on-policy data~\citep{thomas2016data}. The use of counterfactual goals to accelerate learning goal-conditioned RL~\citep{kaelbling1993learning,schaul2015universal,plappert2018multi} is what inspired our local CoDA algorithm.

\section{Discussion}
This Chapter presented some promising initial results within the CoDA framework.
Areas for future work include improved methods for learning LCM structure directly from data, which is closely related to causal structure discovery~\citep{seitzer_causal_2021,ke_learning_2020}.
CoDA has so far only been applied to fully observed MDP where $s$ informs the agent which objects are in the scene as well as the relevant state observations for each.
Scaling up CoDA so that it can handle images would be a valuable extension, and some recent work on object-centric representation learning~\citep{wu_slotformer_2023} could be helpful in this endeavor.

At the time of writing, commercial interest in generative models is waxing, mostly due to recent improvements to the fidelity of text and image sampling~\citep{brown_language_2020,ho_denoising_2020,radford_learning_2021,ouyang_training_2022}.
This Chapter provides a complementary view on the benefits of generative models.  
As discussed in Section~\ref{sec:coda:implicit_generative_model} CoDA and MoCoDA are generative models (implicit and explicit, respectively) that function as an auxiliary module to help a different type of learning algorithm (an RL agent, in this case).
When turning our focus to other types of machine learning such classification with fairness~\citep{denton_image_2019,zhang_towards_2020} or expalainability~\citep{chang_explaining_2019} requirements, we can likewise see the benefit of an indirect application of generative models.
We expect methods like MoCoDA (which, unlike CoDA, relies on explicit generative modeling) to get better over time as high-quality generative models of RL environments become more widely available.
Some preliminary success in applying diffusion models to augment RL datasets has already been realized~\citep{yu_scaling_2023}.

The CoDA framework for causally-grounded data augmentation was developed for RL, but in other domains where data coverage is a concern---recommender systems, fair prediction, and robust learning, for example---CoDA could also be applied.
Part of the reason for working in RL first is that the forward progression of time makes the question of causal discovery much easier.
Prior work has shown the merit of using SCMs~\citep{pearl2009causal} to model causal processes in dynamic fairness settings~\citep{creager_causal_2020}.
However, there are some noteworthy objections to the use of causal graphs to model sociotechnial processes.
Firstly, demographic group membership does not represent a truly manipulable variable, which calls into questions whether it is possible to meaningfully (statistically) reason about ``interventions'' on attributes like sex or race~\citep{greiner_causal_2011,kohler-hausmann_eddie_2019}.
Secondly, sociotechnical systems are incredibly sophisticated and nuanced.
Computer scientists may be ill-equipped to model them at the whiteboard using simple graphs~\citep{hu_whats_2020,kasirzadeh_use_2021}.
We remain optimistic about the potential to apply CoDA more broadly, but remind researchers that causal priors are typically untestable.
We found in our experiments that CoDA was relatively insensitive to mis-specification of the LCM, but in high-stakes decision-making settings such observations may not hold.

  \chapter{Conclusion}\label{chap:conc}

\section{Summary: All Data are Imperfect Data}
A favored mantra among statisticians is that \emph{``all models are wrong (but some are useful)''}~\citep{box1976science,box_robustness_1979}.
This can be repeated periodically (or during times of scientific-spiritual turmoil!) as a helpful reminder that the modeler's assumptions (made explicitly or otherwise) are inherited by any inferences made using the model.
The less realistic the assumptions, the less bearing the resulting inferences have on reality.
In other words, your mileage may vary.

We framed this thesis at the outset by describing ML as a process of turning data into computer programs (models).
To continue along this theme, we will conclude by humbly offering an additional mantra: \emph{``all data are imperfect (full stop)''}.
While this is admittedly a weaker statement---``imperfections'' being perhaps easier to recover from than ``wrongness'', while ``usefulness'' is entirely left out of the equation---it is no less informative.
It can help us to understand \emph{why} ML models fail in unexpected ways: 
learning algorithms compound data imperfections by producing wrong models.
The only perfect data are synthetic data.
But this is the exception that proves the rule, as synthetic data are typically used to demonstrate proof-of-concept for a newly proposed method, and do not reflect what we expect to see in the real world.
Yes, machine learning is a practice of data-driven---and thus, by definition, evidence-based---decision making.
But we should not conclude that it provides a detached ``view from nowhere''~\citep{haraway1988situated,d2020data} from which \emph{impartial} decisions can be automated.
Rather, ML researchers and practitioners regularly make assumptions (explicitly or otherwise) about how data is generated and measured~\citep{paullada_data_2021,birhane2022values,scheuerman_datasets_2021}, and these assumptions are inherited by any models trained on the data.

This new mantra also suggests that the aspiration of \emph{Robust ML} is in fact a process rather than a destination.
Not only are available training data imperfect, but so too are datasets used for evaluation.
The process of building reliable models should be an adaptive one, where data imperfections are identified and addressed in perpetuity, with the overall effort keeping the wrongness of the downstream models in reasonable check (and hopefully shrinking over time).

This manuscript has discussed several approaches to addressing data imperfections, which may be of use in the process of robust machine learning.
To some extent, detecting data imperfections may be possible to automate using machine learning (Chapter~\ref{chap:eiil}).
However, useful auxiliary labels pertinent to model failure (Chapter~\ref{chap:fair-reps}) and human-specified prior knowledge (Chapter~\ref{chap:coda}) are nevertheless key components within this process.

\section{What is Next?}

%%%%%%%%%%%%%%%%%%%%%%%%%%%%%%%%%%%%%%%%%%%%%%%%%%%%%%%%%%%%%%%%%%%%%%%%%%%%%%%%
% macros
%%%%%%%%%%%%%%%%%%%%%%%%%%%%%%%%%%%%%%%%%%%%%%%%%%%%%%%%%%%%%%%%%%%%%%%%%%%%%%%%
\newcommand{\contradiction}[1]{

  \begin{centering}
    \emph{#1}
  \end{centering}
}
\newcommand{\contritem}[1]{\item[$\circlearrowright$]\textbf{#1}}
\newcommand{\questitem}[1]{\item[$\looparrowright$]\textbf{#1}}
%%%%%%%%%%%%%%%%%%%%%%%%%%%%%%%%%%%%%%%%%%%%%%%%%%%%%%%%%%%%%%%%%%%%%%%%%%%%%%%%

In each of the preceding Chapters, the proposal of new methods to address data imperfections was punctuated by a brief discussion Section, which suggested some ideas for immediate extension of these ideas.
For the remainder of this manuscript, we set aside short-term directions for future research and turn to the question of what lies further ahead.
Our final discussion is thus more speculative, but it also provides a vantage with which to assess some broader dynamics within the field.
We begin by outlining several contradictions characteristic of the contemporary state of ML research.

\subsection{Some Contradictions Within Contemporary Machine Learning}
New ML methods have developed at a staggering pace in recent years.
And yet, it is difficult to identify a north star, or single unifying goal, guiding the research field towards a specific direction.
How should we make sense of a scientific field that moves quickly, but with no apparent direction or purpose?
In order to add some detail (and hopefully shed some light on the situation) let us begin by unpacking this initial paradox, and listing some follow-on contradictions within the field.
Unlike the self-evident mantra discussed earlier, these observations act more like \emph{koans}: self-refuting propositions that are nevertheless grounded in apparent truth.
To be clear, it is extremely unlikely that these contradictions will be resolved by research activity alone.
However, we hope that their recitation would serve as a murmur of background incoherence that drowns out the loudest and seemingly most urgent research  questions---Which prompt to choose? How to string together APIs to ``solve'' benchmark tasks?---and allows us to formulate a more daring research agenda that will change the shape of the field in the long run.
\begin{itemize}
  \contritem{Is this a scientific or engineering discipline?}
  Research concerning the social impacts of ML has recently been institutionalized, with teams devoted to ``AI Ethics'', ``Technology and Society'', and ``Human Centered AI'' proliferating in academic and corporate spheres.\footnote{
    Admittedly this is more of a re-organization than an institutionalization, given that these new teams typically constitute researchers with participatory some previous institutional credential (or new trainees).
    The failure to include stakeholders from outside institutions in the design of new technology is indeed a known and pressing issue~\citep{abdurahman2019fat,kulynych2020participatory}.
  } 

  When reflecting on the research output from these communities in the past five years, a strong argument can be made that the most impactful projects were those advocating for the adoption of systematic design protocols such as ``datasheets''~\citep{gebru2021datasheets} and ``model cards''~\citep{mitchell2019model}.
  In other words, moving AI and ML research from a scientific endeavor to an engineering discipline is seen by many as an important step towards characterizing existing harms, and preventing future ones.
  Meanwhile, the pressure to commercialize emerging research can be felt most acutely in precisely those specific research subfields (e.g. adapting generative models for use as chatbots or text-to-image samplers) that are most technically challenging to characterize.
  Thus, we arrive at the following contradiction:
  \contradiction{As we move to systematize our field, the object of our focus becomes more difficult to systematize.}
  \contritem{Reproducibility crises}
  In spite of recent progress in ML research, individual results from the literature can be difficult to reproduce~\citep{haibe2020transparency,kapoor2022reproducibility}.
  This could be interpreted as just one example of a broader crisis of reproducibility in the sciences.
  However, certain aspects of ML methodology exacerbate this tendency.
  The fast pace of the field intensifies the publish-or-perish mentality.
  Coupled with a benchmark culture that highly values empirical performance as a measure of scientific validity~\citep{raji_ai_2021}, this can create incentives for researchers to take shortcuts.
  For ML results to be legitimate, careful care must be taken to keep evaluation data out of the training process.
  Because ML is being applied broadly to a variety of scientific disciplines, a practical failure to adhere to these protocols can artificially inflate the numbers.
  To make matters worse, large firms have become less forthcoming in their recent technical notes and public papers, apparently due to a heightened sense of competition surrounding the emergent market for generative models.
  This makes models difficult to reproduce independently, and raises whether some of ML's recent empirical successes---such as scoring highly on standardized testing primarily through language modeling---could be due to leakage of the test data into the training set~\citep{narayanan2023gpt}.
  Thus, we arrive at the following contradiction:
  \contradiction{Machine learning's crescendo of relevance in (economic and cultural) production is accompanied by a diminished capacity for (scientific) reproduction.}
  \contritem{Double underspecification}
  Many prediction problems are underspecified in the sense that they admit a plurality of solutions of equal ``quality'' according to the training objective~\citep{damour_underspecification_2020,fisher2019all}
  The learning algorithm, if trained to convergence, could choose any of these solutions; moreover, it is not obvious which one it \emph{should} choose in light of the fact that they may have very different behaviors under distribution shift.
  This observation, coupled with legal and regulatory concerns surrounding deployment (e.g. algorithmic discrimination) have motivated research on how to learn models that satisfy additional requirements: we search for ML \emph{with} fairness (or explainability, privacy, safety etc.) added on. 
  One might hope that additional requirements would further constrain the solution space and mitigate the underspecification issue.
  However, these requirements are themselves often underspecified: what it means for an algorithm to be ``fair'' or ``explainable'' is fundamentally contested.
  We find ourselves in a sticky situation, attempting to solve learning problems that are \emph{doubly underspecified}, both at the technical and conceptual level.
  Thus, we arrive at the following contradiction:
  \contradiction{Resolving an underspecification in machine terms necessitates introducing an underspecification in human terms.}
  \contritem{Organizing research subfields by adjectives.}
  Aside from the underspecification issue, there are many reasons to seek models that do \emph{more} than just predict accurately.
  Indeed, many ML researchers are committed to working towards ensuring that emerging technologies are socially beneficial.
  We tend to organize these efforts by through the use of qualifying adjectives, organizing new ML subfields such as ``Fair ML'', ``Explainable ML'', and so on.
  This tendency is both productive and counterproductive.
  Organizing subfields by adjectives is an intuitive way to bring people together and catalyze research activity, and indeed tends to be the social norm across scientific and humanistic disciplines~\citep{whitley2000intellectual,abbott2010chaos}
  The trade-off is that when we promote adjectives, we tend to demote verbs.
  This is especially important for ML research given the broad applicability of the methods we work on.
  What does it mean that big tent of ``Fair ML'' stretches over algorithms governing the legal allocation of resources (say, lending) and also those encouraging group parity in facial recognition and law enforcement? 
  What the algorithm \emph{does} (in the real world) is paramount, and yet this is often reduced to a mere implementation detail barely worthy of discussion.
  Thus, we arrive at the following contradiction:
  \contradiction{Organizing subfields by adjectives promotes research activity while simultaneously mystifying the actual uses of the resulting research artifacts.}
  \contritem{``Scaling up'': power laws and the concentration of power.}
  Contemporary approaches for modeling text have benefited considerably by combining neural architectures that enable parallel training with the massive data source of web-scraped text~\citep{brown_language_2020}.
  Reaching state-of-the-art performance requires an ever-increasing number of model parameters and training data~\citep{kaplan_scaling_2020,bender_dangers_2021}.
  This has led to a steep barrier-to-entry in the research space: access to hardware and compute infrastructure now traces a dividing line between those who can train contemporary methods and those who cannot.
  This trend serves as a consolidation of power, with large firms standing to benefit the most ~\citep{luitse_great_2021}.
  However, open-source model sharing and adaptation has also thrived in this new training paradigm.
  This is because of the compute asymmetry inherent in the two stages of pre-training and fine-tuning.
  Once pre-training has been completed, model weights can be distributed just as freely on the web as their original training data.
  Unlike in the data owner/vendor paradigm of Chapter~\ref{chap:fair-reps}, where data representations served as an interface between firms, here the model weights are directly shared.
  Models can then be adapted during fine-tuning under a moderate compute and data budget.
  Thus, we arrive at the following contradiction:
  \contradiction{``Scaling up'' the learning process proliferates high-quality open-source generative models while simultaneously entrenching centralized power among firms.}
  \contritem{The internet as a data source.}
  While the pre-training/fine-tuning recipe is currently the special du jour for ML practitioners (having been successfully applied to image and multi-modal domains~\citep{radford_learning_2021}), there are reasons to believe that the web is a fundamentally limited data source in terms of quantity~\citep{villalobos_will_2022} and quality~\citep{bender_dangers_2021,hanna2020against} of data.
  As more of the web is scraped as a data source, so too increases the likelihood of harmful content being used during training~\citep{birhane2021large}
  Furthermore, as firms race to deploy generative models in massive consumer markets, the text output sampled from these models will become a non-neglible portion of the text available on the web.
  Including today's generative model output in tomorrow's training data creates a new data-to-model-to-data circuit that is very poorly understood.
  Thus, we arrive at the following contradiction:
  \contradiction{Freely accessible internet artifacts serve as ``weak supervision'' for powerful contemporary generative models, but may ultimately undercut their reliability in the long run.}
  \contritem{Statics and dynamics.}
  To complement the data-to-model-to-data circuit, we can also take note of a model-to-data-to-model circuit completed when ML is embedded in a social context and induces a social response.
  People who interact with automated decision-making systems may alter their behavior, and perhaps even attempt to game the system towards their advantage\citep{hardt2016strategic}.
  This type of strategic behavior on part of the human subjects could be realized when the model makes predictions during deployment, or perhaps even to create a strong training signal to anticipate the development of future models~\citep{biggio_poisoning_2013,demontis_why_2019}.
  While the research community has begun to model these processes~\citep{ensign2018runaway,perdomo2020performative}, the presence of data feedback makes the learning considerably more challenging.
  Even approaches such as Reinforcement Learning that tackle feedback head-on typically require limiting assumptions such as Markovness and full observability, and even so are relatively sample inefficient.
  Thus, we arrive at the following contradiction:
  \contradiction{Deploying ML in the real world induces feedback loops, which increases the difficulty of the learning process and undermines the credibility of deployed ML systems.}
\end{itemize}

\subsection{Open Research Questions}
These contradictions need not be read with an air of hopelessness.
Rather, studying the internal tensions within the field may lead us to some promising directions for long-term research that represent an important departure from conventional thinking of how data should be turned into computer programs.
We draw this manuscript to a close with the following (non-exhaustive) list of open research questions worthy of further study:
\begin{itemize}
  \questitem{Characterizing imperfect data.}
  ML methods have been remarkably influential across the academy.
  One notable cluster of activity is referred to as the ``Digital Humanities'', where ML methods such as topic models~\citep{blei2012topic} serve as crucial tools in the effort to extract new scholarly insights from the archives.
  In light of the sensitivities of ML to data imperfections discussed in this thesis, it is time to think more seriously about how to catalog and analyze our own data sources.
  There is a need to go beyond data sheets~\citep{gebru2021datasheets} and thoroughly catalog data origins and known issues.
  Current practices around dataset curation within the ML field are also worthy of scrutiny~\citep{scheuerman_datasets_2021,paullada_data_2021,raji_ai_2021}.
  Fortunately, there is a flurry of ongoing effort in this area~\citep{mitchell_measuring_2023, siddiqui_metadata_2022,jo2020lessons,chang2023speak} demonstrating the potential promise of applying both ML and conventional archiving procedures towards these goals.
  Future work in this direction could also shed light on the question of what role synthetic data (e.g. sampled from language models) in ML.
  \questitem{Participatory priors.}
  Chapter~\ref{chap:coda} discussed how domain knowledge, in the form of causal priors~\citep{pearl2009causal}, can help answer difficult computational questions with limited data.
  We also discussed some critiques of how causal inference is typically implemented, namely that these crucial and often untestable assumptions about social systems tend to be made by computer scientists working at a whiteboard.
  Given the importance causal priors in algorithm design, can we source them in a more participatory way?
  Recent work from~\citet{martin2020extending} suggests a way forward, highlighting the role of collaborative causal theory formulation in “extending the ML abstraction boundary” to realize more inclusive and equitable development processes.
  They focus in particular on techniques from Community-based System Design (CBSD)~\citep{hovmand2014group} for collectively generating causal theories underlying ML systems, drawing on an established literature of participatory design that includes rigorous ways to model the social context in which a model is embedded.
  Perhaps causal inference can be synthesized with CBSD to arrive at participatory models that provide rigorous causal inferences.
  This may unearth new technical problems: as mentioned above, feedback loops can be difficult to model from an ML perspective---we tend to roll them out, but in CBSD graphs this could involve unpacking tens of thousands of nodes---so it remains to be seen whether sourcing causal priors via community engagement would result in practically useful learning algorithms.
  \questitem{Collective action.}
  What is the role of community action in guiding the behavior of ML systems that are unaligned to community values?
  From the perspective of the model, all data subjects are \emph{i.i.d.} and thus act independently. 
  Of course, data subjects need not adhere to the model's assumption about them.
  In settings where the model owner cannot be trusted to ensure socially beneficial properties such as fairness, data subjects may wish to take matters into their own hands by coordinating to change the overall data distribution.
  This type of collective action could empower data subjects to have a more democratic control over how machine learning systems function in their broader social context.
  For example, \emph{algorithmic recourse} considers what changes to an individual's data profile will improve their future prospects following an adverse automated decision (say, a loan denial)~\citep{karimi2022survey}.
  The typical framing for this problem assumes data subjects must each act alone and in their own self interests.
  However, in online settings---where the decision maker's policy is updated over time---collective action can be a more effective strategy~\citep{creager2021online}.
  Exploring the question of collective action more broadly---for example in recommender systems where myopic models can shape long-term user behavior~\citep{mladenov2020optimizing,carroll2022estimating}---is a nascent research direction.
  \citet{hardt2023algorithmic} recently provided a general-purpose analytic framework that could prove helpful here;
  they also provide compelling theoretical results suggesting that, through coordination, small collectives can have an outsides influence on learned models.
  \questitem{Politics of refusal.}
  In theory, expressive models like deep neural networks can fit to any function given a sufficient number of examples~\citep{hornik_multilayer_1989}.
  Meanwhile, in practice, datasets of increasing size and scope are being amassed~\citep{kaplan_scaling_2020,bender_dangers_2021}.
  Perhaps the most consequential research question, then, is what predictive models \emph{should} we build using data~\citep{barocas_when_2020}.
  While policy advocates and organizers have frequently raised the alarm over dubious applications of data-driven prediction, computer science researchers have been quiet by comparison.

  ML research tends to novel methodology above other possible contributions~\citep{sambasivan2021everyone,birhane2022values}, although dataset contributions have recently been more widely accepted as well~\citep{vanschoren2021announcing}.
  \emph{Task definition} is an even more central topic that should be elevated as an equally important area for future contribution and contestation.
  Reorganizing research efforts accordingly is unlikely to produce a consensus, so methods to respectively adjudicate internal conflicts would be needed~\cite{mouffe2013agonistics, young2022confronting}.
\end{itemize}

  \cleardoublepage
  \clearpage
  \appendix
  \chapter{Notes for Chapter \ref{chap:fair-reps} }

\section{Notes on Section \ref{sec:laftr:laftr} (LAFTR)}
\subsection{Bounding Equalized Odds}\label{app:laftr:theory-eq-odds}
We now turn our attention to equalized odds. 
First we extend our shorthand to denote $p(Z| A=a, Y=y)$ as $\mathcal{Z}_{a}^{y}$, the representation of group $a$ conditioned on a specific label $y$.
The equalized odds distance of classifier $g: \Omega_{\mathcal{Z}} \rightarrow \{0, 1\}$ is
\begin{equation}\label{eq:delta-eo}
    \begin{aligned}
        \Delta_{EO}(g) &\triangleq | \E_{\mathcal{Z}_0^0}[g] - \E_{\mathcal{Z}_1^0}[g] | \\
        &+ | \E_{\mathcal{Z}_0^1}[1 - g] - \E_{\mathcal{Z}_1^1}[1 - g] |
,
    \end{aligned}
\end{equation}
which comprises the absolute difference in false positive rates plus the absolute difference in false negative rates.
$\Delta_{EO}(g) = 0$ means $g$ satisfies equalized odds. 
Note that $\Delta_{EO}(g) \leq \Delta(\mathcal{Z}_0^0, \mathcal{Z}_1^0) + \Delta(\mathcal{Z}_0^1, \mathcal{Z}_1^1)$. 

We can make a similar claim as above for equalized odds: given an optimal adversary trained on $Z$ with the appropriate objective, if the adversary also receives the label $Y$, the adversary's loss will upper bound $\Delta_{EO}(g)$ for any function $g$ learnable from $Z$.

\textbf{Theorem.} \textit{
Let the classifier $g: \Omega_{\mathcal{Z}} \rightarrow \Omega_{\mathcal{Y}}$ and the adversary $h: \Omega_{\mathcal{Z}} \times \Omega_{\mathcal{Y}} \rightarrow \Omega_{\mathcal{Z}}$, as binary functions, i.e., $\Omega_{\mathcal{Y}} = \Omega_{\mathcal{A}} = \{0, 1\}$.
Then $L_{Adv}^{EO}(h^*) \geq \Delta_{EO}(g)$: the equalized odds distance of $g$ is bounded above by the optimal objective value of $h$.
}    

\textbf{Proof.} Let the adversary $h$'s objective be
\begin{equation}
\begin{aligned}
    L_{Adv}^{EO}(h) &= \E_{\mathcal{Z}_0^0}[1 - h] + \E_{\mathcal{Z}_1^0}[h]\\
&+ \E_{\mathcal{Z}_0^1}[1 - h] + \E_{\mathcal{Z}_1^1}[h] - 2\\
\end{aligned}
\end{equation}

By definition $\Delta_{EO}(g) \geq 0$.
Let  $| \E_{\mathcal{Z}_0^0}[g] - \E_{\mathcal{Z}_1^0}[g] | = \alpha \in [0, \Delta_{EO}(g)]$ and $| \E_{\mathcal{Z}_0^1}(1 - g) - \E_{\mathcal{Z}_1^1}[1 - g] | = \Delta_{EO}(g) - \alpha$. 
WLOG, suppose $\E_{\mathcal{Z}_0^0}[g] \geq \E_{\mathcal{Z}_1^0}[g]$ and $\E_{\mathcal{Z}_0^1}[1 - g] \geq \E_{\mathcal{Z}_1^1}[1 - g]$. 
Thus we can partition (\ref{eq:delta-eo}) as two expressions, which we write as 
\begin{equation}
\begin{aligned}
\E_{\mathcal{Z}_0^0}[g] + \E_{\mathcal{Z}_1^0}[1 - g] &= 1 + \alpha,\\
\E_{\mathcal{Z}_0^1}[1 - g] + \E_{\mathcal{Z}_1^1}[g] &= 1 + (\Delta_{EO}(g) - \alpha),\\
\end{aligned}
\end{equation}
which can be derived using the familiar identity $\E_{p}[\eta] = 1-\E_{p}[1-\eta]$ for binary functions.

Now, let us consider the following adversary $h$
\begin{equation}
\begin{aligned}
    h(z) = \left\{\begin{array}{lr}
        g(z), & \text{if } y = 1\\
        1 - g(z), & \text{if } y = 0
        \end{array}\right\} 
        .
\end{aligned}
\end{equation}
Then the previous statements become
\begin{equation}
\begin{aligned}
\E_{\mathcal{Z}_0^0}[1 - h] + \E_{\mathcal{Z}_1^0}[h] &= 1 + \alpha\\
\E_{\mathcal{Z}_0^1}[1 - h] + \E_{\mathcal{Z}_1^1}[h] &= 1 + (\Delta_{EO}(g) - \alpha)\\
\end{aligned}
\end{equation}
Recalling our definition of $L_{Adv}^{EO}(h)$, this means that 
\begin{equation}
\begin{aligned}
    L_{Adv}^{EO}(h) &= \E_{\mathcal{Z}_0^0}[1 - h] + \E_{\mathcal{Z}_1^0}[h] + \E_{\mathcal{Z}_0^1}[h] + \E_{\mathcal{Z}_1^1}[h] - 2 \\
&= 1 + \alpha + 1 + (\Delta_{EO}(g) - \alpha) - 2 = \Delta_{EO}(g)\\
\end{aligned}
\end{equation}
That means that for the optimal adversary $h^\star = \sup_h L_{Adv}^{EO}(h)$, we have $L_{Adv}^{EO}(h^\star) \geq L_{Adv}^{EO}(h) = \Delta_{EO}$.
\hfill$\blacksquare$

An adversarial bound for equal opportunity distance, defined as $\Delta_{EOpp}(g) \triangleq | \E_{\mathcal{Z}_0^0}[g] - \E_{\mathcal{Z}_1^0}[g] |$, can be derived similarly.

\subsection{Experimental Details}\label{sec:training-details}

We used single-hidden-layer neural networks for each of our encoder, classifier and adversary, with 20 hidden units for the Health dataset and 8 hidden units for the Adult dataset. 
We also used a latent space of dimension 20 for Health and 8 for Adult. 
We train with $L_C$ and $L_{Adv}$ as absolute error, as discussed in Section \ref{gen-model}, as a more natural relaxation of the binary case for our theoretical results. 
Our networks used leaky rectified linear units and were trained with Adam \citep{kingma2014adam} with a learning rate of 0.001 and a minibatch size of 64, taking one step per minibatch for both the encoder-classifier and the discriminator. 
When training \textsc{ClassLearn} in Algorithm \ref{alg:laftr} from a learned representation we use a single hidden layer network with half the width of the representation layer, i.e., g.	
\textsc{ReprLearn} (i.e., LAFTR) was trained for a total of 1000 epochs, and \textsc{ClassLearn} was trained for at most 1000 epochs with early stopping if the training loss failed to reduce after 20 consecutive epochs.

To get the fairness-accuracy tradeoff curves in Figure \ref{results:pareto-fairness}, we sweep across a range of fairness coefficients $\gamma \in [0.1, 4]$. To evaluate, we use a validation procedure. For each encoder training run, model checkpoints were made every 50 epochs; $r$ classifiers are trained on each checkpoint (using $r$ different random seeds), and epoch with lowest median error $+ \Delta$ on validation set was chosen. We used $r = 7$.
Then $r$ more classifiers are trained on an unseen test set. The median statistics (taken across those $r$ random seeds) are displayed.

For the transfer learning experiment, we used $\gamma=1$ for models requiring a fair regularization coefficient.

\section{Notes on Section \ref{sec:ffvae:ffvae} (FFVAE)}
\subsubsection{Discriminator approximation of total correlation} \label{sec:disc_approx}
This section describes how density ratio estimation is implemented to train the FFVAE encoder.
We follow the approach of \citet{kim2018disentangling}.
\paragraph{Generating Samples}
The binary classifier adversary seeks to discriminate between
\begin{itemize}
    \item $[z,b] \sim q(z,b)$, ``true'' samples from the aggregate posterior; and
    \item $[z',b'] \sim q(z)\prod_j q(b_j)$, ``fake'' samples from the product of the marginal over $z$ and the marginals over each $b_j$.
\end{itemize}
At train time, after splitting the latent code $[z^i,b^i]$ of the $i$-example along the dimensions of $b$ as $[z^i, b_0^i ... b_j^i]$, the minibatch index order for each subspace is then randomized, simulating samples from the product of the marginals; these dimension-shuffled samples retain the same marginal statistics as ``real'' (unshuffled) samples, but with joint statistics between the subspaces broken.
The overall minibatch of encoder outputs contains twice as many examples as the original image minibatch, and comprises equal number of ``real'' and ``fake'' samples.

As we describe below, the encoder output minibatch is used as training data for the adversary, and the error is backpropagated to the encoder weights so the encoder can better fool the adversary.
 If a strong adversary can do no better than random chance, then the desired independence property has been achieved.

\paragraph{Discriminator Approximation}
Here we summarize the approximation of the 
\[D_{KL}(q(z,b)||q(z)\prod_{j} q(b_j))\]
term from equation \ref{eq:ffvae}.
Let $u \in \{0, 1\}$ be an indicator variable with $u=1$ indicating $[z,b] \sim q(z,b)$ comes from a minibatch of ``real'' encoder distributions, while $u=1$ indicating $[z',b'] \sim q(z)\prod_j(b_j)$ is drawn from a ``fake'' minibatch of shuffled samples, i.e., is drawn from the product of the marginals of the aggregate posterior.
The discriminator network outputs the probability that vector $[z,b]$ is a ``real'' sample, i.e., 
$d(u|z,b)=\text{Bernoulli}(u|\sigma(\theta_{d}(z,b))$ 
where $\theta_d(z,b)$ 
is the discriminator and $\sigma$ is the sigmoid function.
If the discriminator is well-trained to distinguish between ``real'' and ``fake'' samples then we have 
\begin{align}
    \log d(u=1|z,b) - \log d(u=0|z,b) &\approx \nonumber \\
    \log q(z,b) - \log q(z)\prod_j q(b_j).
\end{align}
We can substitute this into the KL divergence as 
\begin{align}
    &D_{KL}(q(z,b)||q(z)\prod_{j} q(b_j)) = \nonumber \\
    &\quad\quad \E_{q(z,b)}[ \log q(z,b) - \log q(z)\prod_j q(b_j) ] \approx \nonumber \\
    &\quad\quad \E_{q(z,b)}[ \log d(u=1|z,b) - \log d(u=0|z,b)].
\end{align}

Meanwhile the discriminator is trained by minimizing the standard cross entropy loss
\begin{align}
    L_{\text{Disc}}(d) &= \E_{z,b\sim q(z,b)} [\log d(u=1|z,b)] \nonumber \\
    &\quad\quad + \E_{z',b'\sim q(z)\prod_j q(b_j)} [\log ( 1 - d(u=0|z',b') )],
\end{align}
w.r.t. the parameters of $d(u|z,b)$.
This ensures that the discriminator output $\theta_d(z,b)$ is a calibrated approximation of the log density $\log\frac{q(z,b)}{q(z)\prod_j q(b_j)}$.

$L_{\text{Disc}}(d)$ and $L_{\text{FFVAE}}(p,q)$ (Equation \ref{eq:ffvae}) are then optimized in a min-max fashion.
In our experiments we found that single-step alternating updates using optimizers with the same settings sufficed for stable optimization.

\subsection{Experimental Details}\label{sec:dsprites-training-details}
\paragraph{DSpritesUnfair data generation}
The original DSprites dataset has six ground truth factors of variation (FOV): 
\begin{itemize}
    \item Color: white
    \item Shape: square, ellipse, heart
    \item Scale: 6 values linearly spaced in $[0.5, 1]$
    \item Orientation: 40 values in $[0, 2\pi]$
    \item XPosition: 32 values in $[0, 1]$
    \item YPosition: 32 values in $[0, 1]$
\end{itemize}

In the original dataset the joint distribution over all FOV factorized; each FOV was considered independent.
In our dataset, we instead sample such that the FOVs Shape and X-position correlate.
We associate an index with each possible value of each attribute, and then sample a (Shape, X-position) pair with probability proportional to $(\frac{i_S}{n_S})^{q_S} + (\frac{i_X}{n_X})^{q_X}$, where $i, n, q$ are the indices, total number of options, and a real number for each of Shape and X-position ($S, X$ respectively).
We use $q_S = 1, q_X = 3$. 
All other attributes are sampled uniformly, as in the standard version of DSprites.

We binarized the factors of variation by using the Boolean outputs of the following operations:
\begin{itemize}
    \item Color $\geq 1$
    \item Shape $\geq 1$
    \item Scale $\geq 3$
    \item Rotation $\geq 20$
    \item XPosition $\geq 16$
    \item YPosition $\geq 16$
\end{itemize}

\paragraph{Training details}
All network parameters were optimized using the Adam \citep{kingma2014adam}, with learning rate 0.001.
Our encoders trained $3 \times 10^5$ iterations with minibatch size 64 (as in \citet{kim2018disentangling}).
Our MLP classifier has two hidden layers with 128 units each, and is trained with patience of 5 epochs on validation loss.

\subsection{Mutual Information Gap} \label{sec:mig}
\paragraph{Evaluation Criteria}
Here we analyze the encoder mutual information in the synthetic setting of the DSpritesUnfair dataset, where we know the ground truth factors of variation.
In Fig. \ref{fig:dsprites-mig}, we calculate the \textit{Mutual Information Gap (MIG)} \citep{chen2018isolating} of FFVAE across various hyperparameter settings.
With $J$ latent variables $z_j$ and $K$ factors of variation $v_k$, MIG is defined as 
\begin{equation}
    \frac{1}{K} \sum_{k=1}^{K} \frac{1}{H(v_k)} (MI (z_{j_k}; v_k) - \max_{j \neq j_k} MI (z_j; v_k))
\end{equation}
where $j_k = \underset{j}{\argmax} MI (z_j; v_k)$, $MI(\cdot; \cdot)$ denotes mutual information, and $K$ is the number of factors of variation.
Note that we can only compute this metric in the synthetic setting where the ground truth factors of variation are known.
MIG measures the difference between the latent variables which have the highest and second-highest $MI$ with each factor of variation, rewarding models which allocate one latent variable to each factor of variation.
We test our disentanglement by training our models on a biased version of DSprites, and testing on a balanced version (similar to the ``skewed'' data in \citet{chen2018isolating}).
This allows us to separate out two sources of correlation --- the correlation existing across the data, and the correlation in the model's learned representation.

\paragraph{Results}

%%%%%%%%%%%%%%%%%%%%%%%%%%%%%%%%%%%%%%%%%%%%%%%%%%%%%%%%%%%%%%%%%%%%%%%%%%%%%%%%
% results-dsprites-mig
%%%%%%%%%%%%%%%%%%%%%%%%%%%%%%%%%%%%%%%%%%%%%%%%%%%%%%%%%%%%%%%%%%%%%%%%%%%%%%%%
\begin{figure}[ht!]
\centering
\begin{subfigure}[t]{0.33\textwidth}
\includegraphics[width=\textwidth]{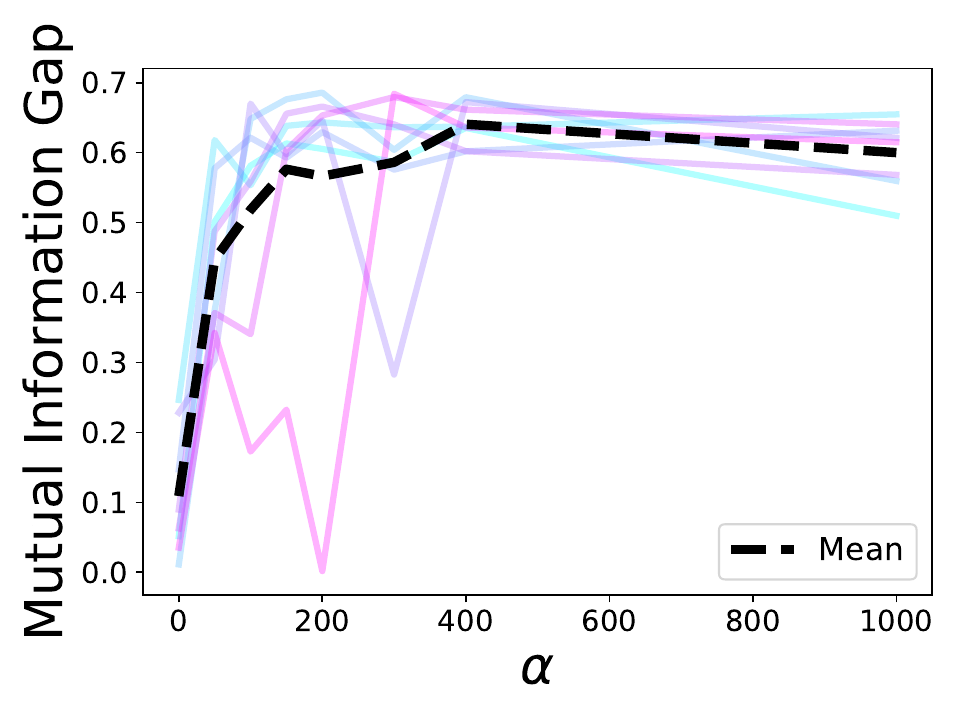}
\caption{Color is $\gamma$, brighter colors $\longrightarrow$ higher values}
\label{fig:dsprites-mig-lines}
\end{subfigure}
\hspace{1cm}
\begin{subfigure}[t]{0.33\textwidth}
\includegraphics[width=\textwidth]{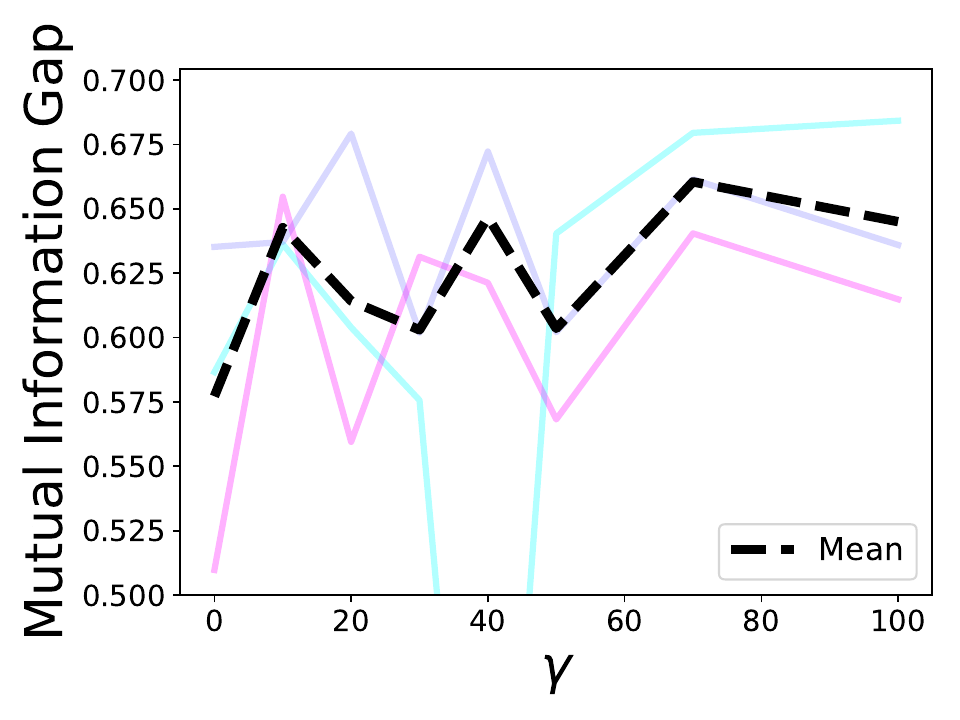}
\caption{Color is $\alpha$, brighter colors $\longrightarrow$ higher values}
\label{fig:dsprites-mig-lines-high-alpha}
\end{subfigure}

\caption{
    Mutual Information Gap (MIG) for various $(\alpha, \gamma)$ settings of the FFVAE. 
    In Fig. \ref{fig:dsprites-mig-lines}, each line is a different value of
    $\gamma \in [10, 20, 30, 40, 50, 70, 100]$, with brighter colors indicating larger values of $\gamma$.
    In Fig. \ref{fig:dsprites-mig-lines-high-alpha}, each line is a different value of
    $\alpha \in [300, 400, 1000]$, with brighter colors indicating larger values of $\alpha$.
    Models trained on DspritesUnfair, MIG calculated on Dsprites.
    Higher MIG is better.
    Black dashed line indicates mean (with outliers excluded).
    $\alpha = 0$ is equivalent to the FactorVAE.
    }
    \label{fig:dsprites-mig}
\end{figure}
%%%%%%%%%%%%%%%%%%%%%%%%%%%%%%%%%%%%%%%%%%%%%%%%%%%%%%%%%%%%%%%%%%%%%%%%%%%%%%%%

In Fig. \ref{fig:dsprites-mig-lines}, we show that MIG increases with $\alpha$, providing more evidence that the supervised structure of the FFVAE can create disentanglement.
This improvement holds across values of $\gamma$, except for some training instability for the highest values of $\gamma$.
It is harder to assess the relationship between $\gamma$ and MIG, due to increased instability in training when $\gamma$ is large and $\alpha$ is small.
However, in Fig. \ref{fig:dsprites-mig-lines-high-alpha}, we look only at $\alpha \geq 300$, and note that in this range, MIG improves as $\gamma$ increases.
See Appendix \ref{sec:mig} for more details.

%%%%%%%%%%%%%%%%%%%%%%%%%%%%%%%%%%%%%%%%%%%%%%%%%%%%%%%%%%%%%%%%%%%%%%%%%%%%%%%%
% results-dsprites-mig-appendix.tex
%%%%%%%%%%%%%%%%%%%%%%%%%%%%%%%%%%%%%%%%%%%%%%%%%%%%%%%%%%%%%%%%%%%%%%%%%%%%%%%%
\begin{figure}[ht!]
\centering
\begin{subfigure}[t]{0.33\textwidth}
\includegraphics[width=\textwidth]{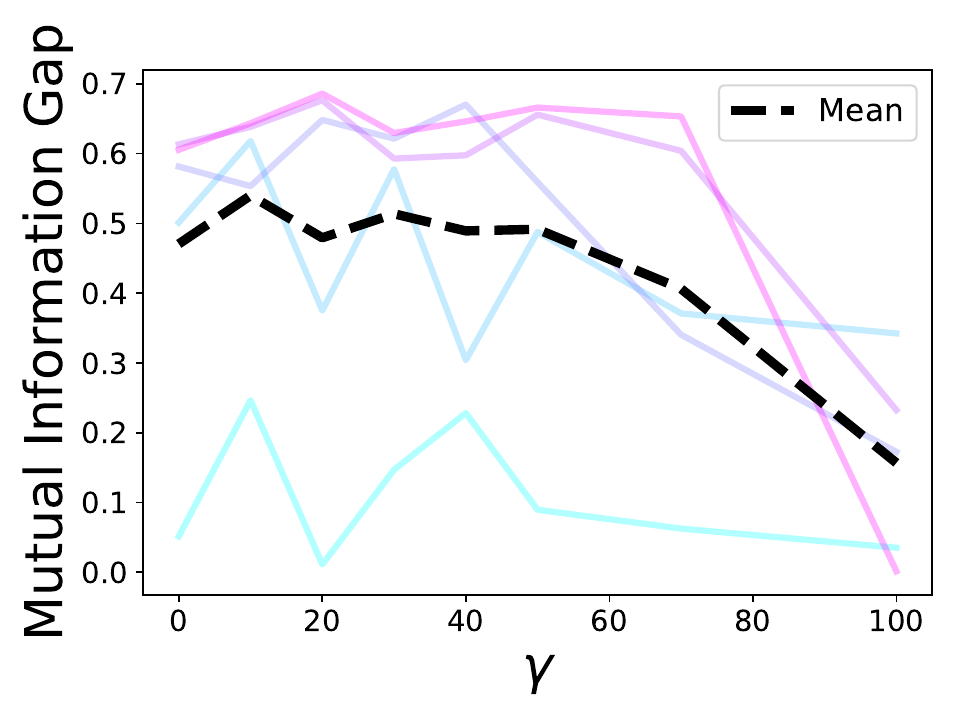}
\caption{Color is $\alpha$}
\label{fig:dsprites-mig-lines-low-alpha}
\end{subfigure}
\hspace{1cm}
\begin{subfigure}[t]{0.33\textwidth}
\includegraphics[width=\textwidth]{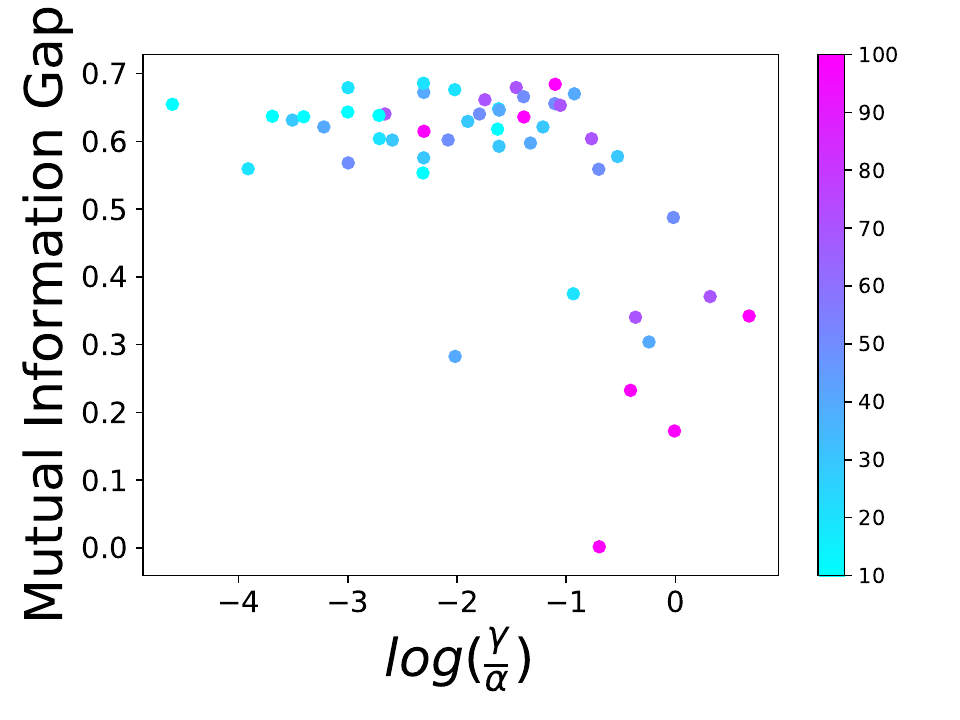}
\caption{Color is $\gamma$}
\label{fig:dsprites-mig-ratio}
\end{subfigure}
\caption{
    Mutual Information Gap (MIG) for various $(\alpha, \gamma)$ settings of the FFVAE. 
    In Fig. \ref{fig:dsprites-mig-lines-low-alpha}, each line is a different value of
    $\alpha \in [0, 50, 100, 150, 200]$, with brighter colors indicating larger values of $\alpha$.
    In Fig. \ref{fig:dsprites-mig-ratio}, all combinations with $\alpha, \gamma > 0$ are shown.
    Models trained on DspritesUnfair, MIG calculated on Dsprites.
    Higher MIG is better.
    Black dashed line indicates mean (outliers excluded).
    $\alpha = 0$ is equivalent to the FactorVAE.
    }
    \label{fig:dsprites-mig-appendix}
\end{figure}
%%%%%%%%%%%%%%%%%%%%%%%%%%%%%%%%%%%%%%%%%%%%%%%%%%%%%%%%%%%%%%%%%%%%%%%%%%%%%%%%

In Fig. \ref{fig:dsprites-mig-lines-low-alpha}, we show that for low values of $\alpha$, increasing $\gamma$ leads to worse MIG, likely due to increased training instability.
This is in contrast to Fig. \ref{fig:dsprites-mig-lines-high-alpha}, which suggests that for high enough $\alpha$, increasing $\gamma$ can improve MIG.
This leads us to believe that $\alpha$ and $\gamma$ have a complex relationship with respect to disentanglement and MIG.

To better understand the relationship between these two hyperparameters, we examine how MIG varies with the ratio $\frac{\gamma}{\alpha}$ in Fig. \ref{fig:dsprites-mig-ratio}.
In We find that in general, a higher ratio yields lower MIG, but that the highest MIGs are around $\log\frac{\gamma}{\alpha} = -2$, with a slight tailing off for smaller ratios.
This indicates there is a dependent relationship between the values of $\gamma$ and $\alpha$.

\paragraph{Discussion}
What does it mean for our model to demonstrate disentanglement on test data drawn from a new distribution?
For interpretation, we can look to the causal inference literature, where one goal is to produce models that are robust to certain interventions in the data generating process \citep{rothlenhauser2018anchor}.
We can interpret Figure \ref{fig:dsprites-mig} as evidence that our learned representations are (at least partially) invariant to \textit{interventions} on $a$.
This property relates to counterfactual fairness, which requires that models be robust with respect to counterfactuals along $a$ \citep{kusner2017counterfactual}.

  \chapter{Notes for Chapter \ref{chap:eiil}}

\section{Environment Inference Psuedocode}\label{sec:pseudocode}

%%%%%%%%%%%%%%%%%%%%%%%%%%%%%%%%%%%%%%%%%%%%%%%%%%%%%%%%%%%%
% algo box
%%%%%%%%%%%%%%%%%%%%%%%%%%%%%%%%%%%%%%%%%%%%%%%%%%%%%%%%%%%%
\begin{algorithm}[tb]
   \caption{Pseudocode for environment inference (EI) with the invariance principle (realized via relaxed IRMv1 penalty) as the EI objective.}
   \label{alg:example}
% NOTE: for some reason \STATE had to be changed to \State in order for this
%       algo box to be compiled as a ut-thesis document
\begin{algorithmic}
   \State {\bfseries Input:} Reference model $\Phi$, dataset $\mathcal{D} = \{x_i, y_i\}$, loss $\ell$, duration $N_{steps}$
   \State {\bfseries Output:} Worst case data splits $\mathcal{D}_1$, $\mathcal{D}_2$ for use with an invariant learner.
   \vspace{.2cm}
   \State \textbf{def} $\tilde R^e(\Phi, \mathbf{q})$:
   \State \ \ \textbf{return} $\frac{1}{N} \sum_i \mathbf{q}_i(e) \ell(\Phi(x_i), y_i)$ \hfill\COMMENT{Equation \ref{eq:soft-risk}}
   \vspace{.2cm}
   \State Randomly init. $\mathbf{q} \in [0, 1]^N$ environment posterior ($\mathbf{q}_i(e) := q(e|x_i, y_i)$)
   \FOR{$n \in 1 \ldots N_{steps}$}
     \State $SoftVari = \sum_{e\in\{1,2\}}||\nabla_{\bar w} \tilde R^e(\bar w \circ \Phi, \mathbf{q})||$
     \hfill \COMMENT{Aggregate reference model variances across soft envs}
     \State $Loss = -1 \cdot SoftVari$
     \hfill \COMMENT{Maximize the EI objective by minimizing this loss}
     \State $\mathbf{q} \leftarrow OptimUpdate(\mathbf{q}, \nabla_{\mathbf{q}} Loss)$
   \ENDFOR
   \State $\hat{ \mathbf{q} }\sim {Bernoulli}(\mathbf{q})$ \hfill\COMMENT{sample splits}
   \State $\mathcal{D}_1 \leftarrow \{x_i, y_i | \hat {\mathbf{q}}_i = 1\}$,
          $\mathcal{D}_2 \leftarrow \{x_i, y_i | \hat {\mathbf{q}}_i = 0\}$ \hfill\COMMENT{split data}
   \State {\bfseries return} $\mathcal{D}_1, \mathcal{D}_2$  
\end{algorithmic}
\label{algo:eiil}
\end{algorithm}
%%%%%%%%%%%%%%%%%%%%%%%%%%%%%%%%%%%%%%%%%%%%%%%%%%%%%%%%%%%%

Algorithm \ref{algo:eiil} provides pseudocode for the environment inference procedure used in our experiments.

\section{Proofs}\label{sec:proofs}
\subsection{Proof of Proposition \ref{thm:1}}\label{sec:proof_of_thm1}
Consider a dataset with some feature(s) $z$ which are spurious, and other(s) $v$ which are valuable/causal w.r.t. the label $y$.
This includes data generated by models where
$v \rightarrow y \rightarrow z$, such that
$P(y|v,z) = P(y|v)$.
Assume further that the observations $x$ are functions
of both spurious and valuable features: $x := f(v, z)$.
The aim of invariant learning is to form a classifier that predicts $y$ from $x$ that focuses solely on the causal features, i.e., is invariant to $z$ and focuses solely on $v$.

Consider a classifier that produces a score $S(x)$ for example $x$.
In the binary classification setting $S$ is analogous to the model $\Phi$, while the score $S(x)$ is analogous to the representation $\Phi(x)$.
To quantify the degree to which the constraint in the Invariant Principle (\ref{eq:irm}) holds, we introduce a measure called the \emph{group sufficiency gap}\footnote{
This was previously used in a fairness setting by \citet{liu2018implicit} to measure differing calibration curves across groups.
}:
\[\Delta(S,e) = 
\E[
\E[(y|S(x),e_1)] - \E[(y|S(x),e_2)]
]
\]

Now consider the notion of an environment: some setting in which the $x \rightarrow y$ relationship varies (based on spurious features).
Assume a single binary spurious feature $z$.
We restate Proposition 1 as follows:

Claim: If environments are defined based on the agreement of the spurious feature $z$ and the label $y$, then a classifier that predicts based on $z$ alone maximizes the group-sufficiency gap (and vice versa -- if a classifier predicts $y$ directly by predicting $z$, then defining two environments based on agreement of label and spurious feature---$e_1 = \{v,z,y|\indicator{y = z}\}$ and $e_2 = \{v,z,y|\indicator{y \neq z}\}$---maximizes the gap).

We can show this by first noting that if the environment is based on spurious feature-label agreement, then with $e \in \{0, 1\}$ we have $e = \indicator{y=z}$.
If the classifier predicts $z$, i.e. $S(x) = z$, then we have
\[
 \Delta(S,e) 
 = 
 \E[
 \E[y|z(x),\indicator{y=z}] - \E[y|z(x),\indicator{y\neq z}]
 ]
\]

For each instance of $x$ either $z=0$ or $z=1$.
Now we note that when $z=1$ we have ${\E(y|z,\indicator{y=z}) = 1}$ and ${\E(y|z,\indicator{y\neq z}) = 0}$, while when $z=0$  ${E(y|z,\indicator{y=z}] = 0}$ and \\
${\E[y|z,\indicator{y \neq z}] = 1}$.
Therefore for each example ${|E(y|z(x),\indicator{y=z}) - E[y|z(x),\indicator{y\neq z}| = 1]}$, contributing to an overall $\Delta(S,e)=1$, which is the maximum value for the sufficiency gap.

\subsection{
Analysis of the Binning Heuristic
}\label{sec:majmin}
Here we analyze the heuristic discussed in Section \ref{sec:methods}.
We want to show that finding environment assignments in this way both maximizes the violation of the softened version of the regularizer (Equation \ref{eq:ei-obj}), and also also maximally violates the invariance principle (\ref{eq:irm}).

Because the invariance principle $\E[Y|\Phi(X), e] = \E[Y|\Phi(X), e'] \forall e, e'$ is difficult to quantify for continuous $\Phi(X)$, we consider a binned version of the representation, with $b$ denoting the discrete index of the bin in representation space.
Let $q_i \in [0, 1]$ denote the soft assignment of example $i$ to environment 1, and $1 - q_i$ denote its converse, the assignment of example $i$ to environment 2.
Denote by $y_i \in \{0, 1\}$ the binary target for example $i$, and $\hat y \in [0, 1]$ as the model prediction on this example.
Assume that $\ell$ represents a cross entropy or squared error loss so that $\nabla_w \ell(\hat y, y) = (\hat y - y)\Phi(x)$.

Consider the IRMv1 regularizer with soft assignment, expressed as
\begin{align}
    D(q) &= \sum_e ||\nabla_{w|w=1.0} \frac{1}{N_e} \sum_i q_i(e) \ell (w \circ \Phi(x_i), y_i)||^2 \nonumber \\
    &= \sum_e || \frac{1}{N_e} \sum_i q_i(e)(\hat y_i - y_i)\Phi(x_i)||^2 \nonumber \\
    &= ||\frac{1}{\sum_i' q_i'} \sum_{i} q_i(\hat y_i - y_i) \Phi(x_i)||^2 + ||\frac{1}{\sum_i' 1 - q_i'} \sum_{i} (1 - q_i)(\hat y_i - y_i) \Phi(x_i)||^2 \nonumber \\
    &= ||\frac{\sum_i q_i \hat y_i \Phi(x_i)}{\sum_{i'} q_{i'}}
    - \frac{\sum_i q_i y_i \Phi(x_i)}{\sum_{i'} q_{i'}}
    ||^2 + 
    ||\frac{\sum_i (1 - q_i) \hat y_i \Phi(x_i)}{\sum_{i'} 1 - q_{i'}}
    - \frac{\sum_i (1 - q_i) y_i \Phi(x_i)}{\sum_{i'} 1 - q_{i'}}
    ||^2.
\end{align}

Now consider that the space of $\Phi(X)$ is discretized into disjoint bins $b$ over its support, using $z_{i,b} \in \{0, 1\}$ to indicate whether example $i$ falls into bin $b$ according to its mapping $\Phi(x_i)$.
Thus we have 

\begin{align}
    D(q) = &\sum_b (||\frac{\sum_i z_{i,b} q_i \hat y_i \Phi(x_i)}{\sum_{i'} z_{i,b} q_{i'}}
    - \frac{\sum_i z_{i,b} q_i y_i \Phi(x_i)}{\sum_{i'} z_{i,b} q_{i'}}
    ||^2 \nonumber\\
    & + 
    ||\frac{\sum_i z_{i,b} (1 - q_i) \hat y_i \Phi(x_i)}{\sum_{i'} z_{i,b} (1 - q_{i'})}
    - \frac{\sum_i z_{i,b} (1 - q_i) y_i \Phi(x_i)}{\sum_{i'} z_{i,b} (1 - q_{i'})}
    ||^2) 
\end{align}

The important point is that within a bin, all examples have roughly the same $\Phi(x_i)$ value, and the same value for $\hat y_i$ as well.
So denoting $K_b^{(1)} := \frac{\sum_i z_{i,b} q_i \hat y_i \Phi(x_i)}{\sum_{i'} z_{i,b} q_{i'}}$ and $K_b^{(2)} := \frac{\sum_i z_{i,b} (1 - q_i) \hat y_i \Phi(x_i)}{\sum_{i'} z_{i,b} (1 - q_{i'})}$ as the relevant constant within-bin summations, we have the following objective to be maximized by EIIL: 
\[
D(q) = \sum_b( ||K_b^{(1)} -  \frac{\sum_i z_{i,b} q_i y_i \Phi(x_i)}{\sum_{i'} z_{i,b} q_{i'}}||^2 + ||K_b^{(2)} - \frac{\sum_i z_{i,b} (1 - q_i) y_i \Phi(x_i)}{\sum_{i'} z_{i,b} (1 - q_{i'})}||^2.
\]

One way to maximize this is to assign all $y_i=1$ values to environment $1$ ($q_i = 1$ for these examples) and all $y_i=0$ to the other environment $(q_i = 0)$.
We can show this is maximized by considering all of the examples except the $i$-th one have been assigned this way, and then that the loss is maximized by assigning the $i$-th example according to this rule.

Now we want to show that the same assignment maximally violates the invariance principle (showing that this soft EIIL solution provides maximal non-invariance). Intuitively within each bin the difference between $\E[y|e=1]$ and
$\E[y|e=2]$ is maximized (within the bin) if one of these expected label distributions is $1$ while the other is $0$. This can be achieved by assigning all the $y_i=1$ values to the first environment and the $y_i=0$ values to the second.

Thus a global optimum for the relaxed version of EIIL (using the IRMv1 regularizer) also maximally violates the invariance principle.

\section{Dataset Details}\label{sec:dataset_details}
\paragraph{CMNIST}
This dataset was provided by \citet{arjovsky2019invariant}\footnote{\url{https://github.com/facebookresearch/InvariantRiskMinimization}}.
The two training environments comprise $25,000$ images each, with 
$Corr(color, label) = 0.8$ for the firs training environment and $Corr(color, label) = 0.8$ for the second.
A held-out test set with $Corr(color, label)=0.1$ is used for evaluation.
Label noise is applied by flipping the binary target $y$ with probability $\lblnoise = 0.25$, with color correlation applied w.r.t. the noisy label.
Given that only two color channels are used, we follow \citet{arjovsky2019invariant} in downsampling the digit images to $14 \times 14$ pixels and $2$ channels.

\paragraph{Waterbirds}
We follow the procedure outlined by \citet{sagawa2019distributionally} to reproduce the Waterbirds dataset. 
As noted by the authors, due to random seed differences our version of the dataset may differ slightly from the one originally used by the paper.
The train/validation/test splits are of size $4,795$/$1,200$/$5,794$.
As noted in the Appendix of \citep{sagawa2019distributionally}, the validation and test distributions represent upweight the minority groups so that the number of examples coming from each habitat is equal (although there are still marginally more landbirds than waterbirds). For example on train set the subgroup sizes are $3,498$/$184$/$56$/$1,057$ while on the test set the sizes are $467$/$466$/$133$/$133$.

\paragraph{CivilComments-WILDS}
We use the train/validation/test splits from \citet{koh2020wilds}; we refer the interested reader the Appendix of their paper for a detailed description of this version of the dataset, including how it differs from the original dataset \citep{borkan2019nuanced}.

\section{Experimental Details}\label{sec:experimental_details}
\paragraph{Model selection}
\citet{krueger2020out} discussed the pitfalls of achieving good test performance on CMNIST by using test data to tune hyperparameters.
Because our primary interest is in the properties of the inferred environment rather than the final test performance, we sidestep this issue in the Synthetic Regression and CMNIST experiments by using the default parameters of IRM without further tuning.

We refer the interested reader to \citet{gulrajani2020search} for an extensive discussion of possible model selection strategies. They also provide a large empirical study showing that ERM is difficult baseline to beat when all methods are put on equal footing w.r.t. model selection.

In our case, we use the most relaxed model selection method proposed by \citet{gulrajani2020search}, which amounts to allowing each method a $20$ test evaluations using hyperparameter chosen at random from a reasonable range, with the best hyperparameter setting selected for each method.
While none of the methods is given an unfair advantage in the search over hyperparameters, the basic model selection premise does not translate to real-world applications, since information about the test-time distribution is required to select hyperparameters.
Thus these results can be understood as being overly optimistic for each method, although the relative ordering between the methods can still be compared.

\paragraph{Training times}

Because EIIL requires a pre-trained reference model and optimization of the EI objective, overall training time is longer than standard invariant learning.
It depends primarily on the number of steps used to train the reference model and number of steps used in EI optimization. 
The extra training time incurred is manageable and varies from dataset to dataset.

In CMNIST, we train the ERM reference model for $1,000$ steps, which is the same duration as the downstream invariant learner that eventually uses the inferred environments.
In this setting the $10,000$ steps required to optimize the EI objective is actually more than used for representation learning.
The overall EIIL train time is $6.6$ minutes to run 10 restarts on a NVIDIA Tesla P100, compared with $2.18$ minutes for ERM and $2.20$ minutes for IRM.

However, as the problem size scales, the relative overhead cost of EIIL becomes progressively discounted.
On Waterbirds, training GroupDRO takes $4.716$ hours on a NVIDIA Tesla P100.
Our reference model trains for $1$ epoch, so taking this into account along with the $20,000$ steps of EI optimization, EIIL runs at 4.737 hours.
This is a relative increase of 0.4\%.

\paragraph{Batch environment inference}
As mentioned in Chapter~\ref{chap:eiil}, we aggregate logits for the entire training set and optimize the EI objective using the entire training batches.
This can be done by cycling through the train set once in minibatches, computing logits per minibatch, and aggregating the logits only (discarding network activations) prior to EI.
We leave minibatched environment inference and amortization of the soft environment assignments to future work.

\paragraph{Experimental infrastructure}
Our experiments were run on a cluster of NVIDIA Tesla P100 machines.

\paragraph{CMNIST}
IRM is trained on the two training environments and tested on a holdout environment constructed from $10,000$ test images in the same way as the training environments, where color is predictive of the noisy label 10\% of the time. So using color as a feature to predict the label will lead to an accuracy of roughly 10\% on the test environment, while it yields 80\% and 90\% accuracy respectively on the training environments. 

To evaluate EIIL we remove the environment identifier from the training set and thus have one training set comprised of $50,000$ images from both original training environments. We then train an MLP with binary cross-entropy loss on the training environments, freeze its weights and use the obtained model to learn environment splits that maximally violate the IRM penalty.
When optimizing the inner loop of EIIL, we use Adam with learning rate 0.001 for $10,000$ steps with full data batches used to computed gradients.

The obtained environment partitions are then used to train a new model from scratch with IRM.
Following \citet{arjovsky2019invariant}, we allow the representation to train for several hundred annealing steps before applying the IRMv1 penalty.

We used the default architecture---an MLP with two hidden layers of $390$ neurons---and hyperparameter values\footnote{\url{https://github.com/facebookresearch/InvariantRiskMinimization/blob/master/code/colored_mnist/reproduce_paper_results.sh}}----learning rate, weight decay, and penalty strength---from \citep{arjovsky2019invariant}.
We do not use minibatches as the entire dataset fits into memory.

\paragraph{Waterbirds}
Following \citet{sagawa2019distributionally}, we use the default \texttt{torchvision} ResNet50 models, using the pre-trained weights as the initial model parameters, and train without any data augmentation using the 
For GroupDRO and ERM, we use hyperparameters reported by the authors\footnote{\url{https://worksheets.codalab.org/worksheets/0x621811fe446b49bb818293bae2ef88c0}}, and note that the authors make use of the the validation set (whose distribution contains less group imbalance than the training data), to select hyperparameters in their experiments (all methods benefit equally from this strategy).
We train for $300$ epochs without any early stopping (to avoid any further influence from the validation data).
For EIIL, we optimize the EI objective of EIIL with learning rate $0.01$ for $20,000$ steps using the Adam optimizer, and use GroupDRO (using the same hyperparameters as the GroupDRO baseline) as the invariant learner.
An ERM model trained for $1$ epoch was used as the reference model.
We also tried using reference modeled trained for longer, but found that EIIL did not perform as well in this case.
We hypothesize that this is because the reference ERM model focuses on background features early in training, leading to stark performance discrepancies across subgroups, which in turn provides a strong learning signal for EIIL to infer effective environments.
While subgroup disparities are present for more well-trained models, the learning signal in the EI phase will weaken.

\paragraph{CivilComments-WILDS}
Following \citep{koh2020wilds}, we finetune DistilBERT embeddings \citep{sanh2019distilbert} using the default HuggingFace implementation and default weights \citep{wolf2019huggingface}.
EIIL uses an ERM reference classifier and its inferred environments are fed to a GroupDRO invariant learner.
During prototyping the EI step, we noticed that the binning heuristic described in Section \ref{sec:methods}
consistently split the training examples into environments according to the error cases of the reference classifier.
Because error splitting is even simpler to implement than confidence binning, we used this heuristic for the EI step; we believe this is a promising approach for scaling EI to large datasets, and note its equivalence to the first stage of the method independently proposed by \citet{liu2021just}, which is published concurrently to ours.
We experimented with gradient-based EI on this dataset, but did not find any improvement over the (faster) heuristic EI.

On this dataset, we treat reference model selection as part of the overall model selection process, meaning that the hyperparameters of the ERM reference model are treated as a subset of the overall hyperparameters tuned during model selection.
Specifically we used a grid search to tune the reference model learning rate (1e-5, 1e-4), optimizer type (Adam, SGD) and scheduler (linear, plateau), and gradient norm clamping (off, clamped at 1.0), as well as the invariant learner (GroupDRO) learning rate (1e-5, 1e-4).
Moreover, we allow all methods to evaluate worst-group validation accuracy to tune these hyperparameters; such validation data will not be available in most settings, so this result can be seen as an optimistic view of the performance of all methods, including EIIL.
We train all methods (including the reference model) for $5$ epochs, with the best epoch chosen according to validation performance.
Interestingly, the reference model chosen in this way was a constant classifier, so the overall EIIL solution is equivalent to GroupDRO using the class label as the environment label.

The oracle GroupDRO method trains on two environments, with one containing comments where \emph{any} of the $8$ sensitive groups was mentioned, and other environment containing the remaining comments.
We experimented with allowing the oracle method access to more fine-grained environment labels by evaluating all $2^8$ combinations of binary group labels, but did not find any significant performance boost (consistent with observations from \citet{koh2020wilds}).

\section{Additional Empirical Results}\label{sec:additional_results}

\subsection{ColorMNIST}\label{sec:cmnist_extra_results}

\begin{table}[ht]
\centering
\begin{tabular}{lll}
\toprule
{}                                                  &     Train accs &      Test accs \\
\midrule
Grayscale (oracle)                                  & 75.3 $\pm$ 0.1 & 72.6 $\pm$ 0.6 \\
IRM (oracle envs)                                   & 71.1 $\pm$ 0.8 & 65.5 $\pm$ 2.3 \\
\midrule
ERM                                                 & 86.3 $\pm$ 0.1 & 13.8 $\pm$ 0.6 \\
EIIL                                                & 73.7 $\pm$ 0.5 & 68.4 $\pm$ 2.7 \\
Binned EI heuristic                                   & 73.9 $\pm$ 0.5 & 69.0 $\pm$ 1.5 \\
$\Phi_{Color}$                                      & 85.0 $\pm$ 0.1 & 10.1 $\pm$ 0.2 \\
EIIL$|\tilde \Phi = \Phi_{Color}$                   & 75.9 $\pm$ 0.4 & 68.0 $\pm$ 1.2 \\
ARL                                                 & 88.9 $\pm$ 0.2 & 20.7 $\pm$ 0.9 \\
GEORGE& 84.6 $\pm$ 0.3 & 12.8 $\pm$ 2.0 \\
LFF;$\mathcal{L}_{bias}=\text{GCE}_{q\rightarrow0}$ & 96.6 $\pm$ 1.3 & 30.6 $\pm$ 1.0\\
LFF;$\mathcal{L}_{bias}=\text{GCE}_{q=0.7}$         & 15.0 $\pm$ 0.1 & 90.0 $\pm$ 0.3\\
\bottomrule
\end{tabular}
\caption{
Additional baselines for the CMNIST experiment reported in Table \ref{tab:table_teaser}.
The mean and standard deviation of accuracy across ten runs ($\lblnoise=0.25$) are reported.
See text for description of the baseline methods.
}
\label{tab:cmnist_acc}
\end{table}
Table \ref{tab:cmnist_acc} expands on the results from Table \ref{tab:table_teaser} by adding the following baselines that do not require environment labels:
\begin{itemize}
    \item Grayscale: a classifier that removes color via pre-processing, which represents an oracle solution
    \item EIIL$|\tilde \Phi = \Phi_{ERM}$ (reported as EIIL in Table \ref{tab:table_teaser}
    \item Binned EI heuristic: the binning heuristic for environment inference described in Section \ref{sec:methods}.
    \item $\Phi_{Color}$: a hard-coded classifier that predicts \emph{only} based on the digit color
    \item EIIL$|\tilde \Phi = \Phi_{Color}$: EIIL using color-based classifier (rather than $\Phi_{ERM}$) as reference.
    \item GEORGE \citep{sohoni2020no}: This two-stage method seeks to learn the ``hidden subclasses'' by fitting a latent cluster model to the (per-class) distribution of logits of a reference model. The inferred hidden subclasses are fed to a GroupDRO learner, so this approach can be seen as an instance of EIIL under particular choices of (unsupervised) EI and (robust optimization) IL objectives.
    \item ARL \citep{lahoti2020fairness}: A variant of DRO that uses an adversary/auxiliary model to learn worst-case per-example importance weights. Unlike with EIIL, the auxiliary model and main model are trained jointly.
    \item LFF \citep{nam2020learning} jointly trains a ``biased'' model $f_B$ and ``debiased'' model $f_D$. $f_B$ is similar to our ERM reference model, but is trained with $\text{GCE}_q(p(x;\theta), y) = \frac{1-p_y(x;\theta)^q}{q}$ with hyperparameter $q \in (0, 1]$,\footnote{as $q \rightarrow 0$ GCE becomes standard cross entropy}
and its per-example losses determine importance weights for $f_D$.
\end{itemize}

When expanding this study we find that, unlike EIIL, the new baselines fail to find an invariant classifier that predicts based on shape rather than color.
Given that GEORGE does a type of unsupervised EI, it is perhaps surprising that it cannot uncover optimal environments for use with its GroupDRO learner. 
We hypothesize that this is due to assumption of the relevant latent environment labels being ``hidden subclasses'', meaning that all examples in an optimal environment must share the same class label value.
In the CMNIST dataset, this assumption does not hold due to label noise.

We find that, on this dataset, LFF is very sensitive to the hyperparameter $q$, which shapes the GCE loss of $f_B$.
Interestingly, using the default value of $q=0.7$, LFF performs optimally on the test set, but this is \emph{not} because the method has learned an invariant classifier based on the digit shape.
The below-chance train set performance reveals that LFF has learned an \emph{anti-color} classifier, exactly the opposite of what ERM does.
When $q$ approaches zero (GCE approaches standard cross entropy), LFF fails to generalizes to the OOD test distribution.

Finally, we found that because the reference classifier predicts with high confidence on the training set, there are only two populated bins in practice.
Consequentially, the binned EI heuristic is equivalent to splitting errors into one environment and correct predictions into the other.

  \chapter{Notes for Chapter \ref{chap:coda} }

\section{Proof of Proposition 1}\label{app:coda:proposition}

\begin{lemma}
If $V^j \in \parents^\L(V^i)$ in DAG $\G^\L$ of (local) causal model $\M^\L$, and $\L \subset \mathcal{X}$, then $V^j \in \parents^\mathcal{X}(V^i)$ in DAG $\G^\mathcal{X}$ corresponding to causal model $\M^\mathcal{X}$. 
\end{lemma}
\vspace{-\baselineskip}
\begin{proof}
By minimality, there exist $\{u^i, v^{-j}, v^j_1\}$ and $\{u^i, v^{-j}, v^j_2 \}$ with $v^{-j} \in \parents^\L(V^i)\setminus V^j$ for which $f^i(\{u^i, v^{-j}, v^j_1\}) \not= f^i(\{u^i, v^{-j}, v^j_2 \})$. Expand $\{v^{-j}, v^j_1\}$ and $\{v^{-j}, v^j_2\}$ to $(s_1, a_1), (s_2, a_2) \in \L$ (with any values of other components). But $\L \subset \mathcal{X}$, so $(s_1, a_1), (s_2, a_2) \in \mathcal{X}$ and it follows from minimality in $\mathcal{X}$ that $V^j \in \parents^\mathcal{X}(V^i)$.
\end{proof}

\begin{corollary}\label{corollary_sparser}
If $\L \subset \mathcal{X}$, $\G^\L$ is sparser (has fewer edges) than $\G^\mathcal{X}$. 
\end{corollary}\vspace{\baselineskip}

\begin{appdxProp}{1}
The causal mechanisms represented by $\G_i, \G_j \subset \G$ are independent in $\G^{\L_1 \cup \L_2}$ if and only if  $\G_i$ and $\G_j$ are independent in both $\G^{\L_1}$ and $\G^{\L_2}$, and $f^{\L_1,i} = f^{\L_2,i}, f^{\L_1,j} = f^{\L_2,j}$.
\end{appdxProp}
\vspace{-\baselineskip}
\begin{proof}

($\Rightarrow$)
If $\G_i$ and $\G_j$ are independent in $\G^{\L_1 \cup \L_2}$, independence in $\G^{\L_1}$ and $\G^{\L_2}$ follows from Corollary \ref{corollary_sparser}. That $f^{\L_1,i} = f^{\L_2,i}$ (and $f^{\L_1,j} = f^{\L_2,j}$), on their shared domain, follows since each is a restriction of the same function $f^{\L_1 \cup \L_2,i}$ (or $f^{\L_1 \cup \L_2,j}$).

($\Leftarrow$)
Suppose $\G_i$ and $\G_j$ are independent in $\G^{\L_1}$ and $\G^{\L_2}$ but not $\G^{\L_1 \cup \L_2}$. By the definition of independence applied to $\G^{\L_1 \cup \L_2}$, we have that, without loss of generality, there is a $V_i \in \G_i, V_j \in \G_j$ with $V_j \in \parents^{\L_1 \cup \L_2}(V_i)$. Then, from the definition of minimality, it follows that there exist $(s_1, a_1), (s_2, a_2) \in \L_1 \cup \L_2$ that differ only in the value of $V_j$, and $u_i \in \textrm{range}(U_i)$ for which $f^i(s_1, a_1, u_i) \not= f^i(s_2, a_2, u_i)$. 

Clearly, if $(s_1, a_1)$ and $(s_2, a_2)$ are both in $\L_1$ (or $\L_2$), there will be an edge from $V_j$ to $V_i$ in $\G^{\L_1}$ (or $\G^{\L_2}$) and the claim follows by contradiction. Thus, the only interesting case is when, without loss of generality, $(s_1, a_1) \in \L_1$ and $(s_2, a_2) \in \L_2$. 
The key observation is that $(s_1, a_1)$ and $(s_2, a_2)$ differ only in the value of node $V_j \not\in \G_i$: since $\G_i$ is an independent causal mechanism in both $\G^{\L_1}$ and $\G^{\L_2}$ and the parents of $V_i$ take on the same values in each, we have that $f^i(s_1, a_1, u_i) = f^{i,\L_1}(s_1, a_1, u_i) = f^{i,\L_2}(s_2, a_2, u_i) = f^i(s_2, a_2, u_i)$ and the claim follows by contradiction. 
\end{proof}

\section{Additional Experiment Details}\label{app:coda:training_details}
This section provides training details for the experiments discussed in Section \ref{sec:coda:experiments} as well as some additional results. Code is available at \url{https://github.com/spitis/mrl} \cite{mrl}.

\subsection{Batch RL}\label{app:coda:batch_rl_details}
Here we detail the procedure used in the Batch RL experiments in the \texttt{Pong} environment.

\paragraph{Implementation}
Our experiment first builds the agent's dataset (consisting of real data, dyna data and/or CoDA data), then instantiates a TD3 agent by filling its replay buffer with the dataset. The replay buffer is always expanded to include the entire enlarged dataset (for the 5x CoDA ratio at 250,000 data size this means the buffer has 1.5E6 experiences). The agent is run for 500,000 optimization steps.

\paragraph{Hyperparameters} We used similar hyperparameters to the original TD3 codebase, with the following differences:
\begin{itemize}
\item We use a discount factor of $\gamma = 0.98$ instead of 0.99. 
\item Since Pong is a sparse reward task with $\gamma = 0.98$, we clip critic targets to $(-50, 50)$.
\item We use networks of size $(128, 128)$ instead of $(256, 256)$. 
\item We use a batch size of $1000$. 
\end{itemize}

\paragraph{Environment} We base our \texttt{Pong} environment on \texttt{RoboSchoolPong-v1} \cite{klimov2017roboschool}. The original environment allowed the ball to teleport back to the center after one of the players scored, offered a small dense reward signal for hitting the ball, and included a stray ``timeout'' feature in the agent's state representation. We fix the environment so that the ball does not teleport, and instead have the episode reset every 150 steps, and also 10 steps after either player scores. The environment is treated as continuous and never returns a done signal that is not also accompanied by a \texttt{TimeLimit.truncated} indicator \cite{brockman2016openai}. We change the reward to be strictly sparse, with reward of $\pm 1$ given when the ball is behind one of the players' paddles. Finally, we drop the stray ``timeout'' feature, so that the state space is 12-dimensional, where each set of 4 dimensions is the x-position, y-position, x-velocity, and y-velocity of the corresponding object (2 paddles and one ball). 

The opponent in the environment uses a small pre-trained policy that was included in the original environment \cite{klimov2017roboschool} (the effectiveness of this policy was unaffected by the removal of the timeout feature). To collect the batch dataset, we use the same policy as an ``expert'' policy, but replace 50\% of its actions by random actions.

\paragraph{Training the CoDA model}
Without access to a ground truth mask, we needed to train a masking function \ref{app:coda:inferring_local_factorization} to identify local disentanglement. We also forewent the ground truth reward, instead training our own reward classifier. In each case we used the batch dataset given, and so we trained different models for each random seed. For our masking model, we stacked two single-head transformer blocks (without positional encodings) and used the product of their attention masks as the mask. 
Each block consists of query $Q$, key $K$, and value $V$ networks that each have 3-layers of 256 neurons each, with the attention computed as usual \cite{vaswani2017attention,lee2018set}. 
The transformer is trained to minimize the L2 error of the next state prediction given the current state and action as inputs. 
The input is disentangled, and so has shape \texttt{(batch\_size, num\_components, num\_features)}. In each row (component representation) of each sample, features corresponding to other components are set to zero. 
The transformer is trained for 2000 steps with a batch size of 256, Adam optimizer \cite{kingma2014adam} with learning rate of 3e-4 and weight decay of 1e-5. 
For our reward function we use a fully-connected neural network with 1 hidden layer of 128 units. 
The reward network accepts an $(s, a, s')$ tuple as input (not disentangled) and outputs a softmax over the possible reward values of $[-1, 0, 1]$. 
It is trained for 2000 steps with a batch size of 512, Adam optimizer \cite{kingma2014adam} with learning rate of 1e-3 and weight decay of 1e-4. 
All hyperparameters were rather arbitrary (we used the default setting, or in case it did not work, the first setting that gave reasonable results, as was determined by inspection). 
To ensure that our model and reward functions are trained appropriately (i.e., do not diverge) for each seed, we confirm that the average loss of the CoDA model is below 0.005 at the training and that the average loss of the reward model is below 0.1, which values were found by inspection of a prototype run. These conditions were met by all seeds.

When used to produce masks, we chose a threshold of $\tau = 0.02$ by inspection, which seemed to produce reasonable results. A more principled approach would do cross-validation on the available data.

\paragraph{Tested configurations}

Table \ref{fig_batch_rl_and_fetch} reports a subset of our tested configurations. We report our results in full below. We considered the following configurations:
\begin{enumerate}
    \item \textbf{Real data only}.
    \item \textbf{CoDA + real data}: after training the CoDA model, we expand the base dataset by either 2, 3, 4 or 6 times.
    \item \textbf{Dyna (using CoDA model)}: after training the CoDA model, we use it as a forward dynamics model instead of for CoDA; we use 1-step rollouts with random actions from random states in the given dataset to expand the dataset by 2x. We also tried with 5-step rollouts, but found that this further hurt performance (not shown).  Note that Dyna results use only 5 seeds.
    \item \textbf{Dyna (using MBPO model)}: as the CoDA model exhibits significant model bias when used as a forward dynamics model, we replicate the state-of-the-art model-based architecture used by MBPO \cite{janner2019trust} and use it as a forward dynamics model for Dyna; we experimented with 1-step and 5-step rollouts with random actions from random states in the given dataset to expand the dataset by 2x. This time we found that the 5-step rollouts do better, which we attribute to the lower model bias together with the ability to create a more diverse dataset (1-step not shown). The MBPO model is described below. We use the same reward model as CoDA to relabel rewards for the MBPO model, which only predicts next state. 
    \item \textbf{MBPO + CoDA}: as MBPO improved performance over the baseline (real data only) at lower dataset sizes, we considered using MBPO together with CoDA. We use the base dataset to train the MBPO, CoDA, and reward models, as described above. We then use the MBPO model to expand the base dataset by 2x, as described above. We then use the CoDA model to expand \textit{the expanded dataset} by 3x the original dataset size. Thus the final dataset is 5x as large as the original dataset (1 real : 1 MBPO : 3 CoDA). 
\end{enumerate}

All configurations alter only the training dataset, and the same agent architecture/hyperparameters (reported above) are used in each case. 

\paragraph{MBPO model}

Since using the CoDA model for Dyna harms rather than helps, we consider using a stronger, state-of-the-art model-based approach. In particular, we adopt the model used by Model-Based Policy Optimization \cite{janner2019trust}. This model consists of a size 7 ensemble of neural networks, each with 4 layers of 200 neurons. We use ReLU activations, Adam optimizer \cite{kingma2014adam} with weight decay of 5e-5, and have each network output a the mean and (log) diagonal covariance of a multi-variate Gaussian. We train the networks with a maximum likelihood loss. To sample from the model, we choose an ensemble member uniformly at random and sample from its output distribution, as in \cite{janner2019trust}. 

\subsection{Goal-conditioned RL}

Here we detail the procedure used in the Goal-conditioned RL experiments on the \texttt{Fetchpush-v1} and \texttt{Slide2} environments.

\paragraph{Hyperparameters}
For \texttt{Fetchpush-v1} we use the default hyperparameters from the codebase (see \cite{pitis2020mega} for details on how they were selected), which outperform the original HER agents of \cite{andrychowicz2017hindsight,plappert2018multi} and follow-up works. We do not tune the CoDA agent (but see additional CoDA hyperparameters below). They are as follows:
\begin{itemize}
\item Off-policy algorithm: DDPG \cite{lillicrap2015continuous}
\item Hindsight relabeling strategy:  \texttt{futureactual\_2\_2} \cite{pitisprotoge}, using exclusively \texttt{future} \cite{andrychowicz2017hindsight} relabeling for the first 25,000 steps
\item Optimizer: Adam \cite{kingma2014adam} with default hyperparameters
\item Batch size: 2000
\item Optimization frequency: 1 optimization step every 2 environment steps after the 5000th environment step
\item Target network updates: update every 10 optimization steps with a Polyak averaging coefficient of 0.05
\item Discount factor: 0.98
\item Action l2 regularization: 0.01
\item Networks: 3x512 layer-normalized \cite{ba2016layer} hidden layers with ReLU activations
\item Target clipping: (-50, 0)
\item Action noise: 0.1 Gaussian noise
\item Epsilon exploration \cite{plappert2018multi}: 0.2, with an initial 100\% exploration period of 10,000 steps
\item Observation normalization: yes
\item Buffer size: 1M
\end{itemize}

On \texttt{Slide2} we tried to tune the baseline hyperparameters somewhat, but note that this is a fairly long experiment (10M timesteps) and so only a few settings were tested due to constraints. In particular, we considered the following modifications:
\begin{itemize}
\item Expanding the replay buffer to 2M (effective)
\item Reducing the batch size to 1000 (effective)* (used for results)
\item Using the \texttt{future\_4} strategy (agent fails to learn in 10M steps)
\item Reducing optimization step frequency to 1 step every 4 environment steps (about the same performance)
\end{itemize}

We tried similar adjustments to our CoDA agent, but found the default hyperparameters (used for results) performed well. We found that the CoDA agent outperforms the base HER agent on all tested settings.

For CoDA, we used the following additional  hyperparameters:
\begin{itemize}
\item CoDA buffer size: 3M
\item Make CoDA data every: 250 environment steps
\item Number of source pairs from replay buffer used to make CoDA data: 2000
\item Number of CoDA samples per source pair: 2
\item Maximum ratio of CoDA:Real data to train on: 3:1
\end{itemize}
\paragraph{Environment} 

\begin{wrapfigure}{r}{0.23\textwidth}
    \vspace{-\baselineskip}
	\centering
    \captionsetup{width=0.22\textwidth}
	\includegraphics[width=0.22\textwidth,height=0.2\textwidth]{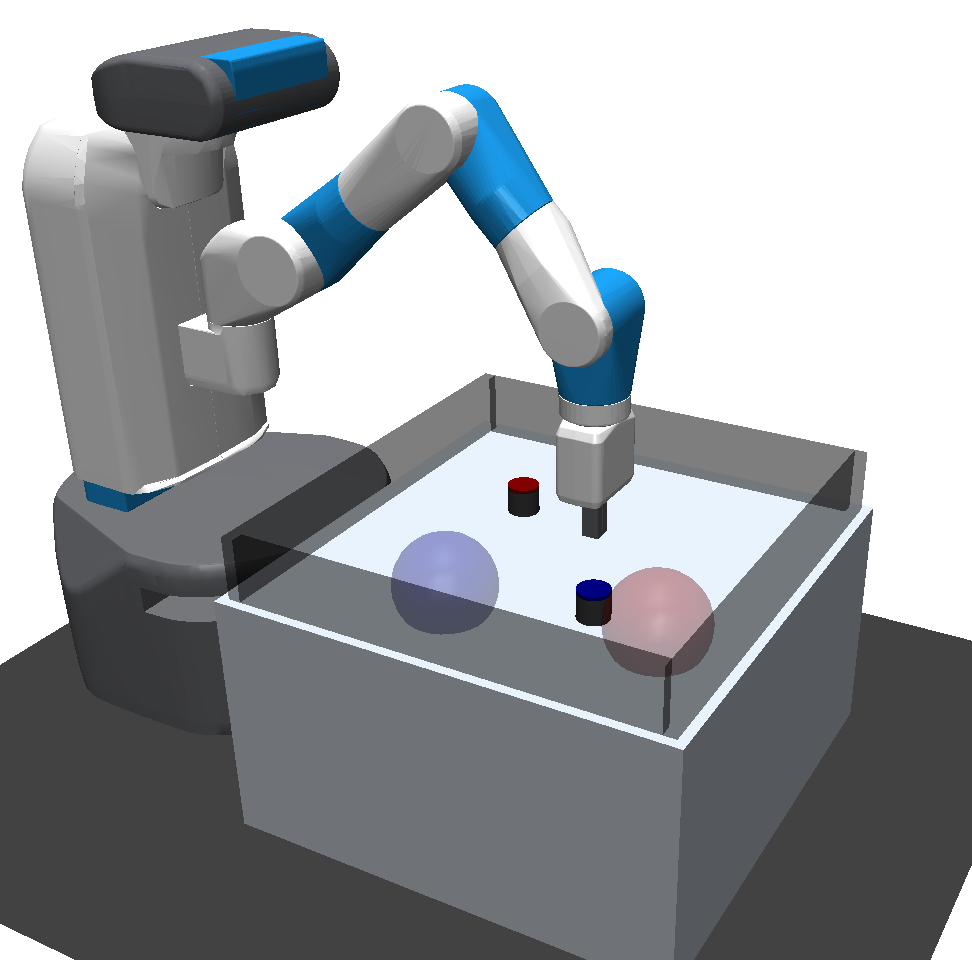}
	\caption{The \texttt{Slide2} environment.}
	\label{fig_env_fetch_appendix}
    \vspace{-\baselineskip}
\end{wrapfigure}

On \texttt{FetchPush-v1} the standard state features include the relative position of the object and gripper, which entangles the two. While this could be dealt with by dynamic relabeling (as used for HER's reward), we simply drop the corresponding features from the state. 

\texttt{Slide2} has two pucks that slide on a table and bounce off of a solid railing. Observations are 40-dimensional (including the 6-dimensional goal), and actions are 4-dimensional. Initial positions and goal positions are sampled randomly on the table. During training, the agent gets a sparse reward of 0 (otherwise -1) if \textit{both} pucks are within 5cm of their ordered target. At test time we count success as having both picks within 7.5cm of the target on the last step of the episode. Episodes last 75 steps and there is no done signal (this is intended as a continuous task). 

\paragraph{CoDA Heuristic} For these experiments we use a hand-coded heuristic designed with domain knowledge. In particular, we assert that the action is always entangled with the gripper, and that gripper/action and objects (pucks or blocks) are disentangled whenever they are more than 10cm apart. This encodes independence due to physical separation, which we hypothesize is a very generally heuristic that humans implicitly rely on all the time.  The pucks have a radius of 2.5cm and height of 4cm, and the blocks are 5cm x 5cm x 5cm, so this heuristic is quite generous / suboptimal. Despite being suboptimal, it demonstrates the ease with which domain knowledge can be injected via the CoDA mask: we need only a high precision (low false positive rate) heuristic---the recall is not as important. It is likely that an agent could learn a better mask that also takes into account velocity.

\paragraph{Results plot} The plot shows mean $\pm 1$ standard deviation of the smoothed data over 5 seeds.

\section{Proof of Theorem 1}\label{appdx_proposition}
The dynamics model assumed is a maximum-likelihood, count-based model that has separate parameters for each causal mechanism, $P_{i, \theta}^\L$, in each local neighborhood.
That is, for a given configuration of the parents $\parents_i = x$ in $P_{i, \theta}^\L$, we define count parameter $\theta_{ij}$ for the $j$-th possible child, $c_{ij}$, so that $P_{i, \theta}^\L(c_{ij}\,|\,x) = \theta_j/\sum_{k=1}^{|c_i|}\theta_k$.

We use the following two lemmas (see source material for proof):
\vspace{\baselineskip}

\begin{lemma}[Proposition A.8 of \citet{agarwal2019reinforcement}]\label{lemma_concentration_l1}
Let $z$ be a discrete random variable that takes values in $\{1, \dots, d\}$, distributed according to $q$.
We write $q$ as a vector where $\vec{q} = [\textrm{Pr}(z=j)]_{j=1}^d$.
Assume we have $n$ i.i.d. samples, and that our empirical estimate of $\vec{q}$ is $[\vec{q}]_j = \sum_{i=1}^n\mathbf{1}[z_i=j]/n$.\\
We have that $\forall\epsilon > 0$:
$$\textrm{Pr}(\Vert\hat{q} - \vec{q}\Vert_2 \geq 1 / \sqrt{n} + \epsilon) \leq e^{-n\epsilon^2}$$
which implies that:
$$\textrm{Pr}(\Vert\hat{q} - \vec{q}\Vert_1 \geq \sqrt{d}(1/\sqrt{n} + \epsilon)) \leq e^{-n\epsilon^2}$$

\end{lemma}
\vspace{\baselineskip}

\begin{lemma}[Corollary 1 of \citet{strehl2007model}]\label{lemma_aggregate_l1}
If for all states and actions, each model $P_{i,\theta}$ of $P_i$ is $\epsilon/k$ close to the ground truth in terms of the $\lone$ norm: $\Vert P_i(s, a) - P_{i,\theta}(s, a) \Vert_1 < \epsilon/k$, then the aggregate transition model $P_\theta$ is $\epsilon$ close to the ground truth transition model: $\Vert P(s, a) - P_{\theta}(s, a) \Vert_1 < \epsilon$.
\end{lemma}
\vspace{\baselineskip}

\begin{appdxTheorem}{1}
Let $n$ be the number of empirical samples used to train the model of each local causal mechanism $P_{i, \theta}^\L$ at each configuration of parents $\parents_i\! =\! x$.
There exists positive constant $c$ such that, if
$$
n \geq \frac{ck^2|c_i|\log(|\S||\A|/\delta)}{\epsilon^2},
$$
then, with probability at least $1-\delta$, we have:
$$\max_{(s, a)} \Vert P(s, a) - P_{\theta}(s, a) \Vert_1 \leq  \epsilon.$$
\end{appdxTheorem}

\begin{proof}
Applying Lemma \ref{lemma_concentration_l1}, we have that for fixed parents $\parents_i = x$, wp. at least $1-\delta$,

$$\Vert P_i(x) - P_{i,\theta}(x) \Vert_1 \leq c\sqrt{\frac{|c_i|\log(1/\delta)}{n}},$$

where $n$ is the number of independent samples used to train $P_{i,\theta}$ and $c$ is a positive constant.
Now consider a fixed $(s,a)$, consisting of $k$ parent sets.
Applying Lemma \ref{lemma_aggregate_l1} we have that, wp. at least $1\!-\!\delta$, 

$$\Vert P(s,a) - P_{\theta}(s,a) \Vert_1 \leq ck\sqrt{\frac{|c_i|\log(1/\delta)}{n}}.$$

We apply the union bound across all states and actions to get that wp. at least $1-\delta$,

$$\max_{(s, a)} \Vert P(s, a) - P_{\theta}(s, a) \Vert_1 \leq ck\sqrt{\frac{|c_i|\log(|S||A|/ \delta)}{n}}.$$

The result follows by rearranging for $n$ and relabeling $c$.
\end{proof}

To compare to full-state dynamics modeling, we can translate the sample complexity from the per-parent count $n$ to a total count $N$.
Recall $m\Pi_i|c_i| = |\S|$, so that $|c_i| = (|\S|/m)^{1/k}$, and $m\Pi_i|\parents_i| \geq |\S||\A|$.
We assume a small constant overlap factor $v \geq 1$, so that $|\parents_i| = v(|\S||\A|/m)^{1/k}$.
We need the total number of component visits to be $n|\parents_i|km$, for a total of $nv(|\S||\A|/m)^{1/k}m$ state-action visits, assuming that parent set visits are allocated evenly, and noting that each state-action visit provides $k$ parent set visits.
This gives:

\begin{appdxCorollary}{1}
To bound the error as above, we need to have
$$N \geq \frac{cmk^2(|\S|^2|\A|/m^2)^{1/k}\log(|\S||\A|/\delta)}{\epsilon^2},$$
total train samples, where we have absorbed the overlap factor $v$ into constant $c$.
\end{appdxCorollary}

To extend this and adapt other results to our setting, we could now apply the Simulation Lemma \cite{agarwal2019reinforcement} to bound the value difference given the model error, or alternatively, develop the theory in the direction of \cite{strehl2007model} and related work.
However, we believe the core insights are already contained in Theorem 1 and Corollary 1. 

\section{Implementation Details}\label{app:coda:mocoda-implementation-details}

Code is available at https://github.com/spitis/mocoda (using https://github.com/spitis/mrl for RL algorithms).
There are numerous components involved that each have several different settings that were mostly just taken ``as-is'' or picked as reasonable defaults (e.g., using a layer size of 512 in most neural networks, or having 5 components in the MDN, or the specific implementation of rejection sampling for \texttt{Mocoda-U}).
The best documentation for specific details is the code itself.
As such, the implementation details below cover the broad strokes so that a reader might understand the general pipeline, and we refer the reader to the provided code for precise details.

\subsection{Causal Transition Structure and Parent Set Definitions}

We implement the local causal model as a mask function $M$ that maps (state, action) tuples to an adjacency matrix of the causal structure.
For example, in \texttt{2d Navigation}, the mask function was implemented as follows:

{\footnotesize
\begin{verbatim}
  def Mask2dNavigation(input_tensor):
    """
    accepts B x num_sa_features, and returns B x num_parents x num_children
    """

    # base local mask
    mask = torch.tensor(
      [[1, 0],
       [0, 1],
       [1, 0],
       [0, 1]]).to(input_tensor.device)

    # change local mask in top right quadrant
    mask = mask[None].repeat((input_tensor.shape[0], 1, 1))
    mask[torch.logical_and(input_tensor[:,0] > 0.5, input_tensor[:, 1] > 0.5)] = 1
    
    return mask
\end{verbatim}
}
As an example, the causal graph for the base local mask, which applies for most of the state space is shown in the figure ~\ref{fig:cg-local-mask}.
We used the base local graph to select the parent sets, in this case, $(x, \Delta x)$ and $(y, \Delta y)$. 

\begin{figure}[!h]
	\centering
	\includegraphics[width=0.2\textwidth]{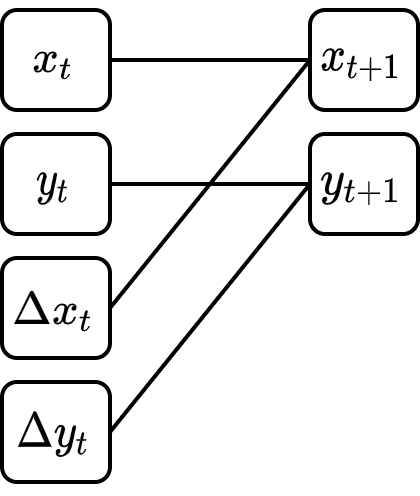}
	\caption{Causal graph for local mask}
	\label{fig:cg-local-mask}
\end{figure}

\subsection{Parent Distribution}\label{appdx_parentdist}

To sample the parent distribution in Step 1 of \methodName, we use the Gaussian Mixture Model (GMM) based approach described in the main text.
The advantage of this approach is that we can easily do conditional sampling in case of overlapping parent sets.
For a given local subset $\L$, we fit a separate GMM to the marginal of each parent set, as it appears in the empirical distribution for $\L$.
To generate a new sample, we optionally shuffle the GMMs, and then sample from one GMM at a time, conditioning on any already generated features.
This process eliminates any spurious correlations between features that are not part of the same parent set, and thus results in the maximum-entropy, marginal matching distribution.

\begin{figure}[!t]
	\centering
	\includegraphics[width=0.8\textwidth]{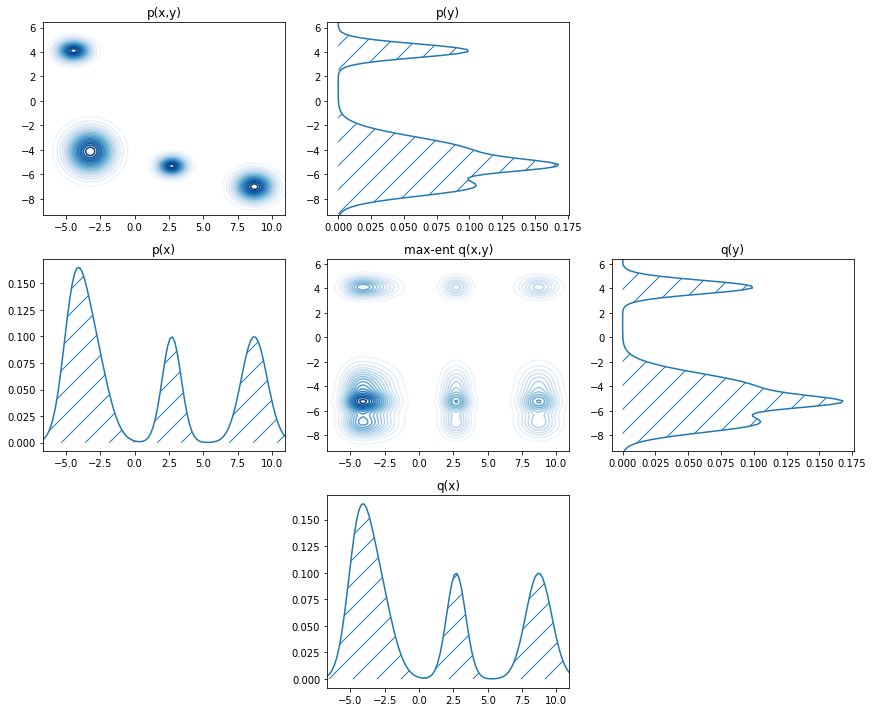}
	\caption{Hypothetical 2D illustration of the GMM-based parent set sampler.
  It is assumed that there are two non-overlapping parent sets $\{x\}$ and $\{y\}$, but that $x$ and $y$ exhibit dependence in the empirical data.
  We fit a GMM to each marginal $P(x)$ and $P(y)$ and sample from them independently to get $Q(x, y)$, which has the same marginal distributions (so that the components in a locally factored dynamics model will generalize), but eliminates the spurious dependence in the empirical data.}
\end{figure}

In cases of multiple local neighborhoods, $\L_1, \L_2, \dots$, one should respect the boundaries of the current local subset $\L$ during both training \textit{and} generation.
If a sample generated with the GMM for $\L$ falls outside of $\L$, that sample should be rejected, as the local causal structure is no longer valid, and the generalization guarantee for the locally factored model no longer holds.

As the local factorization in our experiments is quite simple, we did not stratify the GMM generator, and instead used a single GMM generator for the sparsest local causal structure.
In the case of \texttt{2d Navigation} this did not generate any data that was out-of-distribution for the locally factored model components (as the agent's policy was consistent in all local neighborhoods).
In the case of \texttt{HookSweep2}, there was a bit of locally out-of-distribution data in the local subspace in which there is a block collision; however, most of this data is unreachable as it involves overlapping blocks, and we obtained strong results even with this shortcut.

\subsection{Dynamics Models}\label{sec:dynamics-models}

Our experiments used three different dynamics models.
In each case, we used an ensemble of 5 base models, described below.
The base models output a Gaussian mean and variance for each output variable and are trained independently via a negative log likelihood loss.
All models are trained using Adam Optimizer~\cite{kingma2014adam}.

\begin{enumerate}
    \item \textbf{Unfactored:} The base model is a fully connected neural network with ReLU activations (MLP).
    \item \textbf{Globally Factored:} The base model has one MLP for each causal mechanism in the sparsest local graph.
    For both \texttt{2d Navigation} and \texttt{HookSweep2} the sparest local graph has two components, so the base global model is composed of two MLPs.
    \item \textbf{Locally factored:} The base model is designed as follows.
    For each child node, $c_i$, there is a separately parameterized single MLP that is preceded by a ``Masked Composer'' module.
    The Masked Composer applies a single layer MLP (linear transform followed by ReLU) to each root node, $r_i$ (each parent set has several nodes), to obtain embeddings $\varepsilon_i(r_i)$.
    The $i$-th column of the mask is used to zero out the corresponding embeddings which are then summed, $\sum_i M_{ij}\varepsilon_i(r_i)$, and the result is passed as an input to the MLP.
    
    This architecture works (and enforces local factorization), but is likely poor, because it does not take advantage of potentially useful shared representations between parent nodes across children (since there is a separately parameterized Masked Composer for each child).
    A better architecture would likely use a single parameterization for a single, possibly deeper Masked Composer.
    As this is not the focus of our contribution, we stuck with simple model, as it ``just worked'' for purposes of our experiments. 
\end{enumerate}

\subsection{Training Data for the RL Algorithm}

This varied by experiment, and is described in the next Section.
Notably, we divided the standard deviation returned by our dynamics models by a factor of three when generating data to avoid data that was too far out of distribution.

\subsection{Reinforcement Learning Algorithms}

We use Modular RL \cite{mrl}, adding three offline RL \cite{levine2020offline} algorithms: 
{BCQ} \citep{fujimoto2019off}, {CQL} \citep{kumar2020conservative} \& {TD3-BC} \citep{fujimoto2021minimalist}.

The {BCQ} implementation uses DDPG \cite{lillicrap2015continuous}.
For the generative model we use a Mixture Density Network (MDN) \cite{bishop1994mixture} with 5 components, that produces 20 action samples at each call (both during test rollouts and when creating critic targets).
The MDN was trained for 1000 batches with batch size of 2000.
We did not use a perturbation model. 

The {CQL} implementation uses SAC \cite{haarnoja2018soft}.
Rewards in our environments are sparse, and so value targets can be accurately clipped between two values (depends on the discount factor).
CQL balances two losses: a penalty for Q-values of some non-behavioral distribution/policy (we use a random policy), and a bonus for the Q-values behavioral actions.
We use an L1 penalty toward the lower end of the value target clipping range, and an L1 bonus toward the higher end of the value target clipping range.
We then multiply that by a minimum Q coefficient, as in the original CQL implementation. 

The TD3-BC implementation follows \citet{fujimoto2021minimalist}. 

\section{Experimental Details}\label{exp_details}

\subsection{2D Navigation}  In this environment, the agent must travel from one point in a square arena to another.
States are 2D $(x, y)$ coordinates and actions are 2D $(\Delta x, \Delta y)$ vectors. 

{\footnotesize
\begin{verbatim}
    observation_space = spaces.Box(np.zeros((2,)), np.ones((2,)), dtype=np.float32)
    action_space = spaces.Box(-np.ones((2,)), np.ones((2,)), dtype=np.float32)
\end{verbatim}
}

Episodes run for up to 70 steps.
Rewards are sparse, with a -1 reward everywhere except the goal, where reward is 0.
In most of the state space, the sub-actions $\Delta x$ and $\Delta y$ affect only their respective coordinate.
In the top right quadrant, however, the $\Delta x$ and $\Delta y$ sub-actions each affect \textit{both} $x$ and $y$ coordinates, so that the environment is locally factored.
The two causal graphs are as follows:

\newpage 

\begin{figure}[!h]
	\centering
	\includegraphics[width=0.45\textwidth]{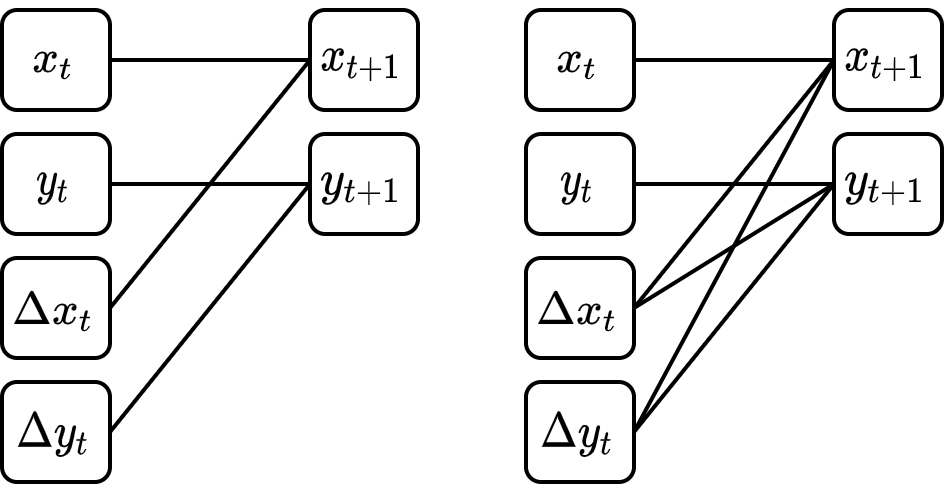}
\end{figure}

The graph on the right applies only in the top-right quadrant; otherwise the graph on the left applies.
The graph on the left has non-overlapping parent sets $(x, \Delta x)$ and $(y, \Delta y)$.
The graph on the right has overlapping parent sets $(x, \Delta x, \Delta y)$ and $(y, \Delta x, \Delta y)$.

The agent has access to an empirical dataset consisting of left-to-right \& bottom-to-top trajectories (20,000 transitions of each type):

\begin{figure}[!h]
	\centering
	\includegraphics[width=0.8\textwidth]{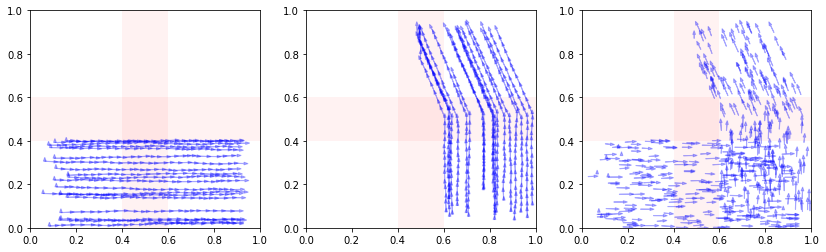}
	\caption{\textbf{Left/Middle:} Random samples of the two types of trajectories the agent has access to. \textbf{Right:} Random sample of transitions from this empirical dataset.}
\end{figure}

We consider a target task where the agent must move from the bottom left to the top right.
In this task there is sufficient empirical data to solve the task by following the ${\Large \textbf{$\lrcorner$} }$ shape of the data, but learning the optimal policy of going directly via the diagonal requires out-of-distribution generalization.

For \env{2d Navigation}, we generated the \ad{Mocoda} distribution by fitting a GMM generator as described in the previous Section.
Each GMM (one for each parent set) had 32 components, and was fit using expectation maximization.
To obtain \ad{Mocoda-U}, we implemented rejection sampling by using a KDE density estimator is as follows:

{\footnotesize
\begin{verbatim}
def prune_to_uniform(proposals, target_size=12000.):
  from sklearn.neighbors import KernelDensity
  sample = proposals[-10000:]

  fmap = lambda s: s[:, :2]
  K = KernelDensity(bandwidth=0.05)
  K.fit(fmap(sample))
  scores = K.score_samples(fmap(proposals))
  scores = np.maximum(scores, np.log(0.01))
  scores = (1. / np.exp(scores))
  scores = scores / scores.mean()  * (target_size / len(proposals))
  
  return proposals[np.random.uniform(size=scores.shape) < scores]   
\end{verbatim}
}

The dynamics models each had 2 layers of 256 neurons and were trained with a batch size of 512 and learning rate of 1e-4.
Hyperparameters were not tuned once a working setting was found.
Of the 40K empirical samples, 35K were used for training, and 5000 for validation.
The models were trained for 600 epochs, with early stopping used in the last 50 epochs to find a locally optimal stopping point. 

Augmented datasets of 200K samples were generated.
In each case except \ad{Emp}, 40K were the original empirical dataset (thus 160K new samples were generated by applying the dynamics model to samples from the augmented distribution).
In case of \ad{Emp}, the 40K original samples were simply repeated 5 times to get a size 200K dataset.
The locally factored network was used to generate the augmented datasets. 

These augmented distributions were then used to train the downstream RL agents.
The agent algorithms used a discount factor of 0.98, a target cutoff range of (-50, 0), batch size of 500, and used 2 layers of 512 neurons in both actor and critic networks.
The agents were trained for 25K batches (for a total of 62.5 passes over the dataset). 

For \env{2D Navigation} we ran 5 seeds, which all yielded similar results.
For each seed we trained new parent set samplers and generated new augmented datasets.

\subsection{HookSweep2}

\env{HookSweep2} is a challenging robotics domain based on Hook-Sweep~\citep{kurenkov2020ac}, in which a Fetch robot must use a long hook to sweep two boxes to one side of the table (either toward or away from the agent).
States, excluding the goal, are 16 dimensional continuous vectors.
Goals are 6 dimensions.
The agents all concatenate the goal to the state, and so operate on 22 dimensional states.
The action space is a 4 dimensional continuous vector. 

The environment contains two boxes that are initialized near the center of the table.

The empirical data contains 1M transitions from trajectories of an expert agent sweeping exactly one box to one side of the table, leaving the other in the center. The target task requires the agent to sweep \textit{both} boxes together to one side of the table.
This is particularly challenging because the setup is entirely offline (no exploration), where poor out-of-distribution generalization typically requires special offline RL algorithms that constrain the agent's policy to the empirical distribution~\citep{levine2020offline,agarwal2020optimistic,kumar2020conservative,fujimoto2021minimalist}.

Episodes run for 75 steps.
Rewards are dense, but structured similarly to a sparse reward, with a base reward of -1 everywhere except the goal and a reward of 0 at the goal.
Additional small rewards are given if the agent keeps the hook near the table (this was required to obtain natural movements from the trained expert agent).

In this environment, \textit{we did not have the ground truth causal graph}, and so a heuristic was used.
The heuristic (wrongly) assumes that the agent/hook \textit{always} causes each of the next object position (hook and objects are always entangled), even though this is only true when the hook and the objects are touching.
The heuristic considers the two boxes to be separate whenever they are further than 5cm from each other.
Here is the implementation of the heuristic:

{\scriptsize
\begin{verbatim}
  def MaskHookSweep2(input_tensor):
  
    # base local mask for when boxes are far apart
    mask = torch.tensor(
      [[1, 1, 1],
      [1, 1, 0],
      [1, 0, 1],
      [1, 1, 1]]
      ).to(input_tensor.device)
    mask = mask[None].repeat((input_tensor.shape[0], 1, 1))
    
    # change local mask when boxes are close to each other
    mask[torch.sum(torch.abs(input_tensor[:,O1X:O1X+2] -\ 
          input_tensor[:,O2X:O2X+2]), axis=1) < 0.05] = 1
    
    return mask
\end{verbatim}
}

where the state-action components are (gripper, box1, box2, action).
This heuristic returns the following two causal graphs (note that goals are not part of the dynamics, and are separately labeled using random goal samples from the environment):

\begin{figure}[!h]
	\centering
	\includegraphics[width=0.4\textwidth]{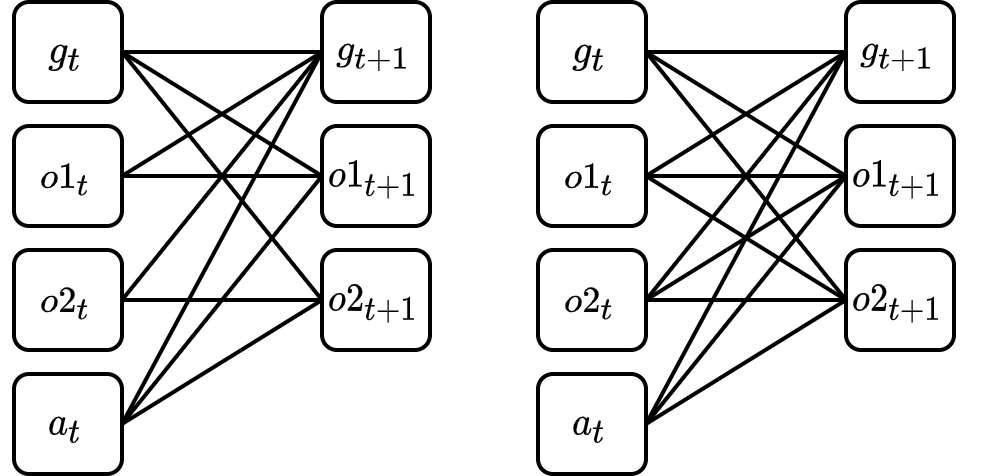}
\end{figure}

The parent sets are (g, o1, a) and (g, o2, a) for the first graph, and (g, o1, o2, a) in the second graph. 

For \env{HookSweep2}, the generation of the \ad{Mocoda} distribution is identical to how it was generated in \env{2d Navigation} (see previous subsection).
To obtain \ad{Mocoda-P}, we implemented rejection sampling as follows:

{\scriptsize
\begin{verbatim}
  def prune_to_uniform2(proposals, target_size=12000., smaller=True):
    proposals = proposals[np.linalg.norm(proposals[:,O1X:O1X+2] - proposals[:,O2X:O2X+2], axis=-1) < 0.3]
    sample = proposals[-5000:]
    
    fmap = lambda s: s[:,[O1X,O1X+1,O2X,O2X+1]]
    K = KernelDensity(bandwidth=0.05)
    K.fit(fmap(sample))
    scores = K.score_samples(fmap(proposals))
    scores = np.maximum(scores, np.log(0.05))
    scores = (1. / np.exp(scores))
    if np.minimum(scores, 1).sum() > 10000:
    while np.minimum(scores, 1).sum() > 10000:
      scores = scores * 0.99
    else:
    while np.minimum(scores, 1).sum() < 10000:
      scores = scores / 0.99
    
    return proposals[np.random.uniform(size=scores.shape) < scores]   
\end{verbatim}
}

The key difference to the \env{2d Navigation} is the definition of the \texttt{fmap} function, which defines the feature map under which the density is computed for rejection sampling.

The dynamics models for \env{HookSweep2} each had 2 layers of 512 neurons and were trained with a batch size of 512 and learning rate of 2e-4.
Hyperparamters were not tuned once a working setting was found (learning rate was increased to make training slightly faster).
Of 1M empirical samples, 5000 were used for validation.
The models were trained for 4000 epochs, where each epoch involved 40K random samples, with early stopping used in the last 50 epochs to find a locally optimal stopping point.

Augmented datasets of 5M samples were generated.
In each case except \ad{Emp}, 1M were the original empirical dataset (thus 4M new samples were generated).
In the case of \ad{Emp}, the 1M original samplers were simply repeated 5 times to get the full augmented dataset.
The locally factored network was used to generate the augmented datasets. 

These augmented distributions were then used to train the downstream RL agents.
The agent algorithms were the same as for \env{2d Navigation}, except that they used 3 layers of 512 neurons in both actor and critic networks.
The agents were trained for 1M steps with batch size 500 (for a total of 100 passes over the dataset). 

\section{\methodName sampling pseudocode}

% http://tug.ctan.org/macros/latex/contrib/algorithmicx/algorithmicx.pdf
%
\algnewcommand{\LeftComment}[1]{\Statex \(\triangleright\) #1}
\newcommand{\Pai}{\text{Pa}(i)}
\newcommand{\already}{\mathcal{AS}}
\begin{algorithm}\footnotesize
  \centering
  \begin{minipage}{1.\linewidth}
	\caption{\methodName for FMDPs, assuming hand-specified dynamics factorization}
  \label{algo:mocoda-fmdp}
	\begin{algorithmic}[1]
	\Function{GenerateMocodaData}{N}:
	\State \textbf{Input:} observed transition dataset $(s,a,s') \in \mathcal{D}$
	\State \textbf{Input:} causal structure of transition dynamics $\mathcal{G} := \{(i, \Pai) \ \forall \ i \in N_s\}$
	a.k.a. ``parent sets''
	\State \hphantom{\textbf{Input:}} NOTE: $\text{Pa}(i)$ is shorthand for $\{j : s_j \in \text{CausalParent}(s_i')\}$
	\State \hphantom{\textbf{Input:}} NOTE: $\text{Pa}(i) \subset [N_s + N_a]$ can index into states or actions
	\vspace{0.5\baselineskip}
	\State $\mathcal{D}_{tr}, \mathcal{D}_{va} = \texttt{train\_val\_split}(\mathcal{D})$
	\Comment{Split data}
  \State $\theta := \textsc{TrainGMMParentsModel}(\mathcal{D}_{tr}, \mathcal{D}_{va}, \mathcal{G})$
	\Comment{train GMM parent distribution $P_\theta(s,a)$}
  
  \State $\phi := \textsc{TrainFactoredDynamics}(\mathcal{D}_{tr}, \mathcal{D}_{va}, \mathcal{G})$
	\Comment{train factored dynamics model $P_\phi(s'|s,a)$}
	
  \State \textbf{return} \textsc{SampleAugmentedDataset}(N, $\theta$, $\phi$, $\mathcal{G}$)
  \EndFunction
  
  \vspace{\baselineskip}
  \Function{TrainGMMParentsModel}{$\mathcal{D}_{tr}, \mathcal{D}_{va}, \mathcal{G}$}:
    \For{ $(i, \Pai) \in \mathcal{G}:$ } \Comment iterate over parent sets for each child
      \State $\{\mu^k_{\Pai},\Sigma^k_{\Pai},\gamma^k_{\Pai}\} = \texttt{init\_gmm\_params}(N_k)$
      \Comment{NOTE: \emph{only} for this parent set}
      \State $\mathcal{D}^{tr}_{\Pai}(s,a) := \{(s[\Pai], a[\Pai \% N_s]) \forall (s, a, s') \in \mathcal{D}^{tr}\}$
      \State $\mathcal{D}^{va}_{\Pai}(s,a) := \{(s[\Pai], a[\Pai \% N_s]) \forall (s, a, s') \in \mathcal{D}^{va}\}$ \Comment subsample relevant dims
      \While{$\texttt{not\_converged}(\text{GMM}(\cdot;\mu^k_{\Pai},\Sigma^k_{\Pai},\gamma^k_{\Pai}), \mathcal{D}^{va}_{\Pai}(s,a))$}
          \State $\{\mu^k_{\Pai},\Sigma^k_{\Pai},\gamma^k_{\Pai}\} \leftarrow \texttt{update\_gmm\_params}(\{\mu^k_{\Pai},\Sigma^k_{\Pai},\gamma^k_{\Pai}\}, \mathcal{D}^{tr}_{\Pai}(s,a))$
      \EndWhile
      \State $\theta \texttt{.append}(\{\mu^k_{\Pai},\Sigma^k_{\Pai},\gamma^k_{\Pai}\})$
  \EndFor
  \State \textbf{return} $\theta$ 
  \EndFunction
  
  \vspace{\baselineskip}
  \Function{TrainFactoredDynamics}{$\mathcal{D}_{tr}, \mathcal{D}_{va}, \mathcal{G}$}:
  \For{ $(i, \Pai) \in \mathcal{G}:$ } \Comment iterate over parent sets for each child
      \State $\phi_i = \texttt{init\_mlp\_params}()$ 
      \Comment{NOTE: \emph{only} for this parent set}
      \State $\mathcal{D}^{tr}_{\Pai}(s,a,s') := \{(s[\Pai], a[\Pai \% N_s], s'[i]) \forall (s, a, s') \in \mathcal{D}^{tr}\}$
      \State $\mathcal{D}^{va}_{\Pai}(s,a,s') := \{(s[\Pai], a[\Pai \% N_s], s'[i]) \forall (s, a, s') \in \mathcal{D}^{va}\}$ \Comment subsample relevant dims
      \While{$\texttt{not\_converged}(\text{MLP}(\cdot|\cdot;\phi_i), \mathcal{D}^{va}_{\Pai}(s,a,s'))$}
          \State $\phi_i \leftarrow \texttt{update\_mlp\_params}(\phi_i, \mathcal{D}^{tr}_{\Pai}(s,a,s'))$
      \EndWhile
  \EndFor
  \State \textbf{return} $\phi$ 
  \EndFunction
  
  \vspace{\baselineskip}
  \Function{SampleAugmentedData}{N, $\theta$, $\phi$, $\mathcal{G}$}
    \For{$\_ \in \texttt{range}(N)$:}
	    \LeftComment{\ \ \ sample parent data, i.e. $(\tilde s, \tilde a)$: sequentially sample the parent set GMMs, conditioning}
      \LeftComment{\ \ \ each GMM on previous samples to handle any overlap between parent sets}
      \State $\already := \{ \ \}$
      \Comment{define an ``already sampled'' set to track any parent set overlap}
      \State $\tilde s := [ \ ]$;
             $\tilde a := [ \ ]$
      \For{$i \in \texttt{range}(N_s)$:}
         \If{$\Pai \cap \already = \emptyset:$} \Comment{no vars in this parent set already sampled}
            \State $(\tilde s_{\Pai}, \tilde a_{\Pai}) \sim \text{GMM}(\cdot;\mu^k_{\Pai},\Sigma^k_{\Pai},\gamma^k_{\Pai})$
            \State $\tilde s \texttt{.extend}(\tilde s_{\Pai})$
            \State $\tilde a \texttt{.extend}(\tilde a_{\Pai})$
         \Else \Comment{some vars in this parent set already sampled and must be conditioned on}
	          \LeftComment{\ \ \ \ \ \ \ \ \ \ \ \ \ \ \ \ condition this GMM already-sampled vars, then sample remaining vars}
            \State $(\tilde s_{\Pai \backslash \already}, \tilde a_{\Pai \backslash \already}) \sim \text{GMM}(\cdot|\already;\mu^k_{\Pai},\Sigma^k_{\Pai},\gamma^k_{\Pai})$
	          \LeftComment{\ \ \ \ \ \ \ \ \ \ \ \ \ \ \ \ NOTE: this sampling is easily realized by conditioning each Gaussian component}
            \LeftComment{\ \ \ \ \ \ \ \ \ \ \ \ \ \ \ \ and updating mixture components in proportion to density of already-sampled vars}
            \State $\tilde s \texttt{.extend}(\tilde s_{\Pai \backslash \already})$
            \State $\tilde a \texttt{.extend}(\tilde a_{\Pai \backslash \already})$
         \EndIf
         \State $\already \leftarrow \already \cup \Pai$
      \EndFor
	    \LeftComment{\ \ \ sample next states, i.e. $\tilde s'|(\tilde s, \tilde a)$: sequentially sample each ``factor'' in the factorized dynamics}
      \State $\tilde s' := [ \ ]$
      \For{$i \in \texttt{range}(N_s)$:}
          \State $\tilde s_i' \sim \text{MLP}(\cdot|\tilde s, \tilde a;\phi_i)$
          \State $\tilde s' \texttt{.append}(\tilde s_i')$
      \EndFor
      \LeftComment{\ \ \ assemble transition}
      \State $\tilde s  = \texttt{array}(\tilde s )$;
             $\tilde a  = \texttt{array}(\tilde a )$;
             $\tilde s' = \texttt{array}(\tilde s')$
      \State $\mathcal{\tilde D}\texttt{.append}((\tilde s, \tilde a, \tilde s'))$
  \EndFor
  \State \textbf{return} $\mathcal{\tilde D}$
  \EndFunction
	\end{algorithmic} 
  \end{minipage}
\end{algorithm}

Algorithm \ref{algo:mocoda-fmdp} shows the pseudocode for \methodName sampling for a FMDP, 
where the causal structure of the transition dynamics is assumed known.
For simplicity of exposition we describe the case where each ``factor'' in the 
factorized dynamics is modeled using an MLP, which corresponds to the ``Globally Factored''
model architecture referred to in Table \ref{tab:toy-mse}.
Realizing the ``Locally Factored'' architecture is simply a matter of replacing $\text{MLP}(\cdot)$
in the pseudocode with $(\text{MLP} \circ \text{MaskedComposer})(\cdot)$ described in Section \ref{sec:dynamics-models}.

Implementing \methodName sampling for a Local Causal Model rather than an FMDP is also a straightforward extension. The dynamics modeling is the same (but for the dynamics being conditioned on $\mathcal{L}$).
The parent sampling procedure is described in \ref{appdx_parentdist}.

  \backmatter
  \bibliographystyle{plainnat-no-links} % NOTE: this custom bst file suppresses
                                        % *all* url/doi links
  \bibliography{
    qualifying-oral-exam,
    phd-thesis,
    original/laftr/refs,
    original/ffvae/refs,
    original/eiil/refs,
    original/eirs/spurious_workshop,
    original/coda/refs,
    original/mocoda/refs,
    scratch,
  }
  
\end{document}